\documentclass{article}

\PassOptionsToPackage{numbers}{natbib}

\usepackage[final]{neurips_2021}

\usepackage[utf8]{inputenc} 
\usepackage[T1]{fontenc}    
\usepackage{hyperref}       
\usepackage{url}            
\usepackage{booktabs}       
\usepackage{amsfonts}       
\usepackage{nicefrac}       
\usepackage{microtype}      
\usepackage{xcolor}         

\usepackage{amsmath}
\usepackage{amsthm}
\usepackage{amssymb}
\usepackage{mathrsfs}

\usepackage{algorithm}
\usepackage{algorithmic}

\newtheorem{assumption}{Assumption}
\newtheorem{lemma}{Lemma}
\newtheorem{theorem}{Theorem}
\newtheorem{remark}{Remark}
\newtheorem{corollary}{Corollary}
\newtheorem{definition}{Definition}
\newtheorem{proposition}{Proposition}

\usepackage{graphicx}
\usepackage{subfigure}
\usepackage{wrapfig}

\usepackage{multirow}

\usepackage{float}

\title{Stochastic Anderson Mixing for Nonconvex Stochastic Optimization}

\author{%
  Fuchao Wei \\
  Department of Computer Science and Technology\\
  Tsinghua University\\
  Beijing, China 100084 \\
  \texttt{wfc16@mails.tsinghua.edu.cn} \\
  \AND
  Chenglong Bao \\
  Yau Mathematical Science Center\\
  Tsinghua University\\
  Beijing, China 100084 \\
  clbao@mail.tsinghua.edu.cn \\
  \And
  Yang Liu \\
  Department of Computer Science and Technology \\
  Tsinghua University \\
  Beijing, China 100084 \\
  \texttt{liuyang2011@tsinghua.edu.cn} \\  
}

\begin{document}

\maketitle

\begin{abstract}

Anderson mixing (AM) is an acceleration method for fixed-point iterations. Despite its success and wide usage in scientific computing, the convergence theory of AM remains unclear, and its applications to machine learning problems are not well explored. In this paper, by introducing damped projection and adaptive regularization to classical AM, we propose a Stochastic Anderson Mixing (SAM) scheme to solve nonconvex stochastic optimization problems. Under mild assumptions, we establish the convergence theory of SAM, including the almost sure convergence to stationary points and the worst-case iteration complexity. Moreover, the complexity bound can be improved when randomly choosing an iterate as the output. To further accelerate the convergence, we incorporate a variance reduction technique into the proposed SAM. We also propose a preconditioned mixing strategy for SAM which can empirically achieve faster convergence or better generalization ability. Finally, we apply the SAM method to train various neural networks including the vanilla CNN, ResNets, WideResNet, ResNeXt, DenseNet and RNN. Experimental results on image classification and language model demonstrate the advantages of our method.

\end{abstract}

\section{Introduction}
 Stochastic optimization is important in various areas such as statistics~\citep{friedman2001elements} and machine learning~\citep{Bottou2018Optim,shalev2014understanding}, power systems~\citep{huang2018stochastic}. In this paper, we consider  the following  stochastic optimization problem:
\begin{align}
	\min\limits_{x \in \mathbb{R}^d} f(x)=\mathbb{E}_{\xi}\left[ F(x;\xi) \right], \label{eq:problem}
\end{align}
where $ F: \mathbb{R}^d \times \mathbb{R}^{d_{\xi}} \rightarrow \mathbb{R}$ is continuously differentiable and possibly nonconvex and the random variable $\xi \in \mathbb{R}^{d_{\xi}} $ may follow an unknown probability distribution.  It is assumed that only noisy information about the  gradient of $f$ is available through calls to some {\em stochastic first-order oracle} ($\mathcal{SFO}$). One special case of \eqref{eq:problem} is the {\em empirical risk minimization} problem:
\begin{align}
	\min\limits_{x \in \mathbb{R}^d} f (x) \overset{\text{def}}{=} \frac{1}{T}\sum_{i=1}^{T}f_{\xi_i}(x), \label{eq:erm}
\end{align}
where $f_{\xi_i}: \mathbb{R}^d \rightarrow \mathbb{R}$ is the loss function corresponding to the $i$-th data sample and $T$ denotes the number of data samples. $ T $ can be extremely large such that it prohibits the computation of the full gradient $\nabla f$. Thus designing efficient and effective numerical algorithm for solving problem \eqref{eq:problem} or \eqref{eq:erm} with rigorous theoretical analysis is a challenging task.

One classical approach for solving \eqref{eq:problem} is the stochastic gradient descent (SGD) method \citep{robbins1951stochastic}. It mimics GD method by using noisy gradients and exhibits optimal convergence rate for some strongly convex stochastic problems~\citep{chung1954stochastic,sacks1958asymptotic}. Some early related works  of SGD in convex optimization can be found in \citep{nemirovski2009robust,polyak1990new,polyak1992acceleration}. For nonconvex cases, \citet{ghadimi2013stochastic} propose a randomized stochastic gradient (RSG) method that randomly selects an solution $\bar{x}$ from previous iterates. To ensure  $ \bar{x} $ satisfying
$ \mathbb{E}\left[ \|\nabla f(\bar{x}) \right\|^2_2] \leq \epsilon $, the total number of $ \mathcal{SFO} $-calls needed by RSG is $ O\left( \epsilon^{-2} \right)$.  
 Spurred by the development of machine learning,  adaptive learning rate methods are proposed to accelerate SGD, e.g. Adagrad \cite{duchi2011adaptive}, RMSprop \cite{tieleman2012lecture} and Adam \cite{kingma2014adam}, though the convergence analyses of these methods only apply to convex cases. 
There are also many works on second-order optimization methods \citep{schraudolph2007stochastic,martens2010deep, mokhtari2020stochastic}.
 One notable work is the framework of stochastic quasi-Newton (SQN) method proposed by \citet{wang2017stochastic}, which  covers a class of SQN methods and has theoretical guarantees in nonconvex stochastic optimization. However,  these second-order methods usually demand more gradient evaluations in every iteration and less noisy gradient information to achieve actual acceleration \citep{Bottou2018Optim}. 
 
 In this paper, we develop a novel second-order method based on Anderson mixing (AM) \citep{anderson1965iterative},  a sequence acceleration method \citep{brezinski2018shanks} in scientific computing. AM is widely used to accelerate the slow convergence of nonlinear fixed-point iterations arisen  in computational physics and quantum chemistry, e.g., the Self-Consistent Field (SCF) iteration in electronic structure calculations \citep{garza2012comparison,cances2000can}, where the function evaluation is costly. AM is appropriate to high dimensional problem since it extrapolates a new iterate using a limit number of historical iterates. It turns out that AM is closely related to multisecant quasi-Newton methods in nonlinear problems \citep{fang2009two,brezinski2018shanks} or the generalized minimal residual (GMRES) method \citep{saad1986gmres} in linear problems \citep{walker2011anderson,potra2013characterization}.
 
 Inspired by the great success of AM in accelerating fixed-point iterations, it is natural to ask whether AM can be applied to accelerate nonlinear optimization since the gradient descent (GD) method for searching a saddle point in unconstrained optimization is a fixed-point iteration when using constant stepsize. This idea has been explored in \citep{scieur2016regularized,scieur2017nonlinear}, but the proposed  Regularized Nonlinear Acceleration (RNA) method is built on the minimal polynomial extrapolation (MPE)  approach \cite{brezinski2018shanks}, a sequence transformation method that has subtle difference from AM. Also, their methods rely heavily on the contraction assumption of the fixed-point map and the strong convexity. For AM, although current research has proved linear convergence of AM for fixed-point iterations under some conditions \citep{toth2015convergence,toth2017local,bian2021anderson}, 
  there exists no version of AM that guarantees convergence for nonconvex optimization, let alone stochastic optimization. 

In this paper, we develop a stochastic extension of AM. 
 Due to the nonconvexity and noise inside the problems and lack of line search or trust-region, a straightforward migration of AM to nonconvex stochastic optimization is infeasible.
 As a result, we make several fundamental modifications to AM. We highlight the main contributions of our works as follows:
\begin{enumerate}
    \item We develop a stochastic version of AM, namely Stochastic Anderson Mixing (SAM), by introducing {\em damped projection} and {\em adaptive regularization}. We prove its almost sure convergence to a stationary point and analyze its work complexity. When a randomly chosen iterate $ x_R $ is returned as the output of SAM, we prove that the worst-case $ \mathcal{SFO}$-calls complexity to guarantee 
$ \mathbb{E}\left[ \| \nabla f(x_R) \|_2^2 \right]\leq \epsilon $ is $ O\left( \epsilon^{-2} \right) $. (See Theorem~\ref{them:nonconvexStochastic} - \ref{them:random_output}.)
    \item We give a variance reduced extension of SAM by borrowing the stochastic variance reduced gradient (SVRG) \citep{johnson2013accelerating} technique and analyze its $ \mathcal{SFO} $-calls complexity. (See Theorem~\ref{them:sam_vr}.) 
    We also propose a  preconditioned mixing strategy for AM and obtain the preconditioned SAM method which can empirically converge faster or generalize better. (See Section~\ref{subsec:enhance}.)
    \item Extensive experiments on training Convolutional Neural Network (CNN), ResNet, WideResNet, ResNeXt, DenseNet, and  Recurrent Neural Network (RNN) on different tasks and datasets show the faster convergence or better generalization ability of our method compared with the state-of-the-art methods. (See Section~\ref{sec:expr}.)
\end{enumerate}


\section{Methodology}
 \subsection{Anderson Mixing}
 AM is proposed for acceleration of fixed-point iterations. We assume the fixed-point iteration is $ x_{k+1} = g(x_k)\overset{\text{def}}{=}x_k+r_k $, where $ r_k \overset{\text{def}}{=} -\nabla f(x_k) $. Then $ g(x_k)-x_k =r_k $. Here, we adopt the description of AM in \citep{fang2009two,walker2011anderson}.
 Let $ \Delta $ denote the forward difference operator, say, $ \Delta x_k = x_{k+1}-x_k $.  Let $ X_k $ and $ R_k $  record the most recent $ m  ( m\leq k )$ iterations:
\begin{align}
X_k = [
\Delta x_{k-m} , \Delta x_{k-m+1} , \cdots , \Delta x_{k-1} ],
R_k = [
\Delta r_{k-m} , \Delta r_{k-m+1} , \cdots , \Delta r_{k-1} ]. \label{Xk_Rk}
\end{align}
 AM can be decoupled into two steps. We call them {\em the projection step} and {\em the mixing step}:
\begin{subequations}
\begin{align}
\bar{x}_k &= x_k - X_k\Gamma_k,  \quad\mbox{(Projection step),} \label{xbar} \\
x_{k+1} &= \bar{x}_k + \beta_k\bar{r}_k, \quad~\mbox{ (Mixing step),} \label{xnext}
\end{align}
\end{subequations}
 where  $ \beta_k $ is the mixing parameter, and $ \bar{r}_k \overset{\text{def}}{=} r_k - R_k\Gamma_k $ is reminiscent of {\em extragradient} \citep{korpelevich1976extragradient}.  $ \Gamma_k $ is determined by solving 
\begin{equation}
\Gamma_k = \mathop{\arg\min}_{\Gamma \in \mathbb{R}^m}\| r_k - R_k\Gamma \|_2.\label{lsq}
\end{equation} 
Combining (\ref{xbar}) and (\ref{xnext}), we obtain the full form of AM
 \citep{fang2009two,walker2011anderson,brezinski2018shanks,pratapa2016anderson}:
\begin{equation}
x_{k+1} = x_k + \beta_k r_k - \left( X_k + \beta_k R_k \right)\Gamma_k. \label{eq:aa}
\end{equation}
\begin{remark} \label{remark:minimal}
To see the rationality of AM, we assume $ f $ is  twice continuously differentiable. 
  Then a quadratic approximation of $ f $ implies $ \nabla^2 f(x_k) \left(x_j-x_{j-1}\right) \approx \nabla f(x_j) - \nabla f(x_{j-1}) $ in a local small region around $ x_k $, so it is reasonable to assume $ R_k \approx -\nabla^2f(x_k)X_k $. Thus we see 
$ \| r_k - R_k\Gamma \|_2 \approx \| r_k + \nabla^2f(x_k) X_k\Gamma \|_2 $. Hence,  we can recognize (\ref{lsq})  as solving 
$ \nabla^2 f(x_k)p_k = \nabla f(x_k) $ in a least-squares sense, where $ p_k =X_k\Gamma_k $. When the quadratic approximation is exact, solving (\ref{lsq}) is a minimal residual procedure, thus being verified as a {\em residual projection} method \citep{saad2003iterative}. From this viewpoint, the fixed-point assumption is unnecessary as long as  $ R_k \approx -\nabla^2f(x_k)X_k $. 
Moreover, let $H_k$  be the solution to a constrained optimization problem \citep{fang2009two}:
\begin{align}
\min\limits_{H_k}\| H_k-\beta_k I \|_F \text{ subject to } H_kR_k = -X_k,\label{mulsecant}
\end{align}
then iterate \eqref{eq:aa} is $ x_{k+1} = x_k + H_kr_k $, which is indeed a multisecant quasi-Newton method.
Note that a key simplification in AM is using differences of historical gradients $ R_k $  to approximate $ -\nabla^2f(x_k)X_k $, which reduces the heavy cost to compute Hessian-vector products~\citep{byrd2016stochastic,gower2019rsn,huang2020span}.

\end{remark}
 
\subsection{Stochastic Anderson Mixing} \label{sec:sam}
 We describe our method Stochastic Anderson Mixing (SAM) in this section.  
 At the $k$-th iteration, let $ S_k\subseteq \left[T\right] \overset{\text{def}}{=} \{ 1,2,\ldots,T \} $ be the sampled mini-batch and the corresponding objective function value is $f_{S_k}\left(x_k\right) = \frac{1}{\vert S_k \vert}\sum_{i\in S_k}f_{\xi_i}(x_k)$. Then $r_k\overset{\text{def}}{=} -\nabla f_{S_k}\left(x_k\right) $ and the noisy $ R_k $ is defined correspondingly (cf. (\ref{Xk_Rk})). Due to the instability and inaccurate estimation of $R_k$, we stabilize the projection step by proposing {\em damped projection} and {\em adaptive regularization} techniques. 
 Algorithm~\ref{alg:sam} is a sketch of our method. 
 We elaborate the mechanism of this algorithm in the following subsections. 
 
\begin{algorithm}[tb]
\caption{Stochastic Anderson Mixing (SAM)}
\label{alg:sam}
\textbf{Input}: $ x_0\in\mathbb{R}^d, m=10, \alpha_k=1, \beta_k=1,  \mu=10^{-8}, max\_iter>0 $.\\
\textbf{Output}: $ x\in\mathbb{R}^d $
\begin{algorithmic}[1] 
\FOR{$k = 0,1,\dots, max\_iter$ }
\STATE $ r_k = -\nabla f_{S_k}\left(x_k\right) $ 
\IF {$k = 0$}
\STATE {$ x_{k+1} = x_k+\beta_kr_k $}
\ELSE
\STATE $ m_k = \min\{m,k\} $
\STATE $ X_k = [
\Delta x_{k-m_k} , \Delta x_{k-m_k+1} , \cdots , \Delta x_{k-1} ] $ 
\STATE $ R_k = [
\Delta r_{k-m_k} , \Delta r_{k-m_k+1} , \cdots , \Delta r_{k-1} ] $ 
\STATE Check Condition~(\ref{ineq:check_alpha_k}) and use smaller $\alpha_k$ if (\ref{ineq:check_alpha_k}) is violated.
\STATE $ x_{k+1} = x_k + \beta_k r_k - \left( \alpha_k X_k + \alpha_k\beta_k R_k \right) \left( R_k^{\mathrm{T}}R_k +\delta_k X_k^{\mathrm{T}}X_k\right)^{\dagger} R_k^{\mathrm{T}}r_k $ 
\ENDIF
\STATE Apply learning rate schedule of $ \alpha_k, \beta_k $
\ENDFOR

\STATE \textbf{return} $ x_k $
\end{algorithmic}
\end{algorithm}

\paragraph{Damped projection}
 From Remark~\ref{remark:minimal}, we see the determination of $ \Gamma_k $  in (\ref{lsq}) relies on the local quadratic approximation of (\ref{eq:erm}), which can be rather inexact in general nonlinear optimization. To improve the stability, we propose a {\em damped projection} method for \eqref{xbar}. Let $\alpha_k$ be the damping parameter, we obtain $\bar x_k$ via 
\begin{align}
\bar{x}_k = (1-\alpha_k)x_k+\alpha_k(x_k-X_k\Gamma_k) 
 =x_k - \alpha_k X_k\Gamma_k, \label{xbar_new}
\end{align}
 Combining (\ref{xbar_new}) and (\ref{xnext})  and noting that $ \bar{r}_k=r_k-\alpha_kR_k\Gamma_k $, the new iterate $ x_{k+1} $ is given by
\begin{equation}
x_{k+1} = x_k + \beta_k r_k - \alpha_k\left(  X_k + \beta_k R_k \right)\Gamma_k.
\label{eq:sam}
\end{equation} 
 It is worth noting that $\beta_k$ and $ \alpha_k$ in \eqref{eq:sam} behave like stepsize or learning rate  in SGD, and the extra term $  \left( \alpha_k X_k + \alpha_k\beta_k R_k \right)\Gamma_k $ can be viewed as a generalized momentum term.
 
\paragraph{Adaptive regularization}
Since $R_k$ may be rank deficient and no safeguard method is used in AM, the least square problem \eqref{lsq} can be unstable. A remedy is to add {\em regularization}~\citep{brezinski2020shanks} to \eqref{lsq}. One well known choice is the Tikhonov regularization introduced in~\citep{scieur2016regularized,scieur2020regularized}, which can be viewed as  forcing $ \|\Gamma_k \|_2 $ to be small \cite{scieur2020regularized}, leading to a penalty term to  (\ref{lsq}):
 \begin{equation}
\Gamma_k = \mathop{\arg\min}_{\Gamma \in \mathbb{R}^m}\| r_k - R_k\Gamma \|_2^2 + \delta \| \Gamma \|_2^2 ,\label{rlsq}
\end{equation} 
where $ \delta \geq 0 $ is the penalized constant. The solution of \eqref{rlsq} is 
$\Gamma_k = \left( R_k^{\mathrm{T}}R_k +\delta I\right)^{\dagger}R_k^{\mathrm{T}}r_k, $
where ``$\dagger$" denotes the Penrose-Moore inverse. We name this regularized variant of AM as RAM. 
 
 Here, we propose a new regularization, namely {\em adaptive regularization}, to better suit the  stochastic optimization. Since $ -X_k \Gamma_k =  \bar{x}_k - x_k  $ denotes the update from $ x_k $ to $ \bar{x}_k $, a large magnitude of $ \| X_k\Gamma_k \|_2 $ tends to make the intermediate step $ \bar{x}_k $ overshoot the trust region  around $ x_k $. Thus it is more reasonable to force $ \| X_k\Gamma_k \|_2 $ rather than $ \| \Gamma_k \|_2 $ to be small. We formulate this idea as 
 \begin{equation}
\min\limits_{\Gamma}\| r_k - R_k \Gamma  \|_2^2 + \delta_k \| X_k \Gamma  \|_2^2,\label{adalsq}
\end{equation}
where $ \delta_k \geq 0 $ is a variable determined in each iteration. Explicitly  solving (\ref{adalsq}) leads to
\begin{equation}
\Gamma_k = \left( R_k^{\mathrm{T}}R_k +\delta_k X_k^{\mathrm{T}}X_k\right)^{\dagger} R_k^{\mathrm{T}}r_k, \label{AdaRAA:Gamma_k}
\end{equation}
We call AM with this regularization and damped projection as SAM, i.e. the algorithm given in Algorithm~\ref{alg:sam}. 
The choice of $ \delta_k $ should reflect the curvature change in the vicinity of $ x_k $, so we give a special choice of $ \delta_k $:
\begin{equation}
\delta_k = \max\left\lbrace \frac{c_1\| r_k \|_2^2}{\| x_k - x_{k-1}  \|_2^2+\epsilon}, c_2 \beta_k^{-2}\right\rbrace  ,\label{deltak}
\end{equation}
where $c_1, c_2 $ are constants, $ \epsilon $ is a small constant to prevent the denominator from being zero.   Such form of $ \delta_k $ is reminiscent of AdaDelta \cite{zeiler2012adadelta}. A large $ \|r_k\|_2 $ indicates a potential dramatic change in landscape, suggesting using a precautious tiny stepsize. The denominator in (\ref{deltak}) behaves like {\em annealing}. In Secant Penalized BFGS \cite{irwin2020secant}, this term measures the noise in gradients. In Section~\ref{anal2}, we will further explain the rationality of (\ref{deltak}). We name this new method as AdaSAM.
 
\paragraph{Positive definiteness.}
 From \eqref{eq:sam} and \eqref{AdaRAA:Gamma_k},  the SAM update is $ x_{k+1} = x_k+H_kr_k$, where
$  H_k = \beta_k I - \alpha_k Y_k Z_k^{\dagger}R_k^{\mathrm{T}}$, $ Y_k = X_k+\beta_kR_k $, $ Z_k = R_k^{\mathrm{T}}R_k +\delta_k X_k^{\mathrm{T}}X_k $.  $ H_k $ is generally not symmetric.  A critical condition for the convergence analysis of SAM is the {\em positive definiteness} of $ H_k $, i.e. 
 \begin{equation}
   p_k^{\mathrm{T}}H_kp_k \geq \beta_k\mu \|p_k\|_2^2, \quad \forall p_k \in \mathbb{R}^d, \label{ineq:pos_Hk}
 \end{equation}
 where $ \mu\in (0,1) $ is a constant. Next, we give an approach to guarantee it.
 
  Let $ \lambda_{min}(\cdot) $ denote the smallest eigenvalue, $ \lambda_{max}(\cdot)$ denote the largest eigenvalue. Since $ p_k^{\mathrm{T}}H_kp_k = \frac{1}{2}p_k^{\mathrm{T}} (H_k+H_k^{\mathrm{T}})p_k $, Condition (\ref{ineq:pos_Hk}) is equivalent to $ \lambda_{min}\left( \frac{1}{2} \left(H_k+H_k^{\mathrm{T}} \right) \right) \geq \beta_k\mu $.
  With some simple algebraic operations, we obtain  $ \lambda_{min}\left( \frac{1}{2}\left(H_k+H_k^\mathrm{T}\right) \right) =\beta_k - \frac{1}{2}\alpha_k\lambda_{max}(Y_kZ_k^{\dagger}R_k^\mathrm{T}+R_kZ_k^{\dagger}Y_k^\mathrm{T})$. Let $ \lambda_k \overset{\text{def}}{=}\lambda_{max}(Y_kZ_k^{\dagger}R_k^\mathrm{T}+R_kZ_k^{\dagger}Y_k^\mathrm{T})$, then
   Condition (\ref{ineq:pos_Hk}) is equivalent to 
  \begin{equation}
 \alpha_k\lambda_k \leq
 2\beta_k(1-\mu). \label{ineq:check_alpha_k}
 \end{equation}
 To check Condition~(\ref{ineq:check_alpha_k}),  note that
 \begin{align}
\lambda_k
 = \lambda_{max} \left( \begin{pmatrix}
 Y_k & R_k
 \end{pmatrix} 
 \begin{pmatrix}
 0 & Z_k^{\dagger} \\
 Z_k^{\dagger} & 0 
 \end{pmatrix}
 \begin{pmatrix}
 Y_k^\mathrm{T} \\
 R_k^\mathrm{T}
 \end{pmatrix} \right) 
 =\lambda_{max}\left( 
 \begin{pmatrix}
 Y_k^\mathrm{T} \\
 R_k^\mathrm{T}
 \end{pmatrix} 
 \begin{pmatrix}
 Y_k & R_k
 \end{pmatrix}
 \begin{pmatrix}
 0 & Z_k^{\dagger} \\
 Z_k^{\dagger} & 0 
 \end{pmatrix} 
 \right). \label{eig}
 \end{align}
 Since $ \begin{pmatrix}
 Y_k^\mathrm{T} \\
 R_k^\mathrm{T}
 \end{pmatrix} 
  \begin{pmatrix}
 Y_k & R_k
 \end{pmatrix}, 
 \begin{pmatrix}
 0 & Z_k^{\dagger} \\
 Z_k^{\dagger} & 0 
 \end{pmatrix} \in \mathbb{R}^{2m\times 2m} $, and $ m \ll d $,  $ \lambda_k $ can be computed efficiently, say, using an eigenvalue decomposition algorithm with the time complexity of $ O(m^3)$. This cost is negligible compared with those  to form $ X_k^\mathrm{T}X_k, R_k^\mathrm{T}R_k $, which need $ O(m^2d) $ flops. After that, to guarantee the positive definiteness, we check if $ \alpha_k $  satisfies (\ref{ineq:check_alpha_k})
 and use a smaller $ \alpha_k $ if necessary. 

\subsection{Enhancement of Stochastic Anderson Mixing}  \label{subsec:enhance}
  To further enhance SAM, we introduce two techniques, namely {\em variance reduction} and {\em preconditioned mixing}. 
 
\textbf{Variance reduction.}
 Variance reduction techniques are proved to be effective  if a scan over the full dataset is feasible \citep{allen2016variance,reddi2016stochastic}. Similar to SdLBFGS-VR proposed in \citep{wang2017stochastic}, we also incorporate SVRG to SAM, which we call SAM-VR (Algorithm~\ref{alg:sam-vr}), for solving (\ref{eq:erm}). To simplify the description, we denote one iteration of SAM in Algorithm \ref{alg:sam} as $SAM\_update(x_k,g_k)$, i.e. one update of $ x_k $ given the gradient estimate $ g_k$. 
 
\textbf{Preconditioned mixing.}
 Motivated by the great success of preconditioning in solving linear systems and eigenvalue computation \citep{golub2013matrix}, we present a preconditioned version of SAM. The key modification is the mixing step (\ref{xnext}). We replace the simple mixing $ x_{k+1} = \bar{x}_k+\beta_k\bar{r}_k $ with $ x_{k+1} = \bar{x}_k+M_{k}^{-1}\bar{r}_k $ where $ M_{k} $ approximates the Hessian. Combining it with (\ref{xbar_new}) and (\ref{AdaRAA:Gamma_k}), we obtain 
 \begin{align}
 x_{k+1} = x_{k}+\left( M_k^{-1}- \alpha_k\left( X_k+M_k^{-1}R_k\right)\left( R_k^{\mathrm{T}}R_k +\delta_k X_k^{\mathrm{T}}X_k \right)^{\dagger}R_k^{\mathrm{T}} \right)r_{k}. \label{psam}
 \end{align}
 Setting $\alpha_k\equiv 1$ and $ \delta_k\equiv 0 $, (\ref{psam})  reduces to a preconditioned AM update, which can be recast as the solution to the constrained optimization problem:
$ \min\limits_{H_k}\| H_k-M_k^{-1} \|_F \text{ subject to } H_kR_k = -X_k,$ 
 a direct extension of (\ref{mulsecant}). This preconditioned version of AM is related to quasi-Newton updates \citep{gower2017randomized}. We also point out that the action of $ M_k^{-1} $ can be implicitly done via an update of any optimizer at hand, i.e. $ x_{k+1} = optim(\bar{x}_k,-\bar{r}_k) $, where $ optim $ updates $ \bar{x}_k$ given the {\em extragradient} $-\bar{r}_k$. If $ R_k = -\nabla^2f(x_k)X_k $, which is the case in deterministic quadratic optimization, the projection step in preconditioned AM is still a minimal residual procedure.

 \begin{algorithm}[tb]
\caption{Stochastic Anderson Mixing with variance reduction (SAM-VR)}
\label{alg:sam-vr}
\textbf{Input}: {$ \tilde{x}_0\in\mathbb{R}^d $; $ \beta_t^k, \alpha_t^k, \delta_t^k $  for $SAM\_update(x_{t}^{k},g_t^k)$; Batch size $ n\geq 1 $ }.\\
\textbf{Output}: $ x \in\mathbb{R}^d $
\begin{algorithmic}[1] 
\FOR{$k = 0,\dots, N-1$ }
\STATE $ x_0^k = \tilde{x}_k $ 
\STATE $ \nabla f(\tilde{x}_k) = \frac{1}{T}\sum_{i=1}^{T}\nabla f_{\xi_i}(\tilde{x}_k) $
 \FOR{$ t = 0,\dots,q-1 $}
 \STATE Sample a subset $ \mathcal{K}\subseteq\left[ T\right] $ with $\vert \mathcal{K} \vert = n$
 \STATE {$ g_t^k = \nabla f_\mathcal{K}(x_t^k) -\nabla f_\mathcal{K}(\tilde{x}_k)  + \nabla f(\tilde{x}_k)$ where $ \nabla f_\mathcal{K}(x_t^k) = \frac{1}{\vert \mathcal{K} \vert }\sum_{i\in\mathcal{K}}\nabla f_{\xi_i}(x_t^{k}) $}
 \STATE {$ x_{t+1}^{k} = SAM\_update(x_{t}^{k},g_t^k) $}
 \ENDFOR
\STATE {Set $ \tilde{x}_{k+1} = x_q^{k} $}
\ENDFOR
\STATE {\textbf{return} Iterate $ x $ chosen uniformly random from $ \lbrace\lbrace x_t^{k} \rbrace_{t=0}^{q-1} \rbrace_{k=0}^{N-1} $ }
\end{algorithmic}
\end{algorithm}  

 \begin{remark} \label{remark:impl}
Similar to SdLBFGS \citep{wang2017stochastic}, SAM needs another $ 2md $ space to store $ X_k $ and $ R_k $. The extra main computational cost for SAM compared with SGD is $ O(m^2d)+O(m^3)$, which accounts for the matrix multiplications ($ \mathbb{R}^{m\times d} \times \mathbb{R}^{d\times m} $) and matrix decomposition of  a small $ \mathbb{R}^{m\times m} $ matrix.  Since dense matrix multiplication can be ideally parallelized  and the cost of gradient evaluations often dominates the computing, the benefit from SAM pays for this extra cost. Besides, we incorporate sanity check of the positive definiteness, alternating iteration and moving average in our implementation and the details are given in the supplementary materials.
\end{remark}

\section{Theory}\label{sec:theory}
 In this section, we give the main results about the convergence and  complexity of SAM. All the proofs are left to the Appendix.   It is worth noting that 
 since the approximate Hessian $ H_k $  of SAM  may depend on the data samples $\lbrace \xi_i\rbrace_{i\in S_k}$ of current mini-batch, which violates the assumption AS.4 in \citep{wang2017stochastic}, the framework of \citep{wang2017stochastic} does not apply. 
 
 We first give  assumptions about the objective function $ f $.

\begin{assumption} \label{assume:Lips}
 $ f: \mathbb{R}^d \mapsto \mathbb{R}$ is {\em continuously differentiable}. $ f(x)\geq f^{low}>-\infty $ for any $ x\in \mathbb{R}^d $. $ \nabla f $ is globally  $ L $-Lipschitz continuous; namely
 $ \|\nabla f(x)-\nabla f(y)\|_2 \leq L\|x-y\|_2 \ $  for any $ x,y \in \mathbb{R}^d $.
\end{assumption}

\begin{assumption} \label{assume:noise}
For any iteration $ k$, the stochastic gradient $ \nabla f_{\xi_k}(x_k) $ satisfies
$\mathbb{E}_{\xi_k}[\nabla f_{\xi_k}(x_k)] =\nabla f(x_k),$
$\mathbb{E}_{\xi_k}[\| \nabla f_{\xi_k}(x_k)-\nabla f(x_k)\|_2^2] \leq \sigma^2, 
$
where $ \sigma>0$, and $ \xi_k, k=0,1,\ldots $ are independent samples that are independent of $ \{ x_j\}_{j=0}^{k} $.
\end{assumption}
 
  We also state the diminishing condition about $ \beta_k $ as
\begin{equation}
\sum_{k=0}^{+\infty}\beta_k = +\infty, \sum_{k=0}^{+\infty}\beta_k^2 < +\infty. 
\label{cond:diminish}
\end{equation}

 \subsection{Convergence and complexity}
 \begin{theorem} \label{them:nonconvexStochastic}
Suppose that Assumptions~\ref{assume:Lips} and \ref{assume:noise} hold for $ \lbrace x_k\rbrace $ generated by SAM with batchsize $ n_k = n $ for all $ k$. $ C>0$ is a constant. If $ \beta_k $ satisfies (\ref{cond:diminish}) and  $ 0<\beta_k\leq \frac{\mu}{4L(1+C^{-1})}, \delta_k\geq C\beta_k^{-2} , 0\leq\alpha_k\leq\min\{ 1, \beta_k^{\frac{1}{2}}\}$ and satisfies (\ref{ineq:check_alpha_k}) , then
\begin{equation}
\liminf\limits_{k\rightarrow\infty} \|\nabla f(x_k) \|_2 = 0 \ \text{with probability} \ 1.
\label{them1:ineq1}
\end{equation} 
Moreover, there exists a positive constant $ M_f $ such that
\begin{equation}
\mathbb{E}[f(x_k)]\leq M_f \quad \forall k. \label{them1:ineq2}
\end{equation}
\end{theorem}

\begin{theorem}\label{them:bounded_noisy_gradient}
Assume the same assumptions hold as in Theorem~\ref{them:nonconvexStochastic}.
If we require that the noisy gradient is bounded, i.e.,
\begin{equation}
\mathbb{E}_{\xi_k}[\| \nabla f_{\xi_k}(x_k)\|_2^2] \leq M_g,\label{cond:bounded}
\end{equation}
where $ M_g>0$ is a constant, we can obtain a stronger convergence result:
\begin{equation}
\lim\limits_{k\rightarrow\infty} \| \nabla f(x_k)\| = 0 \ \text{with probability} \ 1.
\label{them2:ineq1}
\end{equation}
\end{theorem} 

Now, we give the iteration complexity of SAM.
\begin{theorem}\label{them:complexity}
Suppose that Assumptions~\ref{assume:Lips} and \ref{assume:noise} hold for $ \lbrace x_k\rbrace $ generated by SAM with batchsize $ n_k = n $ for all $ k$. $ C>0$ is a constant. $ \beta_k $ is specially chosen as 
$ \beta_k = \frac{\mu}{4L(1+C^{-1})}(k+1)^{-r} $ with $r\in(0.5,1)$.   $ \delta_k\geq C\beta_k^{-2} $, 
$ 0\leq\alpha_k\leq\min\{ 1, \beta_k^{\frac{1}{2}}\}$ and  satisfies (\ref{ineq:check_alpha_k}). Then
\begin{align}
 \frac{1}{N}\sum_{k=0}^{N-1}\mathbb{E}\|\nabla f(x_k)\|_2^2 
\leq \frac{16L(1+C^{-1})(M_f-f^{low})}{\mu^2}N^{r-1}  +\frac{(1+L^{-1}\mu^{-1})\sigma^2}{(1-r)n}(N^{-r}-N^{-1}), \label{them3:ineq1}
\end{align}
where $ N $ denotes the iteration number. Moreover, for a given $ \epsilon\in(0,1) $, to guarantee that $ \frac{1}{N}\sum_{k=0}^{N-1}\mathbb{E}\|\nabla f(x_k)\|_2^2 < \epsilon$, the number of iterations $ N $ needed is at most $ O(\epsilon^{-\frac{1}{1-r}}) $.
\end{theorem}

 We analyze the $\mathcal{SFO}$-calls complexity of SAM when the output $x_R$ is randomly selected from previous iterates  according to some specially defined probability mass function $ P_R $. 
 We show below that under similar conditions, SAM has the same complexity $ O(\epsilon^{-2}) $ as RSG \citep{ghadimi2013stochastic} and SQN \citep{wang2017stochastic}.
 
 \begin{theorem}\label{them:random_output}
Suppose that Assumptions~\ref{assume:Lips} and \ref{assume:noise} hold. Batch size $ n_k = n $ for $ k=0,\dots,N-1 $. $ C>0$ is a constant. $ \beta_k = \frac{\mu}{4L(1+C^{-1})} $.   $ \delta_k\geq C\beta_k^{-2} $, 
$ 0\leq\alpha_k\leq\min\{ 1, \beta_k^{\frac{1}{2}}\}$ and  satisfies (\ref{ineq:check_alpha_k}).  Let $ R $ be a random variable following $  P_R(k)\overset{\mathrm{def}}{=} Prob\lbrace R=k \rbrace =1/N $,
 and $ \bar{N} $ be the total number of $ \mathcal{SFO}$-calls needed to calculate stochastic gradients $ \nabla f_{S_k}(x_k) $ in SAM. 
 
 For a given accuracy $ \epsilon > 0 $, we assume that 
$
 \bar{N}\geq \left\lbrace  \frac{C_1^2}{\epsilon^2}+\frac{4C_2}{\epsilon}, \frac{\sigma^2}{L^2\tilde{D}} \right\rbrace, 
 $
 where
 $
 C_1 = \frac{32D_f(1+C^{-1})\sigma}{\mu^2\sqrt{\tilde{D}}}+(L+\mu^{-1})\sigma\sqrt{\tilde{D}}, 
 C_2 = \frac{32D_fL(1+C^{-1})}{\mu^2}, 
 $
 where $ D_f \overset{\mathrm{def}}{=} f(x_0)-f^{low}$ and $ \tilde{D} $ is a problem-independent positive constant. Moreover, we assume that the batch size satisfies
 $
 n_k=n:=\left\lceil \min\left\lbrace \bar{N},\max \left\lbrace 1,\frac{\sigma}{L}\sqrt{\frac{\bar{N}}{\tilde{D}}}\right\rbrace \right\rbrace \right\rceil. 
 $
  Then we obtain $ \mathbb{E}\left[ \|\nabla f(x_R)\|_2^2\right] \leq \epsilon $, where the expectation is taken with respect to $ R $ and  $\lbrace S_j\rbrace_{j=0}^{N-1}$. In other words, to ensure $ \mathbb{E}\left[ \|\nabla f(x_R)\|_2^2\right] \leq \epsilon $, the number of $ \mathcal{SFO}-$calls is $  O(\epsilon^{-2})$.

\end{theorem}

 We analyze the $ \mathcal{SFO}$-calls complexity of Algorithm \ref{alg:sam-vr}.
 
 \begin{theorem}\label{them:sam_vr}
For Algorithm~\ref{alg:sam-vr}, suppose that Assumptions~\ref{assume:Lips} and \ref{assume:noise} hold. $ x_* $ denotes the minima. Batch size $ n_k = n $ for all $ k$. Let $ \nu, \mu_0 \in (0,1) $ be two positive constants satisfying
$\frac{\mu_0\mu}{2(1+C^{-1})^{1/2}}-\frac{\mu_0^2(2L+\mu^{-1})(e-1)}{L}-2\mu_0^2n
 - \frac{2\mu_0^3(2L+\mu^{-1})(e-1)n}{L} \geq \nu, $
where $e$ is the Euler's number. $ C>0$ is a constant. Set $ \beta_t^k = \beta :=\frac{\mu_0n}{L(1+C^{-1})^{1/2}T^{2/3}}$, $ \delta_t^k\geq C\beta^{-2} $, 
$ 0\leq\alpha_t^k\leq\min\{ 1, \beta^{\frac{1}{2}}\}$ and  satisfies (\ref{ineq:check_alpha_k}). $ q=\left\lfloor \frac{T}{\mu_0nd_0} \right\rfloor$, where $d_0 = 6+\frac{2}{L(1+C^{-1})^{1/2}}$. Then
$ \mathbb{E}\left[ \|\nabla f(x)\|_2^2\right]\leq \frac{T^{2/3}L\left( f(x_0)-f(x_*)\right)}{qNn\nu}.$
 To ensure $ \mathbb{E}\left[ \|\nabla f(x)\|_2^2\right] \leq \epsilon $, the total number of $\mathcal{SFO}$-calls is $O(T^{2/3}/\epsilon)$.
\end{theorem}

\subsection{Convergence of AdaSAM} \label{anal2}
Now we  state the convergence of AdaSAM. From the definition of $ \delta_k $ (\ref{deltak}), $ \delta_k \geq c_2\beta_k^{-2} $, thus fulfilling the condition in previous theorems. 
We discuss the rationality of the first term, which is somewhat heuristic. Since $ r_k^TH_kr_k \geq \beta_k\mu \|r_k\|_2^2,$ and $ \|H_{k}r_k\|_2^2 \leq 2\left(\beta_{k}^2+\alpha_{k}^2\delta_{k}^{-1}\right)\|r_k\|_2^2 \leq 2\beta_k^2\left(1+C^{-1}\right)\|r_k\|_2^2 \ $ if $ \alpha_k \leq 1, \delta_k \geq C\beta_k^{-2} $  (see the supplement for this inequality), 
it is sensible to suppose $ \|x_k-x_{k-1}\|_2 \approx \|x_{k+1}-x_k \|_2 = \|H_kr_k\|_2 =\beta_k h_k \|r_k\|_2 $, where $ 0<h_k\leq h <+\infty $ is a number related to $  H_k$. Therefore, 
\begin{align}
\frac{c_1\| r_k \|_2^2}{\| x_k - x_{k-1}  \|_2^2} \approx
\frac{c_1\|r_k\|_2^2}{\beta_k^2h_k^2\|r_k \|_2^2} \geq c_1h^{-2}\beta_k^{-2},
\label{ineq:delta_k_approx}
\end{align}
which coincides with the requirement that $ \delta_k \geq C\beta_k^{-2} $. This observation together with $ \delta_k \geq c_2\beta_k^{-2} $ ensures AdaSAM converges in nonconvex stochastic optimization under some proper assumptions.

\section{Experiments} \label{sec:expr}
 Our implementation is based on Algorithm~\ref{alg:sam}, with additional details mentioned in  Remark~\ref{remark:impl}. The pseudocode is given in the Appendix.
 To evaluate the effectiveness of our method AdaSAM and its preconditioned variant pAdaSAM, we compared them with several first-order and second-order optimizers for mini-batch training of neural networks, which can be highly nonlinear and stochastic, on different machine learning tasks. The datasets are MNIST \citep{lecun1998gradient}, 
 CIFAR-10/CIFAR-100 \citep{krizhevsky2009learning}
 for image classification and Penn TreeBank \citep{marcus1993building} for language model.  For AdaSAM, the individual hyperparameter needs to be tuned is $c_1$ in \eqref{deltak}, others are set as default in Algorithm~\ref{alg:sam}. 
 More Experimental details and hyper-parameter tuning are referred to the Appendix.


\textbf{Experiments on MNIST}.  We trained a simple convolutional neural network (CNN) \footnote{Based on the official PyTorch implementation \href{https://github.com/pytorch/examples/blob/master/mnist}{https://github.com/pytorch/examples/blob/master/mnist}.} 
on MNIST, for which  we are only  concerned about  the minimization of the empirical risk (\ref{eq:erm}), i.e. the training loss, with large batch sizes. The training dataset was preprocessed by randomly selecting 12k images from the total 60k images for training. Neither weight-decay nor dropout was used. We compared AdaSAM with SGDM, Adam \citep{kingma2014adam}, SdLBFGS \citep{wang2017stochastic}, and RAM  (cf. \eqref{rlsq}).  The learning rate was tuned and fixed for each optimizer. The historical length for SdLBFGS, RAM and AdaSAM was set as 20. $ \delta = 10^{-6}$ for RAM and $ c_1=10^{-4} $ for AdaSAM. 

\begin{figure}[ht]
\centering 
\subfigure[Batchsize=6K]{
\includegraphics[width=0.23\textwidth]{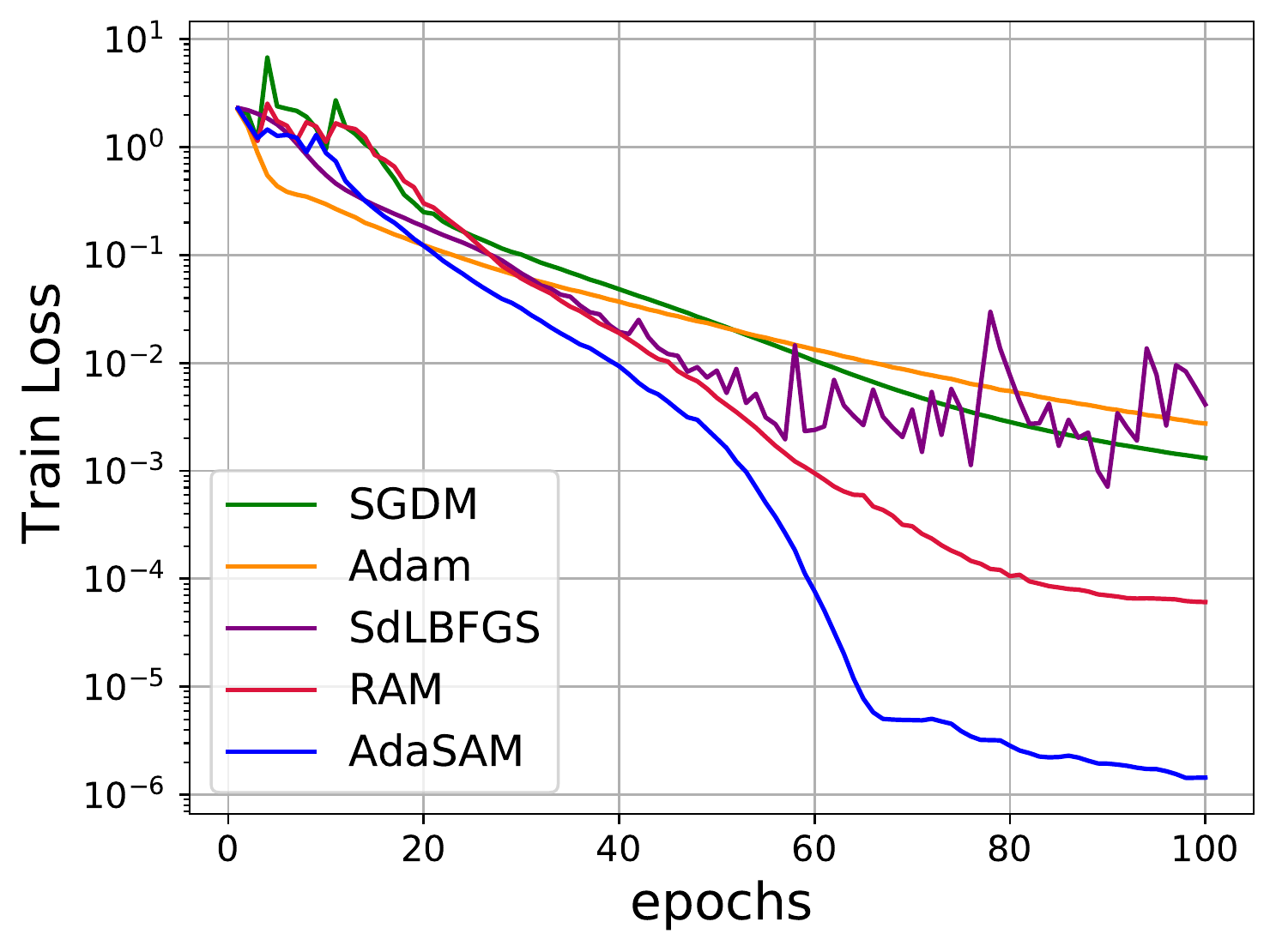}
}
\subfigure[Batchsize=3K]{
\includegraphics[width=0.23\textwidth]{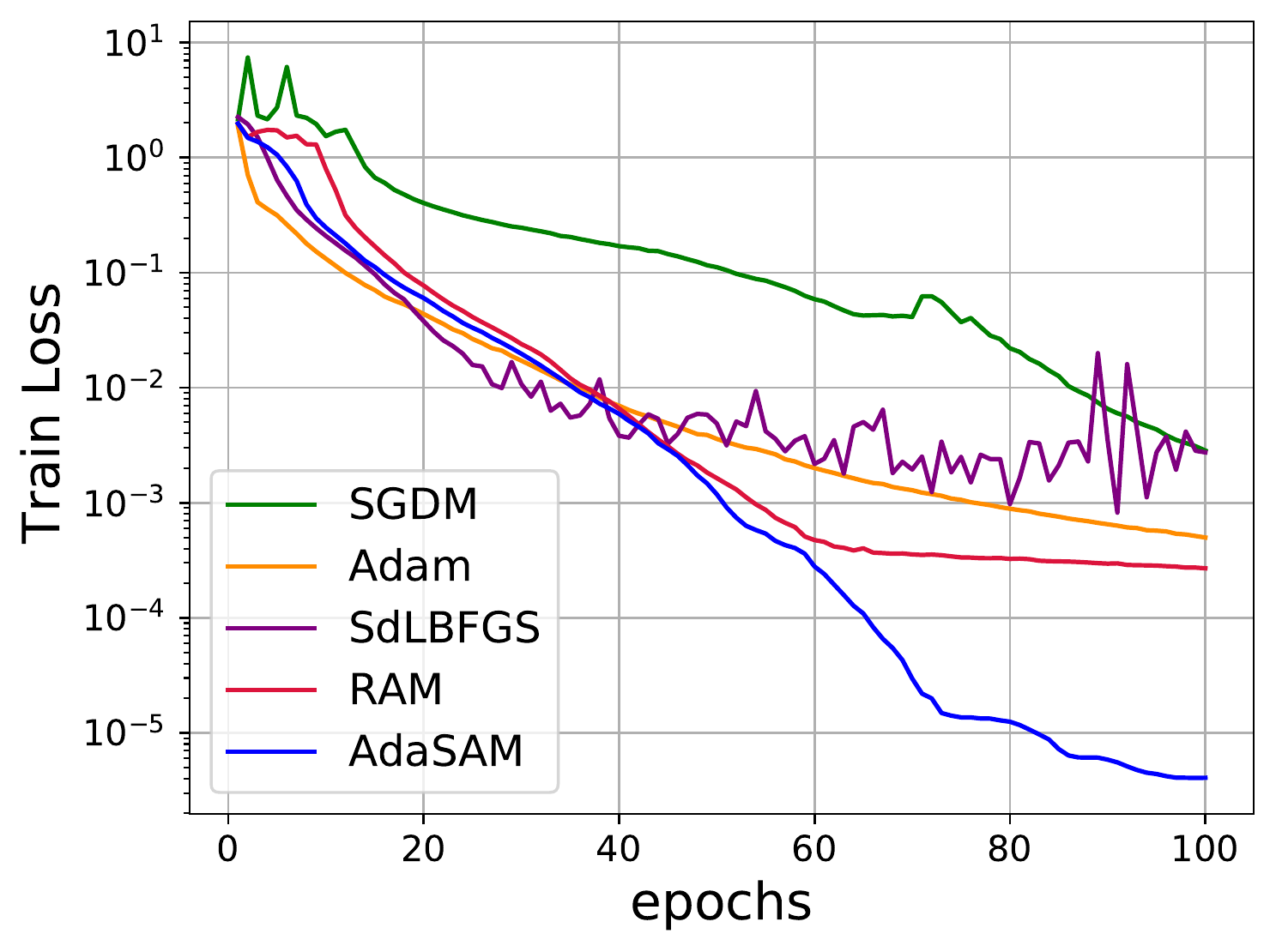}
}
\subfigure[Variance reduction]{
\includegraphics[width=0.23\textwidth]{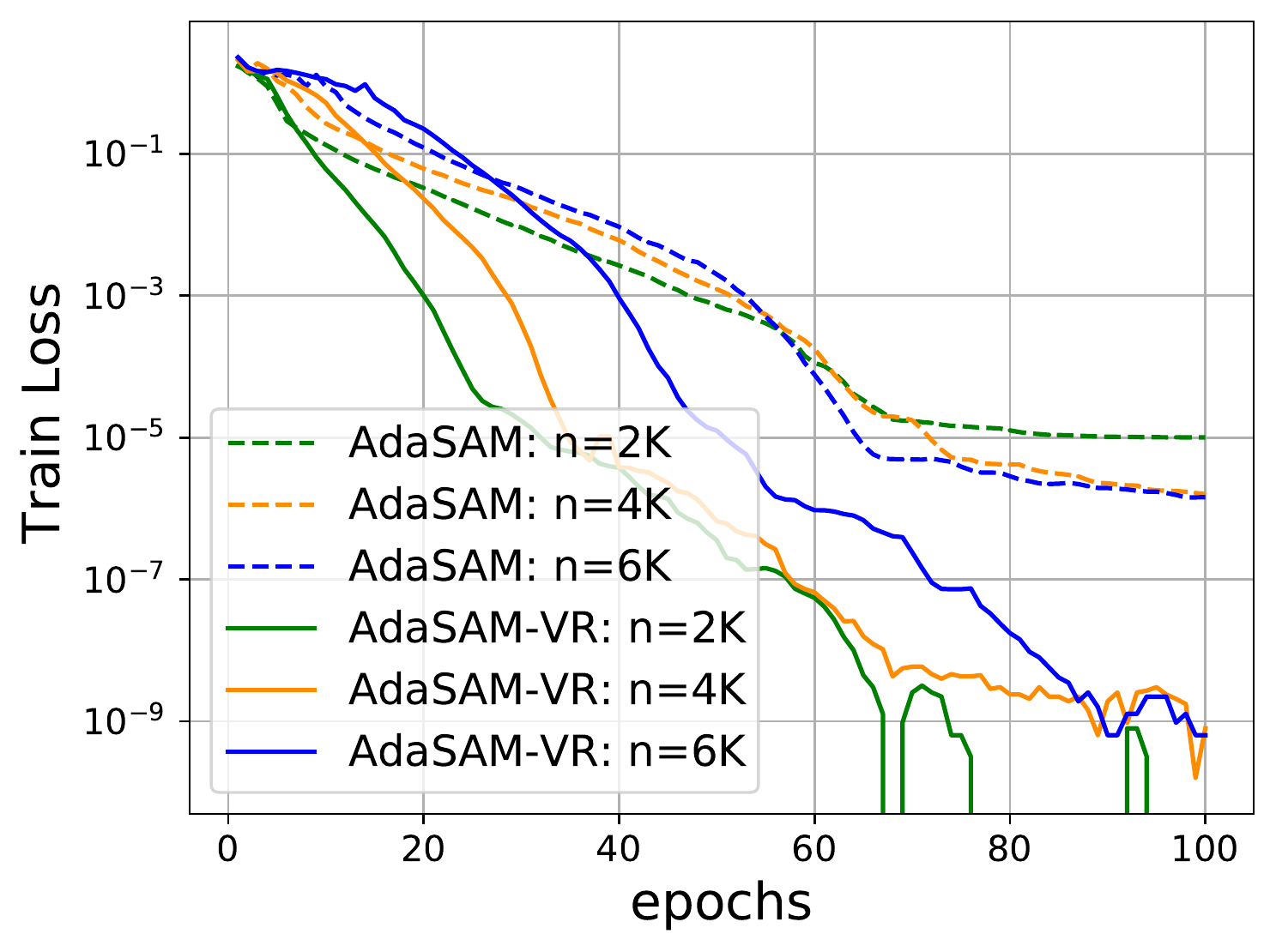}
}
\subfigure[Preconditioning]{
\includegraphics[width=0.23\textwidth]{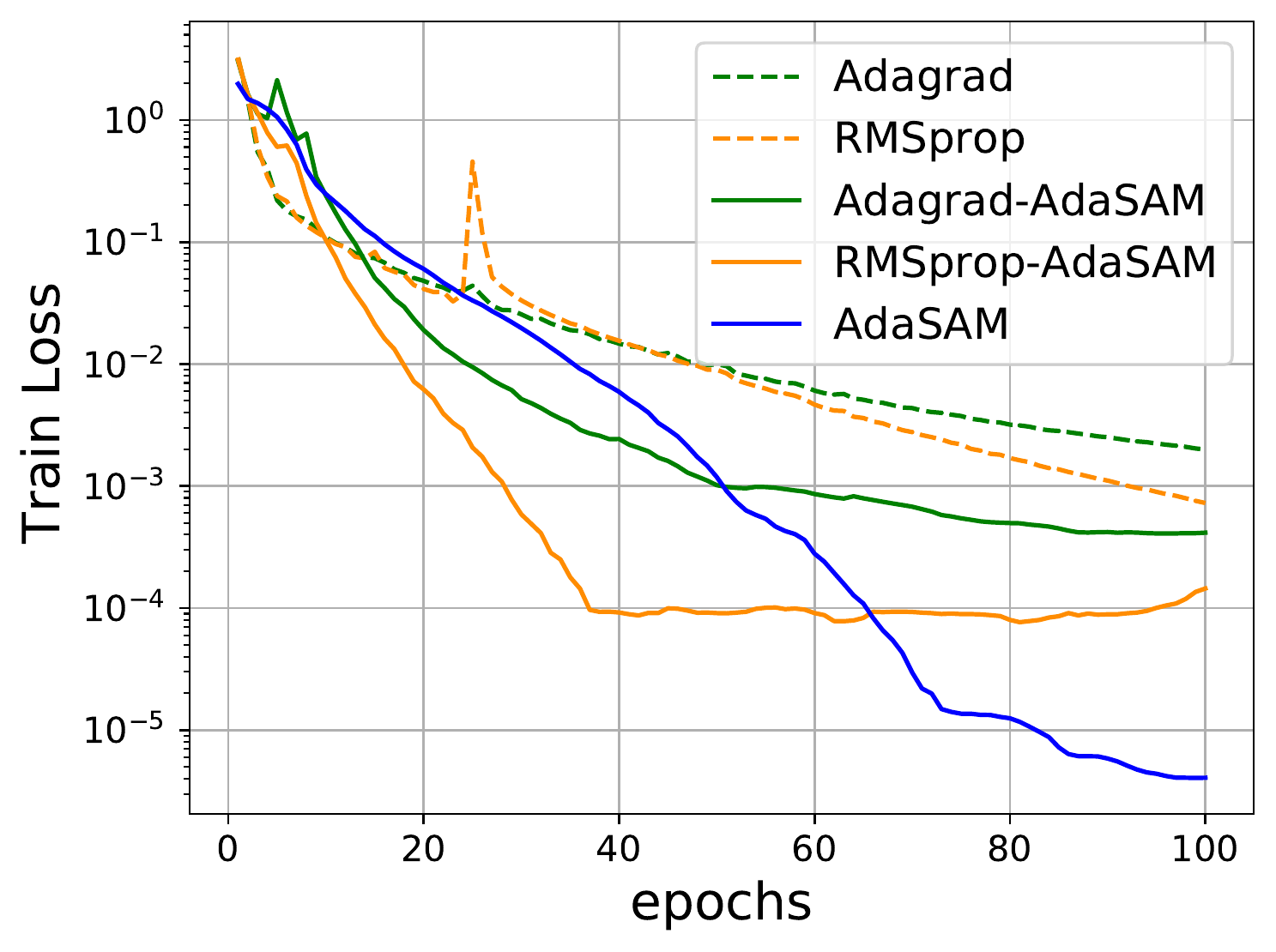}
}
\caption{Experiments on MNIST. (a) Train Loss using batchsize = 6K; (b) Train Loss using batchsize = 3K; (c) AdaSAM with variance reduction; (d) Preconditioned AdaSAM with batchsize = 3K. }
\label{fig:mnist_1}
\end{figure}

 Figure~\ref{fig:mnist_1} (a) and (b) show the curves of training loss when training 100 epochs with batch sizes of 6K and 3K, which indicate that AdaSAM can significantly minimize the empirical risk in large mini-batch training. The comparison with RAM verifies the benefit of adaptive regularization. We also notice that there hardly exists any oscillation in AdaSAM during training except for the first few epochs, which  demonstrates AdaSAM's tolerance to noise. We also tested the effectiveness of variance reduction and preconditioning introduced in Section~\ref{subsec:enhance}. The variance reduced extension of AdaSAM is denoted as AdaSAM-VR and was compared with AdaSAM for different batch sizes. The variants of AdaSAM preconditioned by Adagrad  \citep{duchi2011adaptive} and RMSprop \citep{tieleman2012lecture} are denoted as Adagrad-AdaSAM and RMSprop-AdaSAM respectively. Though AdaSAM-VR demands more gradient evaluations and the preconditioned variants seem to deteriorate the final training loss, we point out that AdaSAM-VR can achieve lower training loss ($ 10^{-9} $) and the preconditioned variants   converge faster to an acceptable training loss (e.g. $ 10^{-3} $).


\begin{table*}[ht]
\centering
\caption{ Training various neural networks on CIFAR10/CIFAR100. The entry (mean $\pm$ standard deviation) denotes the percentage of the final TOP1 accuracy in test dataset. The bold numbers highlight the best results. WideResNet is abbreviated as WResNet. }
\label{table:cifar}
\resizebox{\textwidth}{!}{
\begin{tabular}{l c c c c c c c c c}
\toprule
 \multirow{2}{*}{Method} & \multicolumn{6}{c}{CIFAR10} & \multicolumn{3}{c}{CIFAR100}\\
 \cmidrule(lr){2-7} \cmidrule(lr){8-10}
 & ResNet18 & ResNet20 & ResNet32 & ResNet44 & ResNet56 & WResNet & ResNet18  & ResNeXt & DenseNet \\
 \cmidrule(lr){2-7} \cmidrule(lr){8-10}
 SGDM & 94.82$\pm$.15 & 92.03$\pm$.16 & 92.86$\pm$.15 & 93.10$\pm$.23 & 93.47$\pm$.28 & 94.90$\pm$.09 & 77.27$\pm$.09 & 78.41$\pm$.54 & 78.49$\pm$.12 \\
 Adam & 93.03$\pm$.07 & 91.17$\pm$.13 & 92.03$\pm$.28 & 92.28$\pm$.62 & 92.39$\pm$.23 & 92.45$\pm$.11 & 72.41$\pm$.17 & 73.57$\pm$.17 & 70.80$\pm$.23 \\
 AdaBelief & 94.65$\pm$.13 & 91.15$\pm$.21 & 92.15$\pm$.17 & 92.79$\pm$.24 & 93.30$\pm$.07 & 94.46$\pm$.13 & 76.25$\pm$.06 & 78.27$\pm$.16 & 78.83$\pm$.15 \\
 Lookahead & 94.92$\pm$.33 & 92.07$\pm$.04 & 92.86$\pm$.15 &  93.26$\pm$.24 & 93.36$\pm$.13 & 94.90$\pm$.15 & 77.63$\pm$.35 & 78.93$\pm$.12 & 79.37$\pm$.16 \\
 RNA & 93.45$\pm$.21 & 90.73$\pm$.12 & 91.08$\pm$.51 & 91.61$\pm$.37 & 91.23$\pm$.14 & 93.85$\pm$.24 & 75.12$\pm$.39 & 75.88$\pm$.40 & 75.70$\pm$.49 \\
RAM & 95.10$\pm$.05 & 92.21$\pm$.09 & 93.05$\pm$.43 & 93.42$\pm$.13 & 93.76$\pm$.16 &  95.04$\pm$.09 & 76.19$\pm$.12 & 78.65$\pm$.20 & 78.28$\pm$.62 \\
AdaSAM & \textbf{95.17$\pm$.10} & \textbf{92.43$\pm$.19} & \textbf{93.22$\pm$.32} & \textbf{93.57$\pm$.14} & \textbf{93.77$\pm$.12} & \textbf{95.23$\pm$.07} & \textbf{78.13$\pm$.14} & 79.31$\pm$.27 & \textbf{80.09$\pm$.52} \\
 \cmidrule(lr){2-7} \cmidrule(lr){8-10}
AdaSAM-SGD & 95.04$\pm$.22 & 92.26$\pm$.10 & 92.92$\pm$.28 & 93.01$\pm$.15 & 93.71$\pm$.15 & 94.99$\pm$.19 & 77.81$\pm$.12 & \textbf{79.47$\pm$.44} & 79.58$\pm$.39 \\
AdaSAM-Adam & 93.86$\pm$.23 & 92.27$\pm$.29 & 92.67$\pm$.09 & 92.94$\pm$.30 & 93.22$\pm$.12 & 93.88$\pm$.20 & 74.46$\pm$.51 & 75.34$\pm$.20 & 75.21$\pm$.49 \\
\bottomrule
\end{tabular}}
\end{table*}

 \begin{wrapfigure}{R}{0.42\textwidth}
\includegraphics[width=0.42\textwidth]{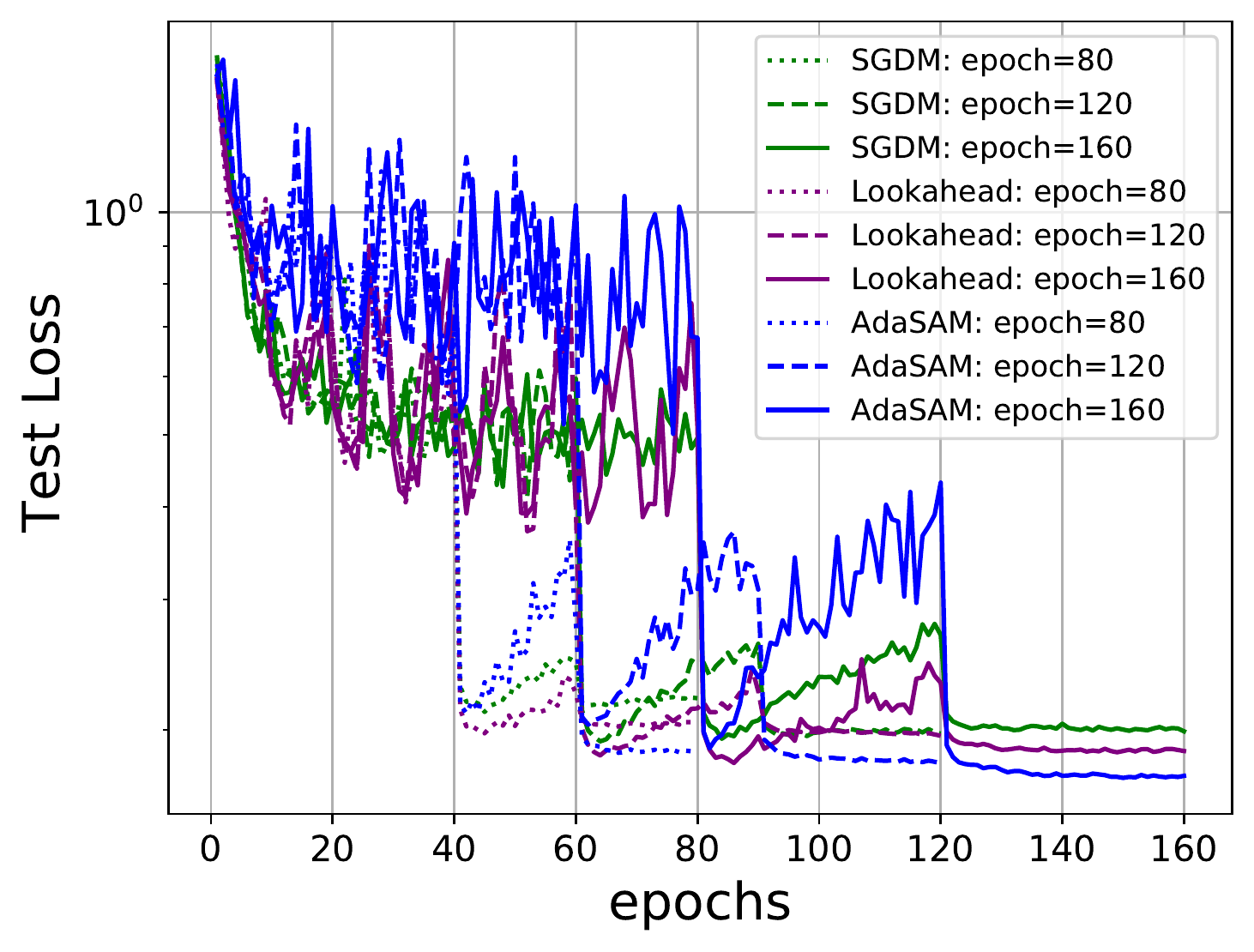}
\caption{Test loss of training ResNet18 on CIFAR-10.} \label{fig:epochs}
\end{wrapfigure}
 
 \textbf{Experiments on CIFAR}. 
 For CIFAR-10 and CIFAR-100, both datasets have 50K images for training and 10K images for test.  The test accuracy at the final epoch was reported as the evaluation metric. We trained ResNet18/20/32/44/56 \cite{he2016deep} and WideResNet16-4 \cite{zagoruyko2016wide} on CIFAR-10, and ResNet18, ResNeXt50 \cite{xie2017aggregated} and DenseNet121 \cite{huang2017densely} on CIFAR-100. The baseline optimizers were SGDM, Adam, AdaBelief \cite{zhuang2020adabelief}, Lookahead \cite{zhang2019lookahead} and RNA \cite{scieur2018nonlinear}. 
 The hyperparameters were kept unchanged across different tests.
 We trained 160 epochs with batch size of 128 and decayed the learning rate at the 80th and 120th epoch. For AdaSAM/RAM, $\alpha_k$ and $\beta_k$ were decayed  at the 80th and 120th epoch.
 
 Table~\ref{table:cifar} demonstrates the   generalization ability of AdaSAM. 
 Compared with SGDM/Lookahead, AdaSAM is built on a noisy quadratic model to extrapolate historical iterates more elaborately, which may explore more information from history.  We also conducted tests on training for 120 epochs and 80 epochs.  Figure~\ref{fig:epochs} shows AdaSAM can achieve comparable or even lower test loss than SGDM/Lookahead when training with fewer epochs, thus saving large number of iterations.   We point out that the slow convergence before the first learning rate decay is attributed to the fact that we use a much larger weight-decay for AdaSAM ($ 1.5\times 10^{-3} $ vs. $ 5\times 10^{-4} $ for SGDM) and large learning rate ($ \alpha_0=\beta_0=1$) which may slow down training but help generalize \citep{li2019towards}.  

 We also explore the scheme of alternating iterations: given an optimizer $ optim $ , in each cycle, we iterate with $ optim $ for $ (p-1) $ steps and then apply AdaSAM in the $ p $-th step, the result of which is the starting point of the next cycle. 
  We tested  vanilla SGD (momentum = 0) alternated with AdaSAM, and Adam alternated with AdaSAM, which are denoted as AdaSAM-SGD and AdaSAM-Adam respectively. The number of steps of one cycle is 5. Results listed at the bottom of Table~\ref{table:cifar} show AdaSAM-Adam gives a thorough improvement over Adam. AdaSAM-SGD can even beat AdaSAM on CIFAR-100/ResNeXt50, and exceed the test accuracy of SGDM by 1.06\%. Hence alternating iteration can reduce computational overhead while achieving comparable test accuracy.

\begin{figure}[ht]
\centering 
\subfigure[1-Layer LSTM]{
\includegraphics[width=0.31\textwidth]{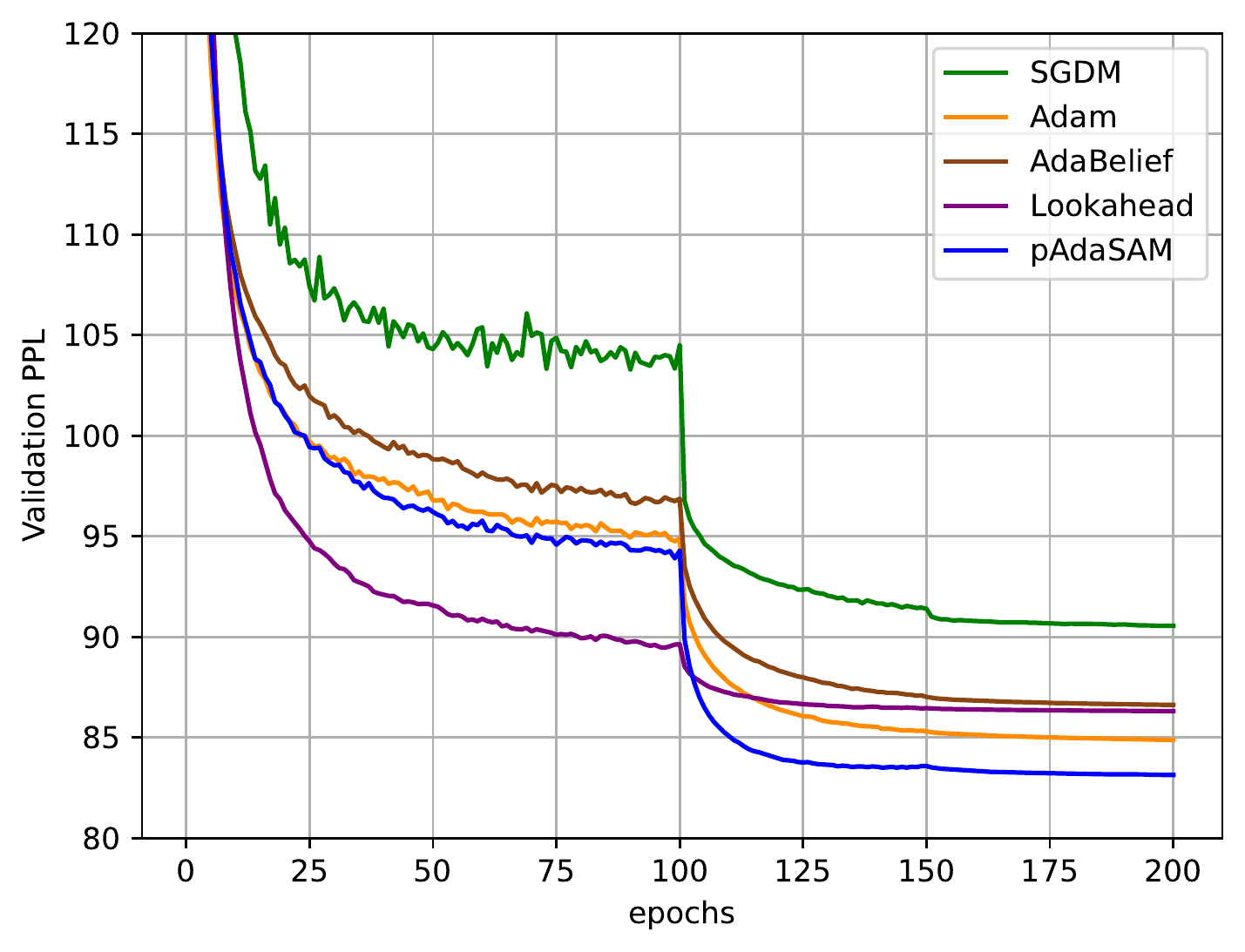}
}
\subfigure[2-Layer LSTM]{
\includegraphics[width=0.31\textwidth]{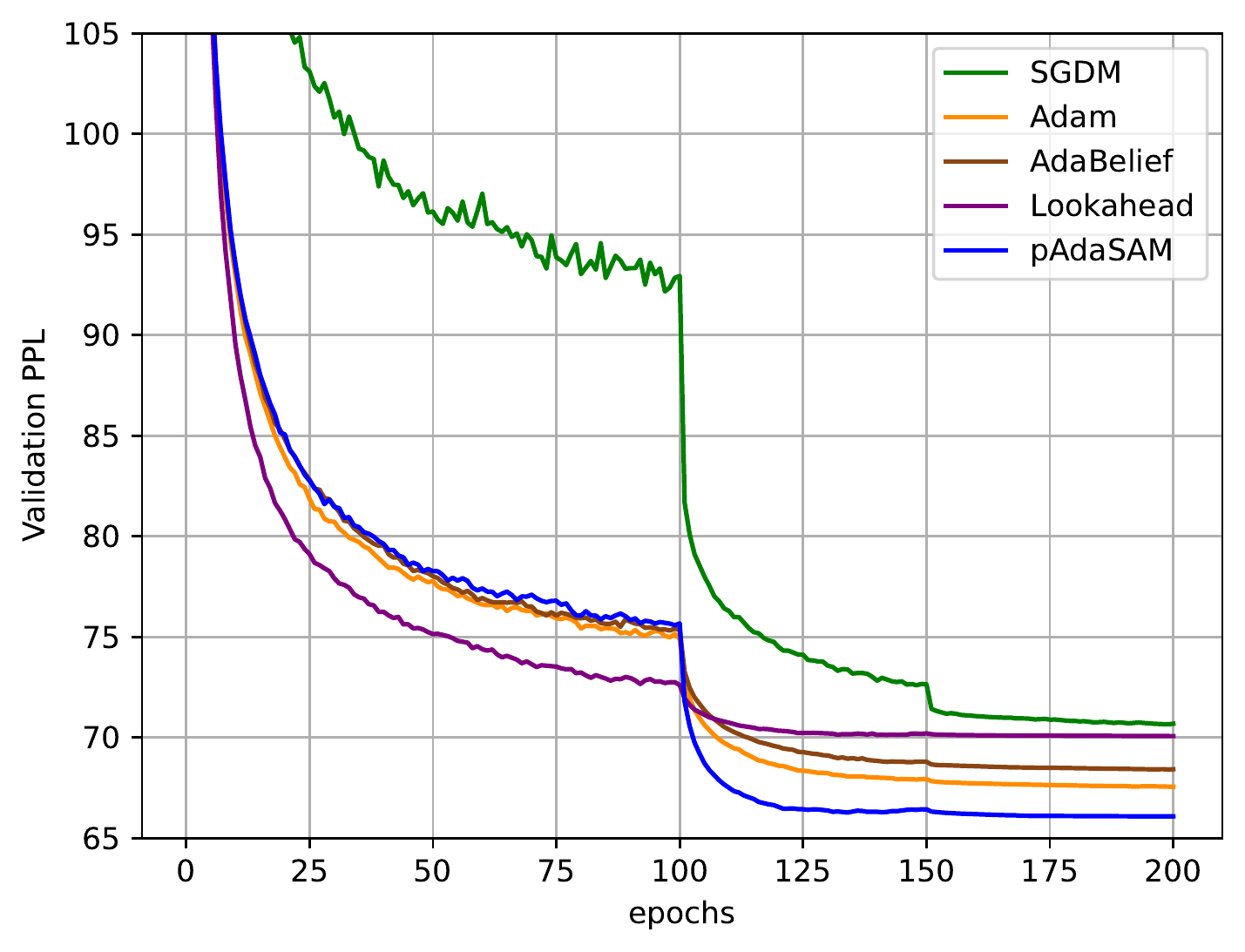}
}
\subfigure[3-Layer LSTM]{
\includegraphics[width=0.31\textwidth]{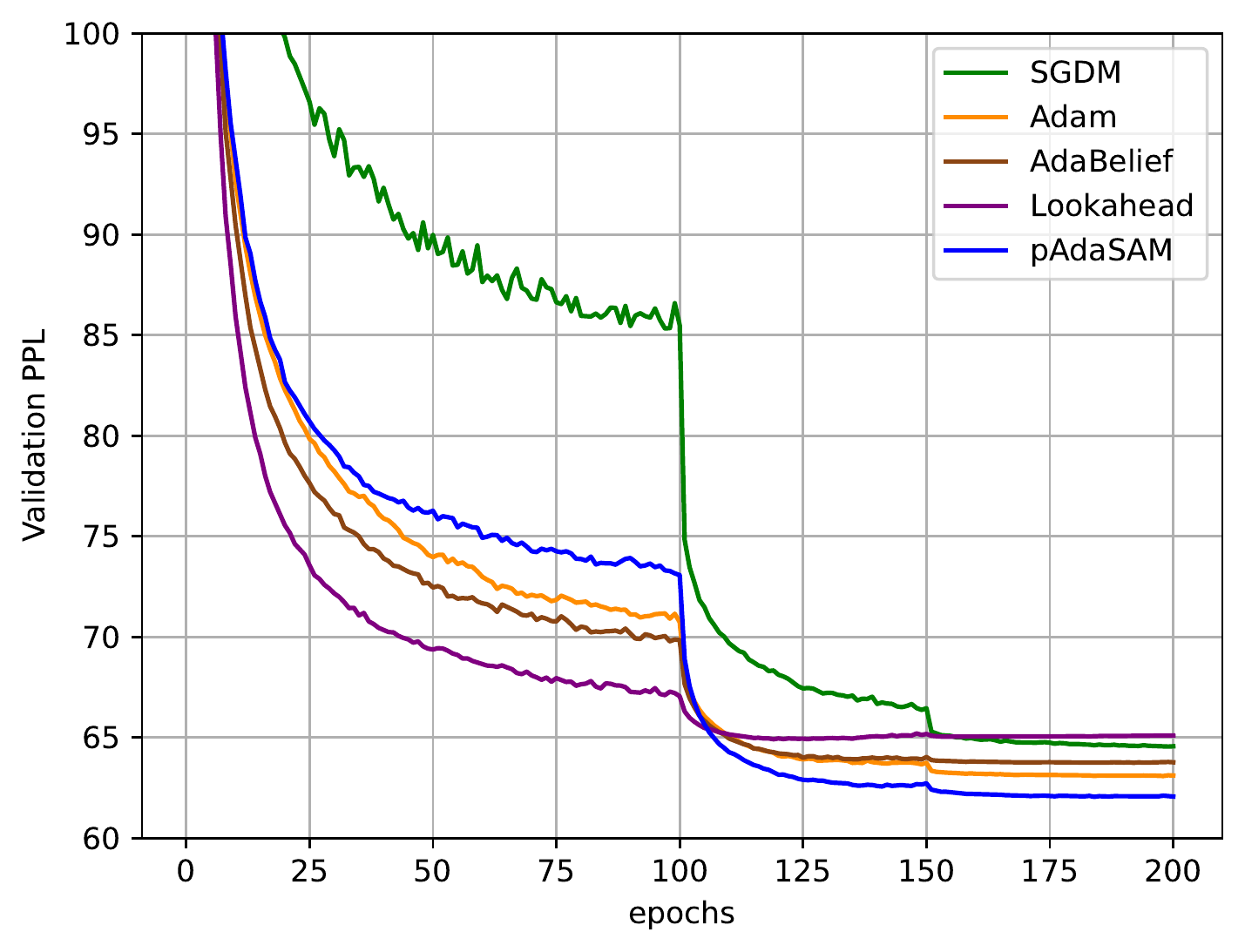}
}
\caption{Experiments on Penn TreeBank. Validation loss of training 1,2,3-Layer LSTM.}
\label{fig:ptb}
\end{figure}

 \begin{wrapfigure}{R}{0.51\textwidth}
 \makeatletter\def\@captype{table}\makeatother
  \caption{Test perplexity on Penn TreeBank for 1,2,3-layer LSTM. Lower is better.} \label{table:ptb}
  \centering
  \resizebox{0.5\textwidth}{!}{
  \begin{tabular}{lccc} 
    \toprule   	
    Method     & 1-Layer     & 2-Layer & 3-Layer \\
    \midrule
    SGDM 		 &  85.21$\pm$.36   & 67.12$\pm$.14     & 61.56$\pm$.14 \\
    Adam     	 &  80.88$\pm$.15   & 64.54$\pm$.18     & 60.34$\pm$.22 \\
    AdaBelief     & 82.41$\pm$.46   & 65.07$\pm$.02    & 60.64$\pm$.14 \\
    Lookahead     & 82.01$\pm$.07   & 66.43$\pm$.33   & 61.80$\pm$.10 \\
    pAdaSAM      & \textbf{79.34$\pm$.09}    & \textbf{63.18$\pm$.22}   & \textbf{59.47$\pm$.08} \\
    \bottomrule
  \end{tabular}
  }
\end{wrapfigure}

\textbf{Experiments on Penn TreeBank}. We trained LSTM on Penn TreeBank and reported the perplexity on the validation set in Figure~\ref{fig:ptb} and testing set in Table~\ref{table:ptb}, where pAdaSAM denotes the variant of AdaSAM preconditioned by Adam. The experimental setting is the same as that in AdaBelief \citep{zhuang2020adabelief}. In our practice, we find that the vanilla AdaSAM with default hyperparameter setting is not suitable for this task.  Nevertheless, a suitable preconditioner (e.g. Adam) can largely improve the behaviour of AdaSAM. Conversely, AdaSAM can also enhance a optimizer when the latter is used as a preconditioner. 

\section{Related work}
AM is  well-known in scientific computing \cite{toth2015convergence, brezinski2018shanks, suryanarayana2019alternating}. 
 Recently, AM also receives attention in machine learning, such as for accelerating EM algorithm \cite{henderson2019damped}, proximal gradient methods \cite{mai2020anderson} and reinforcement learning \cite{geist2018anderson}, but there is no general theoretical convergence analysis. 
 AM utilizes historical iterations through the projection and mixing step.  Besides, there are also other research works which exploit historical information with tools from machine learning, such as using guided policy search \cite{li2016learning} and RNN \cite{andrychowicz2016learning}. However, these methods are much more complicated to implement and apply in practice and the  mechanisms underlying these methods are difficult to interpret.

\section{Conclusion}
 In this paper, we develop an extension of Anderson mixing, namely Stochastic Anderson Mixing, for nonconvex stochastic optimization. By introducing damped projection and adaptive regularization, we establish the convergence theory of our new method. We also analyze its work complexity in terms of $ \mathcal{SFO} $-calls    and show  it can achieve the $ O(1/\epsilon^2)$ complexity for an $\epsilon-$accurate solution. We also give a specific form of adaptive regularization. Then we propose two techniques to further enhance our method. One is the  variance reduction technique, which can further improve the work complexity of our method theoretically and help achieve lower empirical risk in our experiments. The other one is the preconditioned mixing strategy that  directly extends Anderson mixing.  
 Experiments show encouraging results of our method and its enhanced version in terms of convergence rate or generalization ability in training different neural networks in different machine learning tasks.  These results confirm the suitability of Anderson mixing for nonconvex stochastic optimization.

\bibliographystyle{plainnat}
\bibliography{neurips_2021}

\begin{thebibliography}{65}
\providecommand{\natexlab}[1]{#1}
\providecommand{\url}[1]{\texttt{#1}}
\expandafter\ifx\csname urlstyle\endcsname\relax
  \providecommand{\doi}[1]{doi: #1}\else
  \providecommand{\doi}{doi: \begingroup \urlstyle{rm}\Url}\fi

\bibitem[Allen-Zhu and Hazan(2016)]{allen2016variance}
Zeyuan Allen-Zhu and Elad Hazan.
\newblock Variance reduction for faster non-convex optimization.
\newblock In \emph{International conference on machine learning}, pages
  699--707. PMLR, 2016.

\bibitem[Anderson(1965)]{anderson1965iterative}
Donald~G Anderson.
\newblock Iterative procedures for nonlinear integral equations.
\newblock \emph{Journal of the ACM (JACM)}, 12\penalty0 (4):\penalty0 547--560,
  1965.

\bibitem[Andrychowicz et~al.(2016)Andrychowicz, Denil, Colmenarejo, Hoffman,
  Pfau, Schaul, Shillingford, and de~Freitas]{andrychowicz2016learning}
Marcin Andrychowicz, Misha Denil, Sergio~G{\'o}mez Colmenarejo, Matthew~W
  Hoffman, David Pfau, Tom Schaul, Brendan Shillingford, and Nando de~Freitas.
\newblock Learning to learn by gradient descent by gradient descent.
\newblock In \emph{Proceedings of the 30th International Conference on Neural
  Information Processing Systems}, pages 3988--3996, 2016.

\bibitem[Bian et~al.(2021)Bian, Chen, and Kelley]{bian2021anderson}
Wei Bian, Xiaojun Chen, and CT~Kelley.
\newblock Anderson acceleration for a class of nonsmooth fixed-point problems.
\newblock \emph{SIAM Journal on Scientific Computing}, \penalty0 (0):\penalty0
  S1--S20, 2021.

\bibitem[Bottou et~al.(2018)Bottou, Curtis, and Nocedal]{Bottou2018Optim}
Léon Bottou, Frank~E. Curtis, and Jorge Nocedal.
\newblock Optimization methods for large-scale machine learning.
\newblock \emph{SIAM Review}, 60\penalty0 (2):\penalty0 223--311, 2018.
\newblock \doi{10.1137/16M1080173}.
\newblock URL \url{https://doi.org/10.1137/16M1080173}.

\bibitem[Brezinski et~al.(2018)Brezinski, Redivo-Zaglia, and
  Saad]{brezinski2018shanks}
Claude Brezinski, Michela Redivo-Zaglia, and Yousef Saad.
\newblock Shanks sequence transformations and anderson acceleration.
\newblock \emph{SIAM Review}, 60\penalty0 (3):\penalty0 646--669, 2018.

\bibitem[Brezinski et~al.(2020)Brezinski, Cipolla, Redivo-Zaglia, and
  Saad]{brezinski2020shanks}
Claude Brezinski, Stefano Cipolla, Michela Redivo-Zaglia, and Yousef Saad.
\newblock Shanks and anderson-type acceleration techniques for systems of
  nonlinear equations.
\newblock \emph{arXiv preprint arXiv:2007.05716}, 2020.

\bibitem[Byrd et~al.(2016)Byrd, Hansen, Nocedal, and
  Singer]{byrd2016stochastic}
Richard~H Byrd, Samantha~L Hansen, Jorge Nocedal, and Yoram Singer.
\newblock A stochastic quasi-newton method for large-scale optimization.
\newblock \emph{SIAM Journal on Optimization}, 26\penalty0 (2):\penalty0
  1008--1031, 2016.

\bibitem[Canc{\`e}s and Le~Bris(2000)]{cances2000can}
Eric Canc{\`e}s and Claude Le~Bris.
\newblock Can we outperform the diis approach for electronic structure
  calculations?
\newblock \emph{International Journal of Quantum Chemistry}, 79\penalty0
  (2):\penalty0 82--90, 2000.

\bibitem[Chung(1954)]{chung1954stochastic}
Kai~Lai Chung.
\newblock On a stochastic approximation method.
\newblock \emph{The Annals of Mathematical Statistics}, pages 463--483, 1954.

\bibitem[Duchi et~al.(2011)Duchi, Hazan, and Singer]{duchi2011adaptive}
John Duchi, Elad Hazan, and Yoram Singer.
\newblock Adaptive subgradient methods for online learning and stochastic
  optimization.
\newblock \emph{Journal of machine learning research}, 12\penalty0 (7), 2011.

\bibitem[Durrett(2019)]{durrett2019probability}
Rick Durrett.
\newblock \emph{Probability: theory and examples}, volume~49.
\newblock Cambridge university press, 2019.

\bibitem[Fang and Saad(2009)]{fang2009two}
Haw-ren Fang and Yousef Saad.
\newblock Two classes of multisecant methods for nonlinear acceleration.
\newblock \emph{Numerical Linear Algebra with Applications}, 16\penalty0
  (3):\penalty0 197--221, 2009.

\bibitem[Friedman et~al.(2001)Friedman, Hastie, Tibshirani,
  et~al.]{friedman2001elements}
Jerome Friedman, Trevor Hastie, Robert Tibshirani, et~al.
\newblock \emph{The elements of statistical learning}, volume~1.
\newblock Springer series in statistics New York, 2001.

\bibitem[Garza and Scuseria(2012)]{garza2012comparison}
Alejandro~J Garza and Gustavo~E Scuseria.
\newblock Comparison of self-consistent field convergence acceleration
  techniques.
\newblock \emph{The Journal of chemical physics}, 137\penalty0 (5):\penalty0
  054110, 2012.

\bibitem[Geist and Scherrer(2018)]{geist2018anderson}
Matthieu Geist and Bruno Scherrer.
\newblock Anderson acceleration for reinforcement learning.
\newblock In \emph{EWRL 2018-4th European workshop on Reinforcement Learning},
  2018.

\bibitem[Ghadimi and Lan(2013)]{ghadimi2013stochastic}
Saeed Ghadimi and Guanghui Lan.
\newblock Stochastic first-and zeroth-order methods for nonconvex stochastic
  programming.
\newblock \emph{SIAM Journal on Optimization}, 23\penalty0 (4):\penalty0
  2341--2368, 2013.

\bibitem[Golub and Van~Loan(2013)]{golub2013matrix}
Gene~H Golub and Charles~F Van~Loan.
\newblock Matrix computations, 4th.
\newblock \emph{Johns Hopkins}, 2013.

\bibitem[Gower et~al.(2019)Gower, Kovalev, Lieder, and
  Richt{\'a}rik]{gower2019rsn}
Robert Gower, Dmitry Kovalev, Felix Lieder, and Peter Richt{\'a}rik.
\newblock Rsn: Randomized subspace newton.
\newblock In \emph{Conference on Neural Information Processing Systems}, 2019.

\bibitem[Gower and Richt{\'a}rik(2017)]{gower2017randomized}
Robert~M Gower and Peter Richt{\'a}rik.
\newblock Randomized quasi-newton updates are linearly convergent matrix
  inversion algorithms.
\newblock \emph{SIAM Journal on Matrix Analysis and Applications}, 38\penalty0
  (4):\penalty0 1380--1409, 2017.

\bibitem[He et~al.(2016)He, Zhang, Ren, and Sun]{he2016deep}
Kaiming He, Xiangyu Zhang, Shaoqing Ren, and Jian Sun.
\newblock Deep residual learning for image recognition.
\newblock In \emph{Proceedings of the IEEE conference on computer vision and
  pattern recognition}, pages 770--778, 2016.

\bibitem[Henderson and Varadhan(2019)]{henderson2019damped}
Nicholas~C Henderson and Ravi Varadhan.
\newblock Damped anderson acceleration with restarts and monotonicity control
  for accelerating em and em-like algorithms.
\newblock \emph{Journal of Computational and Graphical Statistics}, 28\penalty0
  (4):\penalty0 834--846, 2019.

\bibitem[Horn and Johnson(2012)]{horn2012matrix}
Roger~A Horn and Charles~R Johnson.
\newblock \emph{Matrix analysis}.
\newblock Cambridge university press, 2012.

\bibitem[Huang et~al.(2017)Huang, Liu, Van Der~Maaten, and
  Weinberger]{huang2017densely}
Gao Huang, Zhuang Liu, Laurens Van Der~Maaten, and Kilian~Q Weinberger.
\newblock Densely connected convolutional networks.
\newblock In \emph{Proceedings of the IEEE conference on computer vision and
  pattern recognition}, pages 4700--4708, 2017.

\bibitem[Huang et~al.(2018)Huang, Sun, and Wu]{huang2018stochastic}
Shaojun Huang, Yuanzhang Sun, and Qiuwei Wu.
\newblock Stochastic economic dispatch with wind using versatile probability
  distribution and l-bfgs-b based dual decomposition.
\newblock \emph{IEEE Transactions on Power Systems}, 33\penalty0 (6):\penalty0
  6254--6263, 2018.

\bibitem[Huang et~al.(2020)Huang, Liang, Liu, Li, Yu, and Li]{huang2020span}
Xunpeng Huang, Xianfeng Liang, Zhengyang Liu, Lei Li, Yue Yu, and Yitan Li.
\newblock Span: A stochastic projected approximate newton method.
\newblock In \emph{Proceedings of the AAAI Conference on Artificial
  Intelligence}, volume~34, pages 1520--1527, 2020.

\bibitem[Irwin and Haber(2020)]{irwin2020secant}
Brian Irwin and Eldad Haber.
\newblock Secant penalized bfgs: A noise robust quasi-newton method via
  penalizing the secant condition.
\newblock \emph{arXiv preprint arXiv:2010.01275}, 2020.

\bibitem[Johnson and Zhang(2013)]{johnson2013accelerating}
Rie Johnson and Tong Zhang.
\newblock Accelerating stochastic gradient descent using predictive variance
  reduction.
\newblock \emph{Advances in neural information processing systems},
  26:\penalty0 315--323, 2013.

\bibitem[Kingma and Ba(2014)]{kingma2014adam}
Diederik~P Kingma and Jimmy Ba.
\newblock Adam: A method for stochastic optimization.
\newblock \emph{arXiv preprint arXiv:1412.6980}, 2014.

\bibitem[Korpelevich(1976)]{korpelevich1976extragradient}
Galina~M Korpelevich.
\newblock The extragradient method for finding saddle points and other
  problems.
\newblock \emph{Matecon}, 12:\penalty0 747--756, 1976.

\bibitem[Krizhevsky et~al.(2009)Krizhevsky, Hinton,
  et~al.]{krizhevsky2009learning}
Alex Krizhevsky, Geoffrey Hinton, et~al.
\newblock Learning multiple layers of features from tiny images.
\newblock 2009.

\bibitem[LeCun et~al.(1998)LeCun, Bottou, Bengio, and
  Haffner]{lecun1998gradient}
Yann LeCun, L{\'e}on Bottou, Yoshua Bengio, and Patrick Haffner.
\newblock Gradient-based learning applied to document recognition.
\newblock \emph{Proceedings of the IEEE}, 86\penalty0 (11):\penalty0
  2278--2324, 1998.

\bibitem[Li and Malik(2016)]{li2016learning}
Ke~Li and Jitendra Malik.
\newblock Learning to optimize.
\newblock \emph{arXiv preprint arXiv:1606.01885}, 2016.

\bibitem[Li et~al.(2019)Li, Wei, and Ma]{li2019towards}
Yuanzhi Li, Colin Wei, and Tengyu Ma.
\newblock Towards explaining the regularization effect of initial large
  learning rate in training neural networks.
\newblock \emph{arXiv preprint arXiv:1907.04595}, 2019.

\bibitem[Mai and Johansson(2020)]{mai2020anderson}
Vien Mai and Mikael Johansson.
\newblock Anderson acceleration of proximal gradient methods.
\newblock In \emph{International Conference on Machine Learning}, pages
  6620--6629. PMLR, 2020.

\bibitem[Marcus et~al.(1993)Marcus, Santorini, and
  Marcinkiewicz]{marcus1993building}
Mitchell Marcus, Beatrice Santorini, and Mary~Ann Marcinkiewicz.
\newblock Building a large annotated corpus of english: The penn treebank.
\newblock 1993.

\bibitem[Martens(2010)]{martens2010deep}
James Martens.
\newblock Deep learning via hessian-free optimization.
\newblock In \emph{ICML}, volume~27, pages 735--742, 2010.

\bibitem[Mokhtari and Ribeiro(2020)]{mokhtari2020stochastic}
Aryan Mokhtari and Alejandro Ribeiro.
\newblock Stochastic quasi-newton methods.
\newblock \emph{Proceedings of the IEEE}, 108\penalty0 (11):\penalty0
  1906--1922, 2020.

\bibitem[Nemirovski et~al.(2009)Nemirovski, Juditsky, Lan, and
  Shapiro]{nemirovski2009robust}
Arkadi Nemirovski, Anatoli Juditsky, Guanghui Lan, and Alexander Shapiro.
\newblock Robust stochastic approximation approach to stochastic programming.
\newblock \emph{SIAM Journal on optimization}, 19\penalty0 (4):\penalty0
  1574--1609, 2009.

\bibitem[Polyak(1990)]{polyak1990new}
Boris~T Polyak.
\newblock New stochastic approximation type procedures.
\newblock \emph{Automat. i Telemekh}, 7\penalty0 (98-107):\penalty0 2, 1990.

\bibitem[Polyak and Juditsky(1992)]{polyak1992acceleration}
Boris~T Polyak and Anatoli~B Juditsky.
\newblock Acceleration of stochastic approximation by averaging.
\newblock \emph{SIAM journal on control and optimization}, 30\penalty0
  (4):\penalty0 838--855, 1992.

\bibitem[Potra and Engler(2013)]{potra2013characterization}
Florian~A Potra and Hans Engler.
\newblock A characterization of the behavior of the anderson acceleration on
  linear problems.
\newblock \emph{linear Algebra and its Applications}, 438\penalty0
  (3):\penalty0 1002--1011, 2013.

\bibitem[Pratapa et~al.(2016)Pratapa, Suryanarayana, and
  Pask]{pratapa2016anderson}
Phanisri~P Pratapa, Phanish Suryanarayana, and John~E Pask.
\newblock Anderson acceleration of the jacobi iterative method: An efficient
  alternative to krylov methods for large, sparse linear systems.
\newblock \emph{Journal of Computational Physics}, 306:\penalty0 43--54, 2016.

\bibitem[Reddi et~al.(2016)Reddi, Hefny, Sra, Poczos, and
  Smola]{reddi2016stochastic}
Sashank~J Reddi, Ahmed Hefny, Suvrit Sra, Barnabas Poczos, and Alex Smola.
\newblock Stochastic variance reduction for nonconvex optimization.
\newblock In \emph{International conference on machine learning}, pages
  314--323. PMLR, 2016.

\bibitem[Robbins and Monro(1951)]{robbins1951stochastic}
Herbert Robbins and Sutton Monro.
\newblock A stochastic approximation method.
\newblock \emph{The annals of mathematical statistics}, pages 400--407, 1951.

\bibitem[Saad and Schultz(1986)]{saad1986gmres}
Youcef Saad and Martin~H Schultz.
\newblock Gmres: A generalized minimal residual algorithm for solving
  nonsymmetric linear systems.
\newblock \emph{SIAM Journal on scientific and statistical computing},
  7\penalty0 (3):\penalty0 856--869, 1986.

\bibitem[Saad(2003)]{saad2003iterative}
Yousef Saad.
\newblock \emph{Iterative methods for sparse linear systems}.
\newblock SIAM, 2003.

\bibitem[Sacks(1958)]{sacks1958asymptotic}
Jerome Sacks.
\newblock Asymptotic distribution of stochastic approximation procedures.
\newblock \emph{The Annals of Mathematical Statistics}, 29\penalty0
  (2):\penalty0 373--405, 1958.

\bibitem[Schraudolph et~al.(2007)Schraudolph, Yu, and
  G{\"u}nter]{schraudolph2007stochastic}
Nicol~N Schraudolph, Jin Yu, and Simon G{\"u}nter.
\newblock A stochastic quasi-newton method for online convex optimization.
\newblock In \emph{Artificial intelligence and statistics}, pages 436--443,
  2007.

\bibitem[Scieur et~al.(2016)Scieur, d'Aspremont, and
  Bach]{scieur2016regularized}
Damien Scieur, Alexandre d'Aspremont, and Francis Bach.
\newblock Regularized nonlinear acceleration.
\newblock In \emph{Proceedings of the 30th International Conference on Neural
  Information Processing Systems}, pages 712--720, 2016.

\bibitem[Scieur et~al.(2017)Scieur, Bach, and d'Aspremont]{scieur2017nonlinear}
Damien Scieur, Francis Bach, and Alexandre d'Aspremont.
\newblock Nonlinear acceleration of stochastic algorithms.
\newblock In \emph{Proceedings of the 31st International Conference on Neural
  Information Processing Systems}, pages 3985--3994, 2017.

\bibitem[Scieur et~al.(2018)Scieur, Oyallon, d'Aspremont, and
  Bach]{scieur2018nonlinear}
Damien Scieur, Edouard Oyallon, Alexandre d'Aspremont, and Francis Bach.
\newblock Nonlinear acceleration of cnns.
\newblock In \emph{ICLR Workshop track}, 2018.

\bibitem[Scieur et~al.(2020)Scieur, d’Aspremont, and
  Bach]{scieur2020regularized}
Damien Scieur, Alexandre d’Aspremont, and Francis Bach.
\newblock Regularized nonlinear acceleration.
\newblock \emph{Mathematical Programming}, 179\penalty0 (1):\penalty0 47--83,
  2020.

\bibitem[Shalev-Shwartz and Ben-David(2014)]{shalev2014understanding}
Shai Shalev-Shwartz and Shai Ben-David.
\newblock \emph{Understanding machine learning: From theory to algorithms}.
\newblock Cambridge university press, 2014.

\bibitem[Suryanarayana et~al.(2019)Suryanarayana, Pratapa, and
  Pask]{suryanarayana2019alternating}
Phanish Suryanarayana, Phanisri~P Pratapa, and John~E Pask.
\newblock Alternating anderson--richardson method: An efficient alternative to
  preconditioned krylov methods for large, sparse linear systems.
\newblock \emph{Computer Physics Communications}, 234:\penalty0 278--285, 2019.

\bibitem[Tieleman and Hinton(2012)]{tieleman2012lecture}
Tijmen Tieleman and Geoffrey Hinton.
\newblock Lecture 6.5-rmsprop: Divide the gradient by a running average of its
  recent magnitude.
\newblock \emph{COURSERA: Neural networks for machine learning}, 4\penalty0
  (2):\penalty0 26--31, 2012.

\bibitem[Toth and Kelley(2015)]{toth2015convergence}
Alex Toth and CT~Kelley.
\newblock Convergence analysis for anderson acceleration.
\newblock \emph{SIAM Journal on Numerical Analysis}, 53\penalty0 (2):\penalty0
  805--819, 2015.

\bibitem[Toth et~al.(2017)Toth, Ellis, Evans, Hamilton, Kelley, Pawlowski, and
  Slattery]{toth2017local}
Alex Toth, J~Austin Ellis, Tom Evans, Steven Hamilton, CT~Kelley, Roger
  Pawlowski, and Stuart Slattery.
\newblock Local improvement results for anderson acceleration with inaccurate
  function evaluations.
\newblock \emph{SIAM Journal on Scientific Computing}, 39\penalty0
  (5):\penalty0 S47--S65, 2017.

\bibitem[Walker and Ni(2011)]{walker2011anderson}
Homer~F Walker and Peng Ni.
\newblock Anderson acceleration for fixed-point iterations.
\newblock \emph{SIAM Journal on Numerical Analysis}, 49\penalty0 (4):\penalty0
  1715--1735, 2011.

\bibitem[Wang et~al.(2017)Wang, Ma, Goldfarb, and Liu]{wang2017stochastic}
Xiao Wang, Shiqian Ma, Donald Goldfarb, and Wei Liu.
\newblock Stochastic quasi-newton methods for nonconvex stochastic
  optimization.
\newblock \emph{SIAM Journal on Optimization}, 27\penalty0 (2):\penalty0
  927--956, 2017.

\bibitem[Xie et~al.(2017)Xie, Girshick, Doll{\'a}r, Tu, and
  He]{xie2017aggregated}
Saining Xie, Ross Girshick, Piotr Doll{\'a}r, Zhuowen Tu, and Kaiming He.
\newblock Aggregated residual transformations for deep neural networks.
\newblock In \emph{Proceedings of the IEEE conference on computer vision and
  pattern recognition}, pages 1492--1500, 2017.

\bibitem[Zagoruyko and Komodakis(2016)]{zagoruyko2016wide}
Sergey Zagoruyko and Nikos Komodakis.
\newblock Wide residual networks.
\newblock In \emph{British Machine Vision Conference 2016}. British Machine
  Vision Association, 2016.

\bibitem[Zeiler(2012)]{zeiler2012adadelta}
Matthew~D Zeiler.
\newblock Adadelta: an adaptive learning rate method.
\newblock \emph{arXiv preprint arXiv:1212.5701}, 2012.

\bibitem[Zhang et~al.(2019)Zhang, Lucas, Ba, and Hinton]{zhang2019lookahead}
Michael Zhang, James Lucas, Jimmy Ba, and Geoffrey~E Hinton.
\newblock Lookahead optimizer: k steps forward, 1 step back.
\newblock In \emph{Advances in Neural Information Processing Systems}, pages
  9597--9608, 2019.

\bibitem[Zhuang et~al.(2020)Zhuang, Tang, Tatikonda, Dvornek, Ding,
  Papademetris, and Duncan]{zhuang2020adabelief}
Juntang Zhuang, Tommy Tang, Sekhar Tatikonda, Nicha Dvornek, Yifan Ding,
  Xenophon Papademetris, and James~S Duncan.
\newblock Adabelief optimizer: Adapting stepsizes by the belief in observed
  gradients.
\newblock \emph{arXiv preprint arXiv:2010.07468}, 2020.

\end{thebibliography}

\clearpage

\appendix

\section{Proofs} \label{sec:proof} 
 
 \subsection{Nonconvex stochastic optimization} \label{proof:nonconvexStochastic}
 We give proofs of the theorems in section \ref{sec:theory}.
 
 From Assumption~\ref{assume:noise}, for the mini-batch gradient $ f_{S_k}(x_k) = \frac{1}{n_k}\sum_{i\in S_k}f_{\xi_i}(x_k) $, where $ n_k = \vert S_k \vert $, we have 
\begin{subequations}
\begin{align}
&\mathbb{E}[\nabla f_{S_k}(x)\vert x_k] = \nabla f(x_k), \label{minibatch:avg}  \\
&\mathbb{E}[\| \nabla f_{S_k}(x_k)-\nabla f(x_k)\|_2^2 \vert x_k] \leq \frac{\sigma^2}{n_k}. \label{minibatch:std}
\end{align}
\end{subequations} 

  Note that the update of SAM (\ref{eq:sam}) can be written as  $ x_{k+1}=x_k+H_kr_k $, where
 $ r_k = -\nabla f_{S_k}(x_k) $,   
   $ H_0 = \beta_0 I$ and for $ k\geq 1 $,
 \begin{align}
 H_k =  \beta_kI - \left( \alpha_kX_k+\alpha_k\beta_kR_k \right)
 \left( R_k^{\mathrm{T}}R_k+\delta_kX_k^{\mathrm{T}}X_k\right)^{-1}R_k^{\mathrm{T}}.
 \end{align}
 Theorem \ref{them:nonconvexStochastic} - \ref{them:sam_vr} state the same convergence and complexity results as the ones proved in \citep{wang2017stochastic}. To prove these theorems, the critical points are (i) the positive definiteness of the approximate Hessian $ H_k $  and (ii) an adequate suppression of the noise in the gradient estimation. 
 
 We first give some lemmas.

\begin{lemma} \label{lemma:Hkrk_sq}
Suppose that  $ \lbrace x_k\rbrace $ is generated by SAM. If $ \alpha_k\geq 0, \beta_k>0, \delta_k>0 $, then for any $ v_k	\in \mathbb{R}^{d} $, we have
\begin{align}
\| H_kv_k\|_2^2\leq 2\left( \beta_k^2\left(1+2\alpha_k^2-2\alpha_k \right)+\alpha_k^2\delta_k^{-1} \right)\| v_k\|_2^2.
\end{align}
\end{lemma}
\begin{proof}
The result clearly holds when $k=0$ as $ H_0 = \beta_0 I $. For $ k\geq 1 $,
\begin{align}
 H_kv_k = \beta_kv_k-(\alpha_kX_k+\alpha_k\beta_kR_k)\Gamma_k,
\end{align}
where 
\begin{equation}
\Gamma_k = \min\limits_{\Gamma}\| v_k - R_k \Gamma  \|_2^2 + \delta_k \| X_k \Gamma  \|_2^2. 
\end{equation}
Taking $ \Gamma = 0 $, we have
 \begin{equation}
 \|v_k-R_k\Gamma_k\|_2^2+\delta_k\|X_k\Gamma_k\|_2^2 \leq \|v_k\|_2^2. 
 \label{ineq:gamma_k}
 \end{equation}
 Therefore,
 \begin{align}
& \quad \|H_kv_k\|_2^2 \nonumber \\
 &= \|\beta_kv_k-(\alpha_kX_k+\alpha_k\beta_kR_k)\Gamma_k\|_2^2 \nonumber \\
 &=\|\beta_k\left(v_k-\alpha_kR_k\Gamma_k\right) - \alpha_kX_k\Gamma_k\|_2^2 \nonumber \\
 &=\|\beta_k(1-\alpha_k)v_k+\beta_k\alpha_k(v_k-R_k\Gamma_k) 
  -\alpha_k\delta_k^{-\frac{1}{2}}\delta_k^{\frac{1}{2}}X_k\Gamma_k\|_2^2 \nonumber \\
 &\leq \left(\beta_k^2(1-\alpha_k)^2+\beta_k^2\alpha_k^2+\alpha_k^2\delta_k^{-1} \right) \cdot\left( \|v_k \|_2^2+\|v_k-R_k\Gamma_k\|_2^2+\delta_k\|X_k\Gamma_k\|_2^2 \right)  \nonumber \\
&\leq\left( \beta_k^2\left( 1+2\alpha_k^2-2\alpha_k \right)+\alpha_k^2\delta_k^{-1} \right)\left(\|v_k\|_2^2 +\|v_k\|_2^2 \right) \nonumber \\
&= 2\left( \beta_k^2\left( 1+2\alpha_k^2-2\alpha_k \right)+\alpha_k^2\delta_k^{-1} \right)\|v_k\|_2^2. \label{ineq:Hkrk_sq}
 \end{align}
 In the above, the first inequality uses the inequality
 \begin{align}
 \|\sum_{i=1}^{n}a_i \mathbf{x_i}\|_2^2 &\leq \left(\sum_{i=1}^{n}\vert a_i\vert \|\mathbf{x_i} \|_2\right)^2 
 \leq \left(\sum_{i=1}^{n} a_i^2\right)\left( \sum_{i=1}^n \| \mathbf{x_i}\|_2^2\right), \label{ineq:generalCauchy}
 \end{align}
 where $ a_i \in \mathbb{R}, x_i \in\mathbb{R}^d $.
  The second inequality is based on inequality (\ref{ineq:gamma_k}). 
\end{proof}
 
 \begin{lemma}\label{lemma:bound}
Suppose that Assumption~\ref{assume:noise} holds for $ \lbrace x_k \rbrace $ generated by SAM. In addition, if  $  \beta_k > 0, \delta_k >0 $, and $\alpha_k \geq 0$ and satisfies (\ref{ineq:check_alpha_k}), then
\begin{subequations}
\begin{align}
 \mathbb{E}_{S_k}[\| H_kr_k \|_2^2] &\leq 2\left( \beta_k^2\left( 1+2\alpha_k^2-2\alpha_k \right)+\frac{\alpha_k^2}{\delta_k} \right) 
  \cdot\left(\|\nabla f(x_k)\|_2^2+\frac{\sigma^2}{n_k} \right), \label{lemma_bound:ineq1}  \\
\nabla f(x_k)^\mathrm{T}\mathbb{E}_{S_k}[H_kr_k] &\leq -\frac{1}{2}\beta_k\mu \| \nabla f(x_k) \|_2^2 
 +\frac{1}{2}\frac{\alpha_k^2(\delta_k^{-\frac{1}{2}}+\beta_k)^2}{\beta_k\mu}\cdot \frac{\sigma^2}{n_k}, \label{lemma_bound:ineq2}
\end{align}
\end{subequations}
where $ \mu > 0 $ is the constant introduced in \eqref{ineq:pos_Hk}. 
If further assuming $ H_k $ is independent of $S_k$, a better upper bound can be obtained:
\begin{align}
 \nabla f(x_k)^\mathrm{T}\mathbb{E}_{S_k}[H_kr_k] \leq -\beta_k\mu \| \nabla f(x_k) \|_2^2. \label{lemma_bound:ineq3}
\end{align}
\end{lemma}
 
 \begin{proof}
(i) From Lemma~\ref{lemma:Hkrk_sq}, we have
 \begin{align}
 \mathbb{E}_{S_k}[\|H_kr_k\|_2^2] 
 \leq 2\left( \beta_k^2\left( 1+2\alpha_k^2-2\alpha_k \right)+\frac{\alpha_k^2}{\delta_k} \right)\mathbb{E}_{S_k}[\|r_k\|_2^2]. \label{ineq:Hkrk}
 \end{align} 
 From Assumption~\ref{assume:noise}, we have
 \begin{align}
 \mathbb{E}_{S_k}[\|r_k\|_2^2] = \mathbb{E}_{S_k}[\|r_k-\mathbb{E}_{S_k}[r_k]\|_2^2]+\|\mathbb{E}_{S_k}[r_k]\|_2^2  = \|\nabla f(x_k)\|_2^2+\sigma^2/n_k. \label{ineq:rk_sq}
 \end{align}
With (\ref{ineq:Hkrk}), (\ref{ineq:rk_sq}), we obtain (\ref{lemma_bound:ineq1}).

(ii) Recalling that $H_0=\beta_0I$, the result  holds for $k=0$.
Define $ \epsilon_k = \nabla f_{S_k}(x_k)-\nabla f(x_k)=-r_k-\nabla f(x_k) $, then $ H_kr_k = H_k\left( -\epsilon_k -\nabla f(x_k)\right). $ Since $\alpha_k$ satisfies (\ref{ineq:check_alpha_k}), we can ensure $ \lambda_{min}\left(\frac{1}{2}\left(H_k+H_k^\mathrm{T}\right)\right) \geq \beta_k\mu. $ Thus
\begin{align*}
\nabla f(x_k)^\mathrm{T} H_k\nabla f(x_k) = \frac{1}{2}\nabla f(x_k)^\mathrm{T} \left(H_k+H_k^\mathrm{T} \right)\nabla f(x_k) \geq \beta_k\mu\|\nabla f(x_k)\|_2^2,
\end{align*}
which implies 
\begin{equation}
\mathbb{E}_{S_k}[\nabla f(x_k)^\mathrm{T} H_k\nabla f(x_k)] \geq \beta_k\mu\|\nabla f(x_k)\|_2^2. \label{ineq:part1}
\end{equation}
Let $ M_k = \alpha_k\left( X_k +\beta_kR_k \right)\left( R_k^\mathrm{T}R_k+\delta_kX_k^\mathrm{T}X_k \right)^{\dagger}R_k^\mathrm{T} $, then $ H_k = \beta_kI-M_k. $ With the assumption (\ref{minibatch:avg}), i.e. $\mathbb{E}_{S_k}[\epsilon_k]=0$, we have
\begin{align*}
 \mathbb{E}_{S_k}[\nabla f(x_k)^\mathrm{T}H_k\epsilon_k] 
&=\mathbb{E}_{S_k}[\nabla f(x_k)^\mathrm{T}\left( \beta_k\epsilon_k-M_k\epsilon_k \right)] \\
&=\beta_k\nabla f(x_k)^\mathrm{T}\mathbb{E}_{S_k}[\epsilon_k]-
\mathbb{E}_{S_k}[\nabla f(x_k)^\mathrm{T}M_k\epsilon_k] 
= -\mathbb{E}_{S_k}[\nabla f(x_k)^\mathrm{T}M_k\epsilon_k].
\end{align*}
Using the {\em Cauchy-Schwarz inequality with expectations}, we obtain
\begin{align}
 \vert \mathbb{E}_{S_k}[\nabla f(x_k)^\mathrm{T}H_k\epsilon_k]\vert 
&= \vert \mathbb{E}_{S_k}[\nabla f(x_k)^\mathrm{T}M_k\epsilon_k] \vert \leq \sqrt{\mathbb{E}_{S_k}[\|\nabla f(x_k)\|_2^2]} \sqrt{\mathbb{E}_{S_k}[\|M_k\epsilon_k\|_2^2]} \nonumber \\
& \leq \|\nabla f(x_k) \|_2 \sqrt{\mathbb{E}_{S_k}[\|M_k\epsilon_k\|_2^2]}. \label{ineq:cauchy}
\end{align} 
We now bound $ \|M_k\epsilon_k\|_2^2 $. For brevity, let $ Z_k = R_k^\mathrm{T}R_k+\delta_k X_k^\mathrm{T}X_k $, and $ N_1 = X_kZ_k^{\dagger}R_k^\mathrm{T}, N_2 = \beta_kR_kZ_k^{\dagger}R_k^\mathrm{T} $,
then
\begin{equation}
\|M_k\|_2 =\|\alpha_k\left(N_1+N_2\right) \|_2 \leq \alpha_k \left(\| N_1\|_2+\| N_2\|_2 \right). \label{ineq:Mk}
\end{equation}
Clearly, $ X_k^\mathrm{T}X_k, R_k^\mathrm{T}R_k $ and $Z_k$ are symmetric positive semidefinite. Also, we have
$ \delta_kX_k^\mathrm{T}X_k \preceq Z_k, R_k^\mathrm{T}R_k \preceq Z_k $, where the notation ``$ \preceq $" denotes the {\em Loewner partial order}, i.e., $ A\preceq B $ with $ A,B \in \mathbb{R}^{m\times m} $ means that $ B-A $ is positive semidefinite. 

First we point out that 
\begin{align}
 Z_k^{\dagger} = \lim_{t\rightarrow 0^{+}}Z_k^{\frac{1}{2}}\left( Z_k^2+tI\right)^{-1}Z_k^{\frac{1}{2}} , \label{eq:Zk_pinv}
\end{align} 
 where $t>0$, which can be verified as follows: \\
Since $Z_k$ is symmetric positive semidefinite, we have the eigenvalue decomposition: $Z_k=U\wedge U^{\mathrm{T}}$, where $ UU^{\mathrm{T}}=I, 0 \preceq\wedge = diag\lbrace\wedge_1,0\rbrace \in \mathbb{R}^{ m\times m}$, and $ \wedge_1$ is diagonal and nonsingular. Hence 
\begin{align*}
 Z_k^{\frac{1}{2}}\left( Z_k^2+tI\right)^{-1}Z_k^{\frac{1}{2}}
 =U\wedge^{\frac{1}{2}}\left( \wedge^2+tI\right)^{-1}\wedge^{\frac{1}{2}}U^\mathrm{T}
 =U 
 \begin{pmatrix}
 \wedge_1(\wedge_1^2+tI)^{-1} & 0 \\
 0 & 0 
 \end{pmatrix} 
 U^\mathrm{T}.
\end{align*}
It follows that $ \lim_{t\rightarrow 0^{+}}Z_k^{\frac{1}{2}}\left( Z_k^2+tI\right)^{-1}Z_k^{\frac{1}{2}} = U\wedge^{\dagger}U^{\mathrm{T}} $, where
$ \wedge^{\dagger}=diag\lbrace \wedge_1^{-1},0\rbrace $. From the definition of Penrose-Moore inverse, we know Equation~(\ref{eq:Zk_pinv}) holds.

Since $ \delta_kX_k^\mathrm{T}X_k \preceq Z_k, R_k^\mathrm{T}R_k \preceq Z_k $, we have
\begin{align}
 \delta_k Z_k^{\frac{1}{2}}X_k^\mathrm{T}X_k Z_k^{\frac{1}{2}} \preceq Z_k^2 \preceq Z_k^2+tI, \  \
 Z_k^{\frac{1}{2}}R_k^\mathrm{T}R_k Z_k^{\frac{1}{2}} \preceq Z_k^2 \preceq Z_k^2+tI.
\end{align}
Hence, we have  
\begin{align*}
 (Z_k^2+tI)^{-\frac{1}{2}}\delta_kZ_k^{\frac{1}{2}}X_k^\mathrm{T}X_kZ_k^{\frac{1}{2}}(Z_k^2+tI)^{-\frac{1}{2}} \preceq I,  \  \
 (Z_k^2+tI)^{-\frac{1}{2}}Z_k^{\frac{1}{2}}R_k^\mathrm{T}R_kZ_k^{\frac{1}{2}} (Z_k^2+tI)^{-\frac{1}{2}} \preceq I,
\end{align*}
which implies
\begin{align*}
 \| (Z_k^2+tI)^{-\frac{1}{2}}Z_k^{\frac{1}{2}}\left(X_k^\mathrm{T}X_k\right)Z_k^{\frac{1}{2}}(Z_k^2+tI)^{-\frac{1}{2}} \|_2 &\leq \delta_k^{-1},  \\
  \| (Z_k^2+tI)^{-\frac{1}{2}}Z_k^{\frac{1}{2}}\left(R_k^\mathrm{T}R_k\right)Z_k^{\frac{1}{2}}(Z_k^2+tI)^{-\frac{1}{2}}\|_2 &\leq 1.
\end{align*}
With Equation~(\ref{eq:Zk_pinv}), we also have
\begin{align*}
N_1 &= \lim_{t\rightarrow 0^{+}}N_1(t):=X_k Z_k^{\frac{1}{2}}\left( Z_k^2+tI\right)^{-1}Z_k^{\frac{1}{2}} R_k^{\mathrm{T}}, \\
N_2 &= \lim_{t\rightarrow 0^{+}}N_2(t):=\beta_kR_k Z_k^{\frac{1}{2}}\left( Z_k^2+tI\right)^{-1}Z_k^{\frac{1}{2}} R_k^{\mathrm{T}}. 
\end{align*}
Therefore,
\begin{align*}
 \|N_1(t)\|_2^2 
&= \lambda_{max}\left(N_1(t)N_1(t)^\mathrm{T}\right) \\
&=\lambda_{max}\left(X_k Z_k^{\frac{1}{2}}\left( Z_k^2+tI\right)^{-1}Z_k^{\frac{1}{2}} R_k^{\mathrm{T}} \cdot R_k Z_k^{\frac{1}{2}}\left( Z_k^2+tI\right)^{-1}Z_k^{\frac{1}{2}} X_k^{\mathrm{T}}  \right) \\
&=\lambda_{max}\left( Z_k^{\frac{1}{2}} X_k^{\mathrm{T}} X_k Z_k^{\frac{1}{2}}\left( Z_k^2+tI\right)^{-1}Z_k^{\frac{1}{2}} R_k^{\mathrm{T}} R_k Z_k^{\frac{1}{2}}\left( Z_k^2+tI\right)^{-1}\right) \\
&=\lambda_{max}( ( Z_k^2+tI)^{-\frac{1}{2}} Z_k^{\frac{1}{2}} X_k^{\mathrm{T}} X_k Z_k^{\frac{1}{2}} ( Z_k^2+tI)^{-\frac{1}{2}} \\
 &\quad \cdot ( Z_k^2+tI)^{-\frac{1}{2}} Z_k^{\frac{1}{2}} R_k^{\mathrm{T}} R_k Z_k^{\frac{1}{2}}( Z_k^2+tI)^{-\frac{1}{2}}
) \\
&\leq \| ( Z_k^2+tI)^{-\frac{1}{2}} Z_k^{\frac{1}{2}} X_k^{\mathrm{T}} X_k Z_k^{\frac{1}{2}} ( Z_k^2+tI)^{-\frac{1}{2}} \\
 &\quad \cdot ( Z_k^2+tI)^{-\frac{1}{2}} Z_k^{\frac{1}{2}} R_k^{\mathrm{T}} R_k Z_k^{\frac{1}{2}}( Z_k^2+tI)^{-\frac{1}{2}}
) \|_2 \\
&\leq \| ( Z_k^2+tI)^{-\frac{1}{2}} Z_k^{\frac{1}{2}} X_k^{\mathrm{T}} X_k Z_k^{\frac{1}{2}} ( Z_k^2+tI)^{-\frac{1}{2}}\|_2  \\
 &\quad \cdot \|( Z_k^2+tI)^{-\frac{1}{2}} Z_k^{\frac{1}{2}} R_k^{\mathrm{T}} R_k Z_k^{\frac{1}{2}}( Z_k^2+tI)^{-\frac{1}{2}}
) \|_2 \leq \delta_k^{-1}, \\
\|N_2(t)\|_2 
&= \beta_k\lambda_{max}\left( R_k Z_k^{\frac{1}{2}}\left( Z_k^2+tI\right)^{-1}Z_k^{\frac{1}{2}} R_k^{\mathrm{T}} \right) \\
&= \beta_k\lambda_{max}\left( Z_k^{\frac{1}{2}} R_k^{\mathrm{T}}R_k Z_k^{\frac{1}{2}}\left( Z_k^2+tI\right)^{-1} \right) \\
&= \beta_k\lambda_{max}\left( \left( Z_k^2+tI\right)^{-\frac{1}{2}} Z_k^{\frac{1}{2}} R_k^{\mathrm{T}}R_k Z_k^{\frac{1}{2}}\left( Z_k^2+tI\right)^{-\frac{1}{2}}
 \right) \leq \beta_k,
\end{align*} 
which implies $ \|N_1\|_2^2\leq \delta_k^{-1} $ and $ \| N_2\|_2 \leq \beta_k $ from the continuity of singular value (e.g. Theorem~2.6.4 in \citep{horn2012matrix}).
With (\ref{ineq:Mk}), we have
\begin{equation}
\|M_k\|_2 \leq \alpha_k(\delta_k^{-\frac{1}{2}}+\beta_k). \label{ineq:Mk_nd}
\end{equation}
Then $ \|M_k\epsilon_k \|_2\leq \alpha_k(\delta_k^{-\frac{1}{2}}+\beta_k)\|\epsilon_k\| $, 
which implies
\begin{align}
\mathbb{E}_{S_k}[ \|M_k\epsilon_k \|_2^2]
\leq \alpha_k^2(\delta_k^{-\frac{1}{2}}+\beta_k)^2\mathbb{E}_{S_k}[\|\epsilon_k\|_2^2] 
\leq \alpha_k^2(\delta_k^{-\frac{1}{2}}+\beta_k)^2\frac{\sigma^2}{n_k},
\end{align}
where the last inequality is due to (\ref{minibatch:std}). Now we can obtain the bound of $ \vert\mathbb{E}_{S_k}[\nabla f(x_k)^\mathrm{T}H_k\epsilon_k]\vert $ as follows (cf. (\ref{ineq:cauchy})):
\begin{align}
& \quad \vert \mathbb{E}_{S_k}[\nabla f(x_k)^\mathrm{T}H_k\epsilon_k]\vert \nonumber \\
& \leq \|\nabla f(x_k) \|_2 \sqrt{\mathbb{E}_{S_k}[\|M_k\epsilon_k\|_2^2]} \nonumber \\
& \leq \alpha_k(\delta_k^{-\frac{1}{2}}+\beta_k) \|\nabla f(x_k) \|_2 \sqrt{\mathbb{E}_{S_k}[\|\epsilon_k\|_2^2]} \nonumber \\
& \leq \alpha_k(\delta_k^{-\frac{1}{2}}+\beta_k)\frac{\sigma}{\sqrt{n_k}} \|\nabla f(x_k) \|_2 \nonumber \\
&=\sqrt{\beta_k\mu}\|\nabla f(x_k)\|_2\cdot \frac{\alpha_k(\delta_k^{-\frac{1}{2}}+\beta_k)}{\sqrt{\beta_k\mu}}\frac{\sigma}{\sqrt{n_k}} \nonumber \\
&\leq \frac{1}{2}\beta_k\mu\|\nabla f(x_k)\|_2^2+
\frac{1}{2}\frac{\alpha_k^2(\delta_k^{-\frac{1}{2}}+\beta_k)^2}{\beta_k\mu}\cdot \frac{\sigma^2}{n_k}. \label{ineq:part2}
\end{align}
 
 With the inequality (\ref{ineq:part1}) and (\ref{ineq:part2}), we obtain
\begin{align}
&\quad\nabla f(x_k)^\mathrm{T}\mathbb{E}_{S_k}[H_kr_k] \nonumber\\
&=-\nabla f(x_k)^\mathrm{T}\mathbb{E}_{S_k}[H_k\left(\epsilon_k+\nabla f(x_k) \right)] \nonumber \\
&=-\mathbb{E}_{S_k}[\nabla f(x_k)^\mathrm{T}H_k\nabla f(x_k)]
-\mathbb{E}_{S_k}[\nabla f(x_k)^\mathrm{T}H_k\epsilon_k] \nonumber \\
&\leq -\mathbb{E}_{S_k}[\nabla f(x_k)^\mathrm{T}H_k\nabla f(x_k)]+
\vert \mathbb{E}_{S_k}[\nabla f(x_k)^\mathrm{T}H_k\epsilon_k]\vert \nonumber \\
&\leq -\beta_k\mu\|\nabla f(x_k)\|_2^2
+\frac{1}{2}\beta_k\mu\|\nabla f(x_k)\|_2^2 +\frac{1}{2}\frac{\alpha_k^2(\delta_k^{-\frac{1}{2}}+\beta_k)^2}{\beta_k\mu}\cdot \frac{\sigma^2}{n_k} \nonumber \\
&=-\frac{1}{2}\beta_k\mu\|\nabla f(x_k)\|_2^2+\frac{1}{2}\frac{\alpha_k^2(\delta_k^{-\frac{1}{2}}+\beta_k)^2}{\beta_k\mu}\cdot \frac{\sigma^2}{n_k}.
\end{align}

If $ H_k $ is independent of $ S_k $, then
\begin{equation*}
\mathbb{E}_{S_k}[\nabla f(x_k)^\mathrm{T}H_k\epsilon_k]
=\nabla f(x_k)^\mathrm{T}H_k\mathbb{E}_{S_k}[\epsilon_k] = 0.
\end{equation*}
Thus $
\nabla f(x_k)^\mathrm{T}\mathbb{E}_{S_k}[H_kr_k] 
\leq -\beta_k\mu\|\nabla f(x_k)\|_2^2.
$
\end{proof}

By imposing more restrictions to $ \alpha_k, \beta_k, \delta_k$, we can obtain a convenient corollary:
\begin{corollary} \label{coro1}
Suppose that Assumption~\ref{assume:noise} holds for $ \{x_k \} $ generated by SAM. $ C>0 $ is a constant. If $ \beta_k > 0, \delta_k\geq C\beta_k^{-2} $, $ 0\leq\alpha_k\leq \min\{1,\beta_k^{\frac{1}{2}}\} $ and satisfies (\ref{ineq:check_alpha_k}) , then
\begin{subequations}
\begin{align}
\mathbb{E}_{S_k}[\| H_kr_k \|_2^2] &\leq 2\beta_k^2\left(1+C^{-1} \right) \cdot\left( \|\nabla f(x_k) \|_2^2 +\frac{\sigma^2}{n_k} \right), \label{coro:ineq1}  \\
 \nabla f(x_k)^\mathrm{T}\mathbb{E}_{S_k}[H_kr_k] &\leq -\frac{1}{2}\beta_k\mu\|\nabla f(x_k)\|_2^2 
 +\beta_k^2 \cdot\mu^{-1}\left( 1+C^{-1}\right)\frac{\sigma^2}{n_k}. \label{coro:ineq2}
\end{align}
\end{subequations}
\end{corollary}
\begin{proof}
 The first result (\ref{coro:ineq1}) is easy to obtain by considering (\ref{lemma_bound:ineq1}) and noticing that $ 1+2\alpha_k^2-2\alpha_k \leq 1 $ when $ \alpha_k\in [0,1] $ and $ \delta_k^{-1} \leq C^{-1}\beta_k^2. $
 Since $ \alpha_k\leq \beta_k^{\frac{1}{2}}, \delta_k\geq C\beta_k^{-2} $ and $ (1+C^{-\frac{1}{2}})^2 \leq 2(1+C^{-1}) $ we have
 \begin{align*}
 \frac{1}{2}\frac{\alpha_k^2(\delta_k^{-\frac{1}{2}}+\beta_k)^2}{\beta_k\mu}\cdot \frac{\sigma^2}{n_k} 
& \leq \frac{1}{2}\frac{\beta_k( C^{-\frac{1}{2}}\beta_k+\beta_k )^2 }{\beta_k\mu}\cdot \frac{\sigma^2}{n_k} \\
& =\frac{1}{2}\mu^{-1}(C^{-\frac{1}{2}}+1)^2\beta_k^2\cdot \frac{\sigma^2}{n_k} \\
& \leq \beta_k^2\mu^{-1}(1+C^{-1})\frac{\sigma^2}{n_k}. 
 \end{align*}
 Substituting it into (\ref{lemma_bound:ineq2}), we obtain (\ref{coro:ineq2}).
\end{proof}

 Using Corollary~\ref{coro1} we obtain the descent property of SAM:
 \begin{lemma} \label{lemma:descent}
Suppose that Assumptions~\ref{assume:Lips} and \ref{assume:noise} hold for $ \lbrace x_k \rbrace $ generated by SAM. $ C>0$ is a constant. If  $ 0<\beta_k\leq \frac{\mu}{4L(1+C^{-1})}, \delta_k\geq C\beta_k^{-2}, 0\leq\alpha_k\leq\min\{ 1, \beta_k^{\frac{1}{2}}\} $ and satisfies (\ref{ineq:check_alpha_k}),  
then
\begin{align}
\mathbb{E}_{S_k}[f(x_{k+1})] &\leq f(x_k) -\frac{1}{4}\beta_k\mu\|\nabla f(x_k) \|_2^2  + \beta_k^2\left( (L+\mu^{-1})(1+C^{-1}) \right)\frac{\sigma^2}{n_k}. \label{lemma_descent:ineq1}
\end{align}
\end{lemma}
 
 \begin{proof}
According to Assumption~\ref{assume:Lips}, we have
\begin{align}
f(x_{k+1}) &\leq f(x_k)+\nabla f(x_k)^\mathrm{T}(x_{k+1}-x_k) +\frac{L}{2}\|x_{k+1}-x_k\|_2^2 \nonumber \\
&=f(x_k)+\nabla f(x_k)^\mathrm{T}H_kr_k+\frac{L}{2}\| H_kr_k\|_2^2.\label{ineq:Lip}
\end{align}
Taking expectation with respect to the mini-batch $ S_k $ on both sides of (\ref{ineq:Lip}) and using Corollary~\ref{coro1} we obtain
\begin{align}
& \quad \mathbb{E}_{S_k}[f(x_{k+1})]  \nonumber \\
&\leq f(x_k)+\nabla f(x_k)^\mathrm{T}\mathbb{E}_{S_k}[H_kr_k]
+\frac{L}{2}\mathbb{E}_{S_k}\|H_kr_k\|_2^2 \nonumber \\
&\leq f(x_k)-\frac{1}{2}\beta_k\mu\|\nabla f(x_k)\|_2^2 +\beta_k^2\mu^{-1}(1+C^{-1})\frac{\sigma^2}{n_k} +L\beta_k^2(1+C^{-1})\left(\|\nabla f(x_k) \|_2^2+\frac{\sigma^2}{n_k} \right) \nonumber \\
& = f(x_k)-\beta_k\left( \frac{1}{2}\mu-\beta_k L(1+C^{-1}) \right)\|\nabla f(x_k) \|_2^2 +\beta_k^2(L+\mu^{-1})(1+C^{-1})\frac{\sigma^2}{n_k}.
\label{ineq:ExpectLip}
\end{align}
Then (\ref{ineq:ExpectLip}) combined with the assumption $ \beta_k\leq \frac{\mu}{4L(1+C^{-1})} $ implies (\ref{lemma_descent:ineq1}).
\end{proof}

 Following the proofs in \cite{wang2017stochastic}, we  introduce the definition of a {\em supermartingale}.
\begin{definition}\label{def1}
 Let $ \{ \mathcal{F}_k \}$ be an increasing sequence of $ \sigma $-algebras. If $ \{X_k \} $ is a stochastic process satisfying (i) $\mathbb{E}[\vert X_k \vert]<\infty $, (ii) $ X_k\in\mathcal{F}_k $ for all $ k $, and (iii) $ \mathbb{E}[X_{k+1}\vert \mathcal{F}_k]\leq X_k $ for all $ k $, then $ \{ X_k\} $ is called a supermartingale.
\end{definition}
\begin{proposition}[Supermartingale convergence theorem, see, e.g., Theorem 4.2.12 in  \citep{durrett2019probability}] \label{prop1}
 If $ \{ X_k \} $ is a nonnegative supermartingale, then $ \lim_{k\rightarrow\infty}X_k\rightarrow X $ almost surely and $ \mathbb{E}[X] \leq \mathbb{E}[X_0] $.
\end{proposition}

 Now, we prove our theorems.

\begin{proof}[\textbf{Proof of Theorem~\ref{them:nonconvexStochastic}}]
Define $ \zeta_k := \frac{\beta_k\mu}{4}\|\nabla f(x_k)\|_2^2 $ and $ \tilde{L} := (L+\mu^{-1})(1+C^{-1}), $ $\gamma_k := f(x_k)+\tilde{L}\frac{\sigma^2}{n}\sum_{i=k}^{\infty}\beta_i^2 $. Let $ \mathcal{F}_k $ be the $ \sigma$-algebra measuring $ \zeta_k, \gamma_k, $ and $ x_k $. From (\ref{lemma_descent:ineq1}) we know that for any $ k $,
\begin{align}
\mathbb{E}[\gamma_{k+1}\vert \mathcal{F}_k]
&= \mathbb{E}[f(x_{k+1})\vert \mathcal{F}_k]+\tilde{L}\frac{\sigma^2}{n}\sum_{i=k+1}^{\infty}\beta_i^2 \nonumber \\
&\leq f(x_k)-\frac{1}{4}\beta_k\mu\|\nabla f(x_k)\|_2^2+\tilde{L}\frac{\sigma^2}{n}\sum_{i=k}^{\infty}\beta_i^2 = \gamma_k-\zeta_k, \label{proof1:gamma}
\end{align}
which implies that $ \mathbb{E}[\gamma_{k+1}-f^{low}\vert \mathcal{F}_k]\leq \gamma_k-f^{low}-\zeta_k $. Since $ \zeta_k\geq 0 $, we have $ 0\leq\mathbb{E}[\gamma_k-f^{low}]\leq \gamma_0-f^{low}<+\infty$. As the diminishing condition (\ref{cond:diminish}) holds, we obtain (\ref{them1:ineq2}). According to Definition~\ref{def1}, $ \{\gamma_k-f^{low} \} $ is a supermartingale. Therefore, Proposition~\ref{prop1} indicates that there exists a $ \gamma $ such that $ \lim_{k\rightarrow\infty}\gamma_k = \gamma $ with probability 1, and $ \mathbb{E}[\gamma]\leq \mathbb{E}[\gamma_0] $. Note that from (\ref{proof1:gamma}) we have $ \mathbb{E}[\zeta_k]\leq \mathbb{E}[\gamma_k]-\mathbb{E}[\gamma_{k+1}] $. Thus,
\begin{align*}
\mathbb{E}\left[ \sum_{k=0}^{\infty}\zeta_k \right]\leq \sum_{k=0}^{\infty}(\mathbb{E}[\gamma_k]-\mathbb{E}[\gamma_{k+1}]) < +\infty,
\end{align*}
which further yields that 
\begin{align}
\sum_{k=0}^{\infty}\zeta_k=\frac{\mu}{4}\sum_{k=0}^{\infty}\beta_k\| \nabla f(x_k)\|_2^2< +\infty \ \text{with probability} \ 1. \label{proof1:ineq2}
\end{align}
Since $ \sum_{k=0}^{\infty}\beta_k = +\infty $, it follows that (\ref{them1:ineq1}) holds.
\end{proof}

 \begin{proof}[\textbf{Proof of Theorem~\ref{them:bounded_noisy_gradient}}]
For any give $ \epsilon>0$, according to (\ref{them1:ineq1}), there exist infinitely many iterates $ x_k $ such that $ \| \nabla f(x_k)\|\leq \epsilon$. 
Then if (\ref{them2:ineq1}) does not hold, there must exist two infinite sequences of indices $ \{ s_i\}, $ $\{ t_i\} $ with $ t_i >s_i $, 
such that for $ i=0,1,\ldots $, $  k=s_i+1,\ldots,t_i-1, $
\begin{equation}
\|\nabla f(x_{s_i})\|_2\geq 2\epsilon, \|\nabla f(x_{t_i}) \|_2<\epsilon, \|\nabla f(x_k)\|_2 \geq \epsilon. \label{proof2:ineq3}
\end{equation}
Then from (\ref{proof1:ineq2}) it follows that
\begin{align*}
+\infty>\sum_{k=0}^{\infty}\beta_k\|\nabla f(x_k)\|_2^2 \geq \sum_{i=0}^{+\infty}\sum_{k=s_i}^{t_i-1}\beta_k\|\nabla f(x_k)\|_2^2 
\geq \epsilon^2\sum_{i=0}^{+\infty}\sum_{k=s_i}^{t_i-1}\beta_k \ \text{with probability} \ 1,
\end{align*}
which implies that
\begin{equation}
\sum_{k=s_i}^{t_i-1}\beta_k\rightarrow 0 \ \text{with probability} \ 1, \text{as} \ i\rightarrow +\infty. \label{proof2:ineq2}
\end{equation}
According to (\ref{ineq:rk_sq}) and (\ref{ineq:Hkrk_sq}), we have
\begin{align}
& \quad \mathbb{E}[\|x_{k+1}-x_k\|_2\vert x_k] \nonumber \\
&= \mathbb{E}[\|H_kr_k \|_2\vert x_k] \nonumber \\
&\leq \sqrt{2\left( \beta_k^2\left( 1+2\alpha_k^2-2\alpha_k \right)+\alpha_k^2\delta_k^{-1} \right)} \mathbb{E}[\|r_k\|_2\vert x_k] \nonumber \\
&\leq \beta_k\sqrt{2(1+C^{-1})}\mathbb{E}[\|r_k\|_2\vert x_k] \nonumber \\
&\leq \beta_k\sqrt{2(1+C^{-1})}(\mathbb{E}[\|r_k\|_2^2\vert x_k])^{\frac{1}{2}} \nonumber \\
&\leq \beta_k\sqrt{2(1+C^{-1})}(M_g/n)^{\frac{1}{2}}, \label{proof2:ineq1}
\end{align}
where the last inequalities are due to {\em Cauchy-Schwarz inequality} and (\ref{cond:bounded}). Then it follows from (\ref{proof2:ineq1}) that 
\begin{align*}
\mathbb{E}[\| x_{t_i}-x_{s_i}\|_2]\leq \sqrt{2(1+C^{-1})}(M_g/n)^{\frac{1}{2}}\sum_{k=s_i}^{t_i-1}\beta_k,
\end{align*}
which together with (\ref{proof2:ineq2}) implies that $ \|x_{t_i}-x_{s_i}\|_2\rightarrow 0$ with probability 1, as $ i\rightarrow +\infty $. Hence, from the Lipschitz continuity of $ \nabla f $, it follows that $ \|\nabla f(x_{t_i})-\nabla f(x_{s_i}) \|_2\rightarrow 0 $ with probability 1 as $ i\rightarrow +\infty $. However, this contradicts (\ref{proof2:ineq3}). Therefore, the assumption that (\ref{them2:ineq1}) does not hold is not true. 
\end{proof}

 \begin{proof}[\textbf{Proof of Theorem~\ref{them:complexity}}]
Define $ \tilde{L}:=(L+\mu^{-1})(1+C^{-1}) $. 
Taking expectation on both sides of (\ref{lemma_descent:ineq1}) and summing over $ k=0,1,\ldots,N-1 $ yields
\begin{align*}
&\quad \frac{1}{4}\mu\sum_{k=0}^{N-1}\mathbb{E}[\|\nabla f(x_k) \|_2^2] \\
& \leq\sum_{k=0}^{N-1}\frac{1}{\beta_k}(\mathbb{E}[f(x_k)]-\mathbb{E}[f(x_{k+1})])+\tilde{L}\frac{\sigma^2}{n}\sum_{k=0}^{N-1}\beta_k \\
&=\frac{1}{\beta_0}f(x_0)+\sum_{k=1}^{N-1}\left(\frac{1}{\beta_k}-\frac{1}{\beta_{k-1}} \right)\mathbb{E}[f(x_k)] -\frac{1}{\beta_{N-1}}\mathbb{E}[f(x_N)]+ \tilde{L}\frac{\sigma^2}{n}\sum_{k=0}^{N-1}\beta_k  \\
&\leq \frac{M_f}{\beta_0}+M_f\sum_{k=1}^{N-1}\left(\frac{1}{\beta_k}-\frac{1}{\beta_{k-1}} \right) -\frac{f^{low}}{\beta_{N-1}}+\tilde{L}\frac{\sigma^2}{n}\sum_{k=0}^{N-1}\beta_k \\
&=\frac{M_f-f^{low}}{\beta_{N-1}}+\tilde{L}\frac{\sigma^2}{n}\sum_{k=0}^{N-1}\beta_k \\
&\leq \frac{4L(1+C^{-1})(M_f-f^{low})}{\mu}N^r +\frac{(\mu+L^{-1})\sigma^2}{4n(1-r)}(N^{1-r}-1),
\end{align*}
which results in (\ref{them3:ineq1}), where the second inequality is due to (\ref{them1:ineq2}) and the last inequality is due to the choice of $ \beta_k$. Then for a give $ \epsilon>0 $, to guarantee that $ \frac{1}{N}\sum_{k=0}^{N-1}\mathbb{E}\|\nabla f(x_k)\|_2^2 < \epsilon $, it suffices to require that
\begin{align*}
&\frac{16L(1+C^{-1})(M_f-f^{low})}{\mu^2}N^{r-1} +\frac{(1+L^{-1}\mu^{-1})\sigma^2}{(1-r)n}(N^{-r}-N^{-1})<\epsilon.
\end{align*}
Since $ r\in(0.5,1) $, it follows that the number of iterations $ N$ needed is at most $O(\epsilon^{-\frac{1}{1-r}})$.
\end{proof}

 \begin{proof}[\textbf{Proof of Theorem~\ref{them:random_output}}]
 According to (\ref{ineq:ExpectLip}) in Lemma~\ref{lemma:descent}, we have
 \begin{align}
 \sum_{k=0}^{N-1}\beta_k\left(\frac{1}{2}\mu-\beta_k L(1+C^{-1}) \right)\mathbb{E}\|\nabla f(x_k) \|_2^2 \nonumber \\
 \leq f(x_0)-f^{low}+\sum_{k=0}^{N-1}\beta_k^2(L+\mu^{-1})(1+C^{-1})\frac{\sigma^2}{n_k},
 \end{align}
 where the expectation is taken with respect to $\lbrace S_j\rbrace_{j=0}^{N-1}$.
 Define 
 \begin{align}
 P_R(k)\overset{\text{def}}{=} Prob\lbrace R=k \rbrace 
 = \frac{\beta_k\left( \frac{1}{2}\mu-\beta_kL(1+C^{-1})\right)}{\sum_{j=0}^{N-1}\beta_j\left( \frac{1}{2}\mu-\beta_jL(1+C^{-1})\right) },
 \ k=0,\dots,N-1, \label{them_random:P_R}
\end{align} 
 then
 \begin{align}
 \mathbb{E}\left[ \|\nabla f(x_R) \|_2^2\right]
& = \frac{\sum_{k=0}^{N-1}\beta_k\left( \frac{1}{2}\mu-\beta_kL(1+C^{-1})\right)\mathbb{E}\left[ \|\nabla f(x_k) \|_2^2\right]}{\sum_{j=0}^{N-1}\beta_j\left( \frac{1}{2}\mu-\beta_jL(1+C^{-1})\right) } \nonumber \\
 &\leq \frac{D_f+\sigma^2(L+\mu^{-1})(1+C^{-1})\sum_{j=0}^{N-1}{\beta_k^2}/{n_k}}{\sum_{j=0}^{N-1}\beta_j\left( \frac{1}{2}\mu-\beta_jL(1+C^{-1})\right)}.
 \end{align}
 If we choose $ \beta_k =\beta:=\frac{\mu}{4L(1+C^{-1})}$ and $ n_k=n $, 
 then the definition of $P_R$ simplifies to $ P_R(k) = 1/N $ and
 We have
 \begin{align}
 \mathbb{E}\left[ \|\nabla f(x_R) \|_2^2\right] 
 &\leq \frac{D_f+\sigma^2(L+\mu^{-1})(1+C^{-1})N\frac{\beta^2}{n}}{\frac{1}{4}\mu N\beta} \nonumber \\
 &=\frac{4D_f}{\mu N\beta}+\frac{4(L+\mu^{-1})(1+C^{-1})\sigma^2\frac{\beta}{n}}{\mu} \nonumber \\
 &=\frac{16D_fL(1+C^{-1})}{N\mu^2}+\frac{(L+\mu^{-1})\sigma^2}{nL}. \label{proof_random:ineq3}
 \end{align}
 Let $\bar{N}$ be the total number of $\mathcal{SFO}$-calls needed to calculate stochastic gradients in SAM. Then the number of iterations of SAM is at most 
 $ N=\lceil \bar{N}/n \rceil $. Obviously, $ N\geq \bar{N}/(2n) $.  
 
 For a given accuracy tolerance $\epsilon>0$, we assume that
 \begin{align}
 \bar{N}\geq \left\lbrace  \frac{C_1^2}{\epsilon^2}+\frac{4C_2}{\epsilon}, \frac{\sigma^2}{L^2\tilde{D}} \right\rbrace, \label{proof_random:ineq4}
 \end{align}
 where
 \begin{equation}
 C_1 = \frac{32D_f(1+C^{-1})\sigma}{\mu^2\sqrt{\tilde{D}}}+(L+\mu^{-1})\sigma\sqrt{\tilde{D}}, \  \
 C_2 = \frac{32D_fL(1+C^{-1})}{\mu^2}, \label{proof_random:eq1}
 \end{equation}
 where $ \tilde{D} $ is a problem-independent positive constant. Moreover, we assume that the batch size satisfies
 \begin{equation}
 n_k=n:=\left\lceil \min\left\lbrace \bar{N},\max \left\lbrace 1,\frac{\sigma}{L}\sqrt{\frac{\bar{N}}{\tilde{D}}}\right\rbrace \right\rbrace \right\rceil. 
 \label{proof_random:eq2}
 \end{equation}
 The we can prove $  \mathbb{E}\left[ \|\nabla f(x_R) \|_2^2\right] \leq \epsilon $ as follows. \\ 
 From (\ref{proof_random:ineq3}) we have that 
 \begin{align}
 \mathbb{E}\left[ \|\nabla f(x_R) \|_2^2\right] 
 &\leq \frac{32D_f L(1+C^{-1})n}{\mu^2\bar{N}}+\frac{L+\mu^{-1}}{L}\frac{\sigma^2}{n} \nonumber \\
 &\leq \frac{32D_f L(1+C^{-1})}{\mu^2\bar{N}}\left( 1+ \frac{\sigma}{L}\sqrt{\frac{\bar{N}}{\tilde{D}}} \right) 
 + \frac{L+\mu^{-1}}{L}\cdot\max\left\lbrace \frac{\sigma^2}{\bar{N}}, \frac{\sigma L\sqrt{\tilde{D}}}{\sqrt{\bar{N}}} \right\rbrace. \label{proof_random:ineq5}
 \end{align}
 Equation (\ref{proof_random:ineq4}) implies that
 \begin{equation}
 \frac{\sigma^2}{\bar{N}}=\frac{\sigma}{\sqrt{\bar{N}}}\cdot \frac{\sigma}{\sqrt{\bar{N}}}
 \leq L\sqrt{\tilde{D}}\cdot\frac{\sigma}{\sqrt{\bar{N}}} = \frac{\sigma L\sqrt{\tilde{D}}}{\sqrt{\bar{N}}}.
 \end{equation}
 Then from (\ref{proof_random:ineq5}) and (\ref{proof_random:eq1}), we have
 \begin{align}
 \mathbb{E}\left[ \|\nabla f(x_R) \|_2^2\right] 
 \leq \frac{32D_f L(1+C^{-1})}{\mu^2\bar{N}}\left( 1+ \frac{\sigma}{L}\sqrt{\frac{\bar{N}}{\tilde{D}}} \right) 
 + \frac{L+\mu^{-1}}{L}\frac{\sigma L\sqrt{\tilde{D}}}{\sqrt{\bar{N}}}
 =\frac{C_1}{\sqrt{\bar{N}}}+\frac{C_2}{\bar{N}}. \label{proof_random:ineq6}
 \end{align}
 To ensure $ \mathbb{E}\left[ \|\nabla f(x_R) \|_2^2\right] \leq \epsilon $, it is sufficient to let the upper bound $ \frac{C_1}{\sqrt{\bar{N}}}+\frac{C_2}{\bar{N}}\leq \epsilon $, which implies
 \begin{equation*}
 \sqrt{\bar{N}}\geq \frac{\sqrt{C_1^2+4C_2\epsilon}+C_1}{2\epsilon}.
 \end{equation*}
 This is guaranteed by Condition (\ref{proof_random:ineq4}) since
 \begin{equation*}
 \sqrt{\bar{N}}\geq \frac{\sqrt{C_1^2+4C_2\epsilon}}{\epsilon}\geq \frac{\sqrt{C_1^2+4C_2\epsilon}+C_1}{2\epsilon}.
 \end{equation*}
\end{proof}
 
 To prove Theorem~\ref{them:sam_vr}, we first prove the following lemma.
 \begin{lemma}\label{lemma:sam_vr}
 Suppose that Assumptions~\ref{assume:Lips} and \ref{assume:noise} hold. For any $ \eta_t>0 $, set 
 \begin{align*}
 c_t^k &= c_{t+1}^k\left( 1+2\beta_t^k\eta_t+4(\beta_t^k)^2(1+C^{-1})\frac{L^2}{n}
 +2(\beta_t^k)^2\eta_t^{-1}(1+C^{-1})\frac{L^2}{n} \right) \nonumber \\
&\quad + (\beta_t^k)^2(L+\frac{\mu^{-1}}{2})(1+C^{-1})\frac{2L^2}{n}. 
 \end{align*}
 Then
 \begin{align}
 \Psi_t^k\mathbb{E}\left[ \|\nabla f(x_t^{k}) \|_2^2\right] 
 \leq \Phi_t^k-\Phi_{t+1}^k, \label{lemma_vr:ineq1}
 \end{align}
 where $\Phi_t^k = \mathbb{E}\left[ f(x_t^k)+c_t^k\| x_t^k-\tilde{x}_k\|_2^2\right]$ and 
 $ \Psi_t^k = \frac{1}{2}\beta_t^k\mu-2c_{t+1}^{k}\beta_t^k\eta_t^{-1}(1+C^{-1})-2L(\beta_t^k)^2 (1+C^{-1})-4c_{t+1}^k(\beta_t^k)^2(1+C^{-1}) $. 
 
 (\textbf{Note}: Here and in the following proof of Theorem~\ref{them:sam_vr}, the definition of notation $ c_t^k $ has no relation with $c_1, c_2 $ in \eqref{deltak} of AdaSAM.)
\end{lemma} 

\begin{proof}
 Define $ \epsilon_t^k = -r_t^k-\nabla f(x_t^k), M_t^k = \beta_t^kI-H_t^k $. Then
 \begin{align*}
 \mathbb{E}\left[ \nabla f(x_t^k)^{\mathrm{T}}H_t^k\epsilon_t^k\right]
 = \mathbb{E}\left[ \nabla f(x_t^k)^{\mathrm{T}}(\beta_t^k\epsilon_t^k-M_t^k\epsilon_t^k)\right]
 = -\mathbb{E}\left[ f(x_t^k)^{\mathrm{T}}M_t^k\epsilon_t^k \right].
 \end{align*}
 According to Lemma~\ref{lemma:Hkrk_sq}, for any $ v_k\in \mathbb{R}^d $,
 \begin{align}
 \|H_t^kv_k\|_2^2 \leq 2\left( \beta_k^2\left(1+2\alpha_k^2-2\alpha_k \right)+\alpha_k^2\delta_k^{-1} \right)\| v_k\|_2^2
 \leq 2(\beta_t^k)^2(1+C^{-1})\| v_k\|_2^2.
 \end{align}
 We also have
 \begin{align}
 &\quad \nabla f(x_t^k)^{\mathrm{T}}\mathbb{E}\left[ H_t^kr_t^k \right] \nonumber \\
 &=-\nabla f(x_t^k)^{\mathrm{T}}H_t^k\nabla f(x_t^k)-\mathbb{E}\left[ \nabla f(x_t^k)^{\mathrm{T}}H_t^k\epsilon_t^k \right]  \nonumber \\
 &\leq -\beta_t^k\mu\|\nabla f(x_t^k)\|_2^2+ \vert \mathbb{E}\left[ \nabla f(x_t^k)^{\mathrm{T}}M_t^k\epsilon_t^k \right] \vert \nonumber \\
 &\leq -\beta_t^k\mu\|\nabla f(x_t^k)\|_2^2 + \| \nabla f(x_t^k)\|_2 \sqrt{\mathbb{E}\left[ \|M_t^k\epsilon_t^k\|_2^2\right]} \nonumber \\
 &\leq -\beta_t^k\mu\|\nabla f(x_t^k)\|_2^2+  \alpha_t^k((\delta_t^k)^{-\frac{1}{2}}+\beta_t^k)\| \nabla f(x_t^k)\|_2  \sqrt{\mathbb{E}\left[ \| \epsilon_t^k\|_2^2\right]}  \quad (\text{cf. (\ref{ineq:Mk_nd})}) \nonumber \\
 & = -\beta_t^k\mu\|\nabla f(x_t^k)\|_2^2+ \sqrt{\beta_t^k\mu}\| \nabla f(x_t^k)\|_2\cdot \frac{\alpha_t^k((\delta_t^k)^{-\frac{1}{2}}+\beta_t^k)}{\sqrt{\beta_t^k\mu}}
 \sqrt{\mathbb{E}\left[ \| \epsilon_t^k\|_2^2\right]} \nonumber \\
 &\leq -\frac{1}{2}\beta_t^k\mu\| \nabla f(x_t^k)\|_2^2
 + \frac{1}{2}\frac{(\alpha_t^k)^2((\delta_t^k)^{-\frac{1}{2}}+\beta_t^k)^2}{\beta_t^k\mu}\mathbb{E}\left[ \| \epsilon_t^k\|_2^2\right] \nonumber \\
 &\leq -\frac{1}{2}\beta_t^k\mu\| \nabla f(x_t^k)\|_2^2+
 (\beta_t^k)^2\mu^{-1}(1+C^{-1})\mathbb{E}\| \epsilon_t^k\|_2^2.
 \end{align}
 Hence, with Assumption~\ref{assume:Lips} we obtain
 \begin{align}
 &\quad \mathbb{E}\left[ f(x_{t+1}^{k}) \right] \nonumber \\
 &\leq \mathbb{E}\left[ f(x_t^k)+\nabla f(x_t^k)^{\mathrm{T}}(x_{t+1}^k-x_t^k) + \frac{L}{2}\|x_{t+1}^k-x_t^k\|_2^2 \right] \nonumber \\
 &= \mathbb{E}\left[ f(x_t^k)+\nabla f(x_t^k)^{\mathrm{T}}H_t^kr_t^k + \frac{L}{2}\|H_t^kr_t^k\|_2^2 \right] \nonumber \\
 &\leq \mathbb{E}\left[ f(x_t^k) -\frac{1}{2}\beta_t^k\mu\| \nabla f(x_t^k)\|_2^2+
 (\beta_t^k)^2\mu^{-1}(1+C^{-1})\mathbb{E}\| \epsilon_t^k\|_2^2 
  +L(\beta_t^k)^2(1+C^{-1})\| r_t^k\|_2^2
  \right]. \label{proof_lemma_vr:ineq1}
 \end{align}
 Next, we give a bound of $ (x_t^k-\tilde{x}_k)^{\mathrm{T}}H_t^kr_t^k $. Since
 \begin{align*}
  -(x_t^k-\tilde{x}_k)^{\mathrm{T}}H_t^k\nabla f(x_t^k) \nonumber 
 &=-\beta_t^k\cdot (\beta_t^k)^{-1}(x_t^k-\tilde{x}_k)^{\mathrm{T}}H_t^k\nabla f(x_t^k) \nonumber \\
 &\leq \beta_t^k \| x_t^k-\tilde{x}_k\|_2 \| (\beta_t^k)^{-1}H_t^k\nabla f(x_t^k)\|_2 \nonumber \\
 &= \beta_t^k \eta_t^{1/2}\| x_t^k-\tilde{x}_k\|_2 \cdot \eta_t^{-1/2}\| (\beta_t^k)^{-1}H_t^k\nabla f(x_t^k)\|_2 \nonumber \\
 &\leq \frac{1}{2}\beta_t^k \left( \eta_t\| x_t^k-\tilde{x}_k\|_2^2
 +\eta_t^{-1}(\beta_t^k)^{-2}\| H_t^k\nabla f(x_t^k)\|_2^2 \right) \nonumber \\
 &\leq \frac{1}{2}\beta_t^k \left( \eta_t\| x_t^k-\tilde{x}_k\|_2^2
 +2\eta_t^{-1}(1+C^{-1})\| \nabla f(x_t^k)\|_2^2 \right),
 \end{align*}
 and 
 \begin{align*}
  -\mathbb{E}[(x_t^k-\tilde{x}_k)^{\mathrm{T}}H_t^k\epsilon_t^k] 
 &=-\mathbb{E}[(x_t^k-\tilde{x}_k)^{\mathrm{T}}(\beta_t^kI-M_t^k)\epsilon_t^k] \\
 &=\mathbb{E}[(x_t^k-\tilde{x}_k)^{\mathrm{T}}M_t^k\epsilon_t^k] \\
 &\leq \| x_t^k-\tilde{x}_k\|_2\cdot \sqrt{\mathbb{E}[\|M_t^k\epsilon_t^k\|_2^2]} \\
 &=\sqrt{\beta_t^k\eta_t}\| x_t^k-\tilde{x}_k\|_2 \cdot \frac{\alpha_t^k((\delta_t^k)^{-1/2}+\beta_t^k)}{\sqrt{\beta_t^k\eta_t}}\sqrt{\mathbb{E}\|\epsilon\|_2^2} \\
 &\leq \frac{1}{2}\beta_t^k\eta_t\| x_t^k-\tilde{x}_k\|_2^2
 +\frac{1}{2}\frac{(\alpha_t^k)^2((\delta_t^k)^{-1/2}+\beta_t^k)^2}{\beta_t^k\eta_t}\mathbb{E}\|\epsilon\|_2^2 \\
 &\leq \frac{1}{2}\beta_t^k\eta_t\| x_t^k-\tilde{x}_k\|_2^2
 +\eta_t^{-1}(\beta_t^k)^2(1+C^{-1})\mathbb{E}\|\epsilon\|_2^2 ,
 \end{align*}
 we obtain
 \begin{align*}
 &\quad \mathbb{E}[(x_t^k-\tilde{x}_k)^{\mathrm{T}}H_t^kr_t^k] \nonumber \\
 &= \mathbb{E}[-(x_t^k-\tilde{x}_k)^{\mathrm{T}}H_t^k\nabla f(x_t^k)
 -(x_t^k-\tilde{x}_k)^{\mathrm{T}}H_t^k\epsilon_t^k] \nonumber \\
 &\leq \mathbb{E}[\beta_t^k\eta_t^{-1}(1+C^{-1})\|\nabla f(x_t^k)\|_2^2+
 \beta_t^k\eta_t\| x_t^k-\tilde{x}_k\|_2^2 
 +\eta_t^{-1}(\beta_t^k)^2(1+C^{-1})\mathbb{E}\|\epsilon\|_2^2].
 \end{align*}
 Hence, we have
 \begin{align}
 &\quad \mathbb{E}[\| x_{t+1}^k-\tilde{x}_k\|_2^2] \nonumber \\
  &= \mathbb{E}[ \|x_{t+1}^{k}-x_t^k\|_2^2+\|x_t^k-\tilde{x}_k\|_2^2
  +2(x_t^k-\tilde{x}_k)^{\mathrm{T}}(x_{t+1}^k-x_t^k) ] \nonumber \\
 &= \mathbb{E}[ \|H_t^kr_t^k\|_2^2]+ \|x_t^k-\tilde{x}_k\|_2^2 
 + 2(x_t^k-\tilde{x}_k)^{\mathrm{T}}H_t^kr_t^k] \nonumber \\
 &\leq \mathbb{E}[ 2(\beta_t^k)^2(1+C^{-1})\|r_t^k\|_2^2+ \|x_t^k-\tilde{x}_k\|_2^2 \nonumber \\
 & \quad + 2\beta_t^k\eta_t^{-1}(1+C^{-1})\|\nabla f(x_t^k)\|_2^2+
 2\beta_t^k\eta_t\| x_t^k-\tilde{x}_k\|_2^2 
 +2\eta_t^{-1}(\beta_t^k)^2(1+C^{-1})\mathbb{E}\|\epsilon\|_2^2]. \label{proof_lemma_vr:ineq2}
 \end{align}
 Also, the following inequalities hold:
 \begin{align}
 \mathbb{E}\left[\| \epsilon_t^k\|_2^2\right]
 &=\mathbb{E}\left[ \|\nabla f_{\mathcal{K}}(x_t^k)-\nabla f_{\mathcal{K}}(\tilde{x}_k)
 +\nabla f(\tilde{x}_k) - \nabla f(x_t^k) \|_2^2 \right] \nonumber \\
 &=\mathbb{E}\left[ \| \nabla f_{\mathcal{K}}(x_t^k)-\nabla f_{\mathcal{K}}(\tilde{x}_k)
 -\mathbb{E}\left[ \nabla f_{\mathcal{K}}(x_t^k)-\nabla f_{\mathcal{K}}(\tilde{x}_k) \right] \|_2^2 \right] \nonumber \\
 &\leq \mathbb{E}\left[ \| \nabla f_{\mathcal{K}}(x_t^k)-\nabla f_{\mathcal{K}}(\tilde{x}_k)\|_2^2 \right] \nonumber \\
 &=\frac{1}{n}\mathbb{E}\left[ \|\nabla f_i(x_t^k)-\nabla f_i(\tilde{x}_k)\|_2^2
 \right] \leq \frac{L^2}{n}\mathbb{E}\left[ \|x_t^k-\tilde{x}_k \|_2^2\right], \label{proof_lemma_vr:ineq3}
 \end{align}
 \begin{align}
 \mathbb{E}\left[ \| r_t^k\|_2^2\right]
 &= \mathbb{E}\left[ \| \epsilon_t^k+\nabla f(x_t^k)\|_2^2\right] \nonumber \\
 &\leq \mathbb{E}\left[ 2\|\epsilon_t^k\|_2^2+2\|\nabla f(x_t^k)\|_2^2 \right] \nonumber \\
 &\leq 2\mathbb{E}\left[ \| \nabla f(x_t^k)\|_2^2\right] + \frac{2L^2}{n}\mathbb{E}\left[ \|x_t^k-\tilde{x}_k \|_2^2\right]. \label{proof_lemma_vr:ineq4}
 \end{align}
 Combining (\ref{proof_lemma_vr:ineq1}), (\ref{proof_lemma_vr:ineq2}), (\ref{proof_lemma_vr:ineq3}) and (\ref{proof_lemma_vr:ineq4}) yields  that
 \begin{align*}
 \Phi_{t+1}^k &= \mathbb{E}[ f(x_{t+1}^k)+ c_{t+1}^k\|x_{t+1}^{k}-\tilde{x}_k\|_2^2 ] \nonumber \\
 &\leq \mathbb{E}[ 
f(x_t^k) -\frac{1}{2}\beta_t^k\mu\| \nabla f(x_t^k)\|_2^2+
 (\beta_t^k)^2\mu^{-1}(1+C^{-1})\mathbb{E}\| \epsilon_t^k\|_2^2 
  +L(\beta_t^k)^2(1+C^{-1})\| r_t^k\|_2^2  \nonumber \\
 &\quad + c_{t+1}^k2(\beta_t^k)^2(1+C^{-1})\|r_t^k\|_2^2+ c_{t+1}^k\|x_t^k-\tilde{x}_k\|_2^2  + c_{t+1}^k2\beta_t^k\eta_t^{-1}(1+C^{-1})\|\nabla f(x_t^k)\|_2^2 \nonumber \\
 &\quad +c_{t+1}^k2\beta_t^k\eta_t\| x_t^k-\tilde{x}_k\|_2^2 
 +c_{t+1}^k2\eta_t^{-1}(\beta_t^k)^2(1+C^{-1})\mathbb{E}\|\epsilon_t^k\|_2^2
 ]  \nonumber \\
 &= \mathbb{E}[ f(x_t^k)-(\frac{1}{2}\beta_t^k\mu-c_{t+1}^k2\beta_t^k\eta_t^{-1}(1+C^{-1}))\|\nabla f(x_t^k)\|_2^2  \nonumber \\
 &\quad +(L(\beta_t^k)^2(1+C^{-1})+c_{t+1}^k2(\beta_t^k)^2(1+C^{-1}))\|r_t^k\|_2^2  \nonumber \\
 &\quad +(c_{t+1}^k+c_{t+1}^k2\beta_t^k\eta_t)\| x_t^k-\tilde{x}_k\|_2^2 \nonumber \\
 &\quad +((\beta_t^k)^2\mu^{-1}(1+C^{-1})+c_{t+1}^k2(\beta_t^k)^2\eta_t^{-1}(1+C^{-1}))\|\epsilon_t^k\|_2^2 ] \nonumber \\
 &\leq \mathbb{E}[f(x_t^k)- (\frac{1}{2}\beta_t^k\mu-c_{t+1}^k2\beta_t^k\eta_t^{-1}(1+C^{-1}))\|\nabla f(x_t^k)\|_2^2 \nonumber \\
 &\quad +(L(\beta_t^k)^2(1+C^{-1})+c_{t+1}^k2(\beta_t^k)^2(1+C^{-1}))(2 \| \nabla f(x_t^k)\|_2^2 + \frac{2L^2}{n} \|x_t^k-\tilde{x}_k \|_2^2 ) \nonumber \\
 &\quad + (c_{t+1}^k+c_{t+1}^k2\beta_t^k\eta_t)\| x_t^k-\tilde{x}_k\|_2^2  \nonumber \\
 &\quad +((\beta_t^k)^2\mu^{-1}(1+C^{-1})+c_{t+1}^k2(\beta_t^k)^2\eta_t^{-1}(1+C^{-1}))\frac{L^2}{n}\| x_t^k-\tilde{x}_k\|_2^2 ] \nonumber \\
 &= \mathbb{E}[f(x_t^k) + (c_{t+1}^k(1+2\beta_t^k\eta_t+4(\beta_t^k)^2(1+C^{-1})\frac{L^2}{n}
 +2(\beta_t^k)^2\eta_t^{-1}(1+C^{-1})\frac{L^2}{n}) \nonumber \\
 &\quad+(\beta_t^k)^2(1+C^{-1})\frac{2L^3}{n}+(\beta_t^k)^2\mu^{-1}(1+C^{-1})\frac{L^2}{n})\|x_t^k-\tilde{x_k}\|_2^2     \nonumber \\
 &\quad -(\frac{1}{2}\beta_t^k\mu-2c_{t+1}^k\beta_t^k\eta_t^{-1}(1+C^{-1})
 -2L(\beta_t^k)^2(1+C^{-1})		 \nonumber \\
 &\quad -4c_{t+1}^k(\beta_t^k)^2(1+C^{-1}))\|\nabla f(x_t^k)\|_2^2] 
  =\Phi_t^{k}-\Psi_t^{k}\mathbb{E}[ \| \nabla f(x_t^k)\|_2^2], 
 \end{align*}
 which further implies (\ref{lemma_vr:ineq1}). 
\end{proof}

\begin{proof}[\textbf{Proof of Theorem~\ref{them:sam_vr}}]
 Let $
  \eta_t = \eta := \frac{L(1+C^{-1})^{1/2}}{T^{1/3}} . $
 Denote $
  \theta = 2\beta\eta+4\beta^2(1+C^{-1})L^2/n+2\beta^{2}\eta^{-1}(1+C^{-1})L^2/n .$
   It then follows that 
 \begin{align*}
 \theta &= \frac{2\mu_0n}{T}+\frac{4\mu_0^2n}{T^{4/3}}+ \frac{2\mu_0^2n}{TL(1+C^{-1})^{1/2}} \\
 &\leq \frac{\mu_0n}{T}\left(2+4\mu_0+\frac{2\mu_0}{L(1+C^{-1})^{1/2}}\right) 
 \leq \frac{\mu_0n}{T}\left( 6+\frac{2}{L(1+C^{-1})^{1/2}}\right) = \frac{\mu_0n}{T}d_0,
 \end{align*}
 
 which implies $ (1+\theta)^q \leq e$. 
 Let $ c_q^k=c_q=0 $, then for any $ k\geq 0 $, we have 
 \begin{align*}
 c_{t}^k\leq c_0^k &= \beta^2(L+\frac{\mu^{-1}}{2})(1+C^{-1})\frac{2L^2}{n}\cdot\frac{(1+\theta)^q-1}{\theta}        \nonumber \\
 &=\frac{2\mu_0^2n(L+\frac{\mu^{-1}}{2})((1+\theta)^q-1)}{T^{\frac{4}{3}}\theta}  
 =\frac{2\mu_0^2(L+\frac{\mu^{-1}}{2})((1+\theta)^q-1)}{2\mu_0T^{\frac{1}{3}}+4\mu_0^2+\frac{2\mu_0^2}{L(1+C^{-1})^{\frac{1}{2}}}T^{\frac{1}{3}}}    \nonumber \\
 &\leq \frac{2\mu_0^2(L+\frac{\mu^{-1}}{2})((1+\theta)^q-1)}{2\mu_0T^{\frac{1}{3}}}
 \leq \frac{\mu_0(L+\frac{\mu^{-1}}{2})(e-1)}{T^{\frac{1}{3}}}. 
 \end{align*}
 Therefore, it follows that
 \begin{align*}
 \Psi_t^{k} &= \frac{1}{2}\beta\mu-2c_{t+1}^{k}\beta\eta^{-1}(1+C^{-1})-2L\beta^2(1+C^{-1})-4c_{t+1}^{k}\beta^2(1+C^{-1})  \\
 &=\frac{1}{2}\frac{\mu_0n\mu}{L(1+C^{-1})^{\frac{1}{2}}T^{\frac{2}{3}}}
 -2c_{t+1}^k\frac{\mu_0n}{L^2T^\frac{1}{3}}-\frac{2\mu_0^2n^2}{LT^{\frac{4}{3}}}- 4c_{t+1}^k\frac{\mu_0^2n^2}{L^2T^\frac{4}{3}}   \\
 &\geq \frac{1}{2}\frac{\mu_0n\mu}{L(1+C^{-1})^{\frac{1}{2}}T^{\frac{2}{3}}}
 -\frac{2\mu_0^2(L+\frac{\mu^{-1}}{2})(e-1)n}{L^2T^\frac{2}{3}} -\frac{2\mu_0^2n^2}{LT^\frac{4}{3}} - \frac{4\mu_0^3(L+\frac{\mu^{-1}}{2})(e-1)n^2}{L^2T^\frac{5}{3}}  \\
 &\geq \frac{n}{LT^\frac{2}{3}}\left( \frac{\mu_0\mu}{2(1+C^{-1})^\frac{1}{2}}
 -\frac{2\mu_0^2(L+\frac{\mu^{-1}}{2})(e-1)}{L} -2\mu_0^2n
 -\frac{4\mu_0^3(L+\frac{\mu^{-1}}{2})(e-1)n}{L} \right) \\
 &\geq \frac{n\nu}{LT^\frac{2}{3}}.
 \end{align*}
 As a result, we have
 \begin{align*}
 \sum_{t=0}^{q-1}\mathbb{E}\left[ \| \nabla f(x_t^k)\|_2^2\right]
 \leq \frac{\Phi_0^k-\Phi_q^k}{\min_t\Psi_t^{k}}
 = \frac{\mathbb{E}\left[ f(\tilde{x}_k)-f(\tilde{x}_{k+1})\right]}{\min_t\Psi_t^{k}},
 \end{align*}
 which yields that 
 \begin{align*}
 \mathbb{E}\left[ \| \nabla f(x)\|_2^2\right]
 =\frac{1}{qN}\sum_{k=0}^{N-1}\sum_{t=0}^{q-1}\mathbb{E}\left[ \|\nabla f(x_t^k)\|_2^2 \right]\leq \frac{f(x_0)-f(x_*)}{qN\min_t\Psi_t^{k}} 
 \leq \frac{T^{2/3}L(f(x_0)-f(x_*))}{qNn\nu}.
 \end{align*}
 To achieve $ \mathbb{E}\left[ \| \nabla f(x)\|_2^2\right] \leq \epsilon $, the outer iteration number $ N $ of Algorithm~\ref{alg:sam-vr} should be in the order of $ O\left( \frac{T^{2/3}}{qn\epsilon}\right) = O\left( \frac{T^{-1/3}}{\epsilon}\right) $, which is due to the fact that $ qn = O(T) $. As as result, the total number of $\mathcal{SFO}$-calls is $ (T+2qn)N $, which is $ O(T^{2/3}/\epsilon) $.
\end{proof}

\subsection{Relationship with GMRES} 
 Although the previous worst-case analysis shows that Anderson mixing has similar convergence rate as SGD,  Anderson mixing usually performs much better in practice. To explain this phenomenon, we briefly discuss the relationship of Anderson mixing with GMRES    \citep{saad1986gmres} for deterministic quadratic optimization since a twice continuously differentiable objective function can be approximated by a quadratic model in a local region, thus leading to a quadratic optimization. An optimization method that performs well in quadratic optimization is likely to have good convergence property as well in general nonlinear optimization.
 
 We consider the following strongly convex quadratic objective function:
 \begin{equation}
 f(x) = \frac{1}{2}x^\mathrm{T} A x - b^\mathrm{T} x. \label{eq:convexquad}
\end{equation}
where $ A\in\mathbb{R}^{d\times d} $ is symmetric positive definite, $ b\in\mathbb{R}^d $.  Solving (\ref{eq:convexquad}) is equivalent to solving linear system
\begin{equation}
 Ax = b. \label{eq:linear}
\end{equation}
In this case, $ \nabla f(x) = Ax-b, \nabla^2 f(x) = A $, $ r_k = b-Ax_k $ is the residual.
Hence the quadratic approximation in AM is always exact, i.e. $ R_k = -AX_k=-\nabla^2 f(x_k)X_k $. 

 When neither regularization nor damping  is used, i.e. $ \delta_k = 0 $ and $ \alpha_k = 1 $, SAM is identical to AM to accelerate fixed-point iteration $ g(x)= (I-A)x+b $.   It has been proved in \citep{walker2011anderson}  that in exact arithmetic, AA is essentially equivalent to GMRES when starting from the same initial point and  no stagnation occurs.  We restate  the main result here.
 
 Let $ x_k^{\text{G}}, r_k^{\text{G}}\overset{\text{def}}{=}b-Ax_k^{\text{G}} $ denote the $k$-th GMRES iterate and residual, respectively, and $ \mathcal{K}_k(A,v)\overset{\text{def}}{=}span\{ v, Av,\ldots,A^{k-1}v \} $ denotes the $k$-th Krylov subspace generated by $  A$ and $v$. Define  $ e^j\overset{\text{def}}{=}(1,1,\ldots,1)^\mathrm{T} \in \mathbb{R}^j $ for $ j\geq 1 $. For brevity, let $ span(X) $ denote the linear space spanned by the columns of $ X $. Besides, $ \{x_k \}$ are the iterates generated by AM (SAM), and $\{\bar{x}_k\}$ are the intermediate iterates generated by (\ref{xbar}).
we have
 
 \begin{proposition} \label{prop:linear1} 
 To minimize \eqref{eq:convexquad},
suppose that for SAM,  $ \delta_k=0, \alpha_k=\beta_k=1, m=k\geq 1 $.  If $ x_0 = x_0^{\text{\rm G}} $ and $ rank(R_k) = m $, then $ \bar{x}_k = x_k^{\text{\rm G}} $. 
\end{proposition}
\begin{proof}
 Since $ R_k = -AX_k $ and $ A $ is nonsingular, we have $ rank(X_k)=m $. We first show $ span(X_k)=\mathcal{K}_k(A,r_0^{\text{G}}) $ by induction. We abbreviate $\mathcal{K}_k(A,r_0^{\text{G}}) $ as $ \mathcal{K}_k $ in this proof.
 
 First, $\Delta x_0 = r_0 = r_0^{\text{\rm G}} $. If $ k=1$, then the proof is complete. Then,
suppose that $ k>1 $ and, as an inductive hypothesis, that $ span(X_{k-1})=\mathcal{K}_{k-1} $. 
With (\ref{eq:sam}) and noting that  $\alpha_k = \beta_k = 1$, we have 
\begin{align}
\Delta x_{k-1}&=x_k-x_{k-1} \nonumber \\
&=r_{k-1}-(X_{k-1}+R_{k-1})\Gamma_{k-1} \nonumber \\
&=b-Ax_{k-1}-(X_{k-1}-AX_{k-1})\Gamma_{k-1} \nonumber \\
&=b-A(x_0+\Delta x_0+\cdots+\Delta x_{k-2}) -(X_{k-1}-AX_{k-1})\Gamma_{k-1}  \nonumber \\
&=r_0-AX_{k-1}e^{k-1}-(X_{k-1}-AX_{k-1})\Gamma_{k-1}. \label{eq:delta_x_k_1}
\end{align}
 Since $r_0\in\mathcal{K}_{k-1}$, and by the inductive hypothesis $ span(X_{k-1})\subseteq \mathcal{K}_{k-1}$ which also implies $ span(AX_{k-1})\subseteq \mathcal{K}_k$, we know $ \Delta x_{k-1}\in \mathcal{K}_k $, which implies $ span(X_k) \subseteq \mathcal{K}_k $. Since we assume $ rank(X_k) = m = k$ which implies $ \dim (span(X_k))=\dim (\mathcal{K}_k) $, we have $ span(X_k) = \mathcal{K}_k $, thus completing the induction. 
 
 Recalling that to determine $\Gamma_k$, we solve the least-squares problem (\ref{lsq})
and $ R_k = -AX_k $, we have
\begin{equation}
\Gamma_k = \mathop{\arg\min}_{\Gamma \in \mathbb{R}^m}\| r_k + AX_k\Gamma \|_2. \label{eq:lsq_aa}
\end{equation}
Since $ rank(AX_k)=rank(X_k)=m $, (\ref{eq:lsq_aa}) has a unique solution.
Also, since $ r_k = b-Ax_k=b-A(x_0+X_ke^k)=r_0-AX_ke^k $, we have
$ r_k + AX_k\Gamma = r_0-AX_ke^k + AX_k\Gamma = r_0 -AX_k\tilde{\Gamma}$, where $ \tilde{\Gamma}=e^k-\Gamma $. So
$\Gamma_k$ solves (\ref{eq:lsq_aa}) if and only if $ \tilde{\Gamma}_k=e^k-\Gamma_k $  solves 
\begin{align}
\min\limits_{\tilde{\Gamma} \in \mathbb{R}^m}\| r_0-AX_k\tilde{\Gamma} \|_2, \label{eq:gmres}
\end{align}
which is the GMRES minimization problem.
 Since the solution of (\ref{eq:gmres}) is also unique, we have
\begin{align*}
\bar{x}_k&=x_k-X_k\Gamma_k=x_k-X_k(e^k-\tilde{\Gamma}_k) =x_0+X_k\tilde{\Gamma}_k = x_k^{\text{\rm G}}.
\end{align*}
\end{proof}

 If $ R_k $ is rank deficient, then a stagnation occur in AM (SAM). 
 \begin{proposition}
 To minimize \eqref{eq:convexquad}, 
suppose that for SAM,  $ \delta_k=0, \alpha_k=\beta_k=1, m=k \geq 1 $.  If $ rank(R_k) = k $  holds for $ 1\leq k < s  $ while failing to hold for $ k=s $, where $ s>1 $, then $\bar{x}_s = \bar{x}_{s-1}$.
\end{proposition}

\begin{proof}
The rank deficiency of $ X_s $ implies $ \Delta x_{s-1} \in span(X_{s-1}) $. Therefore, there exists $\gamma_s\in\mathbb{R}^{s-1}$ such that $ \Delta x_{s-1} = X_{s-1}\gamma_s $. Partitioning $ \tilde{\Gamma} \in \mathbb{R}^s $ in (\ref{eq:gmres}) as  $ \tilde{\Gamma} = (\eta_1,\eta_2)^{\mathrm{T}} \in \mathbb{R}^{s} $, where $ \eta_1\in\mathbb{R}^{s-1}, \eta_2\in\mathbb{R} $, we have
\begin{align*}
X_s\tilde{\Gamma} &= \begin{pmatrix} X_{s-1} & X_{s-1}\gamma_s \end{pmatrix} \begin{pmatrix} \eta_1 \\ \eta_2 \end{pmatrix} \\
&= X_{s-1}(\eta_1+\gamma_s\eta_2), 
\end{align*}
which implies
\begin{align*}
r_0-AX_s\tilde{\Gamma} =r_0-AX_{s-1}(\eta_1+\gamma_s\eta_2).
\end{align*}
Hence 
\begin{subequations}
\begin{align}
&\quad \min\limits_{\tilde{\Gamma} \in \mathbb{R}^s}\| r_0-AX_s\tilde{\Gamma} \|_2    \label{eq:lsq_aa_s} \\
&=\min\limits_{\eta_1\in\mathbb{R}^{s-1}, \eta_2 \in \mathbb{R}}\| r_0-AX_{s-1}(\eta_1+\gamma_s\eta_2) \|_2. \label{eq:lsq_aa_s_1}
\end{align}
\end{subequations}
Since $ AX_{s-1} $ has full rank, $ \tilde{\Gamma}_{s-1}=\eta_1+\gamma_s\eta_2 $ is the unique solution to minimize (\ref{eq:lsq_aa_s_1}) and  $ \tilde{\Gamma}_s = (\eta_1, \eta_2)^\mathrm{T} $ minimizes (\ref{eq:lsq_aa_s}) while being not unique. Therefore, from the equivalence of (\ref{eq:lsq_aa}) and (\ref{eq:gmres}), we conclude that
\begin{align*}
\bar{x}_s &= x_0+X_s\tilde{\Gamma}_s 
=x_0+X_{s-1}(\eta_1+\gamma_s\eta_2) \\
&=x_0+X_{s-1}\tilde{\Gamma}_{s-1}=\bar{x}_{s-1}.
\end{align*}
\end{proof}

When the stagnation happens, further iterations of AM cannot make any improvement. This is a potential numerical weakness of AM relative to GMRES, which does not break down upon stagnation before the solution has been found. At this point, switching to applying several fixed-point iteration $ x_{k+1} = g(x_k) $ may help jump out of the stagnation \citep{pratapa2016anderson}. 

 In Section~\ref{subsec:enhance}, we introduce the preconditioned mixing strategy. This form of preconditioning for AM is new as far as we know. We reveal its relationship with right preconditioned GMRES \citep{saad2003iterative} here. 
 Let $ x_k^{\text{RG}}, r_{k}^{\text{RG}} $  denote the $ k$-th right preconditioned GMRES iterate and residual. According to Proposition~9.1 in \citep{saad2003iterative}, with a fixed preconditioner $ M $, then $ x_k^{\text{RG}}$ in the right preconditioned GMRES minimize residual in the affine subspace
$
 x_0 +\mathcal{K}_k\lbrace  M^{-1}A, M^{-1}r_0^{\text{RG}} \rbrace.
 $
 
 \begin{proposition} \label{prop:linearprecond} 
 To minimize \eqref{eq:convexquad},
suppose that for preconditioned SAM (cf. \eqref{psam}),  $ \delta_k=0, \alpha_k=\beta_k=1, m=k\geq 1 $ and $ M_k = M $ where M is nonsingular.  If $ x_0 = x_0^{\text{\rm RG}} $ and $ rank(R_k) = m $, then $ \bar{x}_k = x_k^{\text{\rm RG}} $. 
\end{proposition}

\begin{proof}
 Since $ R_k = -AX_k $ and $ A $ is nonsingular, we have $ rank(X_k)=m $. We first show that $ span(X_k)=\mathcal{K}_k(M^{-1}A,M^{-1}r_0^{\text{RG}}) $ by induction. We abbreviate $\mathcal{K}_k(M^{-1}A,M^{-1}r_0^{\text{RG}}) $ as $ \mathcal{K}_k $ in this proof. 
 
 First, $\Delta x_0 = M^{-1}r_0 = M^{-1}r_0^{\text{\rm RG}} $. If $ k=1$, then the proof is complete. Then,
suppose that $ k>1 $ and, as an inductive hypothesis, that $ span(X_{k-1})=\mathcal{K}_{k-1} $. 
With (\ref{psam}) and noting that  $\alpha_k = \beta_k = 1$, we have 
\begin{align}
\Delta x_{k-1}&=x_k-x_{k-1} \nonumber \\
&=M^{-1}r_{k-1}-(X_{k-1}+M^{-1}R_{k-1})\Gamma_{k-1} \nonumber \\
&=M^{-1}(b-Ax_{k-1})-(X_{k-1}-M^{-1}AX_{k-1})\Gamma_{k-1} \nonumber \\
&=M^{-1}b-M^{-1}A(x_0+\Delta x_0+\cdots+\Delta x_{k-2}) -(X_{k-1}-M^{-1}AX_{k-1})\Gamma_{k-1}  \nonumber \\
&=M^{-1}r_0-M^{-1}AX_{k-1}e^{k-1}-(X_{k-1}-M^{-1}AX_{k-1})\Gamma_{k-1}, \label{eq:p_delta_x_k_1}
\end{align}
 where $ \Gamma_{k-1}  $ minimizes \eqref{lsq} since $ \delta_{k-1} = 0 $. 
 
 Since $ M^{-1}r_0\in\mathcal{K}_{k-1}$, and by the inductive hypothesis $ span(X_{k-1})\subseteq \mathcal{K}_{k-1}$ which also implies $ span(M^{-1}AX_{k-1})\subseteq \mathcal{K}_k$, we know $ \Delta x_{k-1}\in \mathcal{K}_k $, which implies $ span(X_k) \subseteq \mathcal{K}_k $. Since we assume $ rank(X_k) = m = k$ which implies $ \dim (span(X_k))=\dim (\mathcal{K}_k) $, we have $ span(X_k) = \mathcal{K}_k $, thus completing the induction. 
 
 Recalling that to determine $\Gamma_k$, we solve the least-squares problem (\ref{lsq})
and $ R_k = -AX_k $, we have
\begin{equation}
\Gamma_k = \mathop{\arg\min}_{\Gamma \in \mathbb{R}^m}\| r_k + AX_k\Gamma \|_2. \label{eq:p_lsq_aa}
\end{equation}
Since $ rank(AX_k)=rank(X_k)=m $, (\ref{eq:p_lsq_aa}) has a unique solution.
Also, since $ r_k = b-Ax_k=b-A(x_0+X_ke^k)=r_0-AX_ke^k $, we have
$ r_k + AX_k\Gamma = r_0-AX_ke^k + AX_k\Gamma = r_0 -AX_k\tilde{\Gamma}$, where $ \tilde{\Gamma}=e^k-\Gamma $. So
$\Gamma_k$ solves (\ref{eq:p_lsq_aa}) if and only if $ \tilde{\Gamma}_k=e^k-\Gamma_k $  solves 
\begin{align}
\min\limits_{\tilde{\Gamma} \in \mathbb{R}^m}\| r_0-AX_k\tilde{\Gamma} \|_2, \label{eq:p_gmres}
\end{align}
which is the right preconditioned GMRES minimization problem.
 Since the solution of (\ref{eq:p_gmres}) is also unique, we have
\begin{align*}
\bar{x}_k&=x_k-X_k\Gamma_k=x_k-X_k(e^k-\tilde{\Gamma}_k) =x_0+X_k\tilde{\Gamma}_k = x_k^{\text{\rm RG}}.
\end{align*}
\end{proof}
 For preconditioned SAM, the preconditioner $ M_k $ can vary from step to step, while the minimal residual property still holds, i.e. 
$
  \bar{x}_k  = \mathop{\arg\min}_{ x\in x_k +span\lbrace X_k \rbrace} \| b-Ax\|_2   .
 $
 \begin{remark}
 When the objective function is only approximately convex quadratic in a local region around the minimum, the relation between SAM and GMRES  can only approximately hold. Nonetheless, SAM can often show superior behaviour in practice.
 \end{remark}

\section{Pseudocode for AdaSAM/pAdaSAM} 
 \begin{algorithm}[tb] 
\caption{AdaSAM, namely our proposed SAM with $ \delta_k $ chosen as \eqref{deltak}. $ optim(x_k,g_k) $ is an optimizer which updates $ x_k $ given stochastic gradient $ g_k $.   }
\label{alg:adasam}
\textbf{Input}: $ x_0\in\mathbb{R}^d, m=10, \alpha_k=1, \beta_k=1, c_1 = 10^{-2},  p=1, \gamma=0.9, \epsilon=10^{-8}, max\_iter > 0 $, optimizer $ optim $.\\
\textbf{Output}: $ x\in\mathbb{R}^d $
\begin{algorithmic}[1] 
\STATE {$ \Delta\hat{x}_0 =  0, \Delta\hat{r}_0 = 0$}
\FOR{$k = 0,1,\cdots, max\_iter$ }
\STATE $ r_k = -\nabla f_{S_k}\left(x_k\right) $ 
\IF {$k>0$}
\STATE $ m_k = \min\{m,k\} $
\STATE $ \Delta\hat{x}_k =  \gamma\cdot\Delta\hat{x}_{k-1}+(1-\gamma)\cdot\Delta x_{k-1} $ 
\STATE $ \Delta\hat{r}_k = \gamma\cdot\Delta\hat{r}_{k-1}+(1-\gamma)\cdot\Delta r_{k-1} $
\STATE $ \hat{X}_k = [
\Delta \hat{x}_{k-m_k} , \Delta \hat{x}_{k-m_k+1} , \cdots , \Delta \hat{x}_{k-1} ] $
\STATE $ \hat{R}_k = [
\Delta \hat{r}_{k-m_k} , \Delta \hat{r}_{k-m_k+1} , \cdots , \Delta \hat{r}_{k-1} ] $
\ENDIF
\IF { $ k>0 $ and $ k $ mod $p = 0 $ }
\STATE $ \delta_k = c_1 \| r_k\|_2^2 / \left(\|\Delta\hat{x}_k\|_2^2 +\epsilon\right) $
\STATE $ \Delta x_k = \beta_k r_k - \alpha_k\left(\hat{X}_k + \beta_k \hat{R}_k \right)\left( \hat{R}_k^{\mathrm{T}}\hat{R}_k +\delta_k \hat{X}_k^{\mathrm{T}}\hat{X}_k\right)^{\dagger} \hat{R}_k^{\mathrm{T}}r_k  $ 
\IF {$ (\Delta x_k)^{\mathrm{T}}r_k > 0 $}
\STATE $ x_{k+1} = x_k+\Delta x_k $ 
\ELSE
\STATE $ x_{k+1} = optim(x_k,-r_k) $
\ENDIF
\ELSE
\STATE $ x_{k+1} =  optim(x_k,-r_k) $
\ENDIF
\STATE Apply learning rate schedule of $ \alpha_k, \beta_k $
\ENDFOR
\STATE \textbf{return} $ x_k $
\end{algorithmic}
\end{algorithm}

 Algorithm~\ref{alg:adasam} gives the pseudocode for AdaSAM. Based on the  prototype Algorithm~\ref{alg:sam}, we incorporate sanity check of the positive definiteness, alternating iteration and moving average in our implementation of AdaSAM:
 \begin{enumerate}
 \item \textbf{Sanity check of the positive definiteness}. 
 Except for calculating the largest eigenvalue to check Condition~(\ref{ineq:check_alpha_k}),  a rule of thumb is checking a necessary condition $ r_k^TH_kr_k > 0  $, i.e. the searching direction $\Delta x_k = H_kr_k $ is a descent direction with respect to the stochastic gradient $ \nabla f_{S_k}\left(x_k\right) $. If this condition is violated, we switch to updating $ x_k $ via $optim$. Although such rule of thumb is not theoretically justified, it causes no difficulty of convergence in our practice. (Line~14-17 in Algorithm~\ref{alg:adasam}.)
 \item \textbf{Alternating iteration}.
 To amortize the computational cost of SAA, it is  reasonable to apply a form of  alternating iteration like \citep{pratapa2016anderson}.  In each cycle, we iterate with $ optim $ for $ (p-1) $ steps and update $ X_k, R_k $ simultaneously, then apply SAA in the $ p $-th step, the result of which is the starting point of the next cycle. (Line~11,19 in Algorithm~\ref{alg:adasam}.)
 \item \textbf{Moving average}. 
 In mini-batch training, moving average can be used to incorporate information from the past and reduce the variability. Specifically, we maintain the moving averages of $ X_k, R_k $ by $ \hat{X}_k, \hat{R}_k $ respectively. 
 Here $ \hat{X}_k = [
\Delta \hat{x}_{k-m} , \Delta \hat{x}_{k-m+1} , \cdots , \Delta \hat{x}_{k-1} ] $, $ \hat{R}_k = [
\Delta \hat{r}_{k-m} , \Delta \hat{r}_{k-m+1} , \cdots , \Delta \hat{r}_{k-1} ] $, where $ \Delta\hat{x}_k =  \gamma\cdot\Delta\hat{x}_{k-1}+(1-\gamma)\cdot\Delta x_{k-1} $, 
 $ \Delta\hat{r}_k = \gamma\cdot\Delta\hat{r}_{k-1}+(1-\gamma)\cdot\Delta r_{k-1} $, $ \Delta\hat{x}_0 =  0, \Delta\hat{r}_0 = 0$ and $ \gamma\in[0,1) $. For deterministic quadratic optimization, $ \hat{R}_k = -\nabla^2 f(x_k)\hat{X}_k $ still holds. (Line~6-9 in Algorithm~\ref{alg:adasam}.)
 \end{enumerate}
 
 We point out that these three techniques are not  required in  our theoretical analysis in Section~\ref{sec:theory}. Nonetheless, they may have positive effects in practice. 
  
 \begin{algorithm}[tb] 
\caption{pAdaSAM, namely the preconditioned AdaSAM. $ optim(x_k,g_k) $ is an optimizer which updates $ x_k $ given stochastic gradient $ g_k $.   }
\label{alg:padasam}
\textbf{Input}: $ x_0\in\mathbb{R}^d, m=10, \alpha_k=1, \beta_k=1, c_1 = 10^{-2},  p=1, \gamma=0.9, \epsilon=10^{-8}, max\_iter > 0 $, optimizer $ optim $.\\
\textbf{Output}: $ x\in\mathbb{R}^d $
\begin{algorithmic}[1] 
\STATE {$ \Delta\hat{x}_0 =  0, \Delta\hat{r}_0 = 0$}
\FOR{$k = 0,1,\cdots, max\_iter$ }
\STATE $ r_k = -\nabla f_{S_k}\left(x_k\right) $ 
\IF {$k>0$}
\STATE $ m_k = \min\{m,k\} $
\STATE $ \Delta\hat{x}_k =  \gamma\cdot\Delta\hat{x}_{k-1}+(1-\gamma)\cdot\Delta x_{k-1} $ 
\STATE $ \Delta\hat{r}_k = \gamma\cdot\Delta\hat{r}_{k-1}+(1-\gamma)\cdot\Delta r_{k-1} $
\STATE $ \hat{X}_k = [
\Delta \hat{x}_{k-m_k} , \Delta \hat{x}_{k-m_k+1} , \cdots , \Delta \hat{x}_{k-1} ] $
\STATE $ \hat{R}_k = [
\Delta \hat{r}_{k-m_k} , \Delta \hat{r}_{k-m_k+1} , \cdots , \Delta \hat{r}_{k-1} ] $
\ENDIF
\IF { $ k>0 $ and $ k $ mod $p = 0 $ }
\STATE $ \delta_k = c_1 \| r_k\|_2^2 / \left(\|\Delta\hat{x}_k\|_2^2 +\epsilon\right) $
\STATE $ \Gamma_k =  \left( \hat{R}_k^{\mathrm{T}}\hat{R}_k +\delta_k \hat{X}_k^{\mathrm{T}}\hat{X}_k\right)^{\dagger} \hat{R}_k^{\mathrm{T}}r_k $
\STATE $ \bar{x}_k = x_k - \alpha_k \hat{X}_k \Gamma_k $
\STATE $ \bar{r}_k = r_k -\alpha_k\hat{R}_k\Gamma_k $
\STATE $ \Delta x_k = optim(\bar{x}_k,-\bar{r}_k) - x_k  $ 
\IF {$ (\Delta x_k)^{\mathrm{T}}r_k > 0 $}
\STATE $ x_{k+1} = x_k+\Delta x_k $ 
\ELSE
\STATE $ x_{k+1} = optim(x_k,-r_k) $
\ENDIF
\ELSE
\STATE $ x_{k+1} =  optim(x_k,-r_k) $
\ENDIF
\STATE Apply learning rate schedule of $ \alpha_k, \beta_k $
\ENDFOR
\STATE \textbf{return} $ x_k $
\end{algorithmic}
\end{algorithm}  

 Our implementation of the RAM method, i.e. using Tikhonov regularization (cf. \eqref{rlsq}), differs from AdaSAM by replacing Line~13 in Algorithm~\ref{alg:adasam} with 
 $
 \Delta x_k = \beta_k r_k - \alpha_k\left(\hat{X}_k + \beta_k \hat{R}_k \right)\left( \hat{R}_k^{\mathrm{T}}\hat{R}_k +\delta I\right)^{\dagger} \hat{R}_k^{\mathrm{T}}r_k.
 $
 In other words, we also incorporate the damped projection into  the Tikhonov regularized AM. Therefore, the comparison between AdaSAM and RAM can show the effect of adaptive regularization with $ \delta_k$  chosen as \eqref{deltak}.
 
 We give the pseudocode of pAdaSAM in Algorithm~\ref{alg:padasam}, which is the preconditioned version of AdaSAM. (See Section~\ref{subsec:enhance}.) The effect of the preconditioner $ optim $ is reflected in Line~16 in Algorithm~\ref{alg:padasam}. 
 
\begin{remark}
  In our implementations, we introduce several extra hyper-parameters to the prototype Algorithm~\ref{alg:sam}. However, we will show in the additional experiments that the only hyper-parameter needed to be tuned is still the regularization parameter $ c_1 $, while setting the other hyper-parameters as default is proper. We also omit the second term $ c_2\beta_k^{-2} $ in \eqref{deltak}, which we will justify in Section~\ref{subsec:adapreg}.
\end{remark}
  
\section{Experimental details} \label{sec:exp_details}
  Our codes were written in PyTorch1.4.0\footnote{ Information about this framework is referred to\href{https://pytorch.org}{https://pytorch.org}.}  and one GeForce RTX 2080 Ti GPU is used for each test. Our methods are AdaSAM and its preconditioned variant pAdaSAM.  Before describing the experimental details, we explain the hyperparameter setting of AdaSAM/pAdaSAM/RAM.

\subsection{Hyperparameter setting of AdaSAM/pAdaSAM/RAM} \label{subsec:para}
 Since these methods are all the variants of AM with minor differences, their hyperparameter setting are similar. 
 The only hyperparameter that needs to be carefully tuned is the regularization parameter, i.e. $ c_1 $ for AdaSAM/pAdaSAM, and $ \delta $ for RAM. We explain reasons of the default setting of other hyperparamters:
 \begin{itemize}
 \item $ \alpha_0=1 $. Setting $ \alpha_0 = 1 $ corresponds to using no damping, which means that the minimal residual procedure is exact for AM in deterministic quadratic optimization. This setting  follows the same philosophy of setting initial learning rate as 1 in Newton method. 
 \item $ \beta_0=1 $. Setting the mixing parameter $ \beta_0=1 $ is a standard setting in AM. (See e.g. \citep{walker2011anderson}.) Tuning $ \beta_k $ may be of help \citep{potra2013characterization,fang2009two}, but we abandon this possibility to reduce the work of hyperparameter tuning.
 \item $ m=10 $. Since the extra space is $ 2md$ and the extra computational cost is $ O(m^2d)+O(m^3) $, using small $ m $ is preferred. $ m=10 $ or $ 20 $ is also the default setting in restarted GMRES \citep{saad2003iterative}. Moreover, large $ m $, say, $ m=100 $, can cause the solution of $ \Gamma_k $ less stable as we solve the normal equation directly.
 \item $ p=1 $. By default, no alternating iteration is used. When the extra computational cost (e.g. Line~13 in Algorithm~\ref{alg:adasam}) dominates the computation, we consider this option to alleviate the  cost.
 \item $ \epsilon = 10^{-8} $.  $ \epsilon$  only serves as the safe-guard to prevent the denominator in \eqref{deltak} from being zero. It does not have meaning like Tikhonov regularization $ \delta $ in RAM. Only when $ \|\Delta \hat{x}_k \|_2^2 \approx \epsilon $, the effect of $ \epsilon$ becomes obvious, but $ x_k$ is supposed to converge at this point.
 \item $ \gamma = 0.9$. This is a default setting for moving average \citep{tieleman2012lecture,kingma2014adam}. 
 \end{itemize}

\subsection{Experiments on MNIST}
 Since SAM is expected to behave like the minimal residual method in deterministic quadratic optimization, this group of experiments focused on large mini-batch training where the variance of noise is relatively small, thus the curvature of the objective function rather than noise dominates the optimization. Moreover, using constant learning rate is proper in this situation.
 
 The baselines are SGDM, Adam and SdLBFGS.
 For SGDM and Adam, we used the built-in PyTorch implementations. For SdLBFGS, 
 in addition to the initial proposal \citep{wang2017stochastic},  the Hessian is always initialized with the identity matrix and the calculated descent direction is normalized because such modifications were found to be more effective for SdLBFGS in our experiments. 
  
 We tuned the learning rates of the baseline optimizers by log-scale grid-searches from $10^{-3}$ to $100$. The learning rates of SGDM, Adam and SdLBFGS were 0.1, 0.001, 1 and 0.1, respectively. 
 The historical length for SdLBFGS, RAM and AdaSAM was set as 20. $ \delta = 10^{-6}$ for RAM and $ c_1=10^{-4} $ for AdaSAM. The $ optim $ in Algorithm~\ref{alg:adasam} is $ optim(x_k,g_k)=x_k-0.1*g_k $. 
 
 For the preconditioned AdaSAM, i.e. Adagrad-AdaSAM and RMSprop-AdaSAM, the learning rates of Adagrad and RMSprop were 0.01, 0.001, respectively.
 
 For all the tests, the model was trained for 100 epochs.
 
 In the main paper, we only report training loss. Here, we  report both the training loss and the squared norm of the gradient (SNG) in Figure~\ref{fig:appendix_mnist_2}. By comparing the training loss and SNG, it can be observed that a smaller SNG typically indicates a smaller training loss, which confirms the way to minimize $ \mathbb{E}\| \nabla f(x)\|_2^2 $. Since AM is closely related to the minimal residual method, where the term ``residual"  is actually the gradient in optimization, AM is expected to achieve small SNG when the quadratic approximation of the objective function is accurate enough. From the experiments, we find the behaviour of AdaSAM is rather stable even in mini-batch training.
 
 \begin{figure}[ht]
\centering 
\subfigure[Batchsize=6K]{
\includegraphics[width=0.23\textwidth]{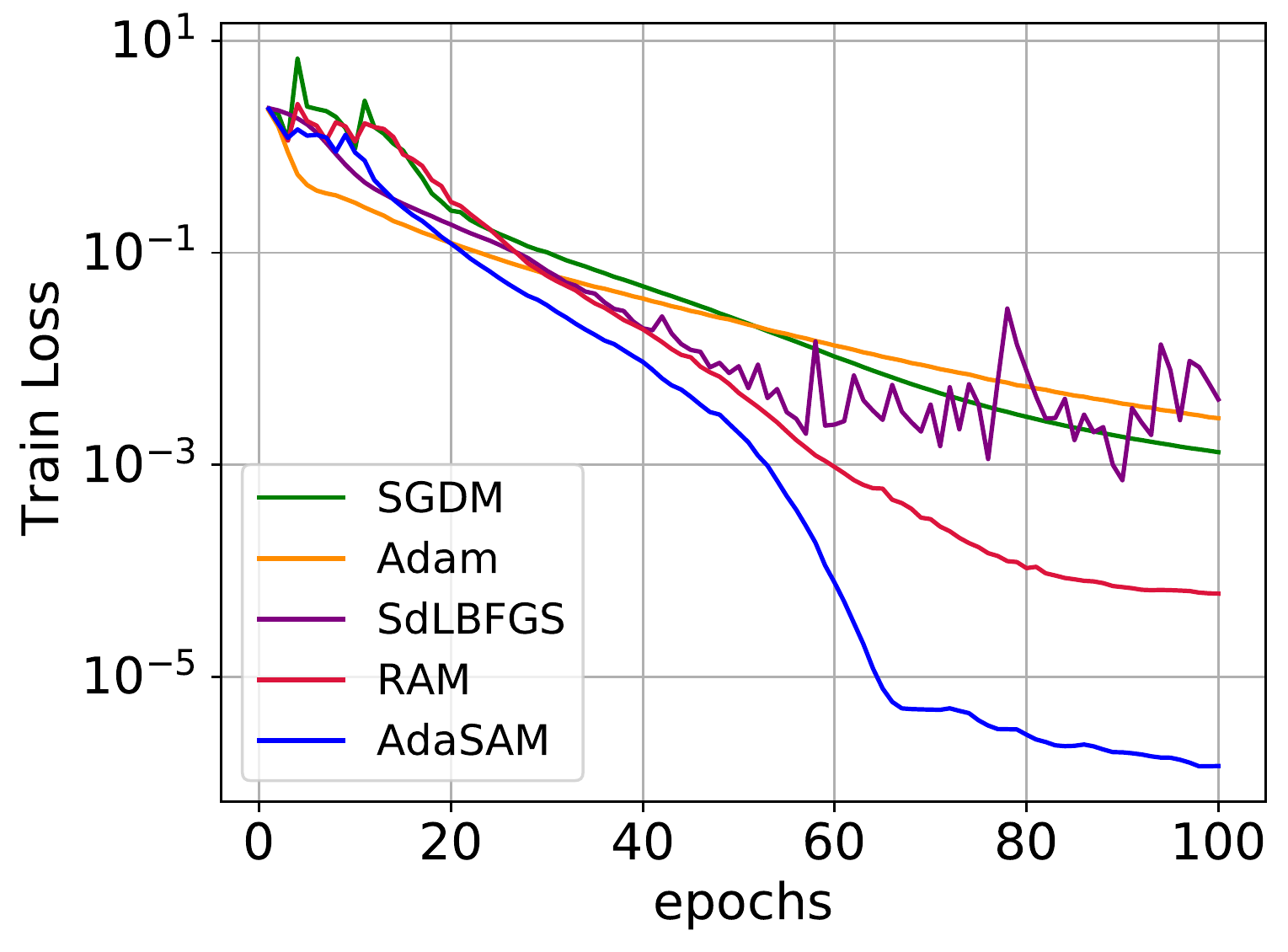}
}
\subfigure[Batchsize=3K]{
\includegraphics[width=0.23\textwidth]{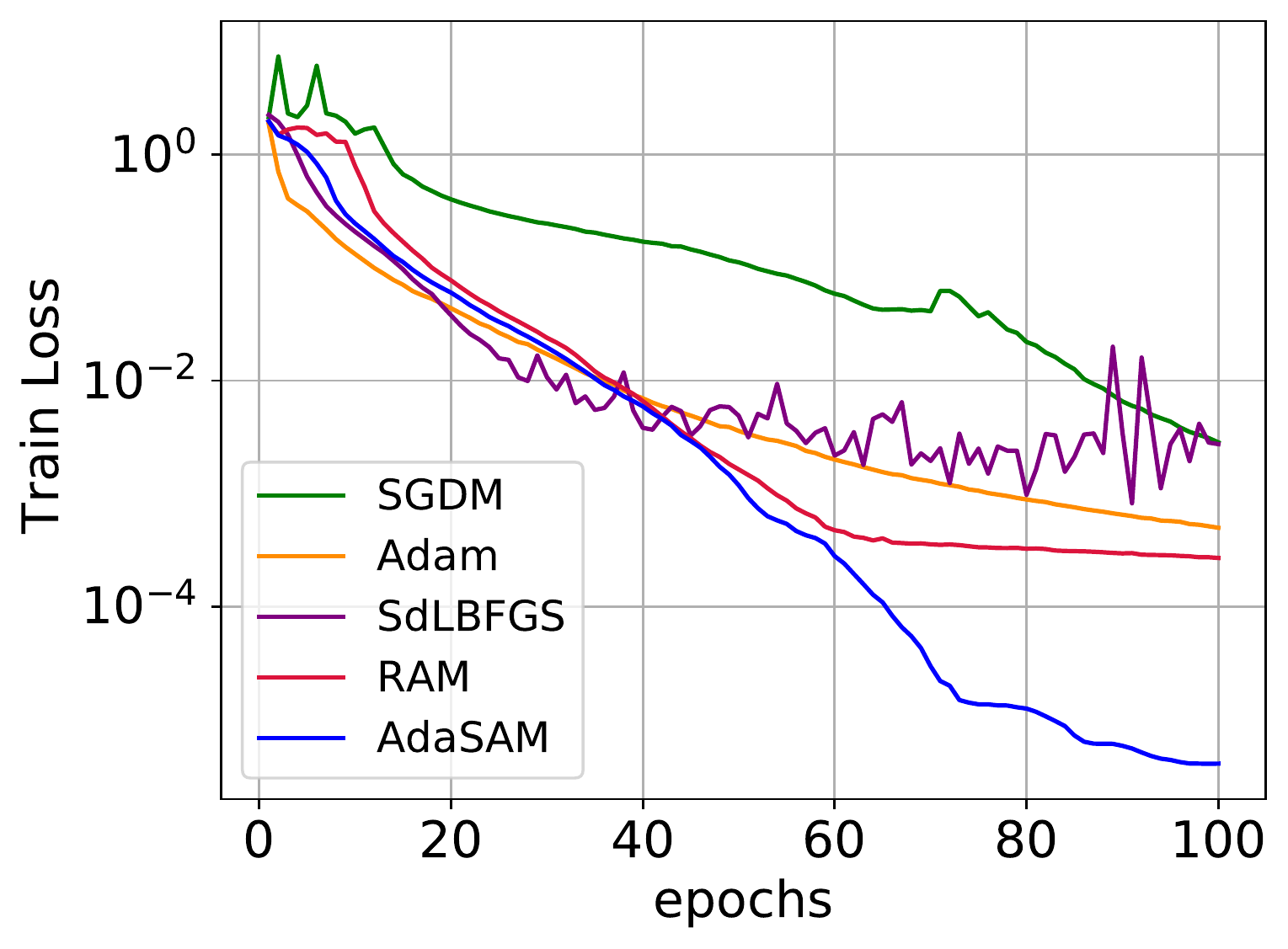}
}
\subfigure[Variance reduction]{
\includegraphics[width=0.23\textwidth]{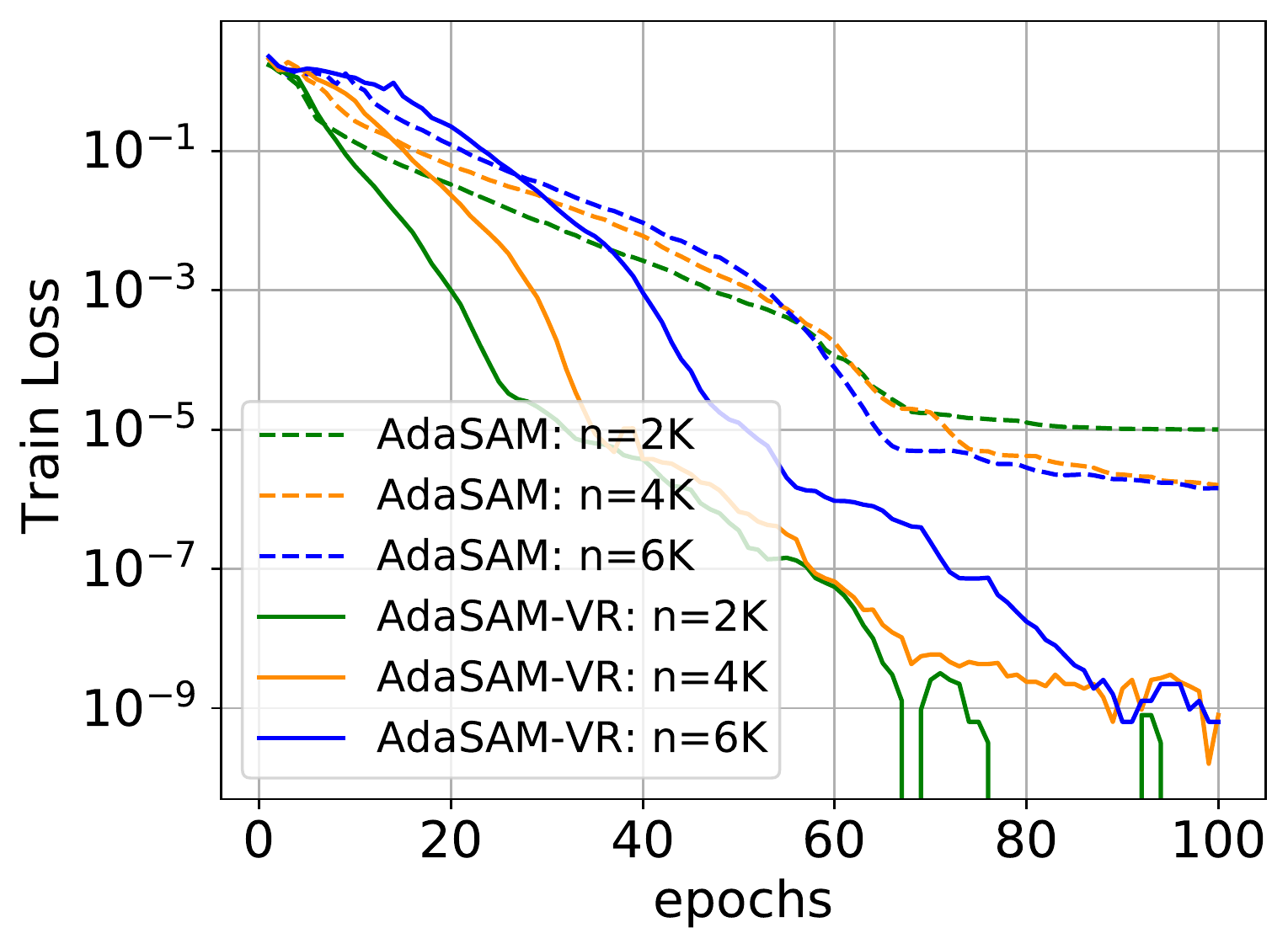}
}
\subfigure[Preconditioning]{
\includegraphics[width=0.23\textwidth]{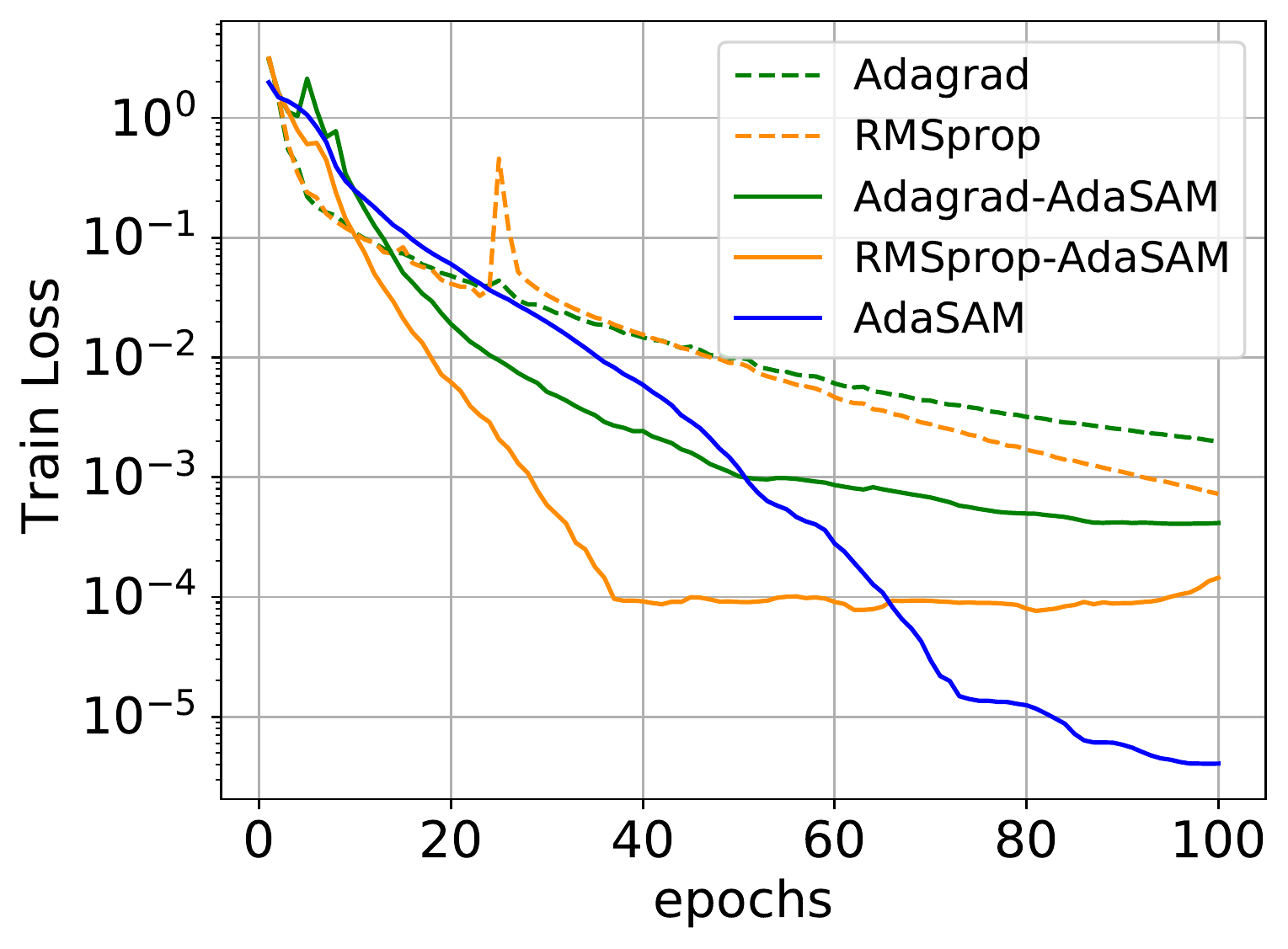}
}
\subfigure[Batchsize=6K]{
\includegraphics[width=0.23\textwidth]{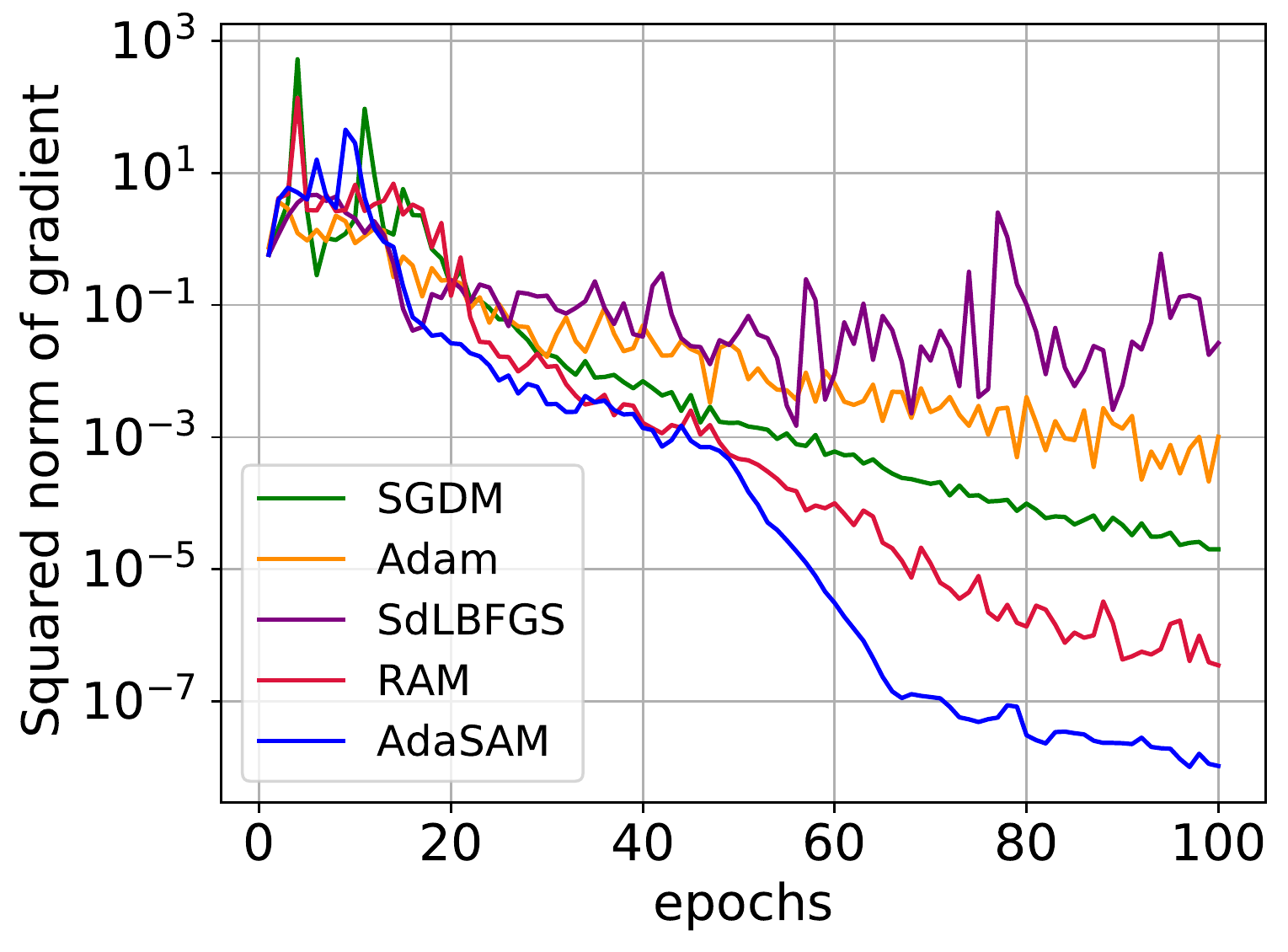}
}
\subfigure[Batchsize=3K]{
\includegraphics[width=0.23\textwidth]{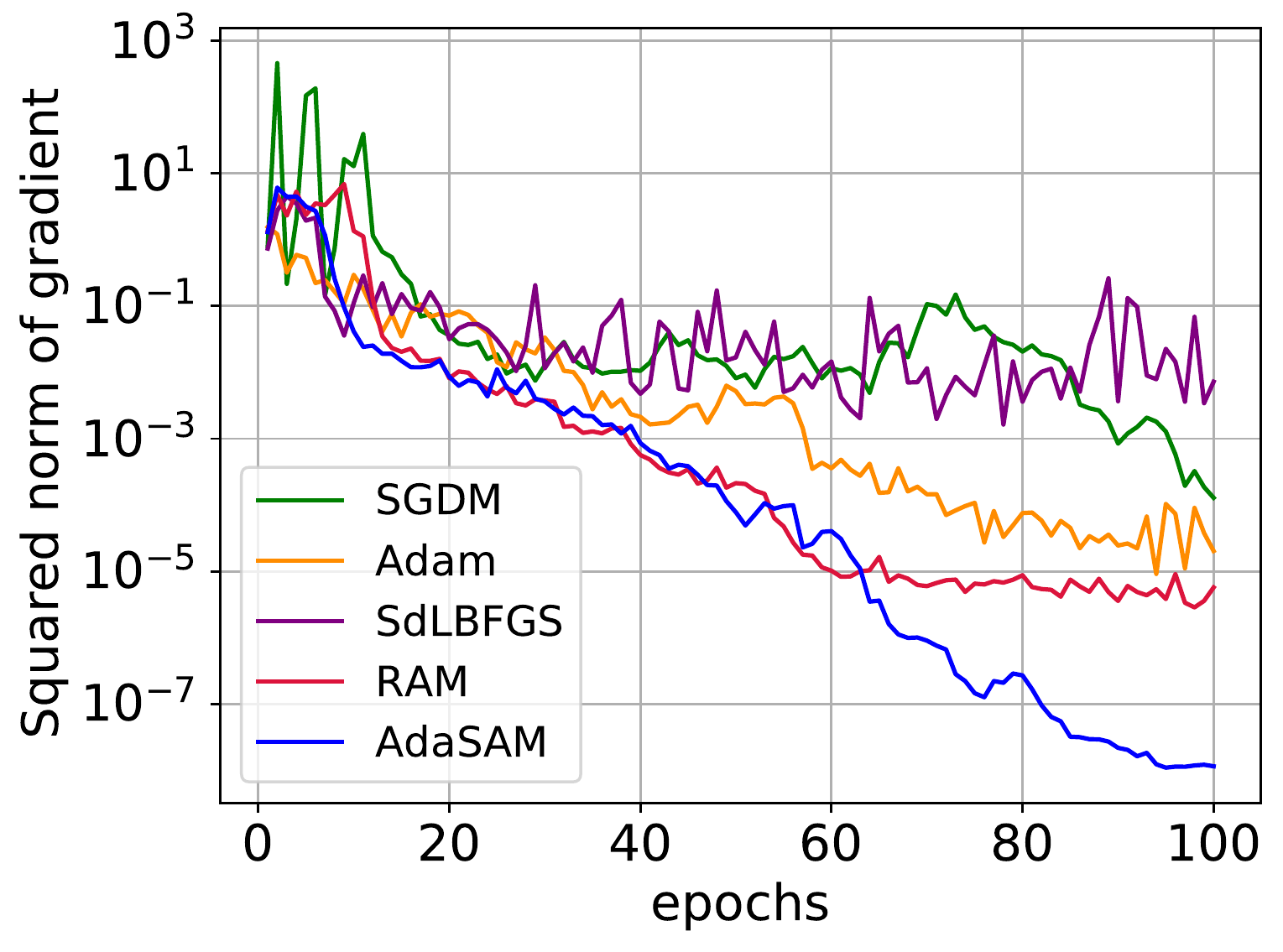}
}
\subfigure[Variance reduction]{
\includegraphics[width=0.23\textwidth]{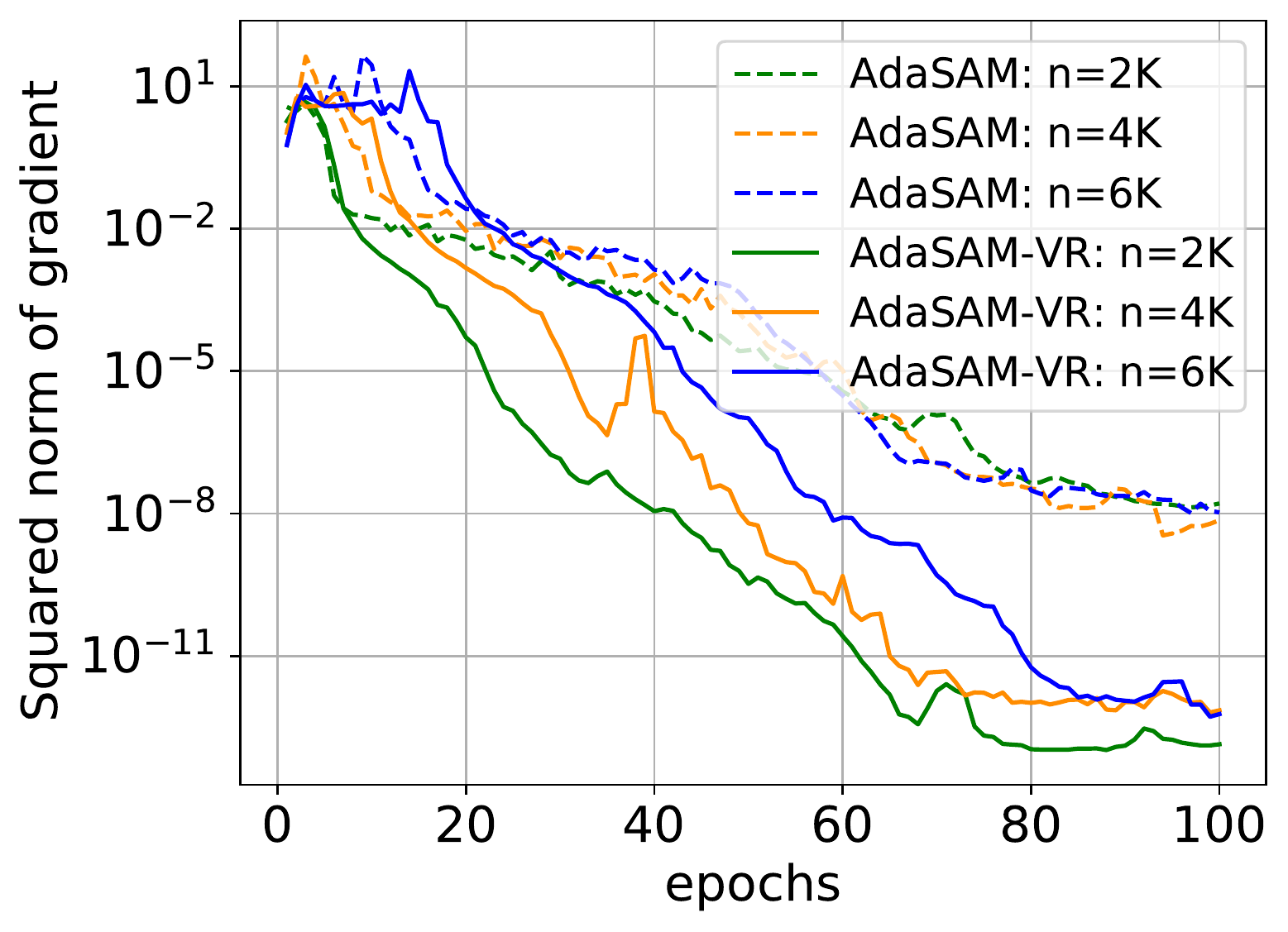}
}
\subfigure[Preconditioning]{
\includegraphics[width=0.23\textwidth]{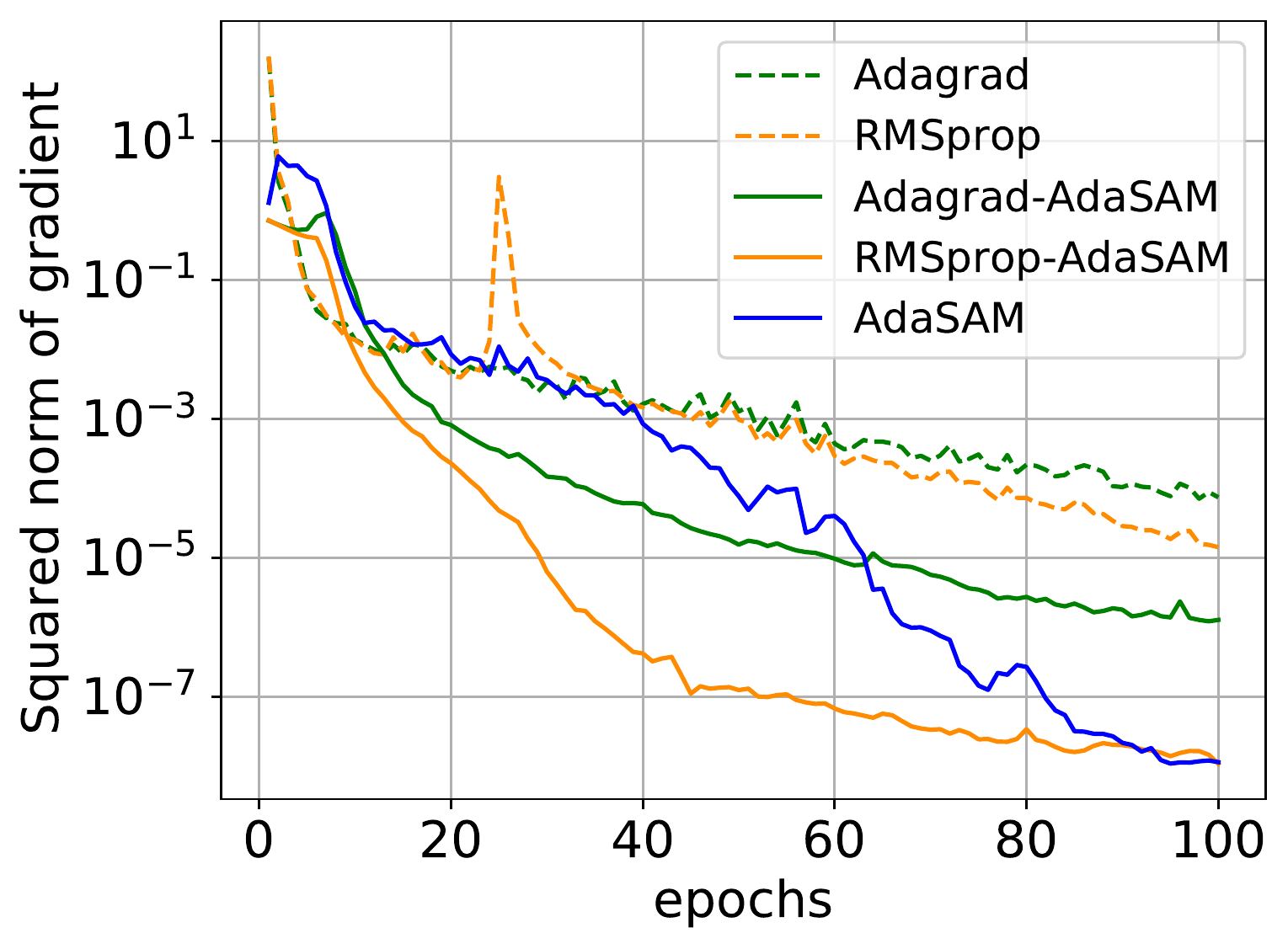}
}
\caption{Experiments on MNIST. Train Loss: (a)  Batchsize = 6K; (b) Batchsize = 3K; (c)  AdaSAM with variance reduction; (d) Preconditioned AdaSAM with batchsize = 3K.
Square norm of the gradient: (a)  Batchsize = 6K; (b) Batchsize = 3K; (c)  AdaSAM with variance reduction; (d) Preconditioned AdaSAM with batchsize = 3K.   }
\label{fig:appendix_mnist_2}
\end{figure}
 
\subsection{Experiments on CIFAR}
 For this group of experiments, we followed the same setting of training ResNet in \citep{he2016deep}. The batchsize is set to be 128 as commonly suggested. 
 (1) When training  for 160 epochs,  the learning rate was decayed at the 80th and 120th epoch; (2) When training  for 120 epochs, the learning rate was decayed at the 60th and 90th epoch; (3) When training  for 80 epochs, the learning rate was decayed at the 40th and 60th epoch. For AdaSAM/RAM, the learning rate decay means decaying $ \alpha_k, \beta_k $ simultaneously.  
 The results in Table~\ref{table:cifar} come from repeated tests with 3 random seeds.
 
 The baseline optimizers were SGDM, Adam, AdaBelief \citep{zhuang2020adabelief}, Lookahead \citep{zhang2019lookahead}, RNA \citep{scieur2018nonlinear}. AdaBelief is a recently proposed adaptive learning rate method  to improve Adam. Lookahead is a $k$-step method, which can be seen as a simple sequence interpolation method. In each cycle, Lookahead  iterates with an inner-optimizer $ optim $ for $k$ steps and then interpolates the first and the last iterates to give the starting point of the next cycle. RNA is also an extrapolation method but based on the minimal polynomial extrapolation approach \citep{brezinski2018shanks}. 
 
 We also explore the scheme of alternating iteration here, i.e. p>1 in Algorithm~\ref{alg:adasam}. SGD alternated with AdaSAM and Adam alternated with AdaSAM are denoted as AdaSAM-SGD and AdaSAM-Adam respectively.

 We tuned the hyperparameters through experiments on CIFAR-10/ResNet20. For AdaSAM/RAM, we only tuned the regularization parameter as explained in Section~\ref{subsec:para}. For each optimizer, the hyperparameter setting that has the best final test accuracy on CIFAR-10/ResNet20 was kept unchanged and used for the other tests. 
 We list the hyperparameters of all the tested optimizers here. (Learning rate is abbreviated as lr.) 
 \begin{itemize}
 \item \textbf{SGDM}: lr = 0.1, momentum = 0.9, weight-decay = $ 5\times 10^{-4} $, lr-decay = 0.1.
 \item \textbf{Adam}: lr = 0.001, $ (\beta_1,\beta_2) = (0.9,0.999) $, weight-decay = $ 5\times 10^{-4} $, lr-decay = 0.1.
 \item \textbf{AdaBelief}: lr = 0.001, $ (\beta_1,\beta_2) = (0.9,0.999) $, eps = $ 1\times 10^{-8} $, weight-decay = $ 5\times 10^{-4} $, lr-decay = 0.1. 
 \item \textbf{Lookahead}: $ optim$: SGDM (lr = 0.1, momentum = 0.9, weight-decay = $ 1\times 10^{-3}$), $ \alpha = 0.8 $, steps = 10, lr-decay = 0.1.
 \item \textbf{RNA}: lr = 0.1, momentum = 0.9, $\lambda = 0.1$, hist-length = 10, weight-decay = $5\times 10^{-4}$, lr-decay = 0.1.
 \item \textbf{RAM}: $optim$: SGDM (lr = 0.1, momentum = 0, weight-decay = $1.5\times 10^{-3}$), $ \alpha_k=1.0, \beta_k=1.0 $, $ \delta=0.1 $, $ p=1, m=10 $, weight-decay = $1.5\times 10^{-3}$, lr-decay = 0.06.
 \item \textbf{AdaSAM}: $ optim $: SGDM (lr = 0.1, momentum = 0, weight-decay = $1.5\times 10^{-3}$), $ \alpha_k=1.0, \beta_k=1.0 $, $ c_1=0.01 $, $ p=1, m=10 $, weight-decay = $1.5\times 10^{-3}$, lr-decay = 0.06.
 \item \textbf{AdaSAM-SGD}: $ optim $: SGDM (lr = 0.1, momentum = 0, weight-decay = $1.5\times 10^{-3}$), $ \alpha_k=1.0, \beta_k=1.0 $, $ c_1=0.01 $, $ p=5, m=10 $, weight-decay = $1.5\times 10^{-3}$, lr-decay = 0.06.
 \item \textbf{AdaSAM-Adam}: $ optim $: Adam (lr = 0.001, weight-decay = $1\times 10^{-3}$), $ \alpha_k=1.0, \beta_k=1.0 $, $ c_1=0.01 $, $ p=5, m=10 $, weight-decay = $1\times 10^{-3}$, lr-decay = 0.06.
 \item \textbf{Lookahead-Adam}: $ optim $: Adam (lr = 0.001, weight-decay = $ 1\times 10^{-3}$), $ \alpha = 0.8 $, steps = 10, lr-decay = 0.1.
 \end{itemize}
 
 \begin{figure}[ht]
\centering 
\subfigure[CIFAR-10/ResNet18]{
\includegraphics[width=0.45\textwidth]{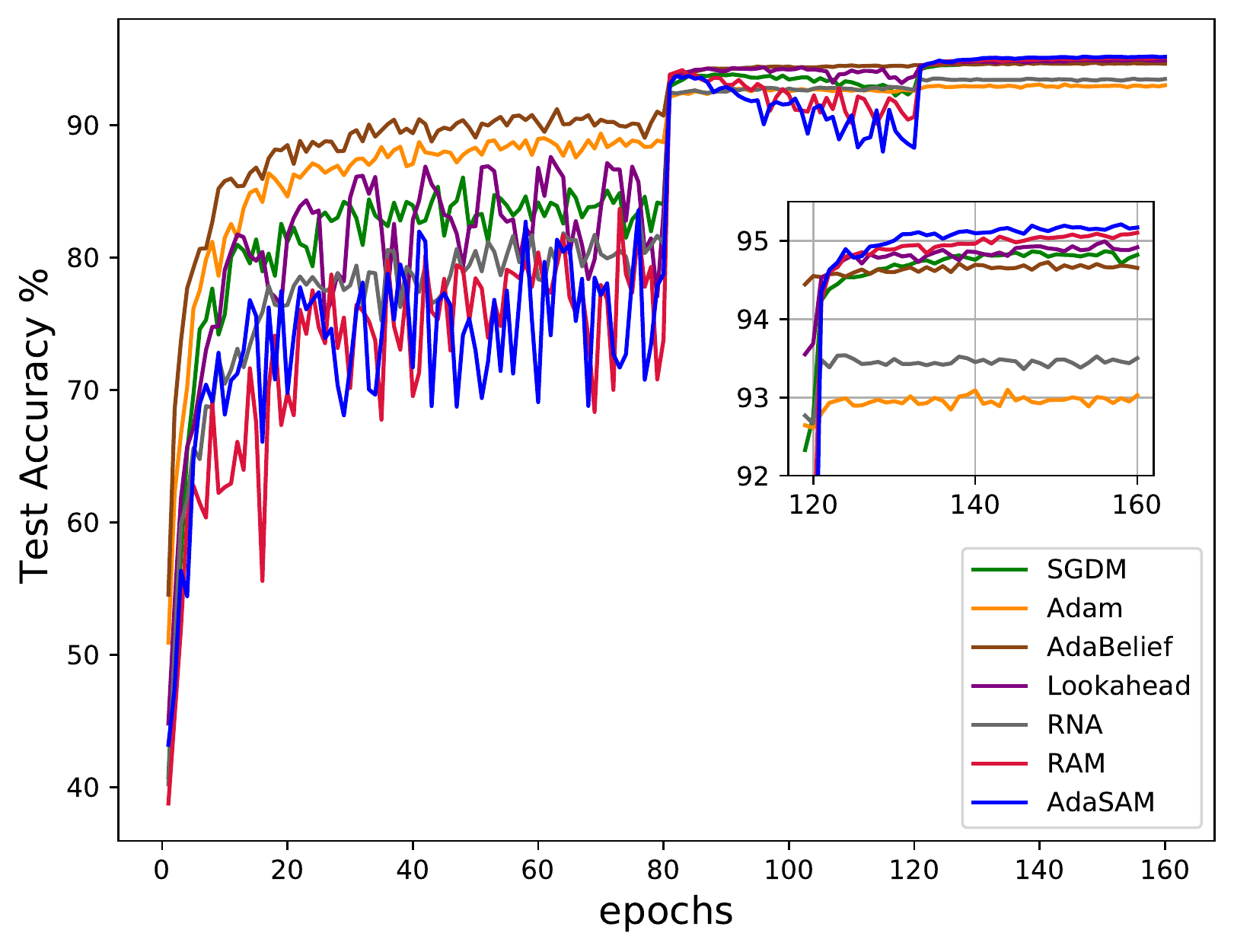}
}
\subfigure[CIFAR-10/WideResNet16-4]{
\includegraphics[width=0.45\textwidth]{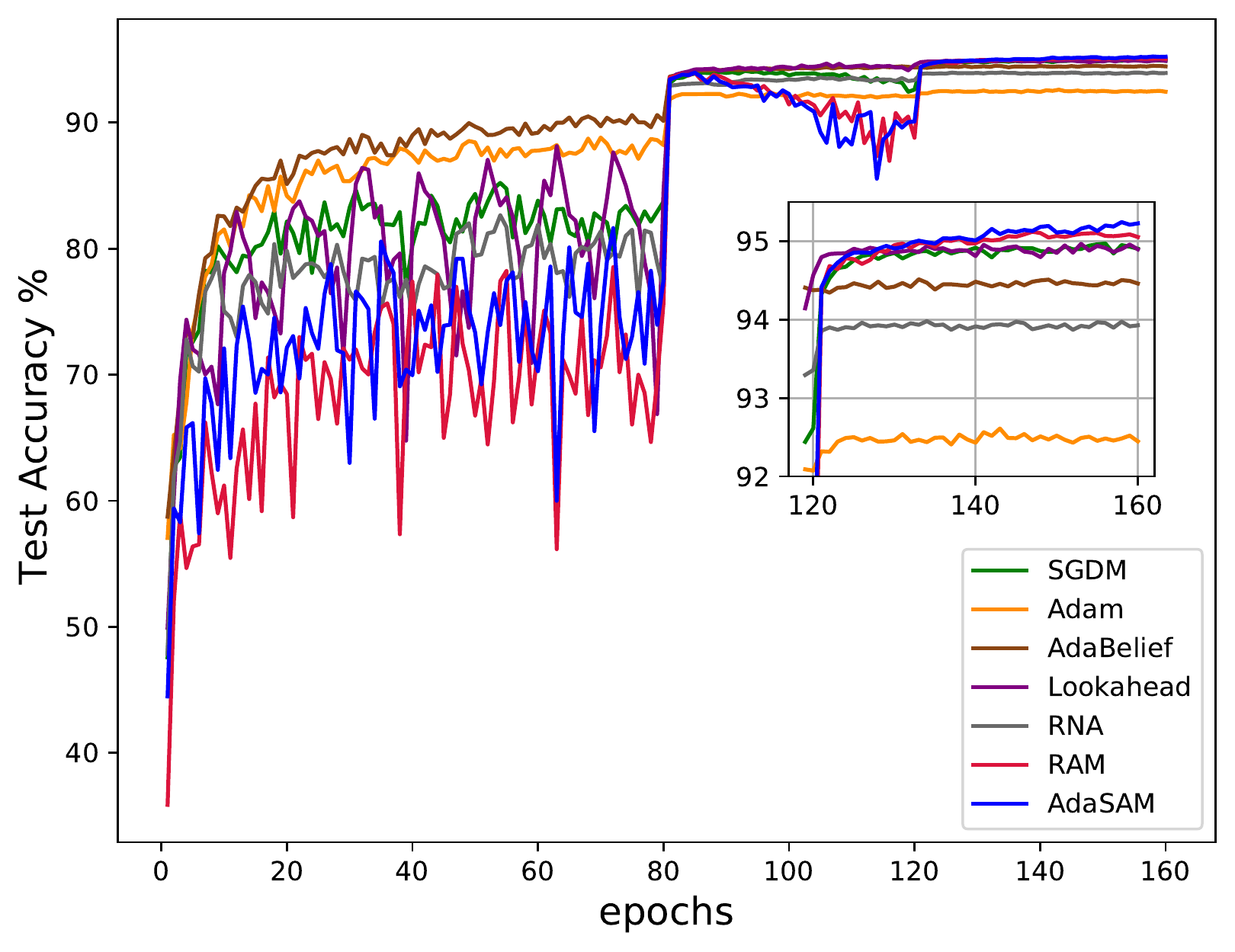}
}
\subfigure[CIFAR-100/ResNeXt50]{
\includegraphics[width=0.45\textwidth]{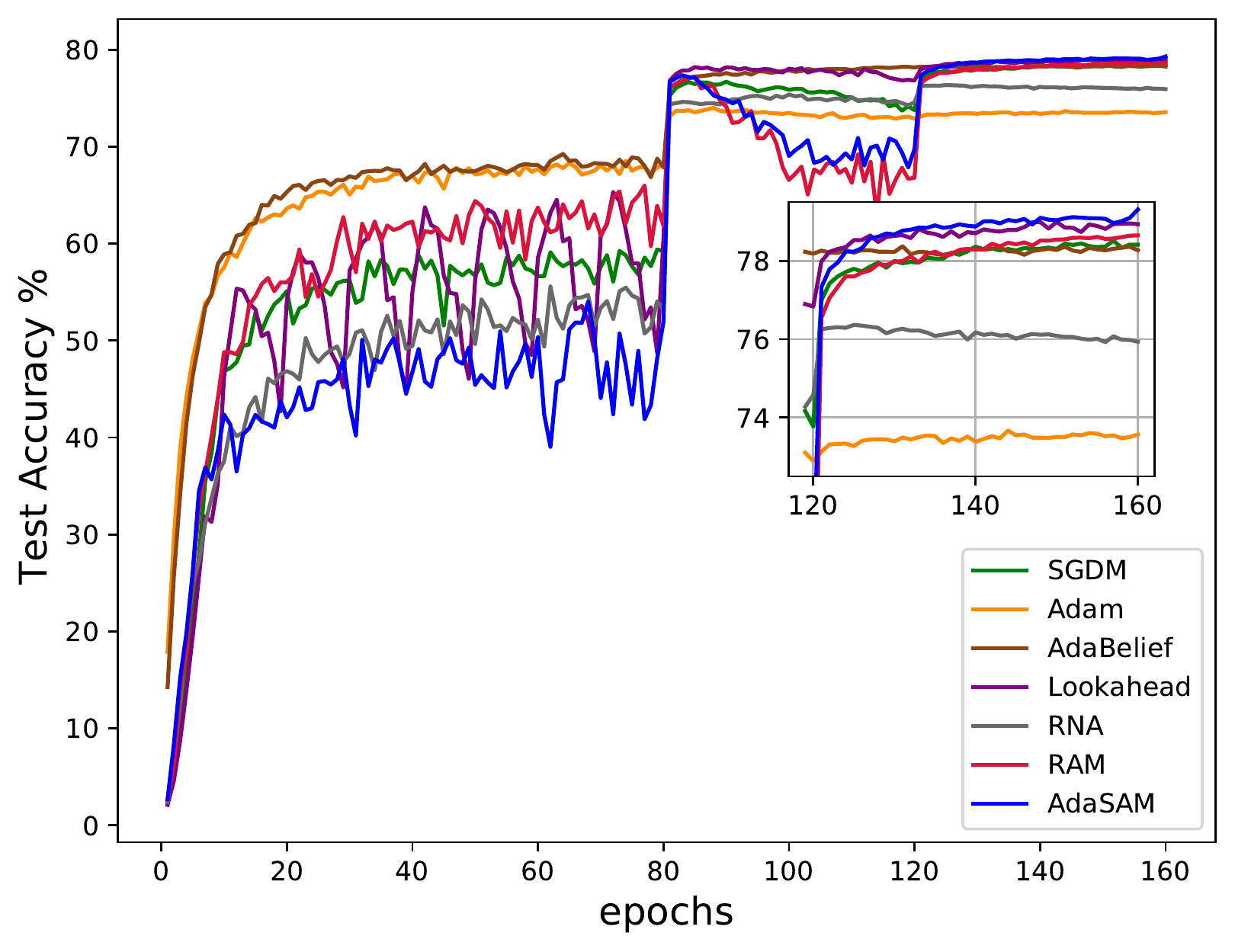}
}
\subfigure[CIFAR-100/DenseNet121]{
\includegraphics[width=0.45\textwidth]{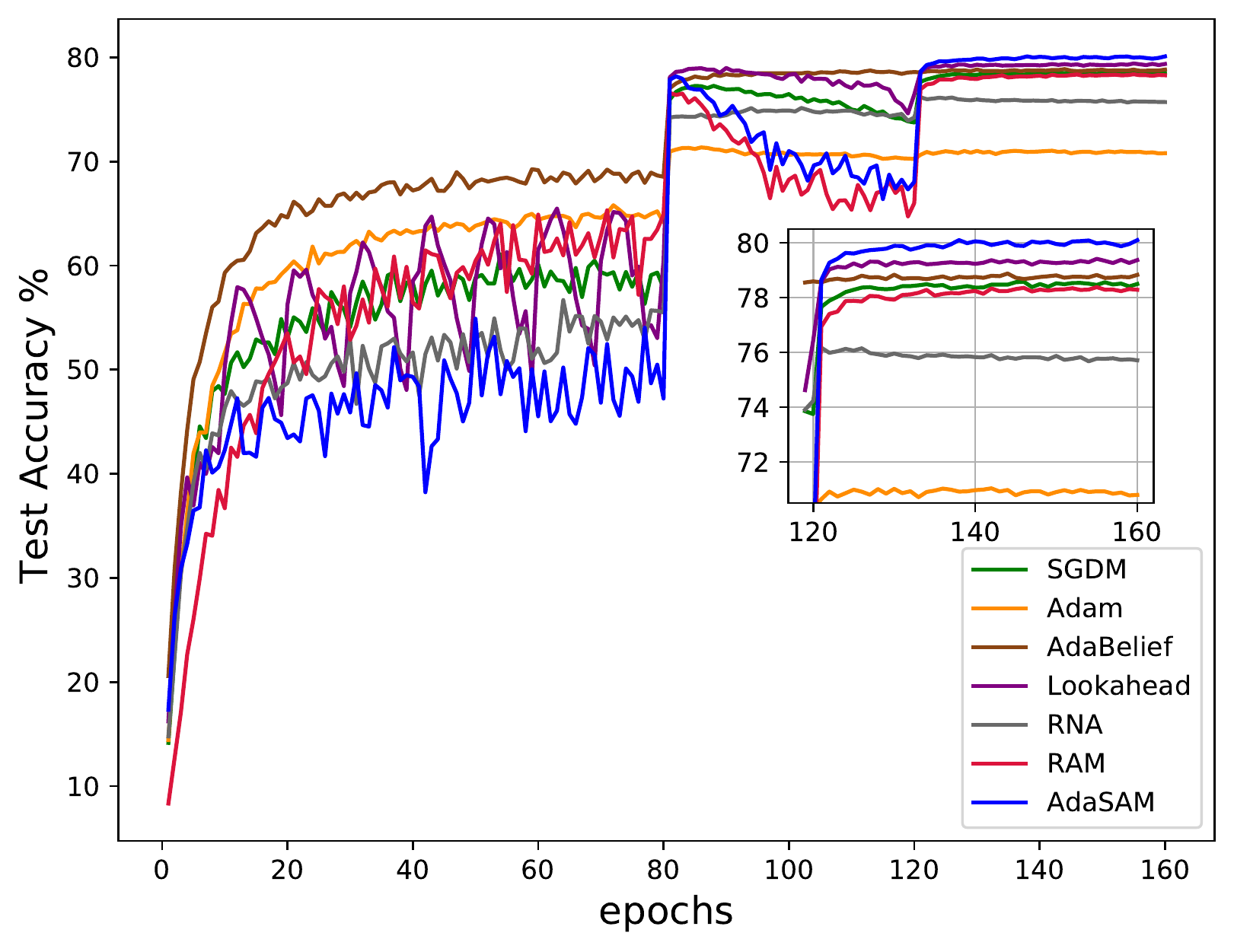}
}
\caption{Test accuracy of ResNet18/WideResNet16-4 on CIFAR-10 and ResNeXt50/DenseNet121 on CIFAR-100.}
\label{fig:appendix_CIFAR_1}
\end{figure} 
 
 Figure~\ref{fig:appendix_CIFAR_1} shows the test accuracy of different optimizers for training ResNet18/WideResNet16-4 on CIFAR-10 and ResNeXt50/DenseNet121 on CIFAR-100. The full results of final test accuracy are listed in Table~\ref{table:cifar} in the main paper. From Figure~\ref{fig:appendix_CIFAR_1}, we find that the convergence behaviour of AdaSAM  is rather erratic during the first 120 epochs. However, it always climbs up and stabilizes to the highest accuracy in the final 40 epochs. This phenomenon is due to the fact that AdaSAM uses a large weight-decay ($ 1.5\times 10^{-3} $ vs. $ 5 \times 10^{-4} $ of SGDM) and large mixing parameter ($  \beta_k=1$ ). We verify this claim by doing tests on CIFAR-10/ResNet20. In Figure~\ref{fig:appendix_CIFAR_wd}, we fixed other hyperparameters and tested the effect of different weight-decays of AdaSAM. It is clear that a smaller weight-decay can lead to faster convergence on the training dataset, but often cause poorer generalization on the test dataset. In Figure~\ref{fig:appendix_CIFAR_beta}, we only changed $ \beta_0 $ while fixing other hyperparameters (weight-decay=$ 1.5\times 10^{-3} $). We can see a smaller $ \beta_0 $ can lead to faster and more stable convergence at the beginning, but the final test accuracy is suboptimal. This phenomenon coincides with the results in \citep{li2019towards}.
 
 Since AdaSAM requires additional matrix computation in each iteration, it consumes more time if training for the same number of epochs as SGDM. Nonetheless, AdaSAM can achieve comparable test accuracy if decaying the learning rate earlier and stopping the training earlier. As indicated in Table~\ref{table:cifar}, SGDM and Lookahead can serve as the strong baselines, so we conducted tests of  comparisons between AdaSAM with SGDM/Lookahead to see the effectiveness of AdaSAM when training with less number of epochs. Results  in Figure~\ref{fig:appendix_CIFAR_stop} show that the final test accuracy of AdaSAM  for training 80 or 120 epochs can match or even surpass the test accuracy of SGDM  for training 160 epochs. Therefore, the generalization benefit from AdaSAM pays for its additional cost.  
 \begin{figure}[ht]
\centering 
\subfigure[Train Loss]{
\includegraphics[width=0.4\textwidth]{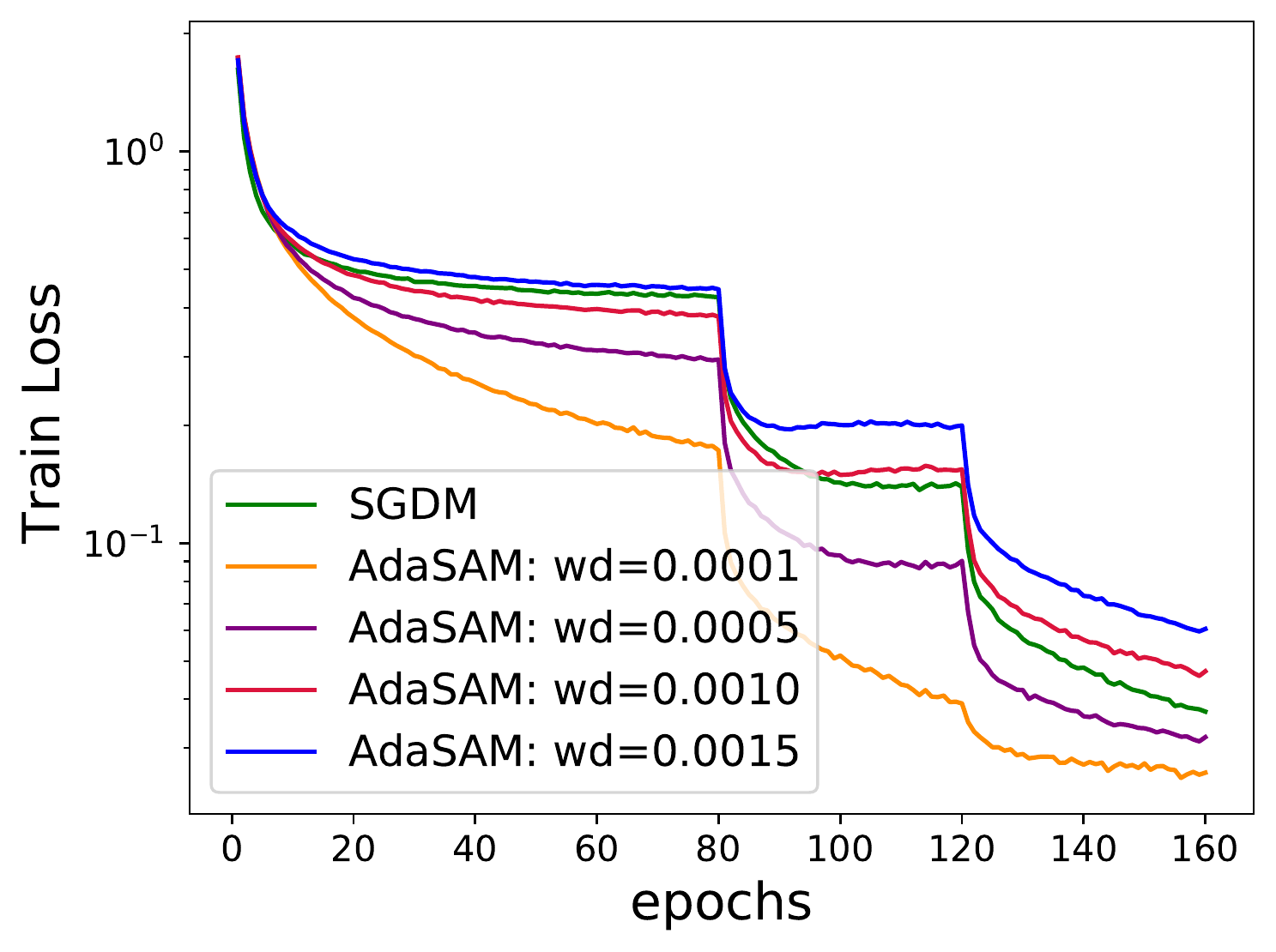}
}
\subfigure[Train Accuracy]{
\includegraphics[width=0.4\textwidth]{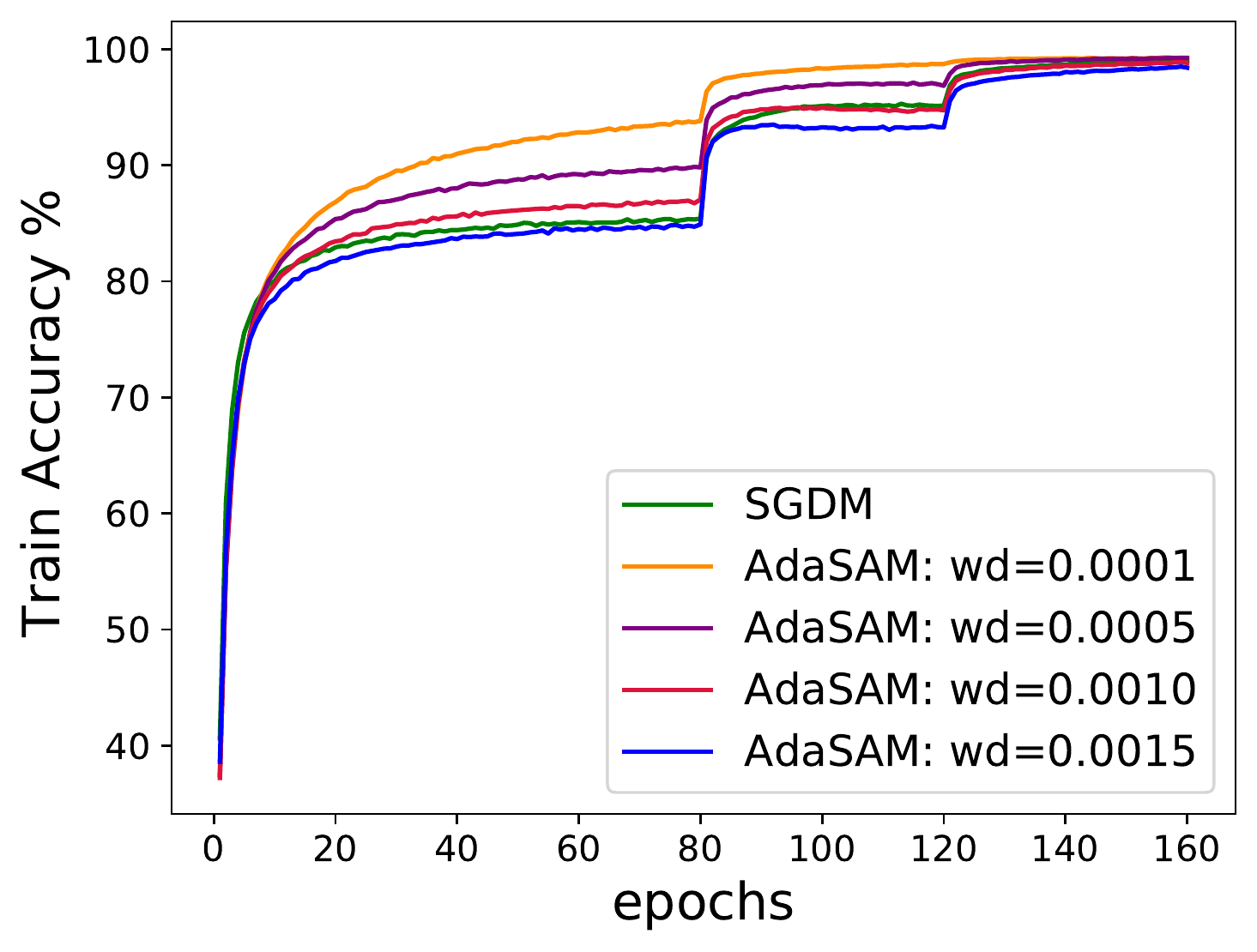}
}
\subfigure[Test Loss]{
\includegraphics[width=0.4\textwidth]{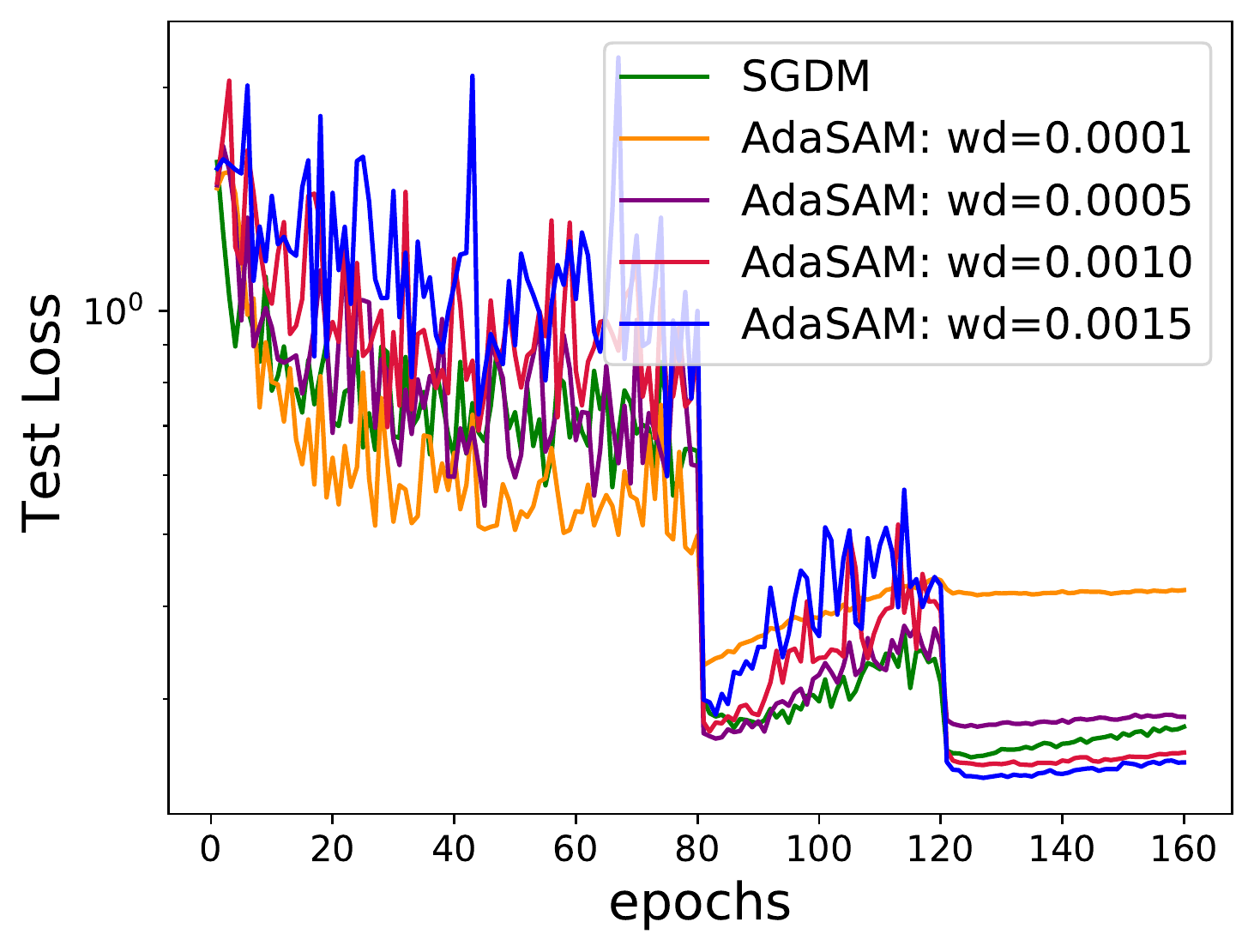}
}
\subfigure[Test Accuracy]{
\includegraphics[width=0.4\textwidth]{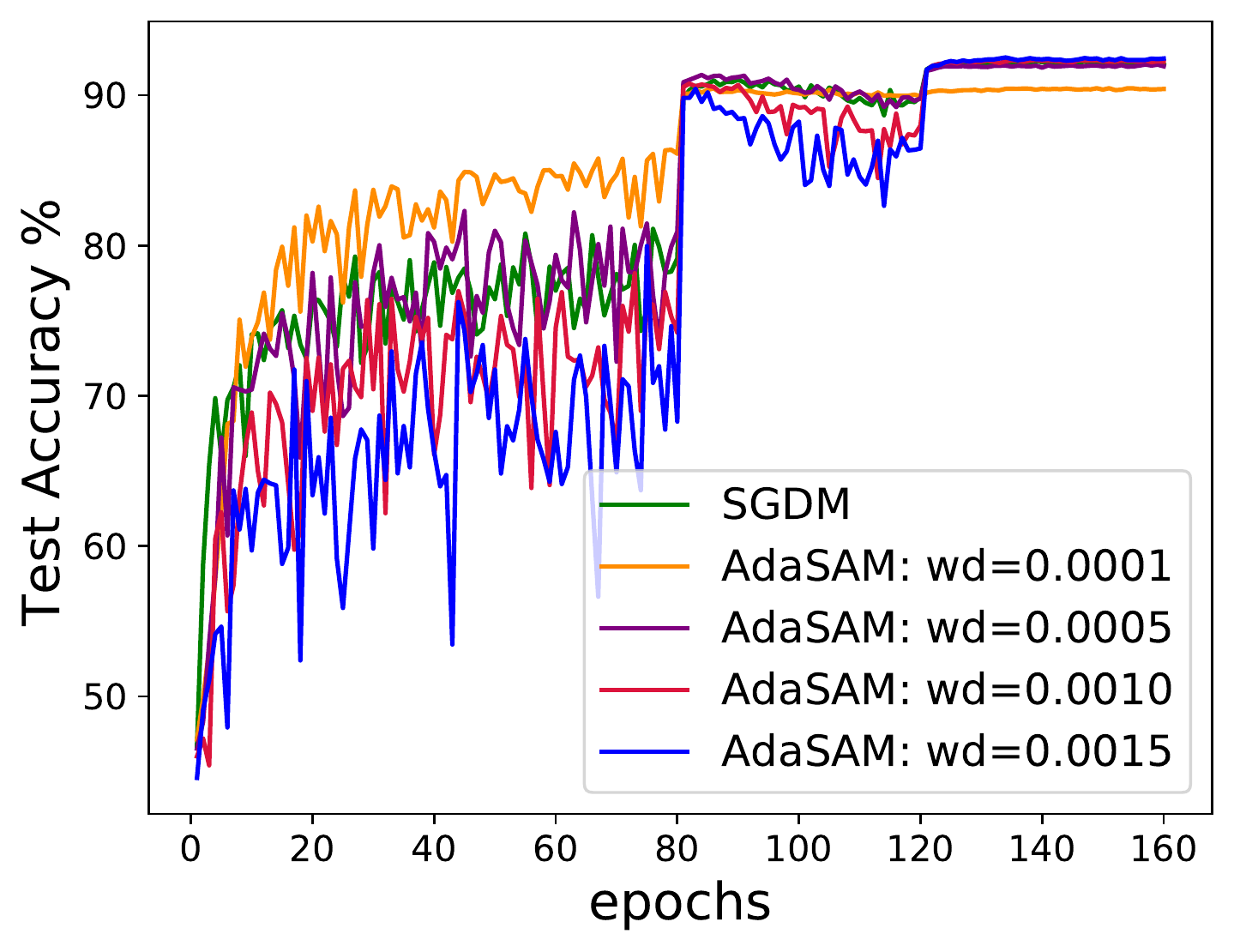}
}
\caption{Experiments on CIFAR-10/ResNet20 with different weight-decay (abbreviated as wd in the legends). }
\label{fig:appendix_CIFAR_wd}
\end{figure} 

\begin{figure}[H]
\centering 
\subfigure[Train Loss]{
\includegraphics[width=0.4\textwidth]{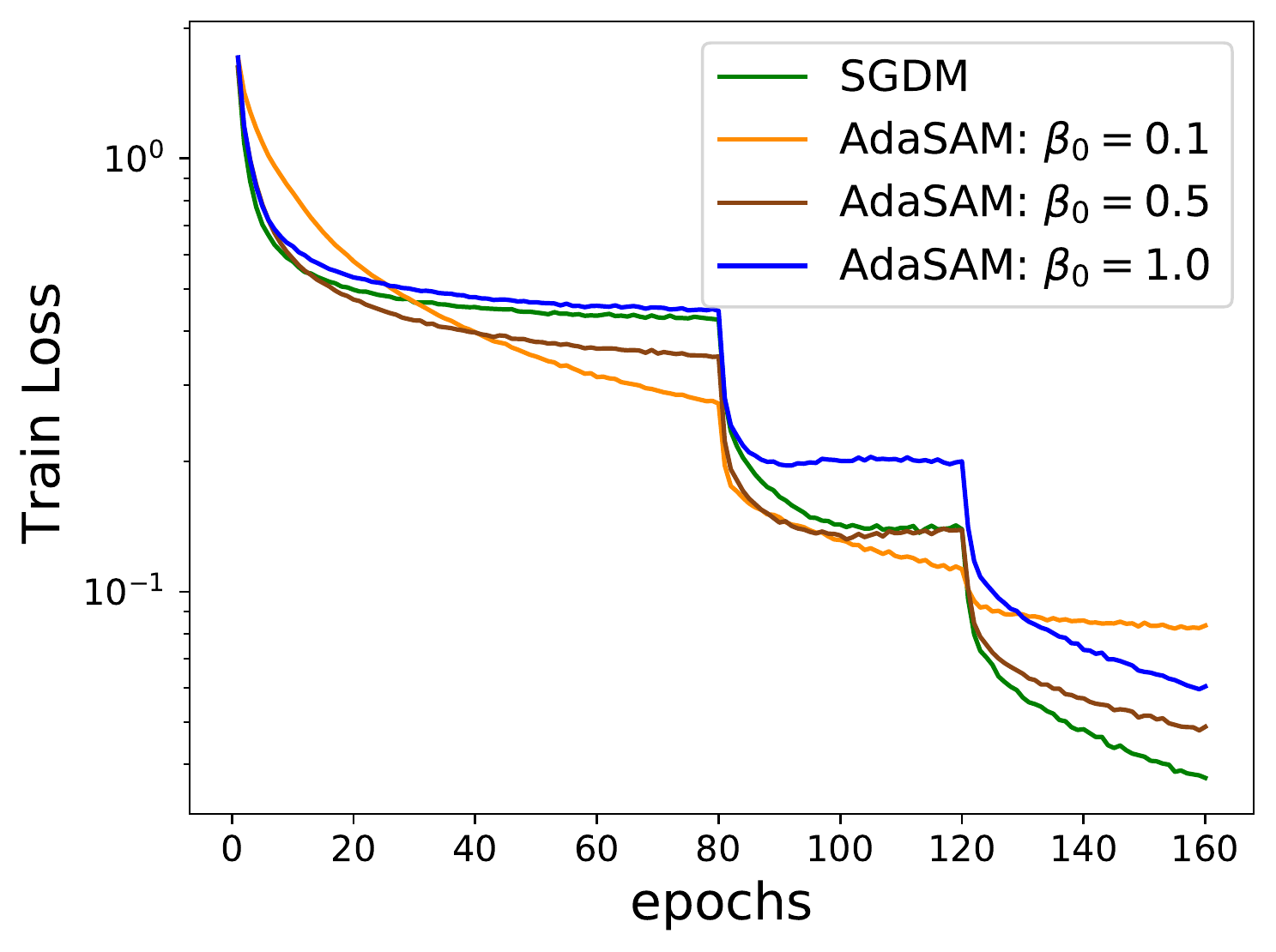}
}
\subfigure[Train Accuracy]{
\includegraphics[width=0.4\textwidth]{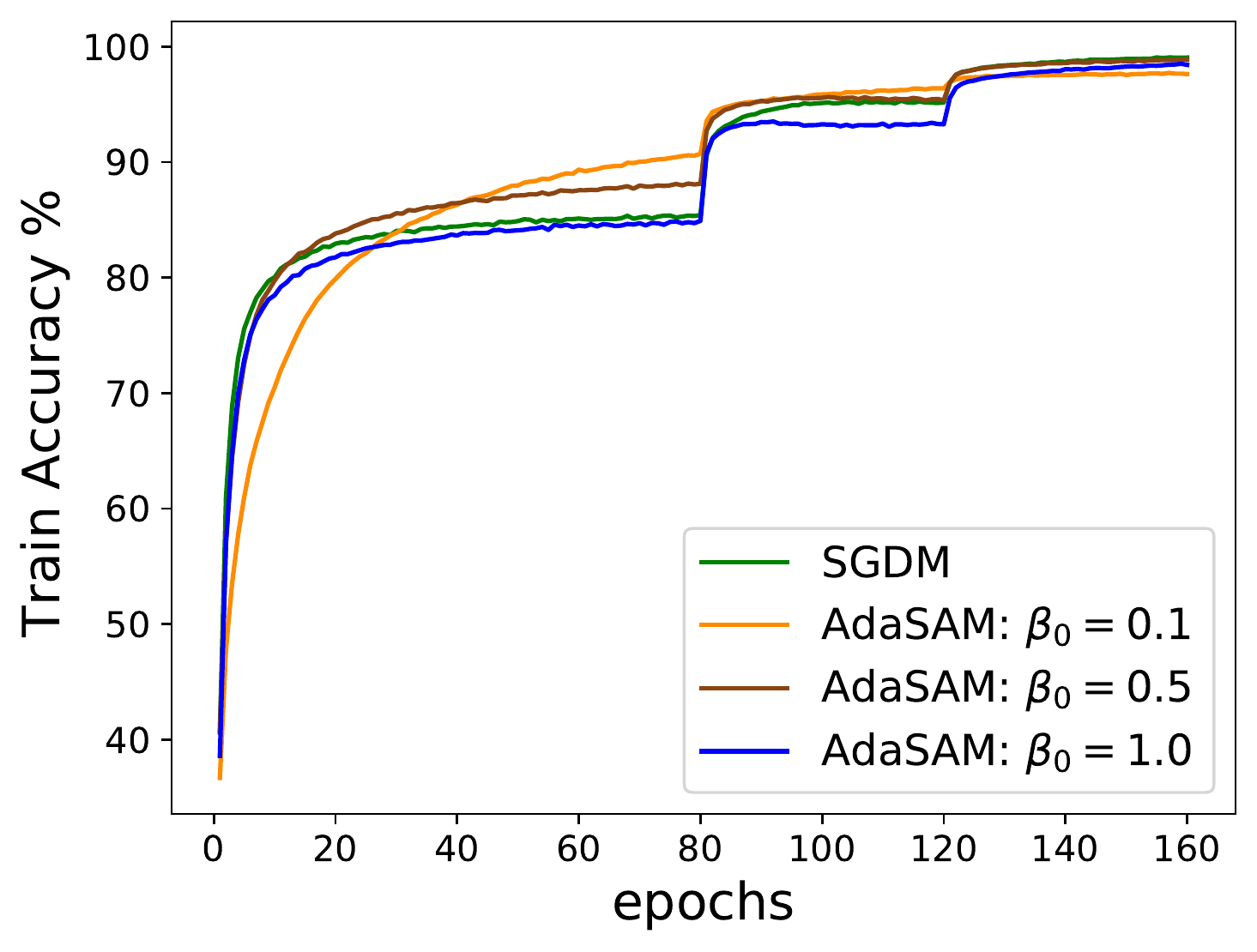}
}
\subfigure[Test Loss]{
\includegraphics[width=0.4\textwidth]{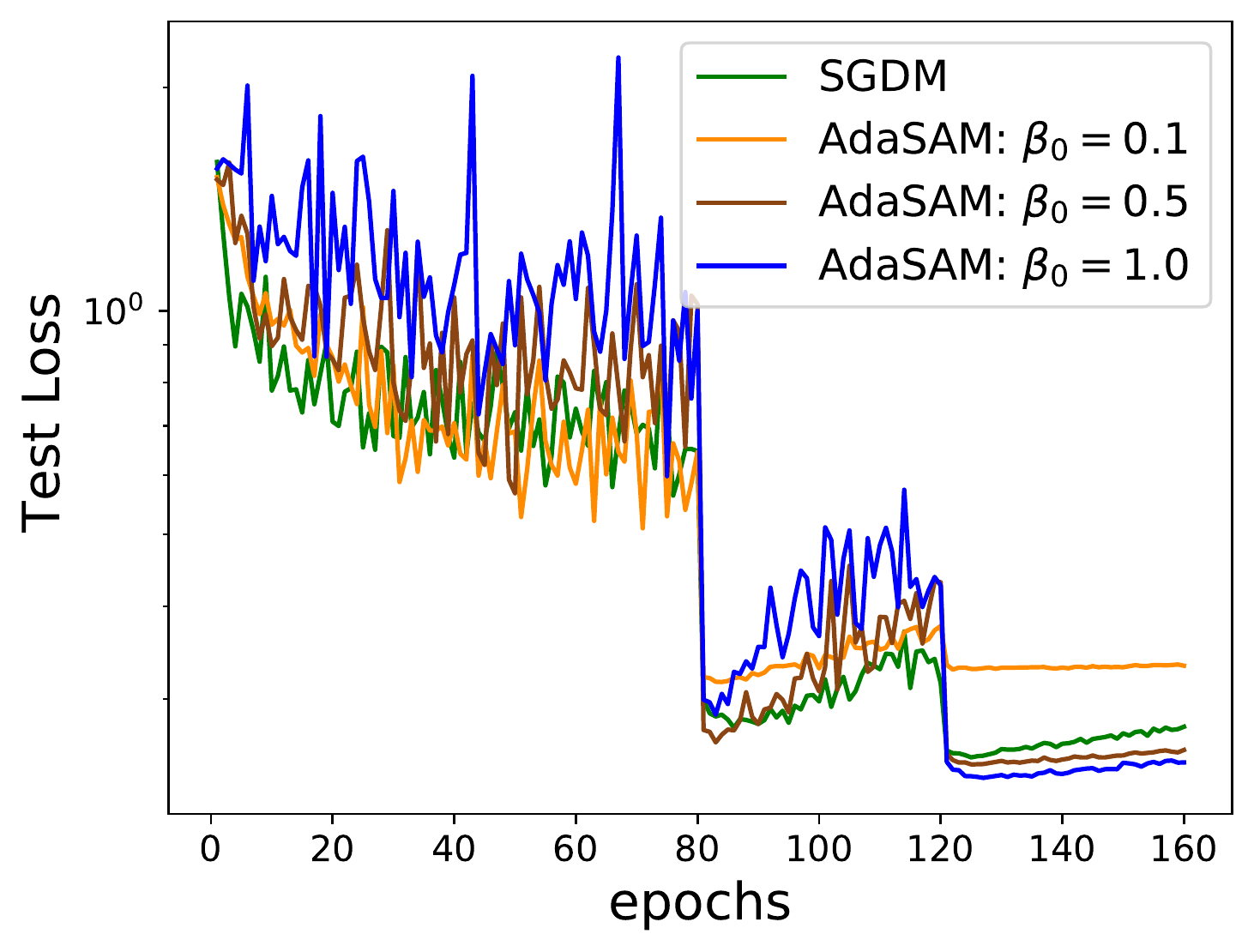}
}
\subfigure[Test Accuracy]{
\includegraphics[width=0.4\textwidth]{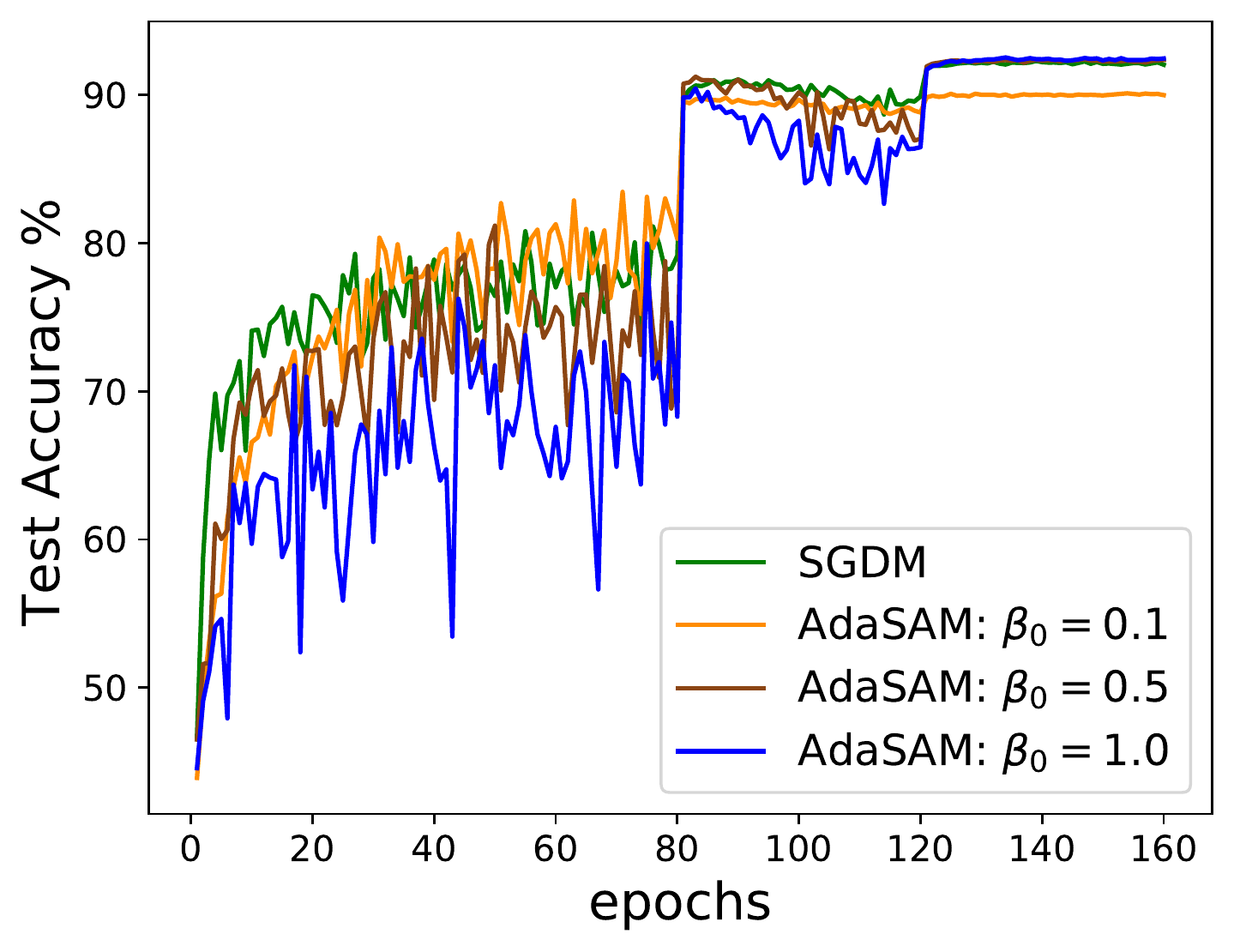}
}
\caption{Experiments on CIFAR-10/ResNet20 with different $ \beta_0 $.}
\label{fig:appendix_CIFAR_beta}
\end{figure} 
 
\begin{figure}[H]
\centering 
\subfigure[Test loss on CIFAR10/ResNet18]{
\includegraphics[width=0.42\textwidth]{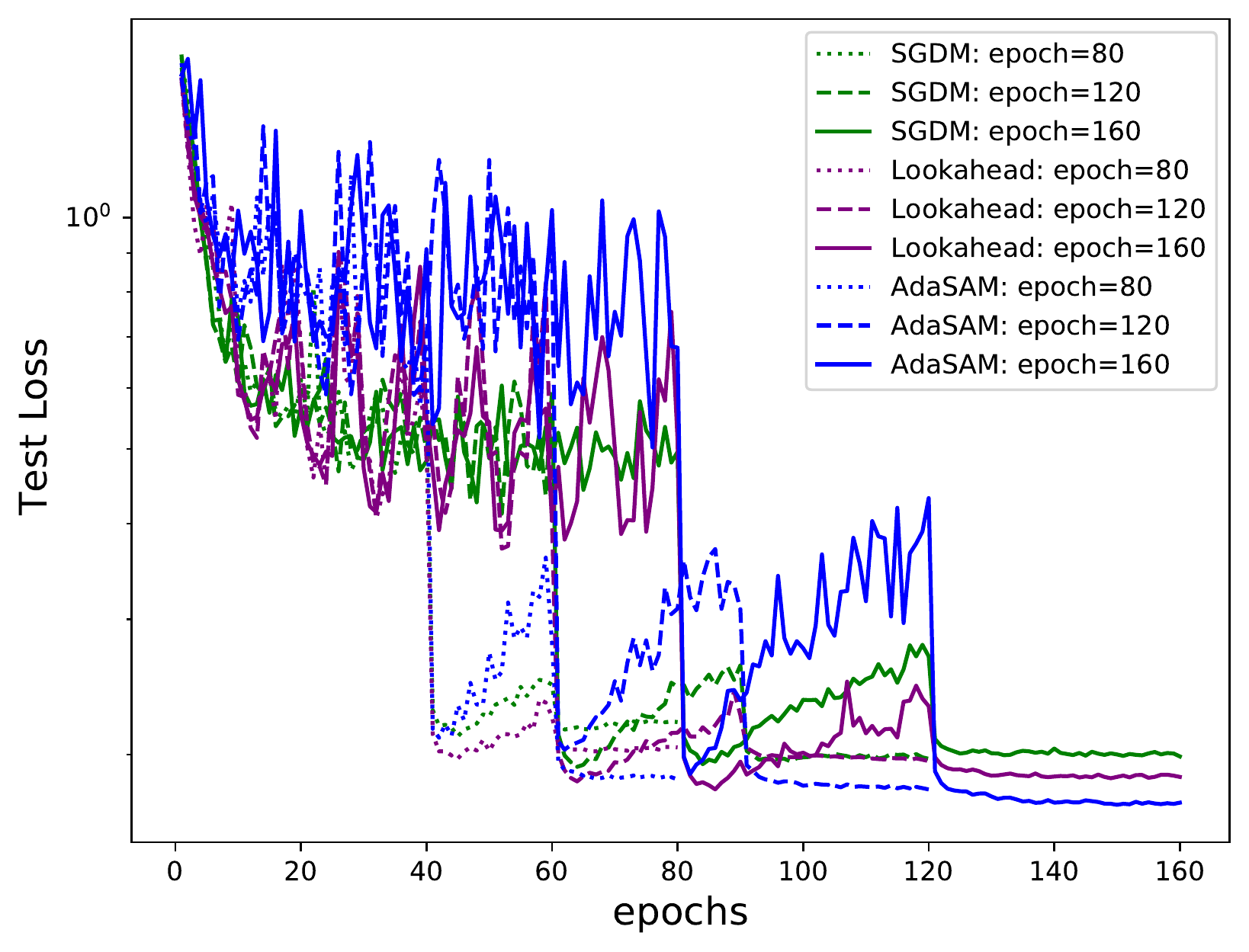}
}
\subfigure[Test accuracy on CIFAR10/ResNet18]{
\includegraphics[width=0.42\textwidth]{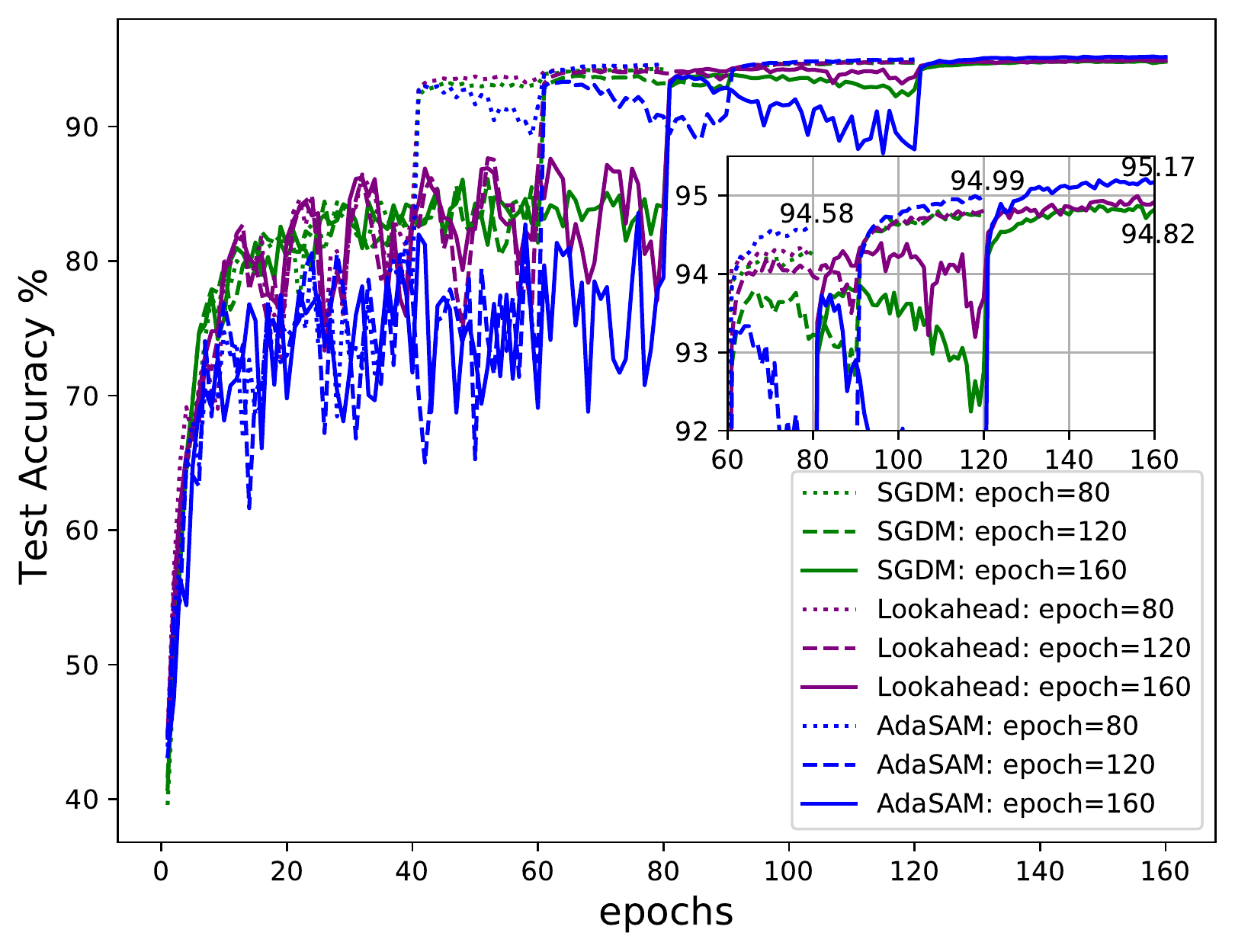}
}
\subfigure[Test loss on CIFAR10/WideResNet16-4]{
\includegraphics[width=0.42\textwidth]{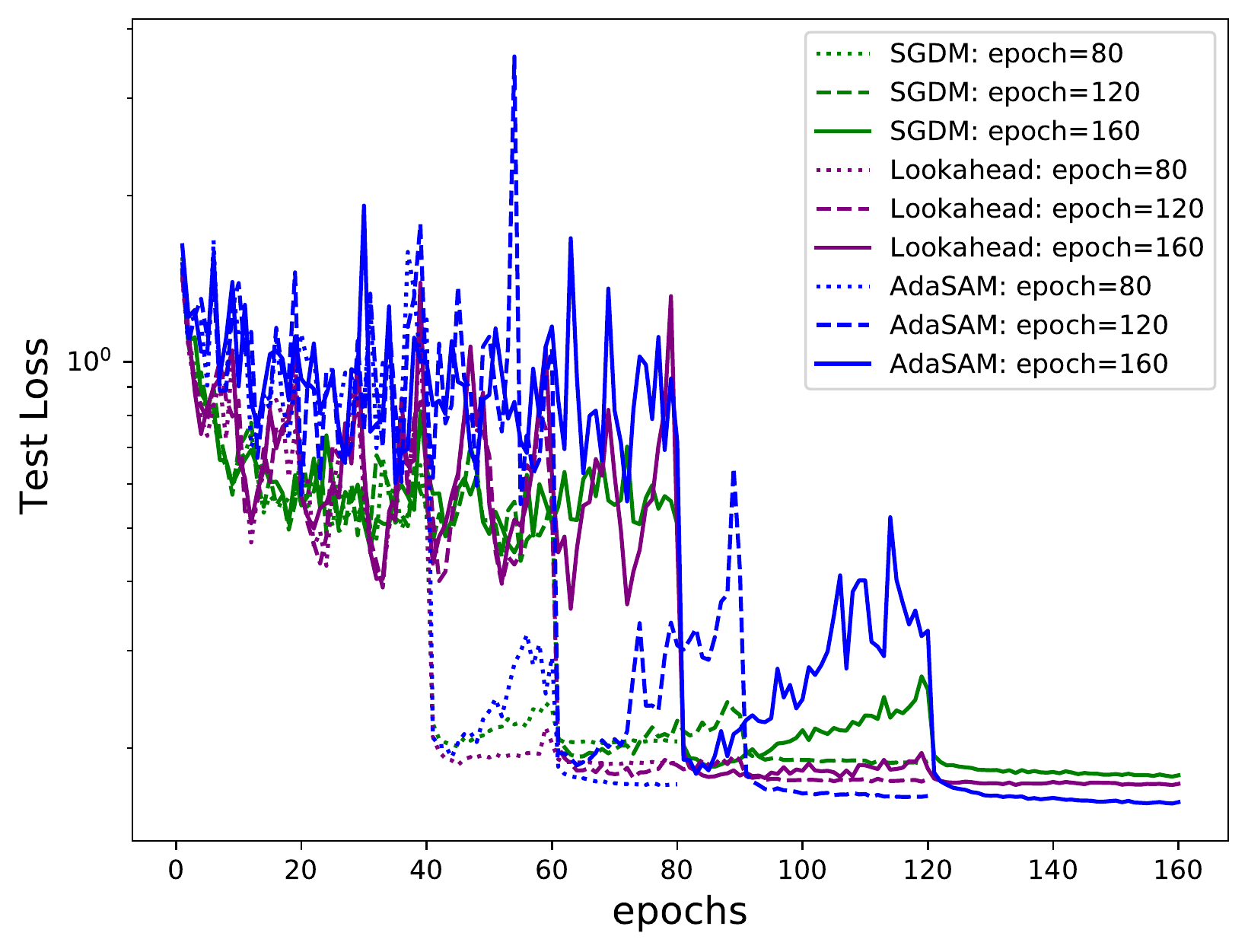}
}
\subfigure[Test accuracy on CIFAR10/WideResNet16-4]{
\includegraphics[width=0.42\textwidth]{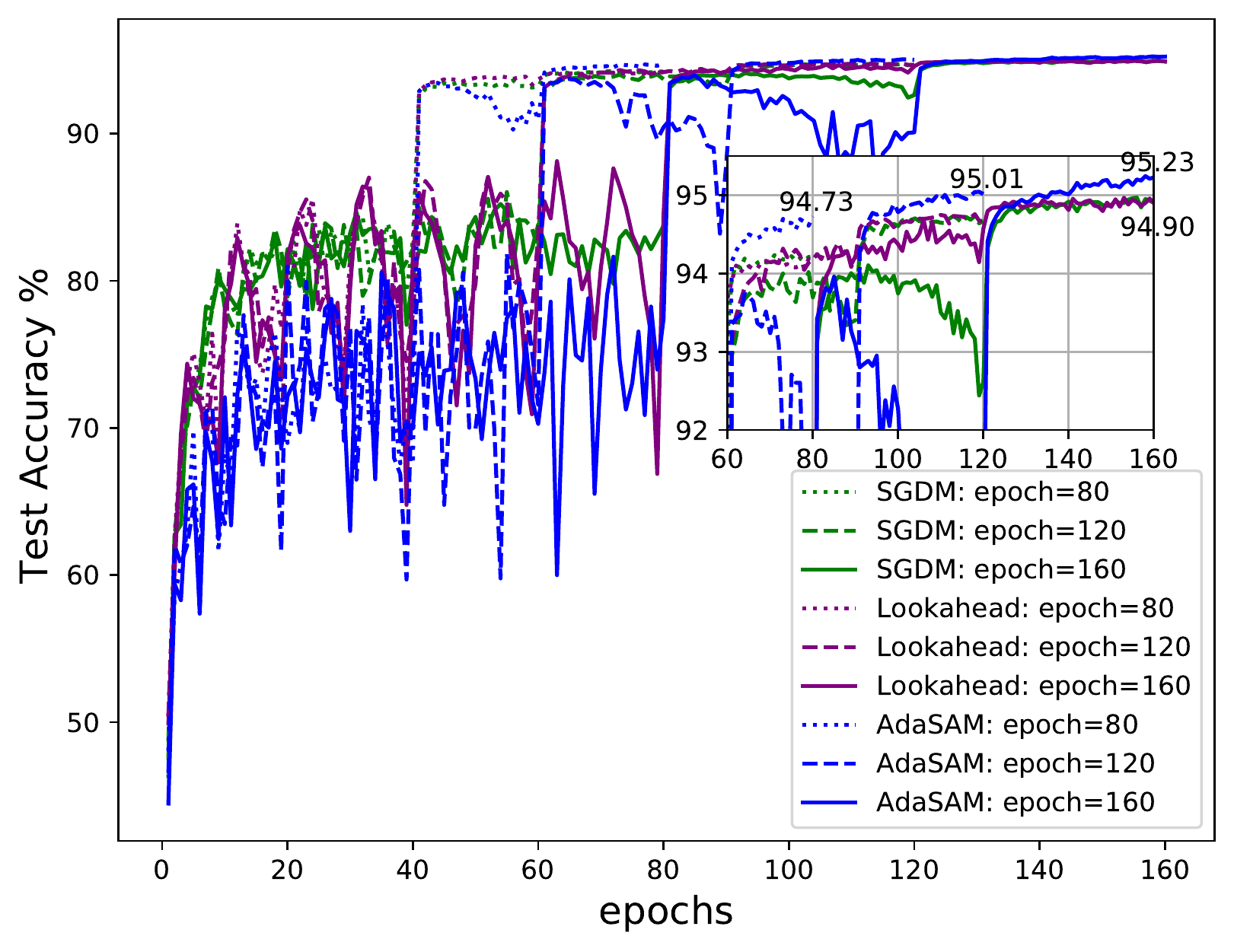}
}
\subfigure[Test loss on CIFAR100/ResNeXt50]{
\includegraphics[width=0.42\textwidth]{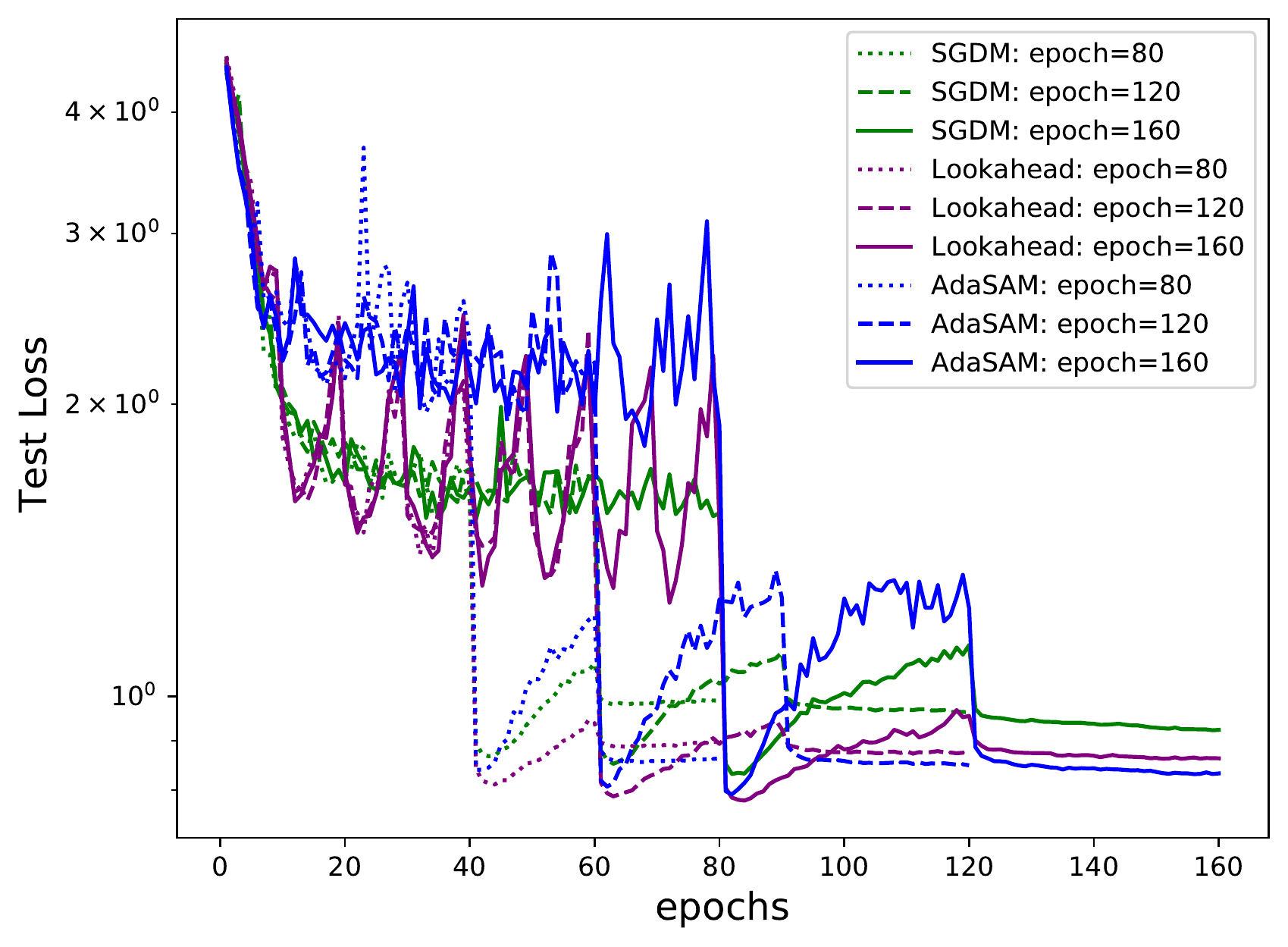}
}
\subfigure[Test accuracy on CIFAR100/ResNeXt50]{
\includegraphics[width=0.42\textwidth]{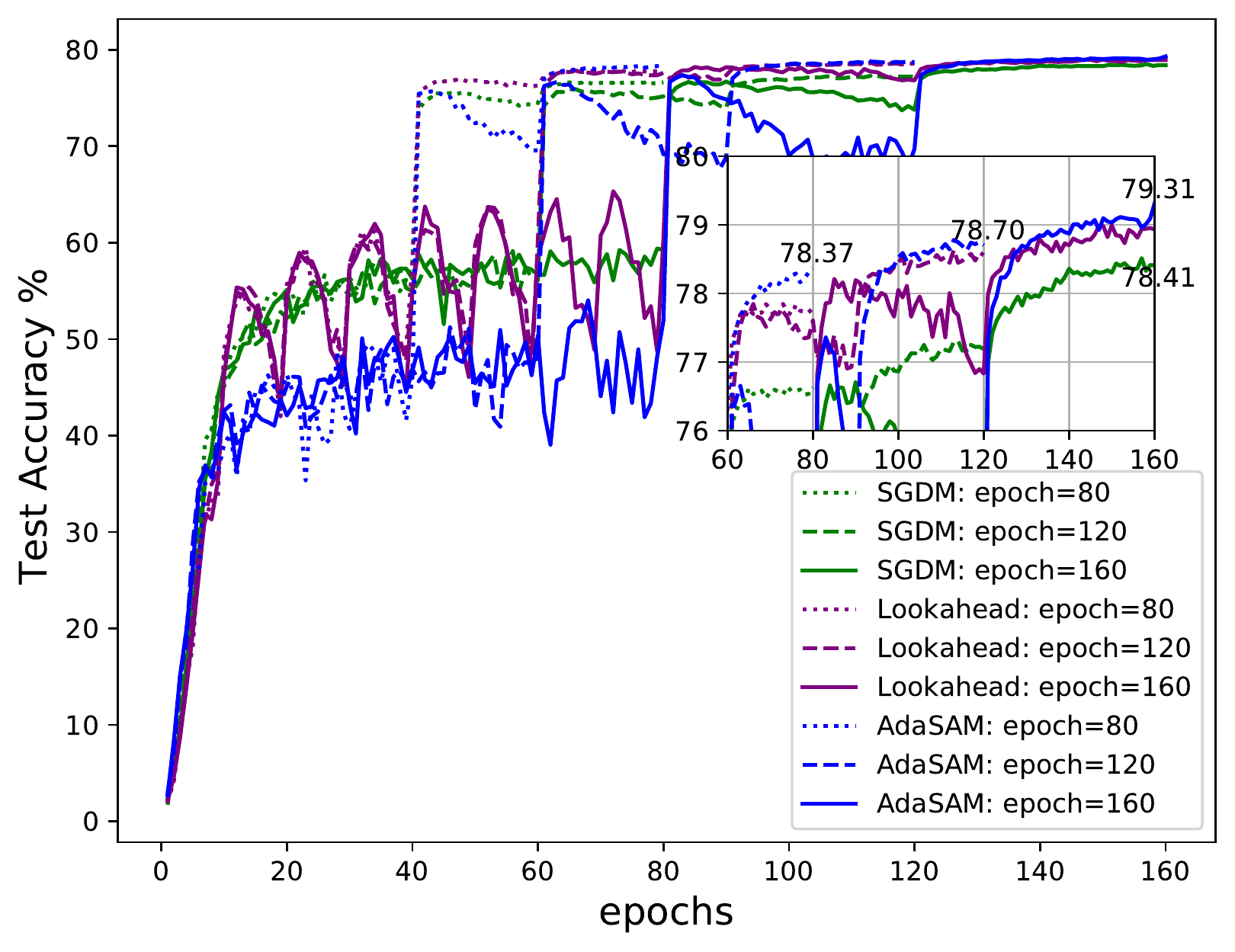}
}
\subfigure[Test loss on CIFAR100/DenseNet121]{
\includegraphics[width=0.42\textwidth]{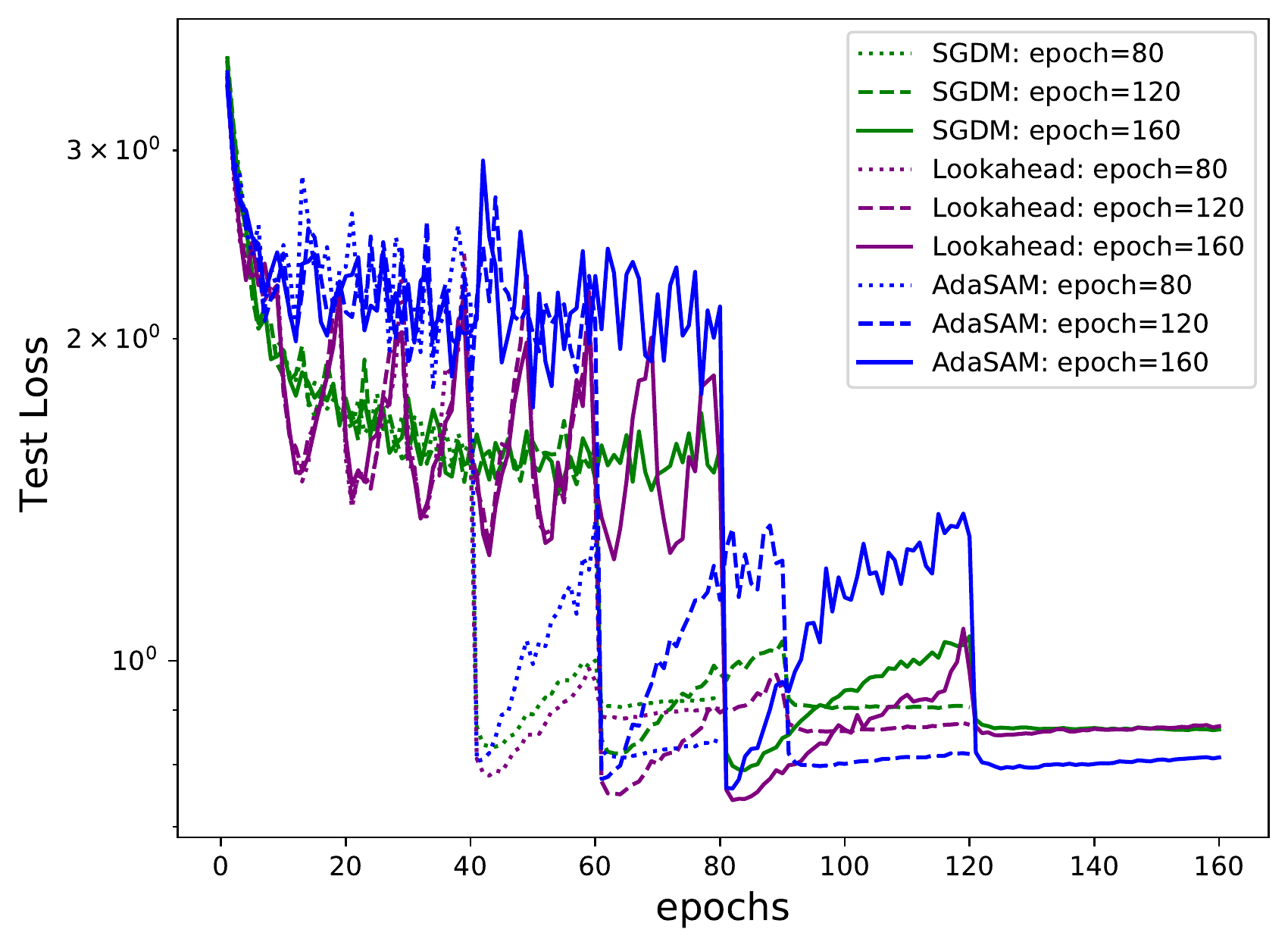}
}
\subfigure[Test accuracy on CIFAR100/DenseNet121]{
\includegraphics[width=0.42\textwidth]{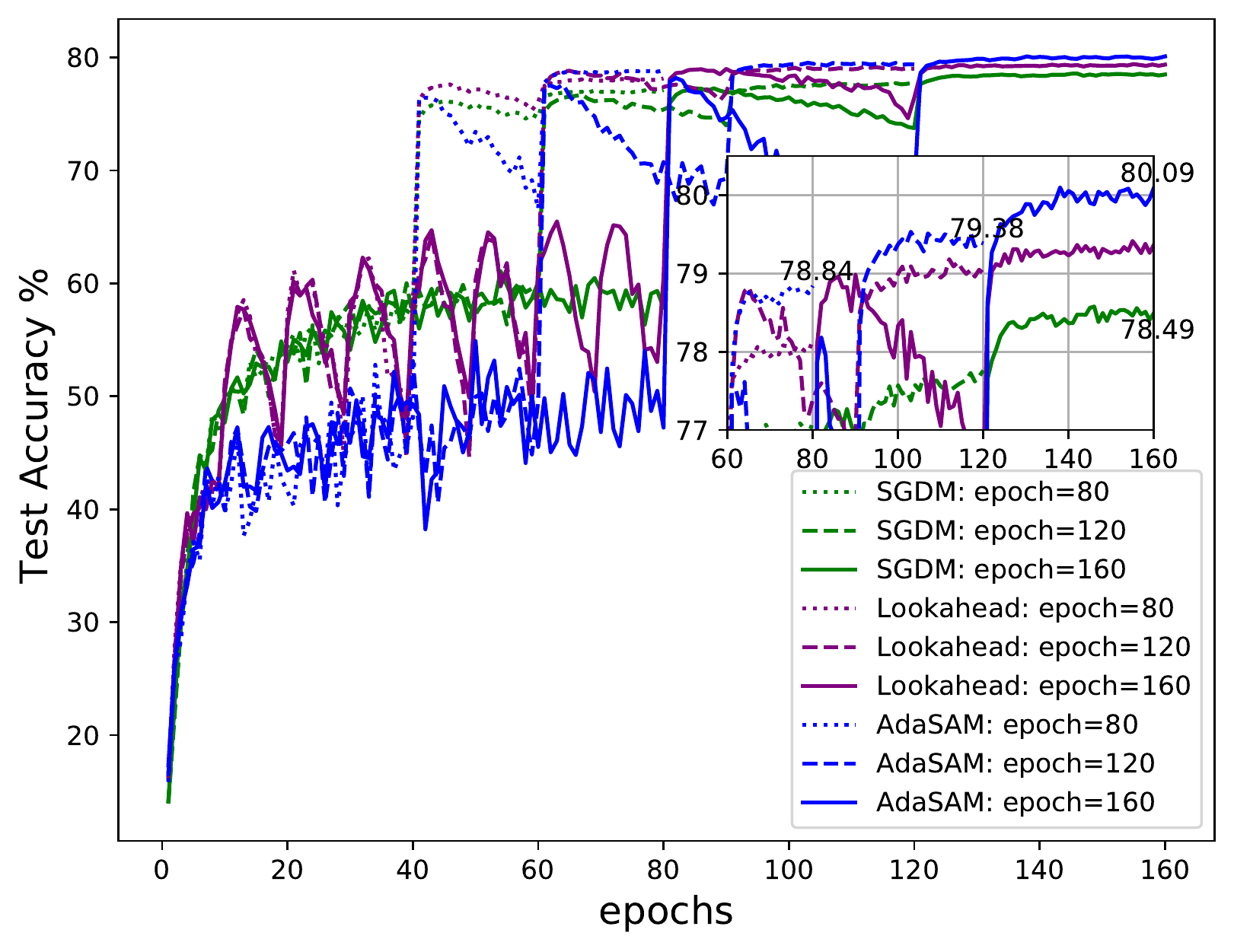}
}
\caption{Training deep neural networks for 80,120,160 epochs. We report the final test accuracy of AdaSAM for training 80,120,160 epochs at nearby point in the nested figure. The final test accuracy of SGDM for training 160 epochs is also reported for comparison.}
\label{fig:appendix_CIFAR_stop}
\end{figure} 
 
 We introduce alternating iteration technique to save the computational cost of AdaSAM, i.e. iterating with an inner optimizer for most of the time while applying AdaSAM periodically. We find AdaSAM can improve the generalization ability of the inner optimizer. Lookahead also has an inner optimizer. However, as shown in Figure~\ref{fig:appendix_CIFAR_alter}, Lookahead cannot improve the generalization ability of Adam. On the contrary, AdaSAM can enhance Adam to match the final test accuracy of SGDM.
 
 \begin{figure}[ht]
\centering 
\subfigure[Test Loss]{
\includegraphics[width=0.45\textwidth]{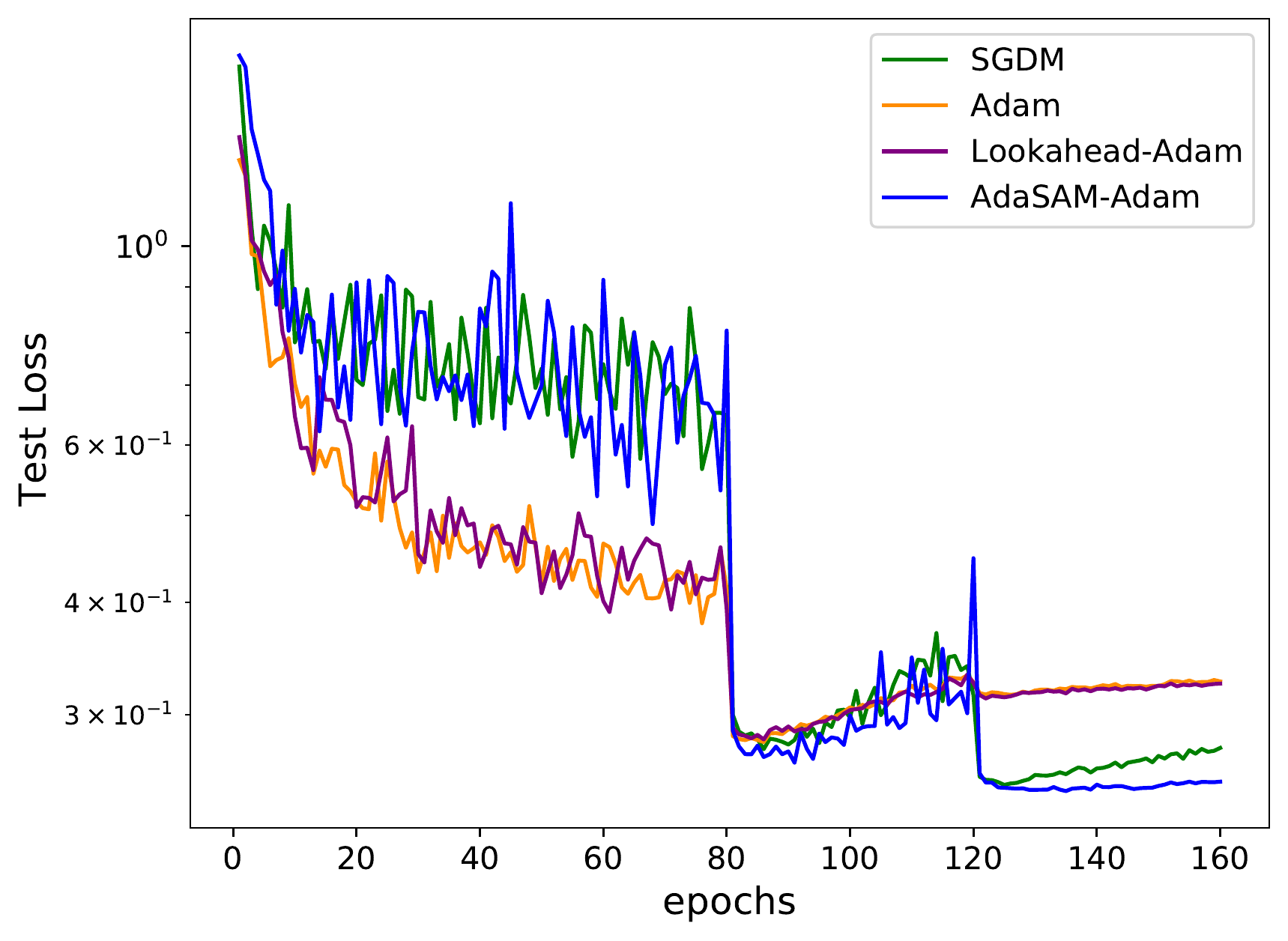}
}
\subfigure[Test Accuracy]{
\includegraphics[width=0.43\textwidth]{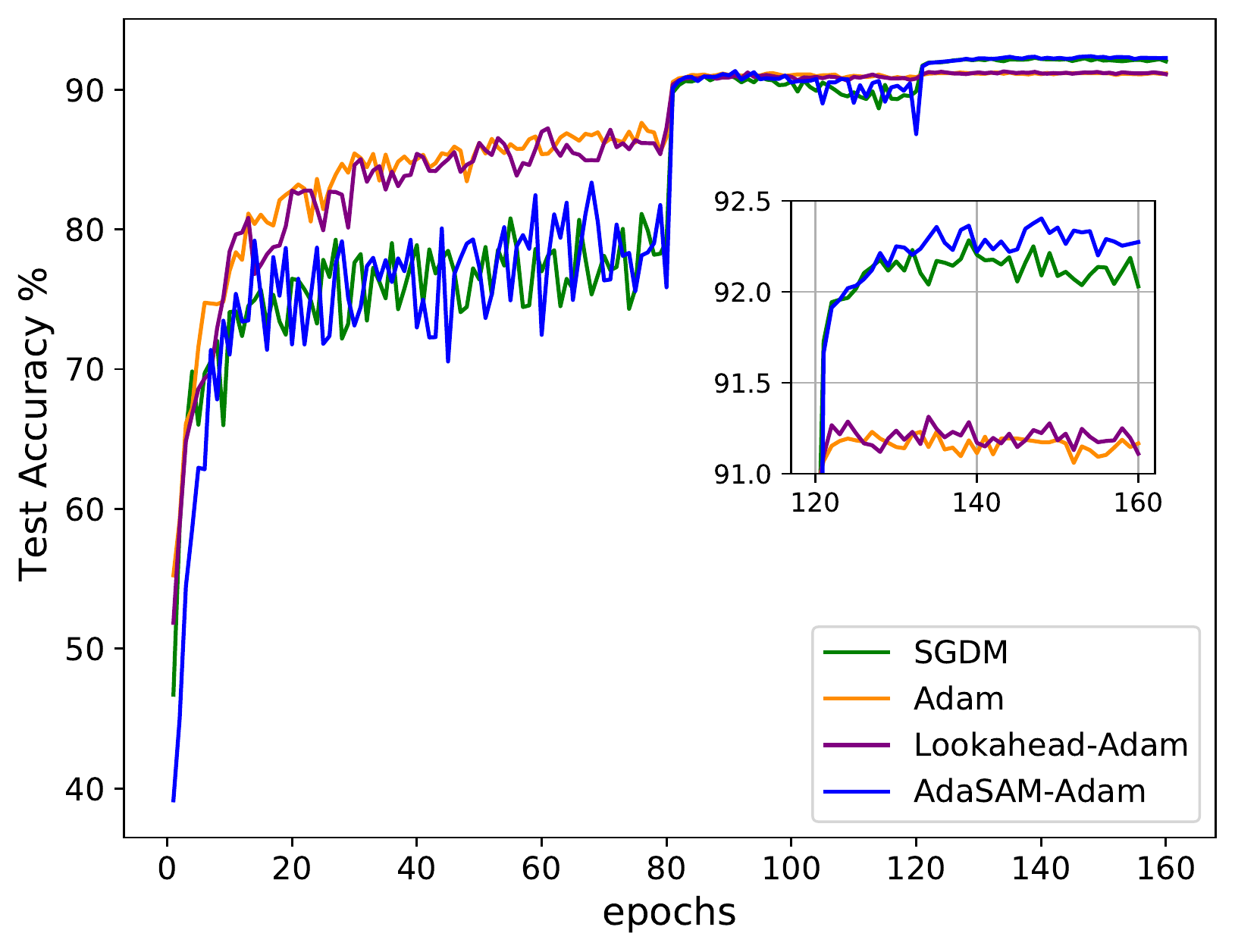}
}
\caption{Testing loss and accuracy on CIFAR-10/ResNet20.}
\label{fig:appendix_CIFAR_alter}
\end{figure} 
 
\subsection{Experiments on Penn TreeBank}
 Our experimental setting on training LSTM models on Penn TreeBank dataset is based on the official implementation of AdaBelief \citep{zhuang2020adabelief}. Results in Table~\ref{table:ptb} were measured across 3 repeated runs with independent initialization. The parameter setting of the LSTM models  are the same as that of AdaBelief.  
 The baseline optimizers are SGDM, Adam, AdaBelief and Lookahead. We tuned hyperparameters on the validation dataset for each optimizer.
 
 For SGDM, we tuned the learning rate (abbr. lr) o via grid-search  in $ \lbrace 1, 10, 30, 100 \rbrace $ and found that lr=10 performs best on 2,3-layer LSTM. For 1-layer LSTM, we set lr=30 and momentum=0 as that in AdaBelief because we found such setting is better.
 
 For Adam, we tuned the learning rate via grid-search in $ \lbrace 1\times 10^{-3}, 2\times 10^{-3}, 5\times 10^{-3}, 8\times 10^{-3}, 1\times 10^{-2}, 2\times 10^{-2} \rbrace $ and found $ 5\times 10^{-3} $ performs best.
 
 For AdaBelief, we tuned the learning rate and found $ 5\times 10^{-3} $ is better than $ 1\times 10^{-2} $ used in \citep{zhuang2020adabelief}.
 
 For Lookahead, as suggested by the authors in \citep{zhang2019lookahead}, Adam with best hyperparameter setting is set as the inner optimizer, then the interpolation parameter $ \alpha=0.5 $ and steps = 5.
 
 The batch size is 20. We trained for 200 epochs and decayed the learning rate by 0.1 at the 100th and 150th epoch. For pAdaSAM, since the learning rate decay has been applied to the inner optimizer, we did not apply decay to $ \alpha_k $ and $ \beta_k $, i.e. $ \alpha_k=\beta_k=1 $ is kept unchanged during the training.
 
 \begin{table*}[ht]
  \caption{Test perplexity on Penn TreeBank for 1,2,3-layer LSTM. Lower is better. AdaSAM* denotes AdaSAM with $ \beta_0=100 $.} \label{table:appendix_ptb}
  \centering
  \begin{tabular}{lccc} 
    \toprule   	
    Method     & 1-Layer     & 2-Layer & 3-Layer \\
    \midrule
    SGDM 		 &  85.21$\pm$.36   & 67.12$\pm$.14     & 61.56$\pm$.14 \\
    Adam     	 &  80.88$\pm$.15   & 64.54$\pm$.18     & 60.34$\pm$.22 \\
    AdaBelief     & 82.41$\pm$.46   & 65.07$\pm$.02    & 60.64$\pm$.14 \\
    Lookahead     & 82.01$\pm$.07   & 66.43$\pm$.33   & 61.80$\pm$.10 \\
    AdaSAM		 & 155.38$\pm$.35  & 159.07$\pm$1.58  & 163.60$\pm$.81 \\
    AdaSAM*	     & 91.23$\pm$.69   & 68.53$\pm$.13    & 63.74$\pm$.09 \\
    pAdaSAM      & \textbf{79.34$\pm$.09}    & \textbf{63.18$\pm$.22}   & \textbf{59.47$\pm$.08} \\
    \bottomrule
  \end{tabular}
\end{table*} 
 
 Our method is pAdaSAM, which set $ optim =$ Adam in Algorithm~\ref{alg:padasam}, where Adam is the tuned baseline. AdaSAM with the default setting is not suitable for this task. To give a full view of the vanilla AdaSAM, we also report the results of AdaSAM with default setting and the tuned AdaSAM ($ \beta_0 = 100 $) in Table~\ref{table:appendix_ptb}. 
 
 \begin{figure}[ht]
\centering 
\subfigure[1-Layer LSTM]{
\includegraphics[width=0.31\textwidth]{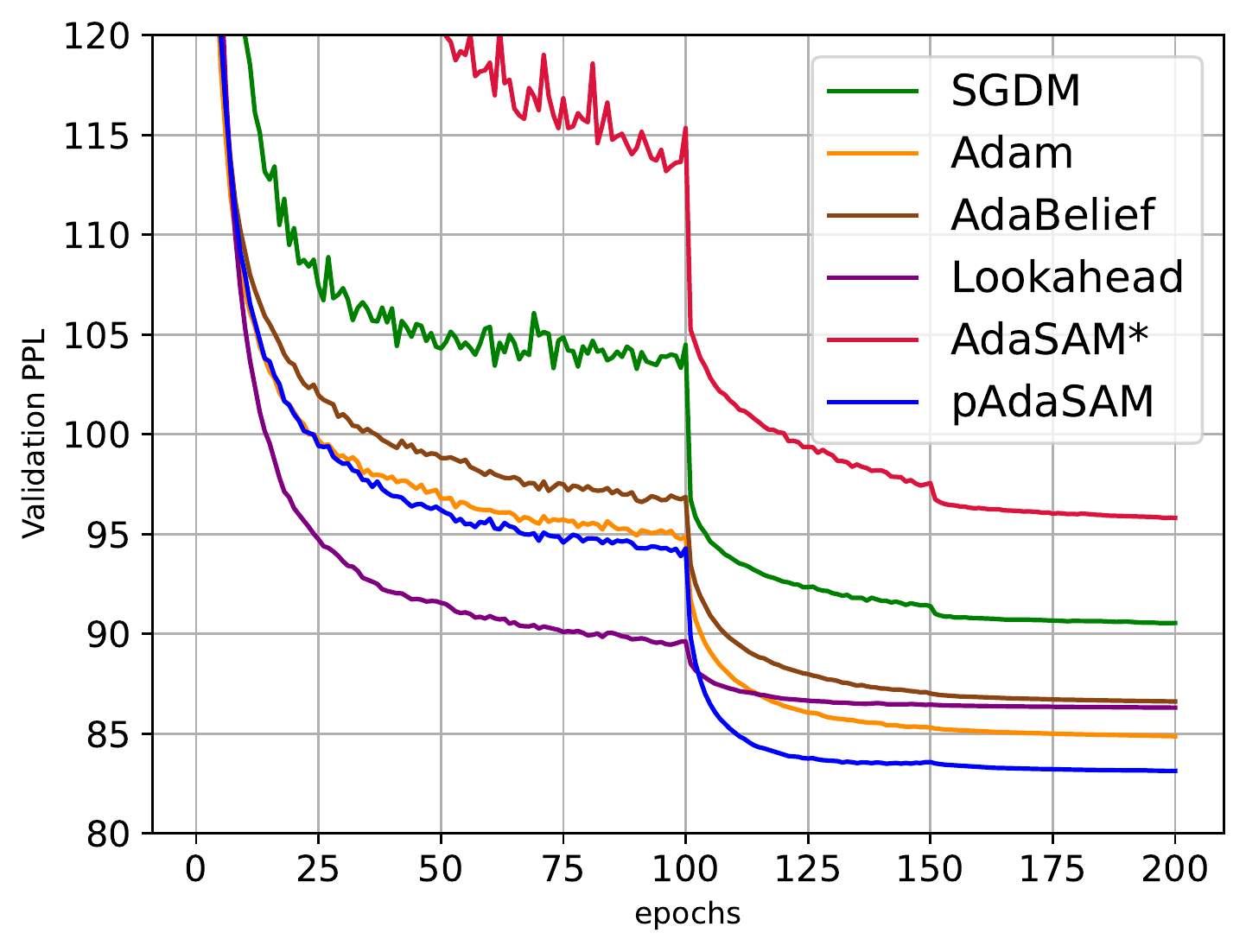}
}
\subfigure[2-Layer LSTM]{
\includegraphics[width=0.31\textwidth]{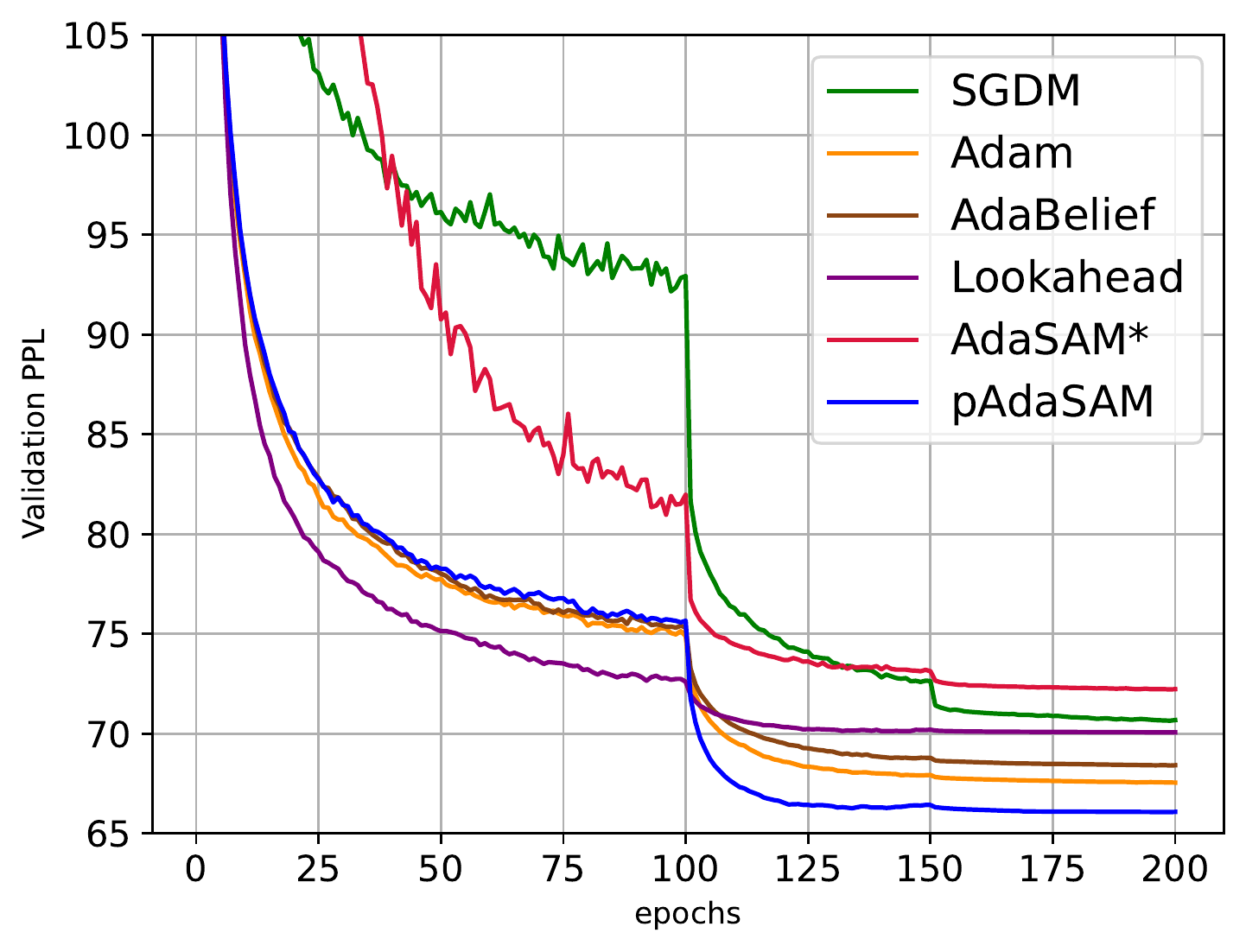}
}
\subfigure[3-Layer LSTM]{
\includegraphics[width=0.31\textwidth]{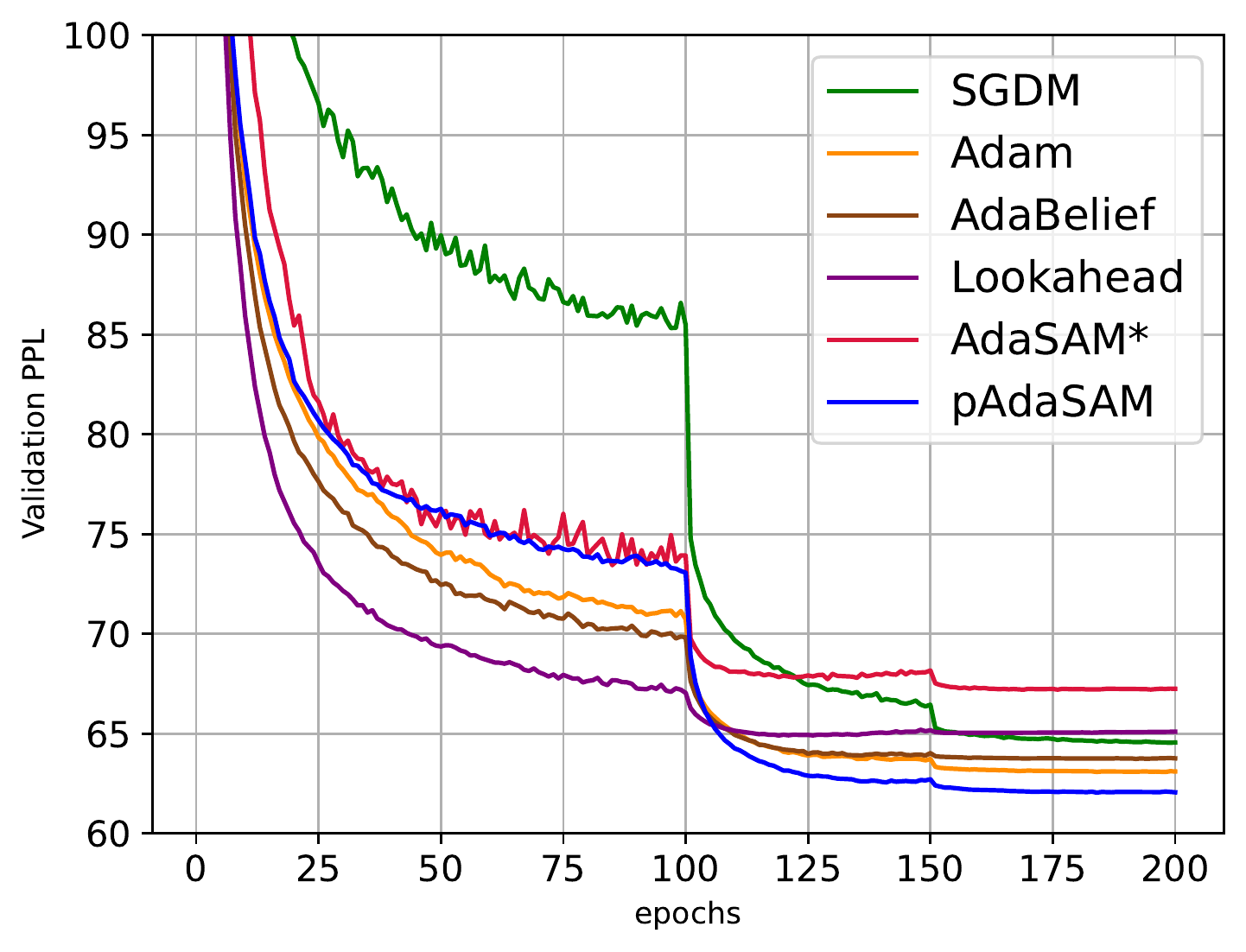}
}
\caption{Experiments on Penn TreeBank. Validation perplexity of training 1,2,3-Layer LSTM.}
\label{fig:appendix_ptb}
\end{figure}
 
 We think  the scaling of the model's parameters is important for  this problem. Since the batch size is very small, the gradient estimation is too noisy to capture the curvature information of the objective function. Hence the quadratic approximation in AdaSAM is rather inaccurate and further scaling by $ \beta_k$ is required. For pAdaSAM,  the scaling of the stochastic gradient is done by the inner optimizer Adam, so $ \beta_k  $ and $ \alpha_k$ can be set as default.

\section{Additional experiments}
 This section is about the techniques and hyperparameters used in our method. The computational cost is also reported in the end.
 
\subsection{Check of positive definiteness }
 In Algorithm~\ref{alg:adasam}, we simplify the check of positive definiteness of $ H_k $ described in Section~\ref{sec:sam} by $ (\Delta x_k)^{\mathrm{T}} r_k > 0 $.  To see the effect of such simplification,  we first give the pseudocode  in Algorithm~\ref{alg:sam_pos} that faithfully follows the procedure of checking positive definiteness in Section~\ref{sec:sam}.  We designate it as AdaSAM0. Note that the check of \eqref{ineq:check_alpha_k} is reflected in Line~16 in Algorithm~\ref{alg:sam_pos}.
 
 We compare AdaSAM (Algorithm~\ref{alg:adasam} with AdaSAM0 via experiments on MNIST and CIFAR-10/ResNet20. The experimental setting is the same as that in Section~\ref{sec:exp_details}. 
 
 The results on MNIST are shown in Figure~\ref{fig:appendix_test_pos1}. We trained 100 epochs with full-batch (batchsize=12K) and mini-batch (batchsize=3K). The evolution of $ \alpha_k $ in Figure~\ref{fig:appendix_test_pos1}(c) implies that the Hessian approximation $ H_k$ in AdaSAM ( $ \alpha_k=\beta_k=1 $) is hardly positive definite. So AdaSAM0 reduces $ \alpha_k $ to ensure Condition~\eqref{ineq:pos_Hk} holds. However, in full-batch training, there exists no noise in gradient evaluations, so $ r_k^{\mathrm{T}}H_kr_k > 0 $ is suffice to ensure $ H_kr_k $ is a descent direction. In other words, Condition~\eqref{ineq:pos_Hk} may be too stringent to prevent the acceleration effect of AdaSAM. We also find that even running without checking of positive definiteness, the result is comparable. The result of AdaSAM0 with $ \mu=0.9 $ suggests that the optimization is trapped in a local minima. In mini-batch training, Condition~\ref{ineq:pos_Hk} is violated more frequently if using $ \alpha_k = 1 $, which can be inferred from the evolution of $ \alpha_k $ of AdaSAM.  Switching to $ optim $ when $ (\Delta x_k)^{\mathrm{T}}r_k \leq 0 $ is better than ignoring the violation of the positive definiteness. For AdaSAM0, we find using small $ \mu $ is proper.
 
 The results on CIFAR-10/ResNet20 are shown in Figure~\ref{fig:appendix_test_pos2}. We test AdaSAM0 with different selections of $ \mu $. We see smaller $ \mu $ is better. We also find that the value of $ \lambda_k$ is restrictive in training ResNet20. For $ \mu = 0 $, the Condition~\eqref{ineq:check_alpha_k} is seldom violated during training, which means $ \alpha_k$ is not need to be reduced to a smaller value to make $H_k$ positive definite. Therefore, AdaSAM0 with $ \mu=0 $ has nearly the same behaviour as AdaSAM. 
 
 With these tests, we confirm that  using the sanity check of positive definiteness of $H_k$ in Algorithm~\ref{alg:adasam} does not lead to any deterioration.

 \begin{algorithm}[ht]
\caption{AdaSAM0. AdaSAM with the check of \eqref{ineq:pos_Hk}}
\label{alg:sam_pos}
\textbf{Input}: $ x_0\in\mathbb{R}^d, m=10, \alpha_k=1, \beta_0 = 0.1, \beta_k=1  (k\geq 1), \gamma=0.9, \mu=10^{-8}, \epsilon = 10^{-8}, max\_iter>0 $.\\
\textbf{Output}: $ x\in\mathbb{R}^d $
\begin{algorithmic}[1] 
\FOR{$k = 0,1,\dots, max\_iter$ }
\STATE $ r_k = -\nabla f_{S_k}\left(x_k\right) $ 
\IF {$k = 0$}
\STATE {$ x_{k+1} = x_k+\beta_kr_k $}
\ELSE
\STATE $ m_k = \min\{m,k\} $
\STATE $ \Delta\hat{x}_k =  \gamma\cdot\Delta\hat{x}_{k-1}+(1-\gamma)\cdot\Delta x_{k-1} $ 
\STATE $ \Delta\hat{r}_k = \gamma\cdot\Delta\hat{r}_{k-1}+(1-\gamma)\cdot\Delta r_{k-1} $
\STATE $ \hat{X}_k = [
\Delta \hat{x}_{k-m_k} , \Delta \hat{x}_{k-m_k+1} , \cdots , \Delta \hat{x}_{k-1} ] $
\STATE $ \hat{R}_k = [
\Delta \hat{r}_{k-m_k} , \Delta \hat{r}_{k-m_k+1} , \cdots , \Delta \hat{r}_{k-1} ] $
\STATE $ \delta_k = c_1 \| r_k\|_2^2 / \left(\|\Delta\hat{x}_k\|_2^2 +\epsilon\right) $
\STATE $ Y_k = \hat{X}_k+\beta_k \hat{R}_k $
\STATE $ Z_k = \hat{R}_k^{\mathrm{T}}\hat{R}_k+\delta_k\hat{X}_k^{\mathrm{T}}\hat{X}_k $
\STATE Compute $ \lambda_k = \lambda_{max}\left( 
 \begin{pmatrix}
 Y_k^\mathrm{T} \\
 \hat{R}_k^\mathrm{T}
 \end{pmatrix} 
 \begin{pmatrix}
 Y_k & \hat{R}_k
 \end{pmatrix}
 \begin{pmatrix}
 0 & Z_k^{\dagger} \\
 Z_k^{\dagger} & 0 
 \end{pmatrix} 
 \right). $
\STATE $ \tilde{\alpha}_k = \alpha_k $
\IF {$ \lambda_k > 0$ }
\STATE $ \tilde{\alpha}_k = \min\lbrace \alpha_k,2\beta_k(1-\mu)/\lambda_k\rbrace $
\ENDIF
\STATE $ x_{k+1} = x_k + \beta_k r_k - \tilde{\alpha}_k Y_k Z_k^{\dagger} \hat{R}_k^{\mathrm{T}}r_k $ 
\ENDIF
\STATE Apply learning rate schedule of $ \alpha_k, \beta_k $
\ENDFOR
\STATE \textbf{return} $ x_k $
\end{algorithmic}
\end{algorithm}

 \begin{figure}[H]
\centering 
\subfigure[Train Loss (batchsize=12K)]{
\includegraphics[width=0.31\textwidth]{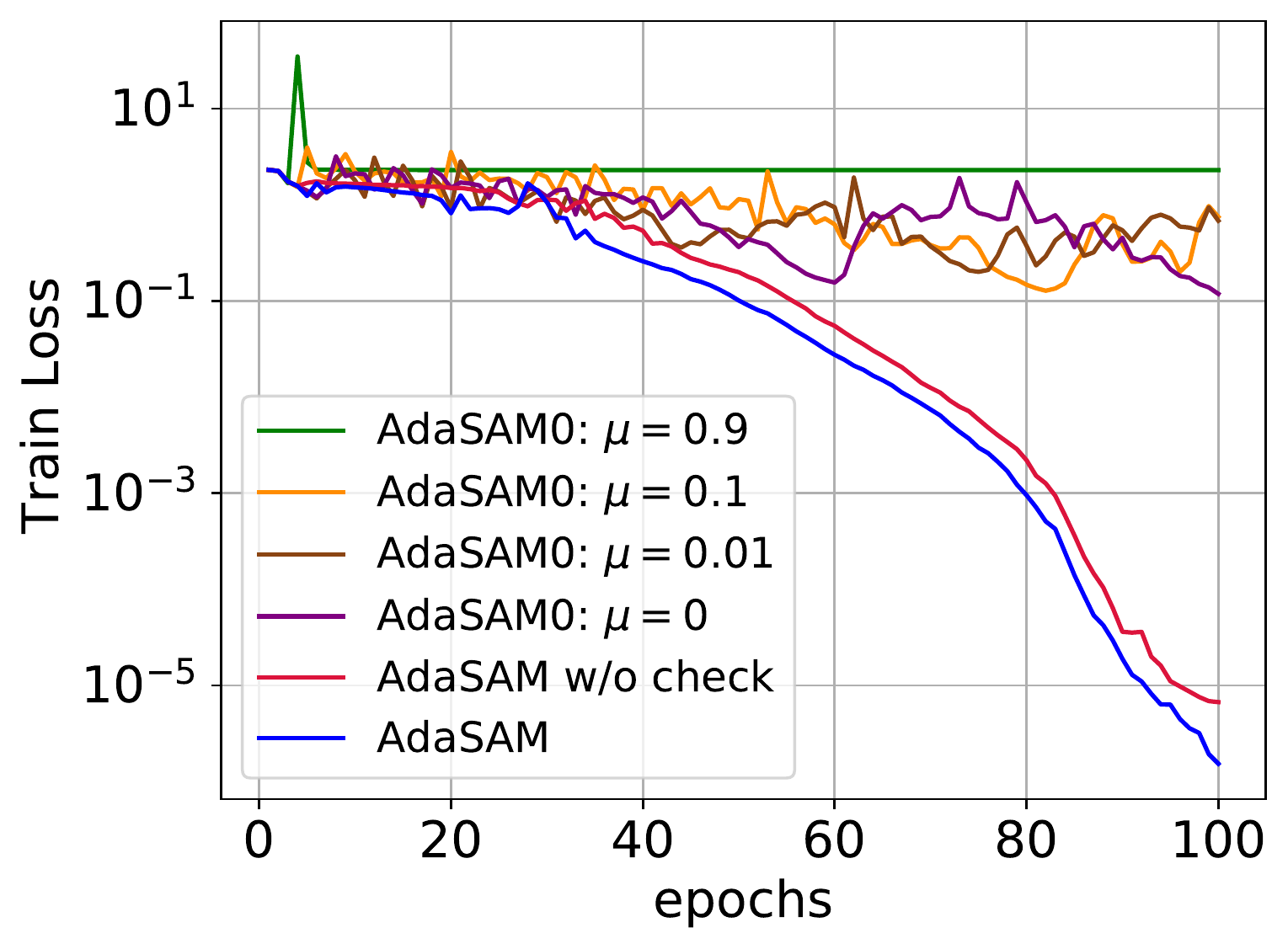}
}
\subfigure[SNG (batchsize=12K)]{
\includegraphics[width=0.31\textwidth]{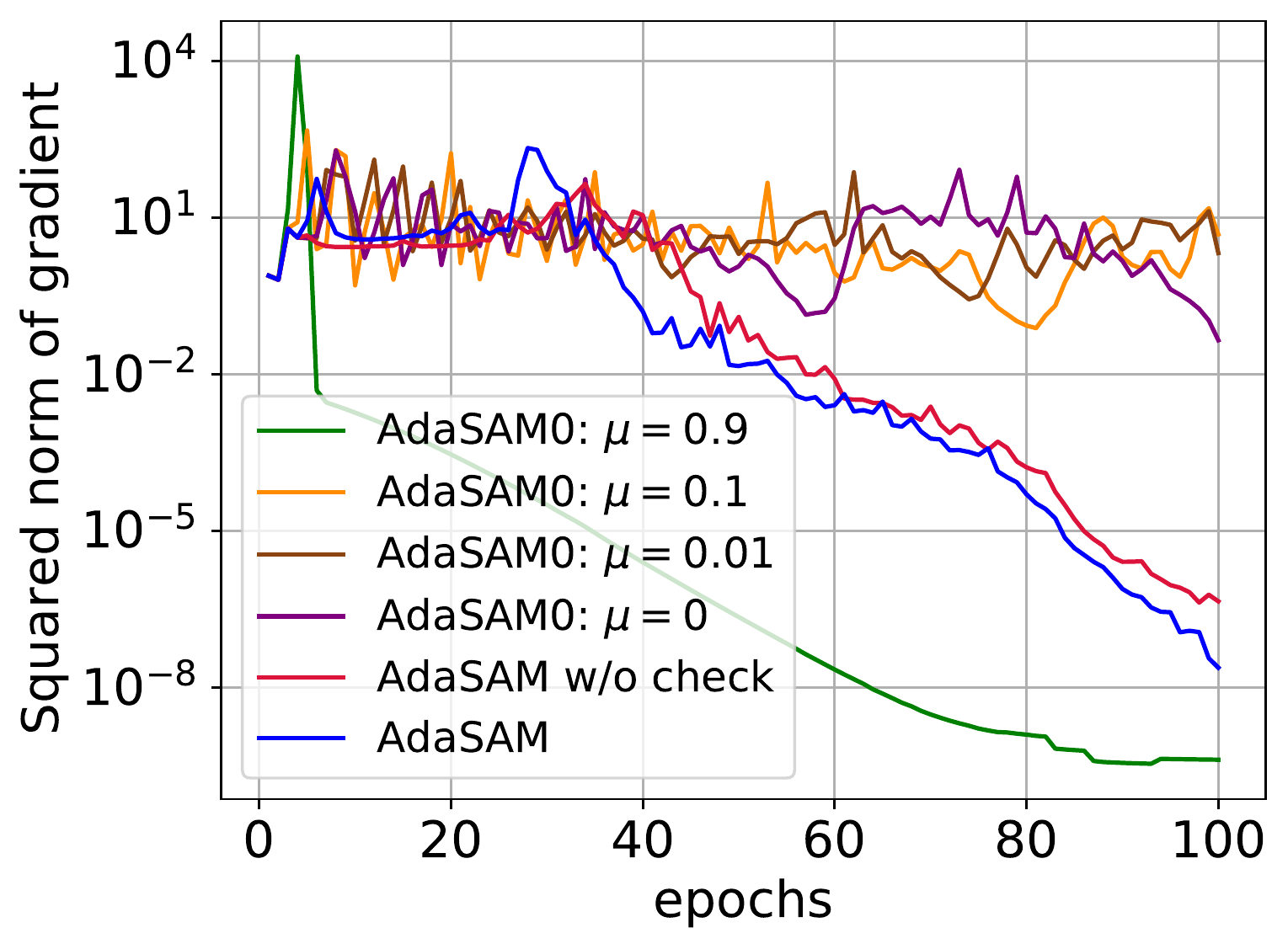}
}
\subfigure[$\alpha_k$ (batchsize=12K)]{
\includegraphics[width=0.31\textwidth]{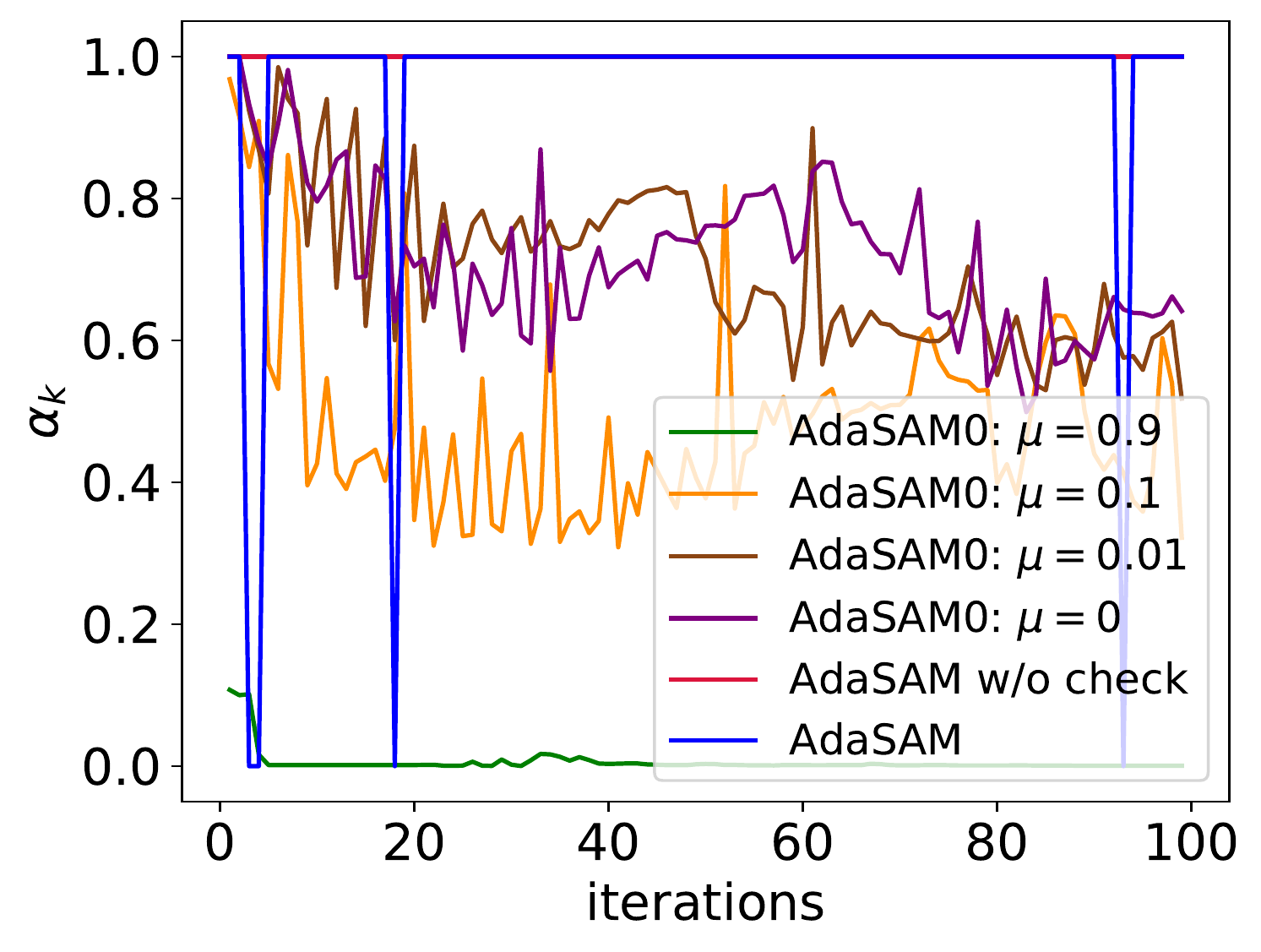}
}
\subfigure[Train Loss (batchsize=3K)]{
\includegraphics[width=0.31\textwidth]{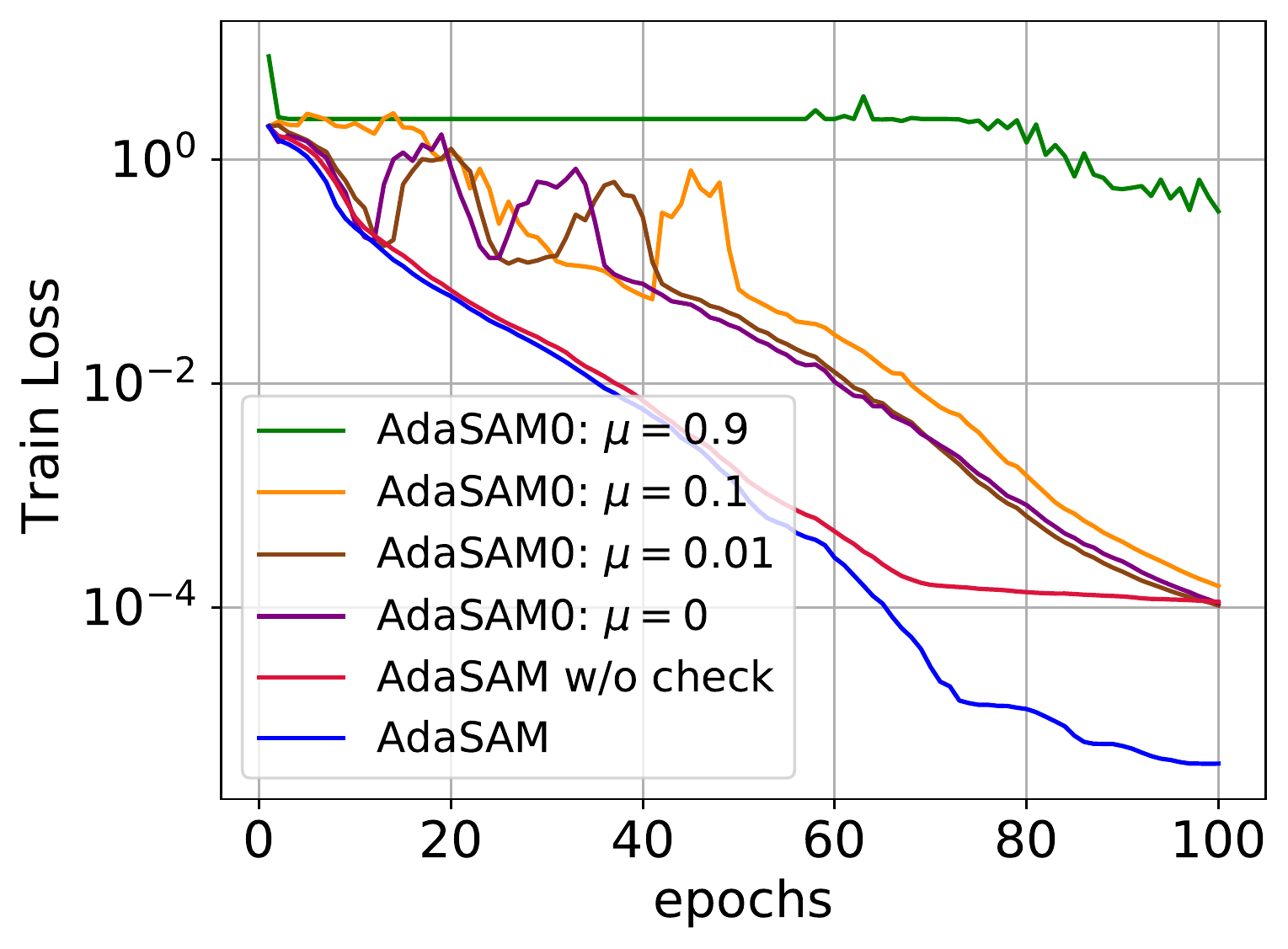}
}
\subfigure[SNG (batchsize=3K)]{
\includegraphics[width=0.31\textwidth]{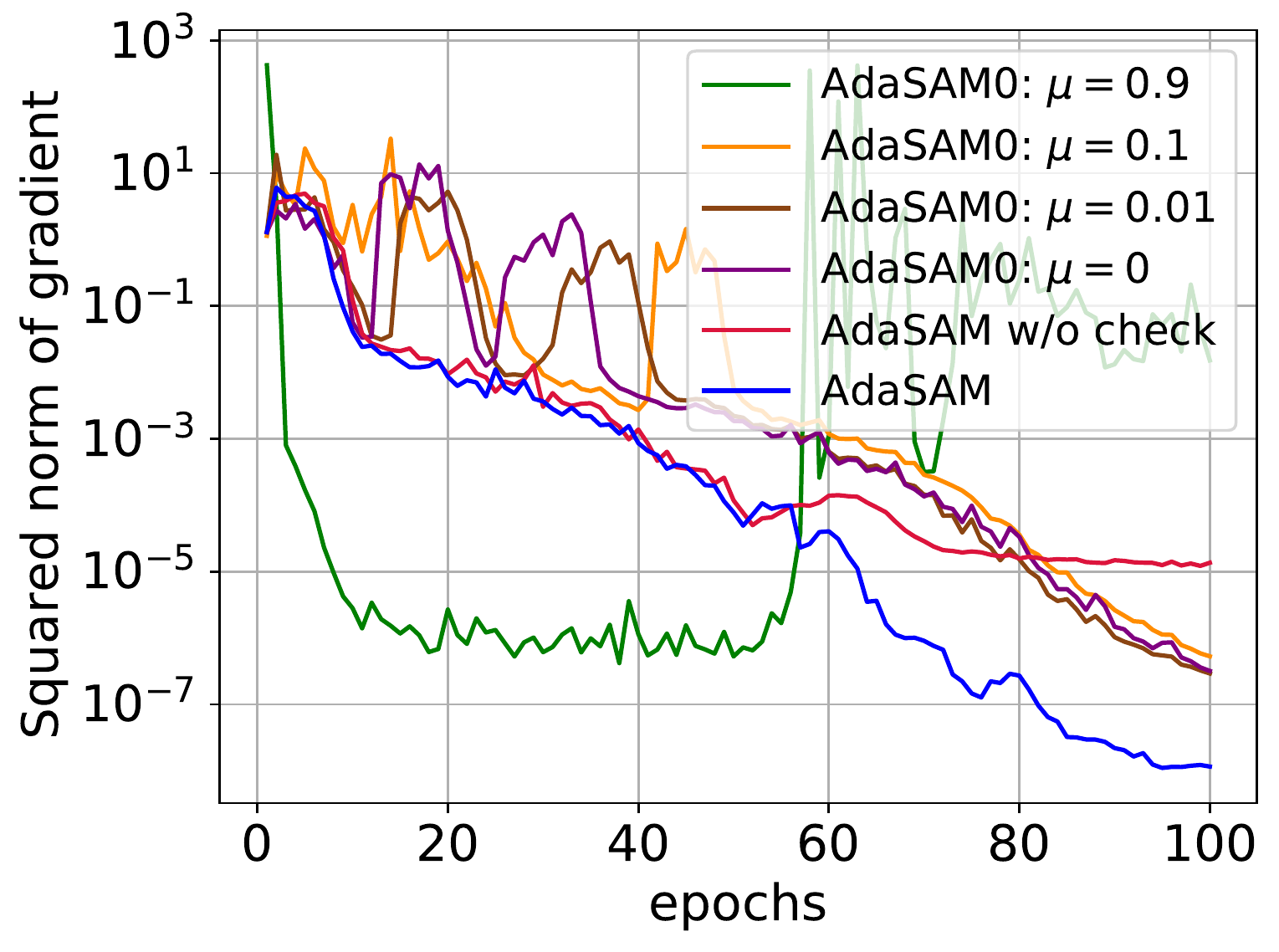}
}
\subfigure[$ \alpha_k$ (batchsize=3K)]{
\includegraphics[width=0.31\textwidth]{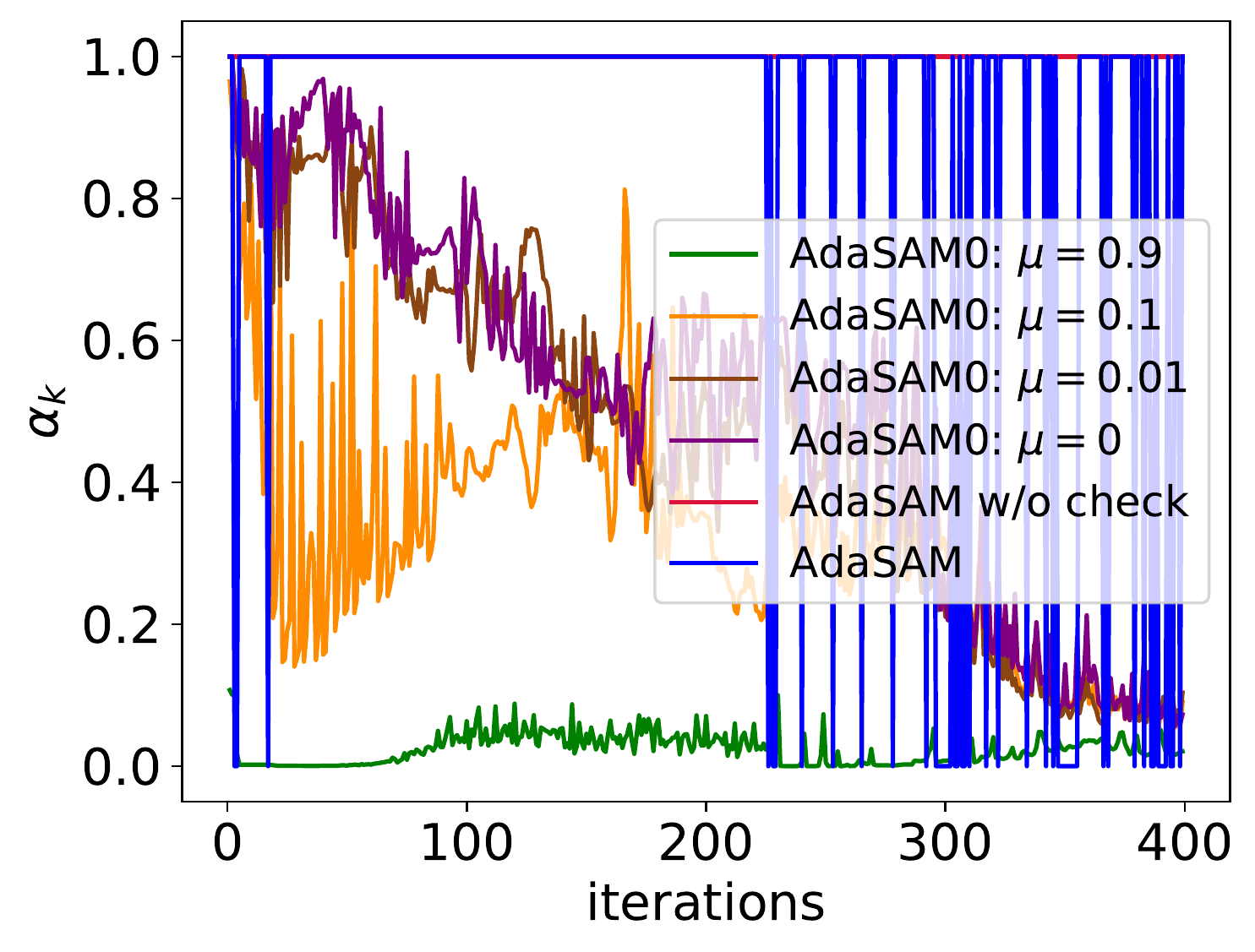}
}
\caption{Experiments on MNIST. Training loss, squared norm of gradient (abbr. SNG) and $ \alpha_k $ for batch size = 12K, 3K.}
\label{fig:appendix_test_pos1}
\end{figure}

\begin{figure}[ht]
\centering 
\subfigure[Train Loss]{
\includegraphics[width=0.31\textwidth]{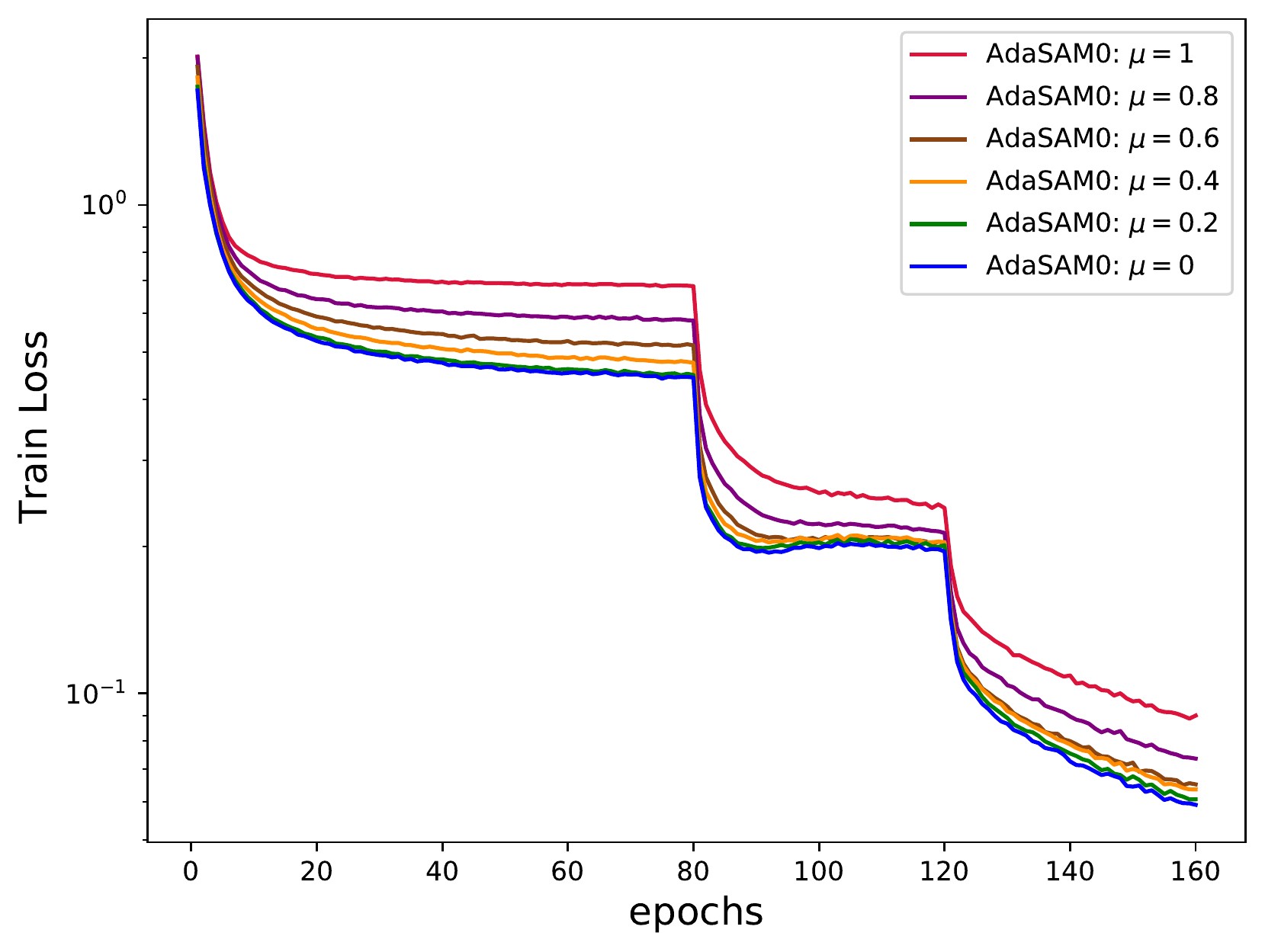}
}
\subfigure[Test Loss]{
\includegraphics[width=0.31\textwidth]{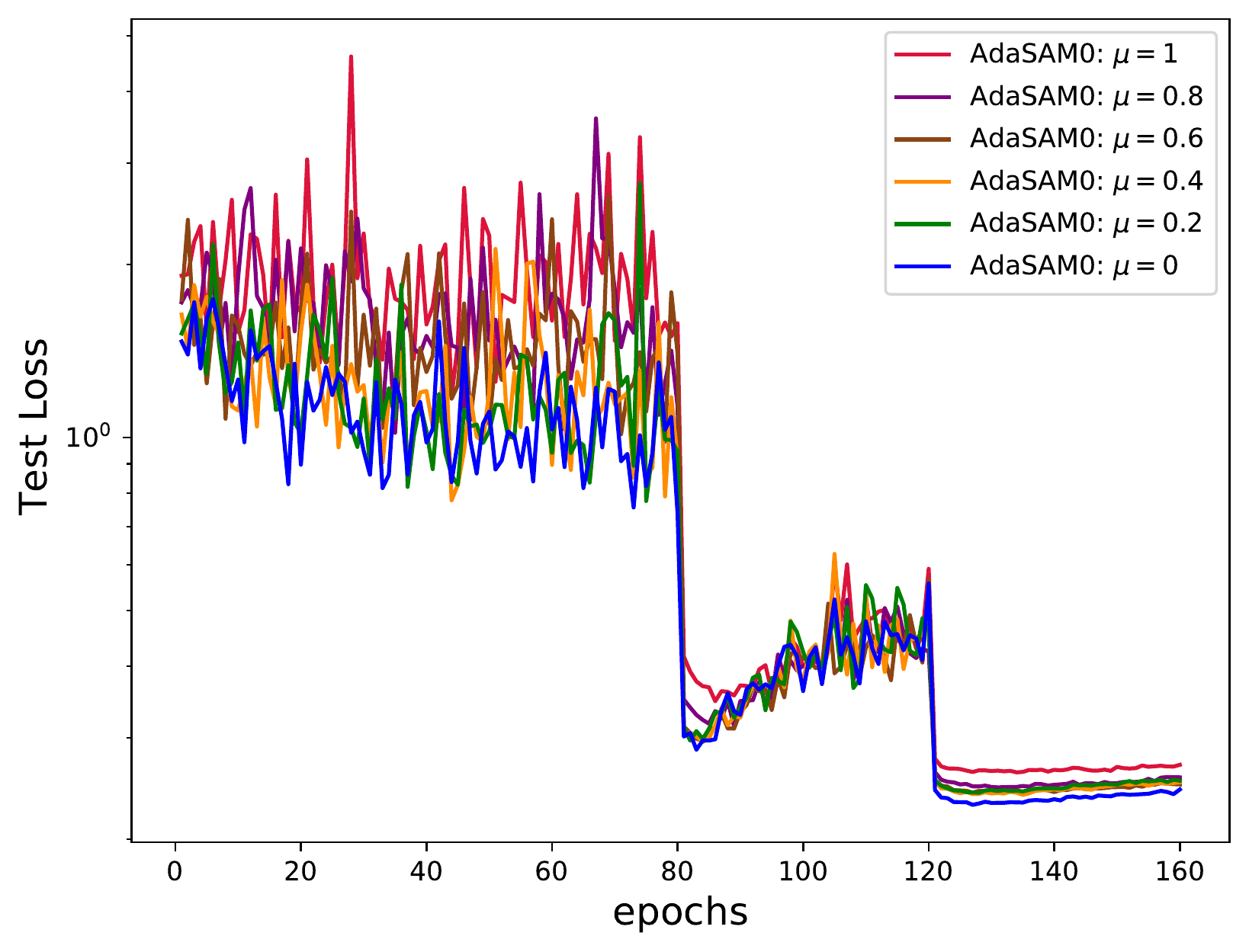}
}
\subfigure[$\lambda_k$]{
\includegraphics[width=0.31\textwidth]{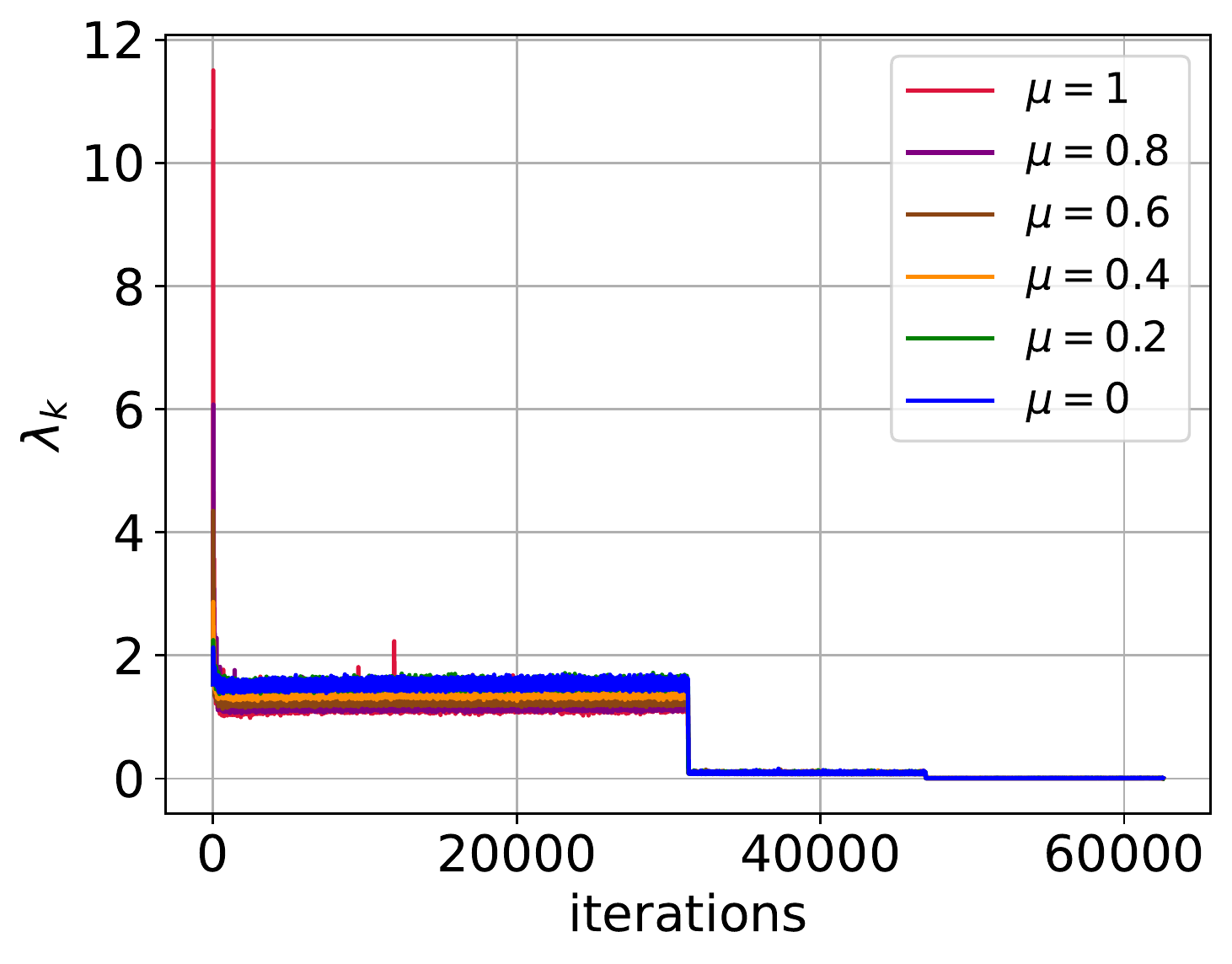}
}
\subfigure[Train Accuracy]{
\includegraphics[width=0.31\textwidth]{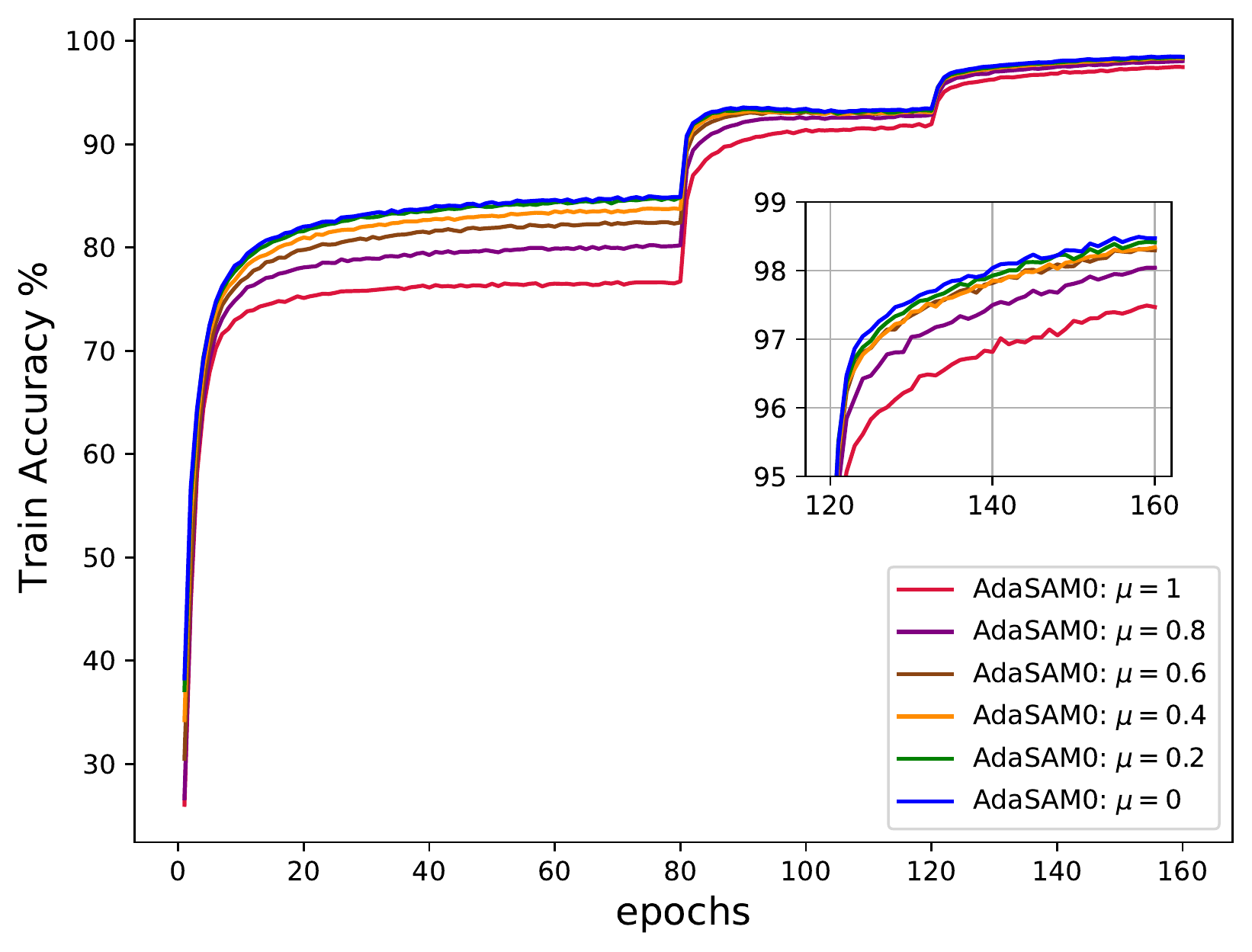}
}
\subfigure[Test Accuracy]{
\includegraphics[width=0.31\textwidth]{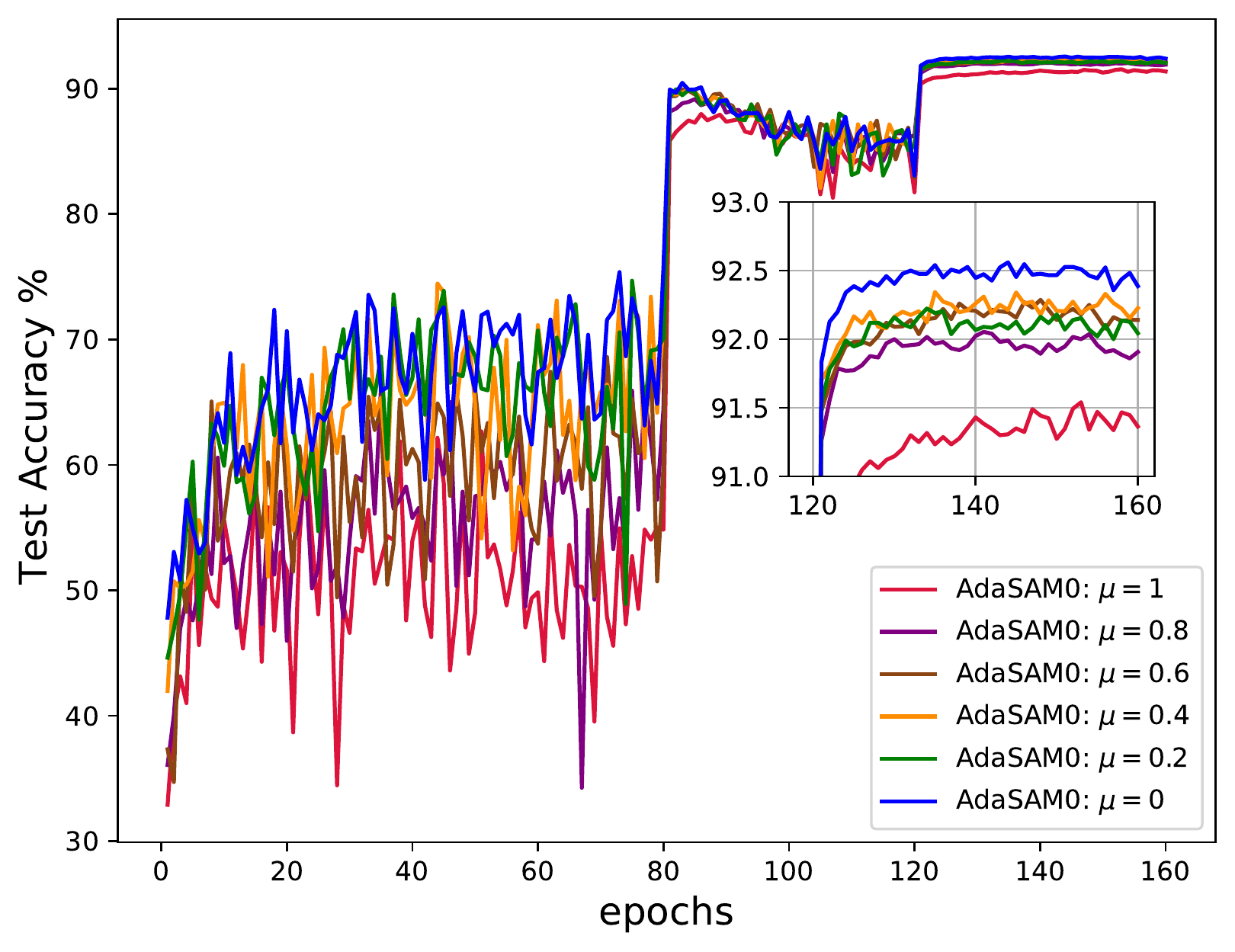}
}
\subfigure[$\alpha_k$]{
\includegraphics[width=0.31\textwidth]{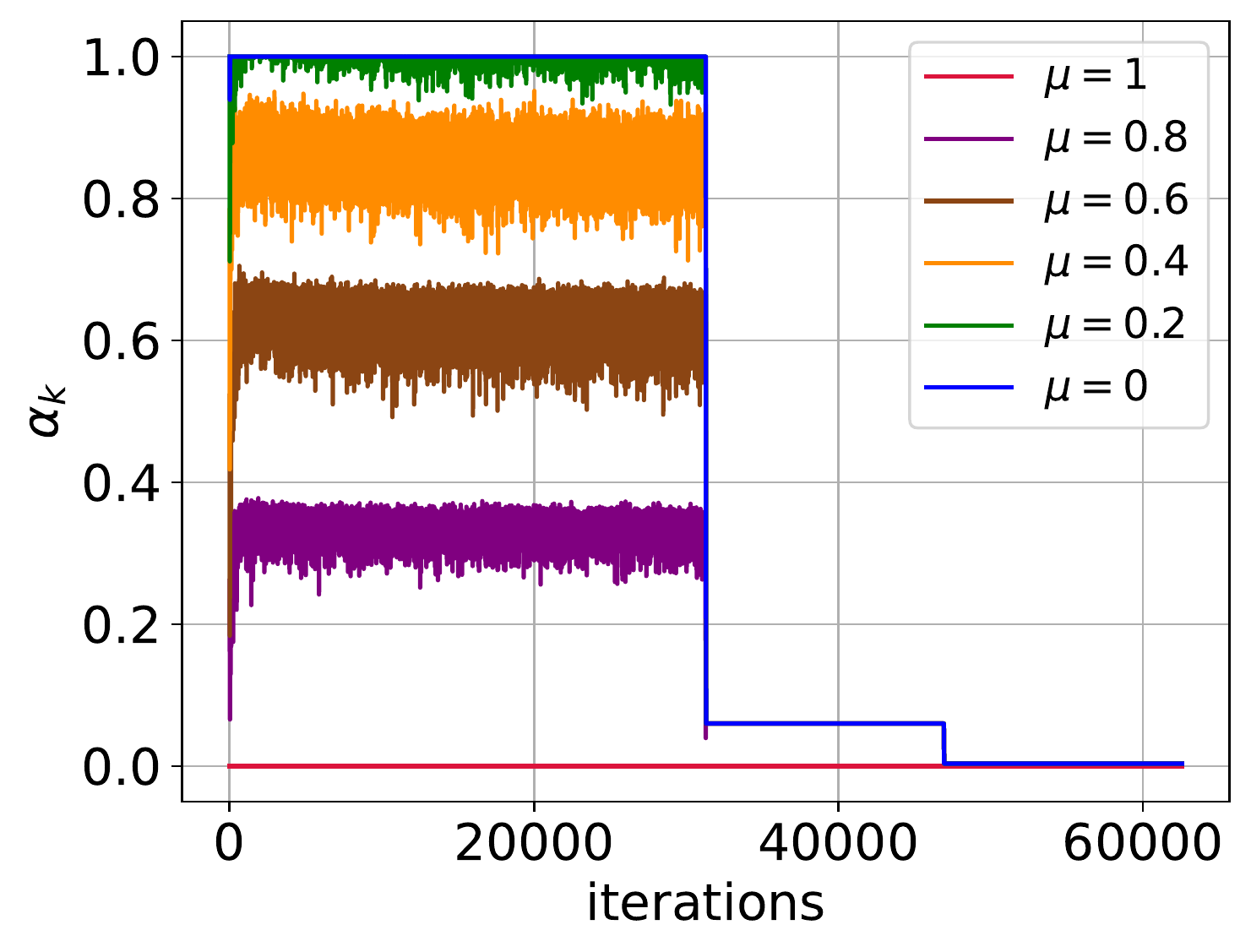}
}
\caption{Experiments on CIFAT-10/ResNet20. Training with AdaSAM0 with $ \mu =$ 0, 0.2, 0.4, 0.6, 0.8, 1. (c) and (d) show the evolutions of $ \lambda_k $ and $\alpha_k$ in AdaSAM0 during training.}
\label{fig:appendix_test_pos2}
\end{figure}


\subsection{Effect of damped projection}
 Damped projection is introduced to overcome the weakness of the potential indefiniteness of $ H_k $ in Anderson mixing. Its necessity has been justified in theory in Section~\ref{sec:theory}. In practice, though we always initially set $ \alpha_k=1 $ in AdaSAM, using damped projection can help improve the effectiveness. As shown in Figure~\ref{fig:appendix_test_pos1}, temporarily setting $ \alpha_k = 0 $ when $ (\Delta x_k)^{\mathrm{T}}r_k \leq 0  $ did improve convergence compared with the way of keeping $ \alpha_k=1$  unchanged. We also conducted tests on CIFAR-10/ResNet20, where the learning rate decay of $ \alpha_k $ was forbidden during training.  The result is shown in Figure~\ref{fig:appendix_test_alpha}. We see the learning rate decay of $ \alpha_k $ improves generalization.
  
 \begin{figure}[ht]
\centering 
\subfigure[Train Loss]{
\includegraphics[width=0.3\textwidth]{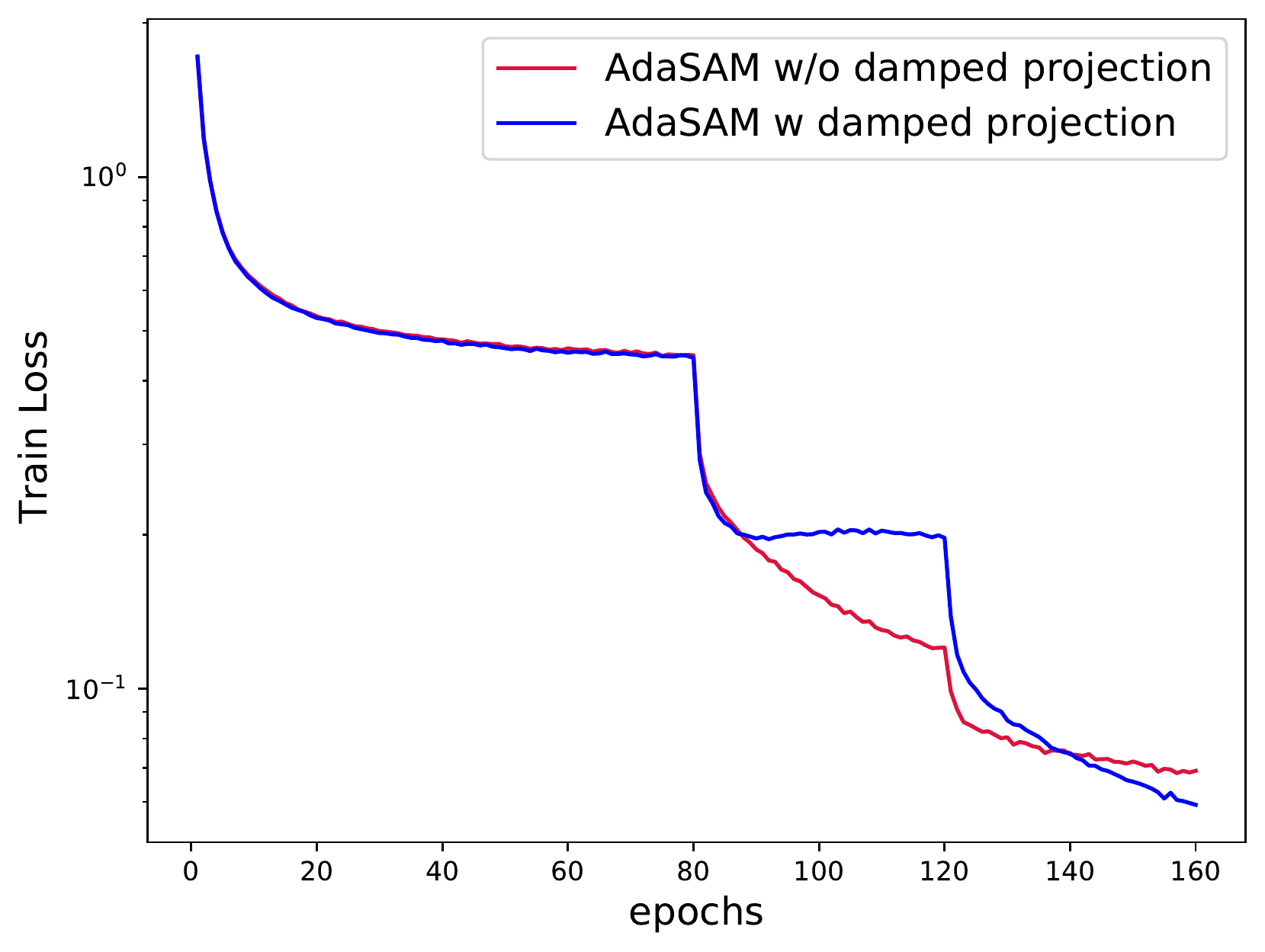}
}
\subfigure[Train Accuracy]{
\includegraphics[width=0.3\textwidth]{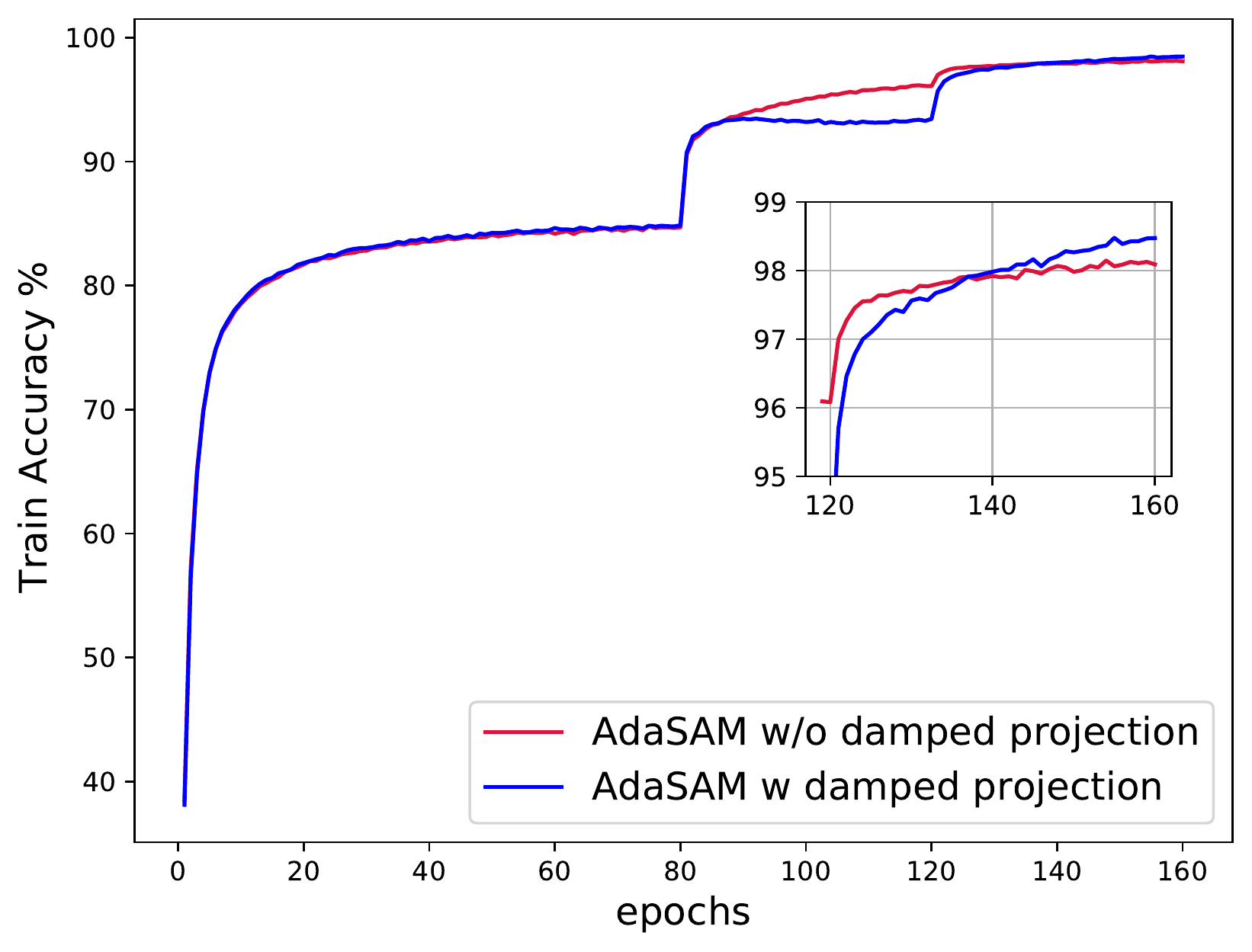}
}
\subfigure[Test Loss]{
\includegraphics[width=0.3\textwidth]{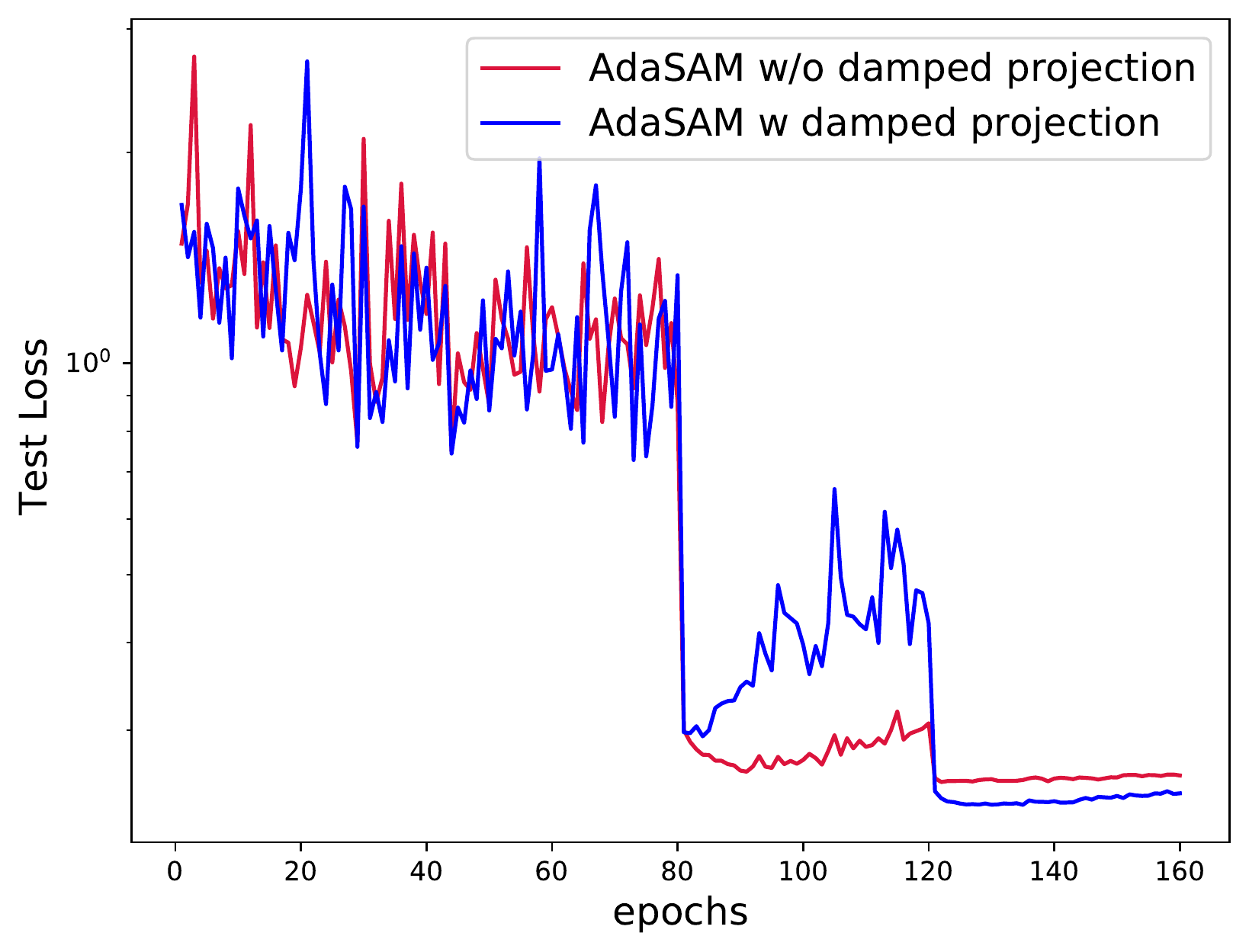}
}
\subfigure[Test Accuracy]{
\includegraphics[width=0.3\textwidth]{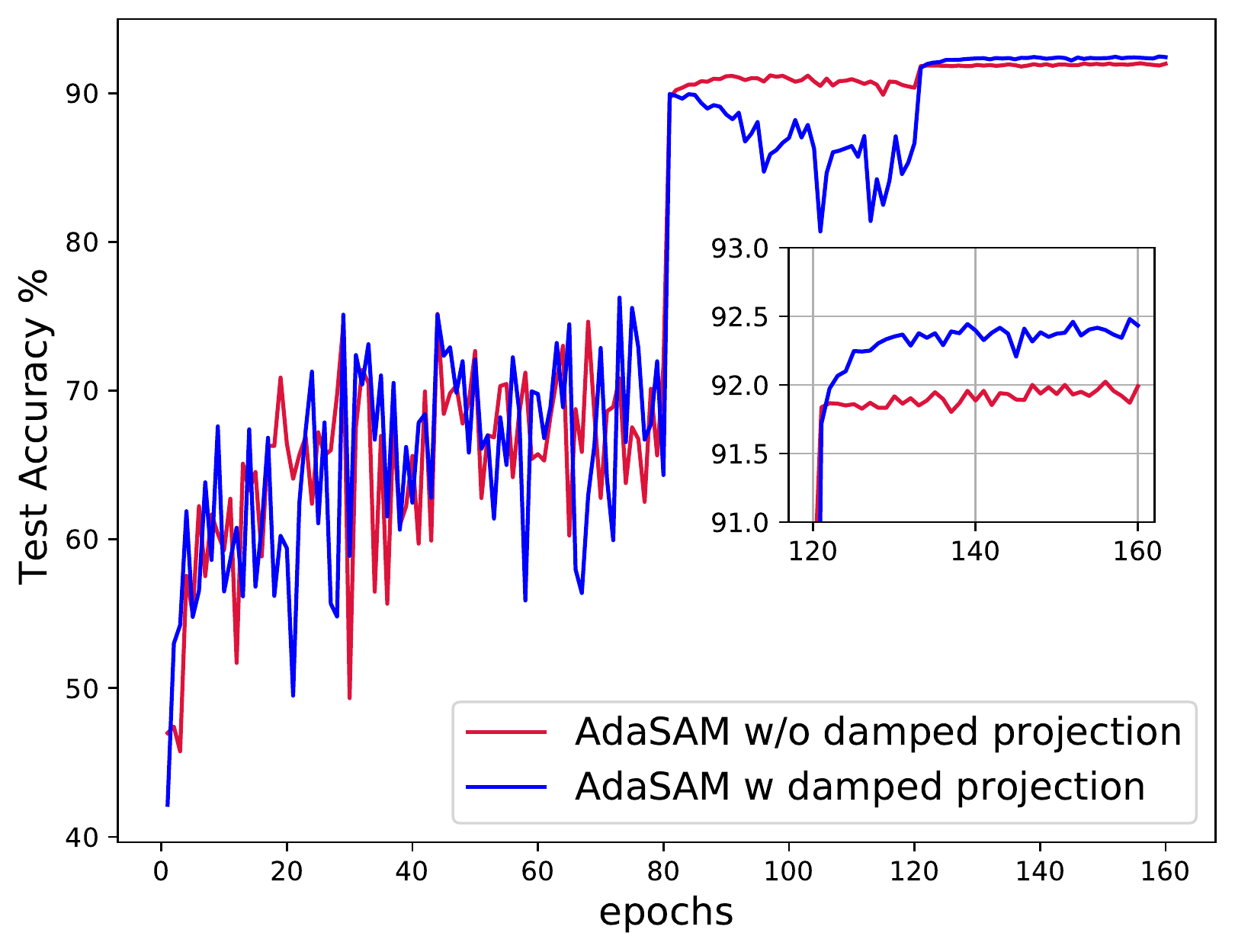}
}
\caption{Experiments on CIFAT-10/ResNet20. Training with/without damped projection.}
\label{fig:appendix_test_alpha}
\end{figure}

\subsection{Effect of adaptive regularization} \label{subsec:adapreg}
 AdaSAM is a special case of SAM with selection of $ \delta_k $ as \eqref{deltak}. Note that in our implementation Algorithm~\ref{alg:adasam}, we omit the term $ c_2\beta_k^{-2}$.  In fact, such special choice is important for SAM to be effective since it can roughly capture the curvature information. 
 The comparison between AdaSAM and RAM in experiments on MNIST and CIFARs confirms the superiority of the regularization term of AdaSAM. Here, we further compare the choice of \eqref{deltak} with two other choices:
 \begin{align*}
 \delta_k &= \delta, \quad\quad\quad\mbox{(Option I),} \\
 \delta_k &= \delta\beta_k^{-2},  ~\quad\mbox{(Option II),} 
 \end{align*}
 We designate SAM with $ \delta_k $ chosen as Option~I (Option~II) as SAM$^\dagger$ (SAM$^\ddagger$). 
 
 \begin{figure}[ht]
\centering 
\subfigure[Train Loss (b=12K)]{
\includegraphics[width=0.23\textwidth]{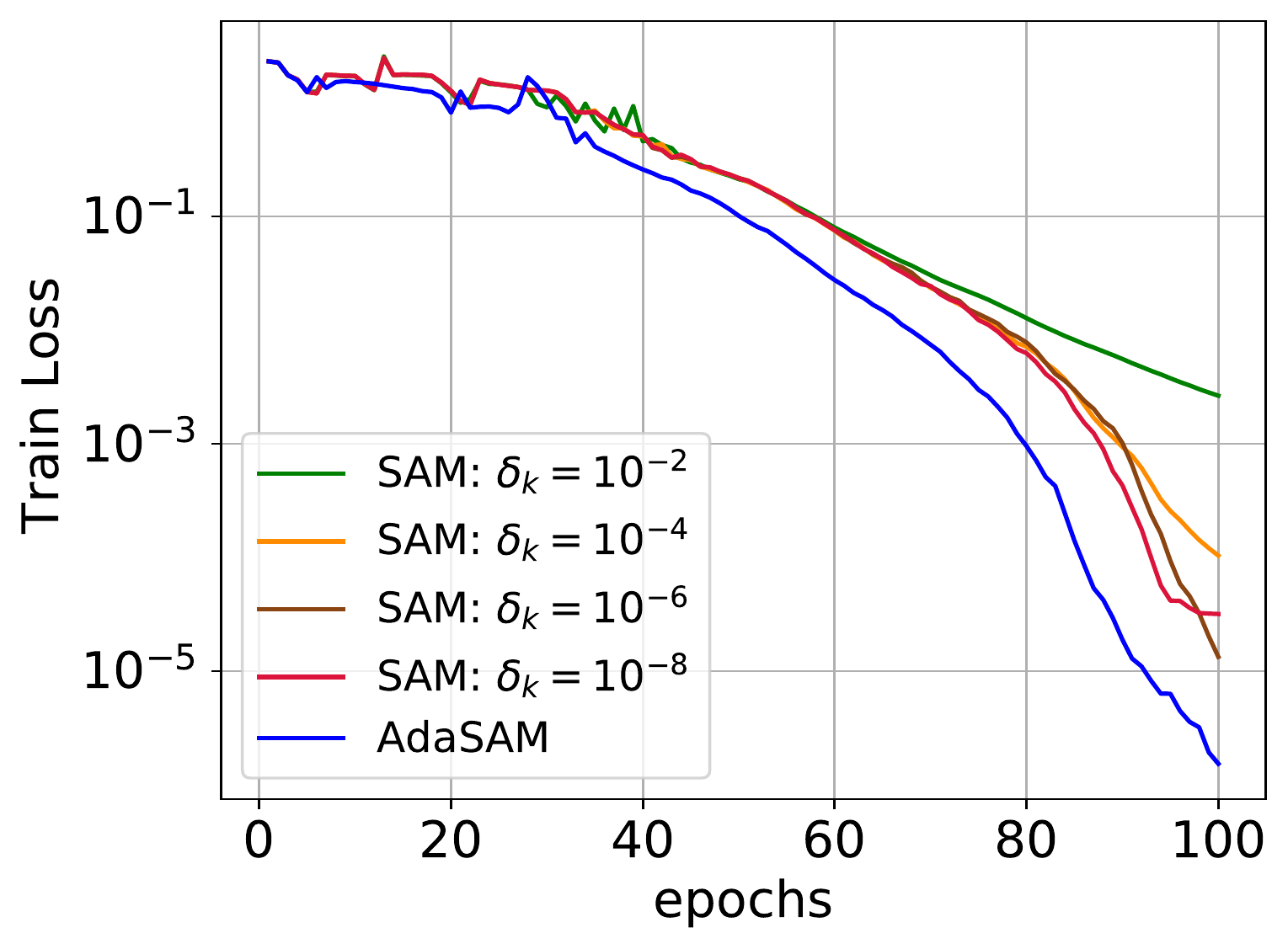}
}
\subfigure[SNG (b=12K)]{
\includegraphics[width=0.23\textwidth]{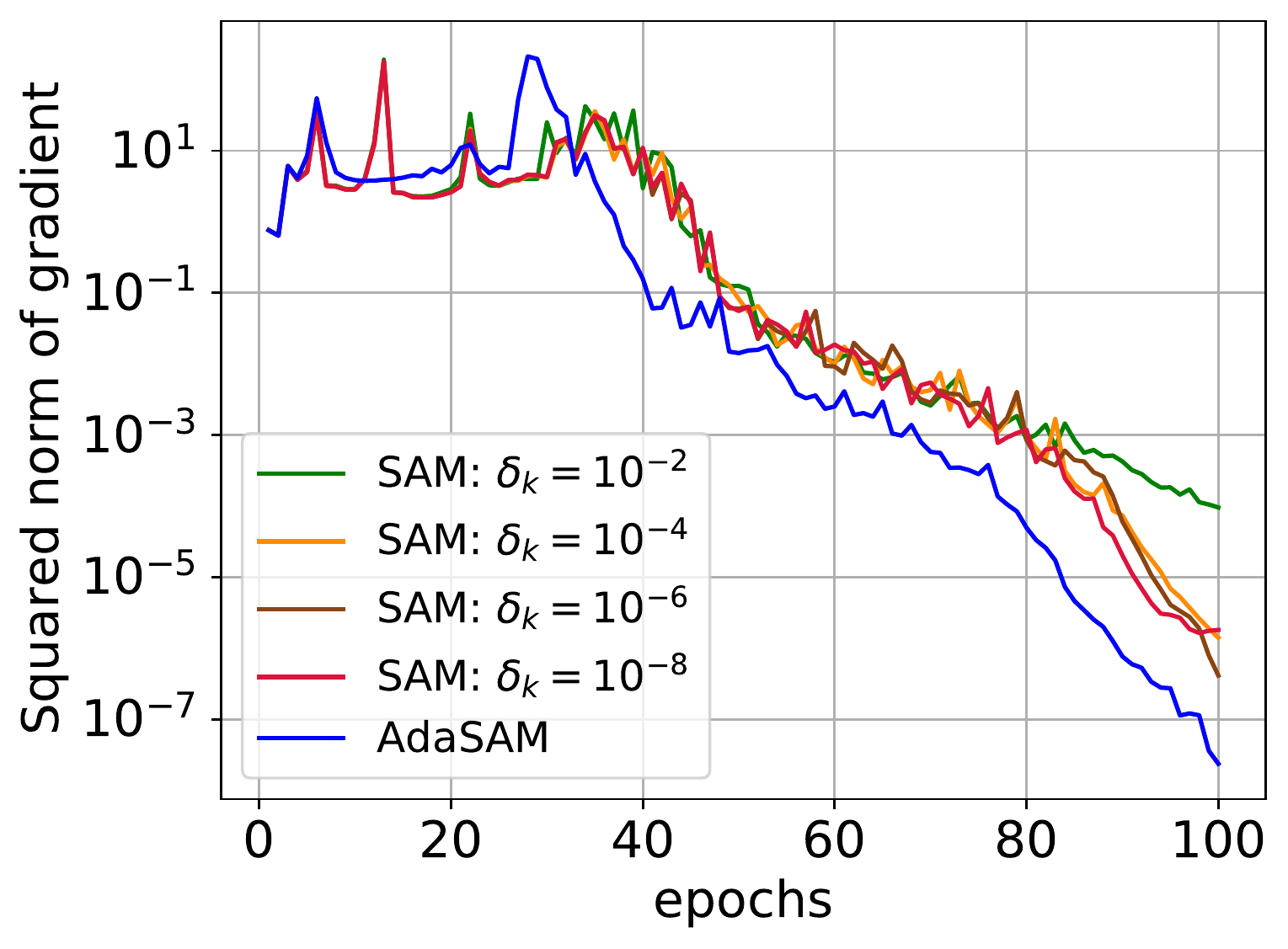}
}
\subfigure[Train accuracy (b=12K)]{
\includegraphics[width=0.23\textwidth]{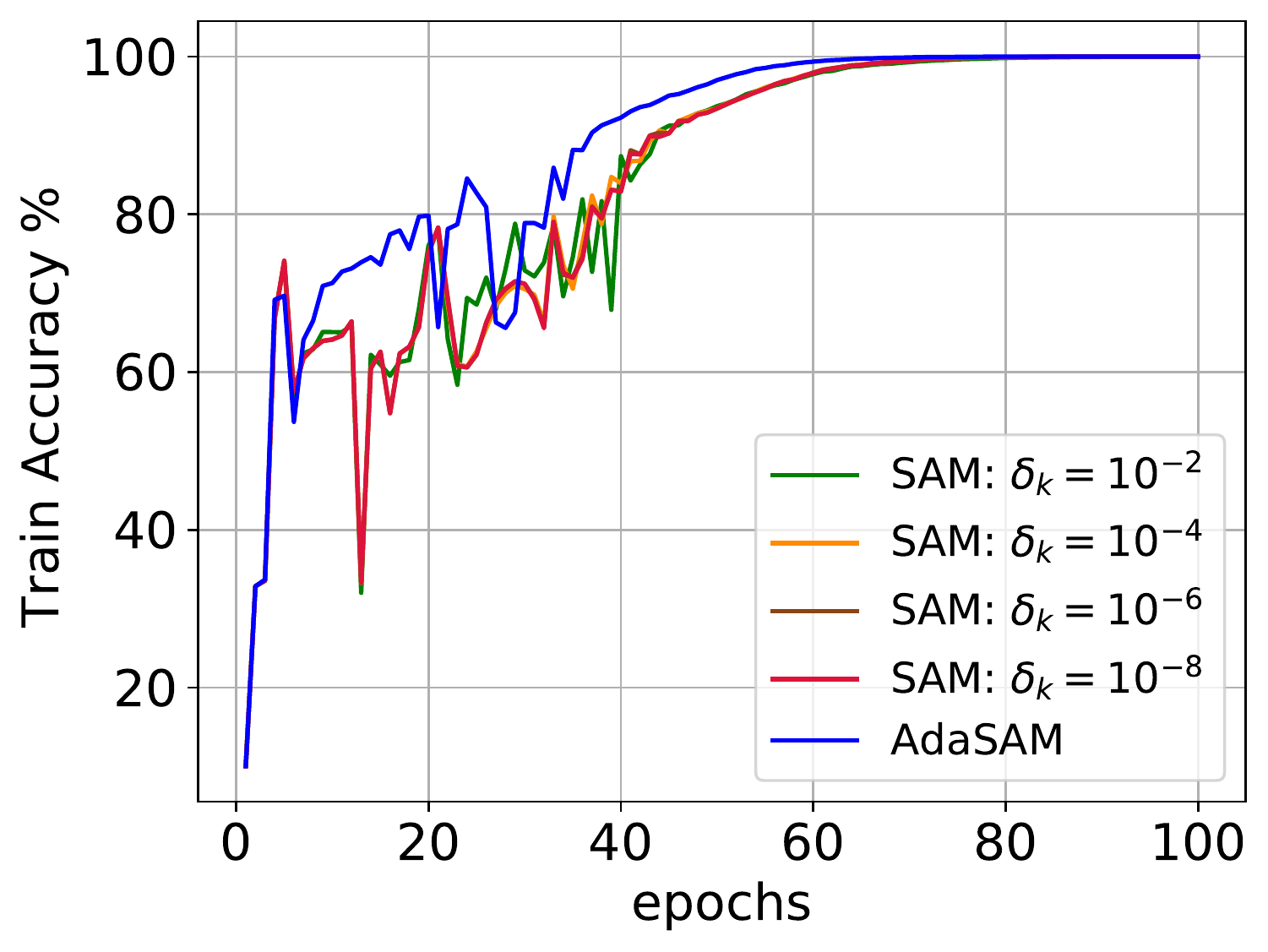}
}
\subfigure[$\delta_k$ (b=12K)]{
\includegraphics[width=0.23\textwidth]{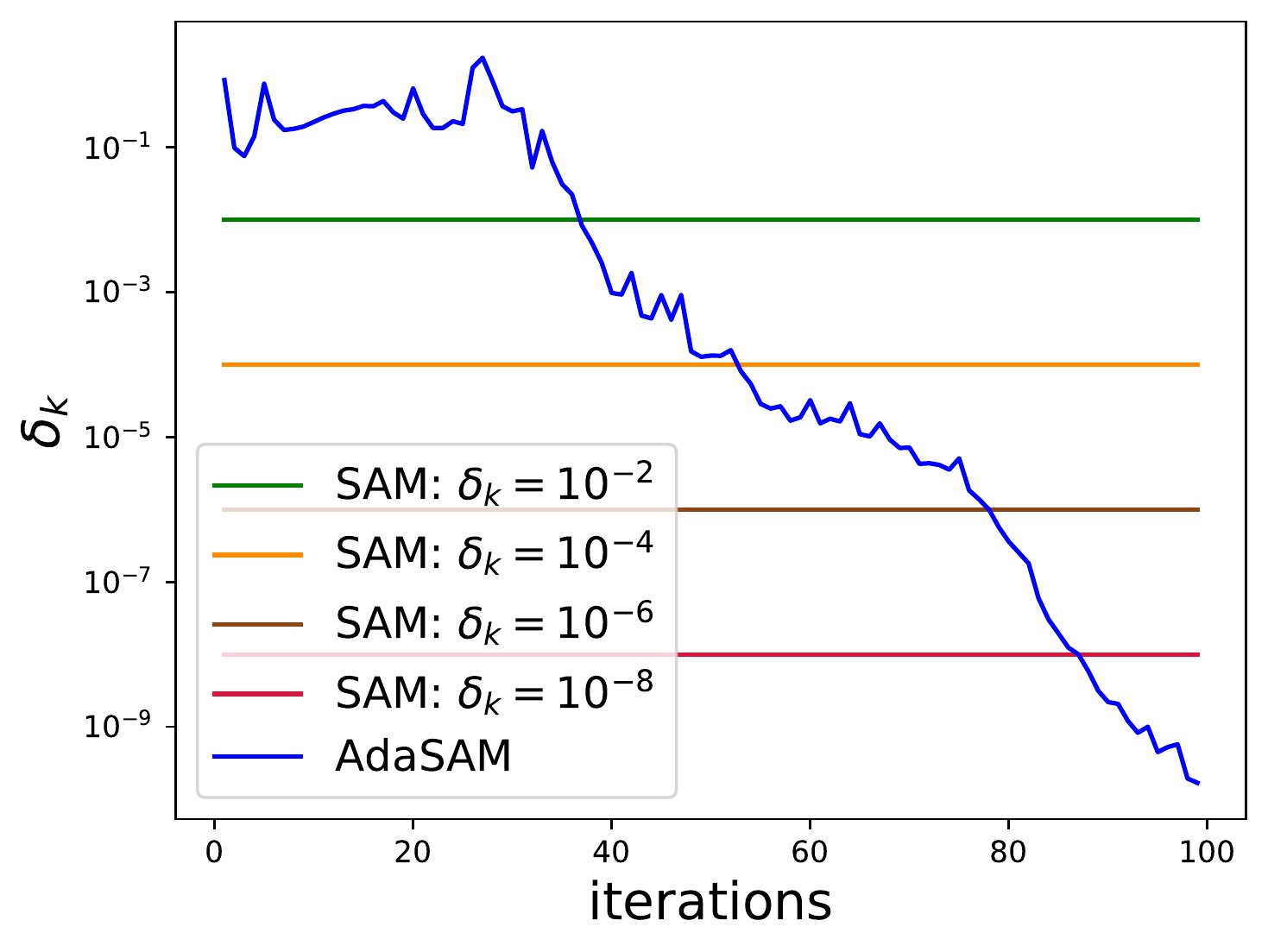}
}
\subfigure[Train Loss (b=3K)]{
\includegraphics[width=0.23\textwidth]{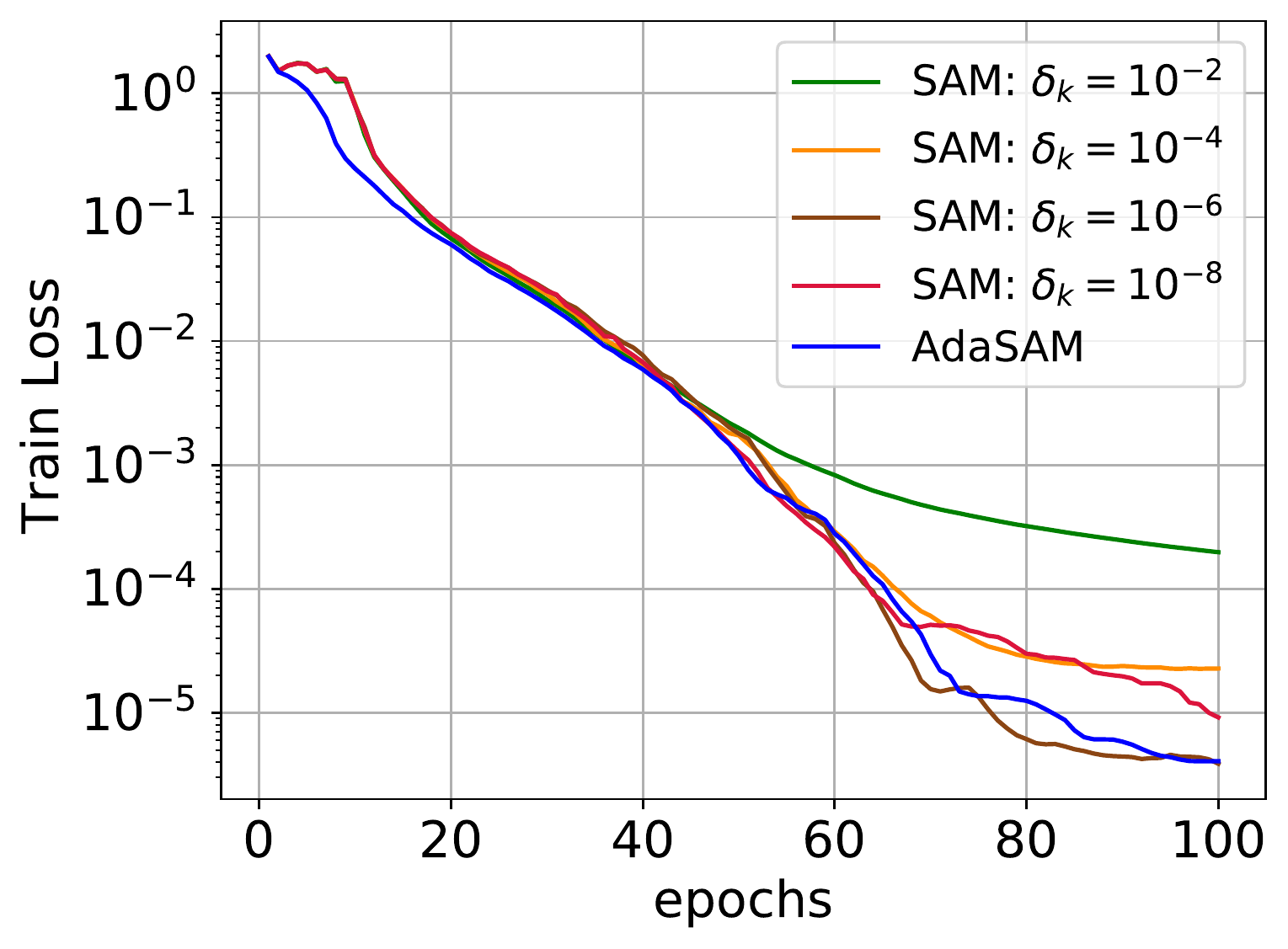}
}
\subfigure[SNG (b=3K)]{
\includegraphics[width=0.23\textwidth]{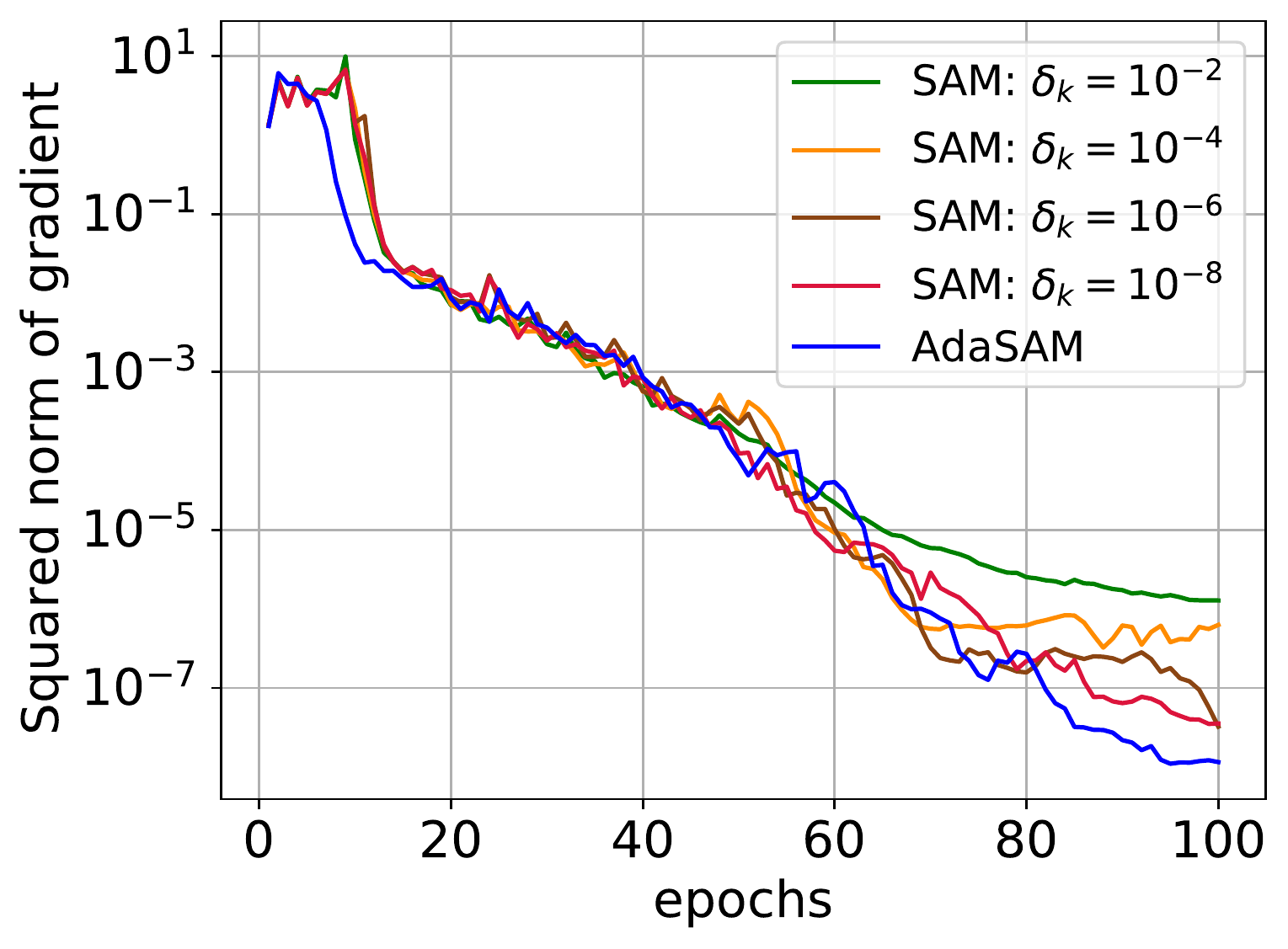}
}
\subfigure[Train accuracy (b=3K)]{
\includegraphics[width=0.23\textwidth]{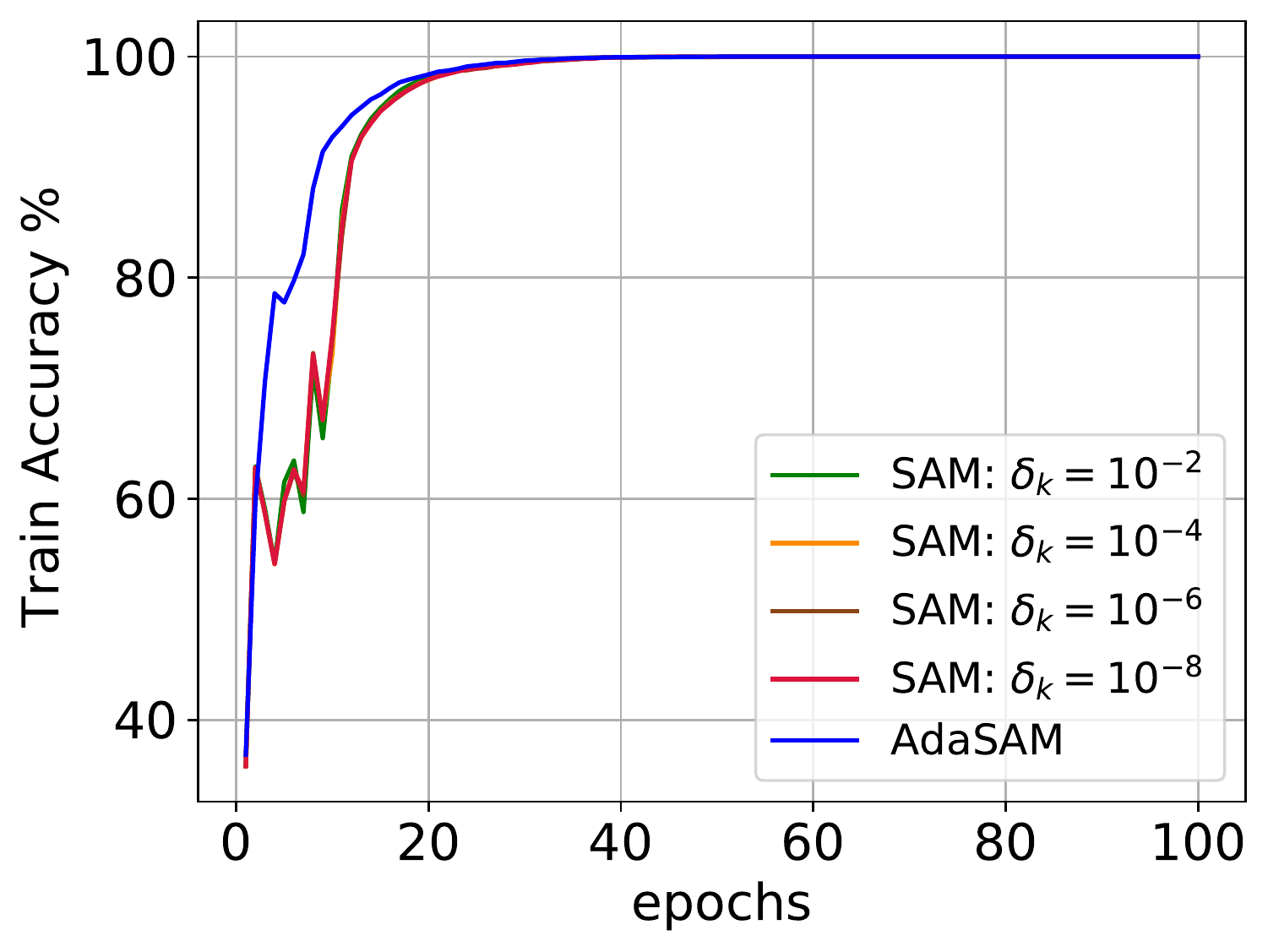}
}
\subfigure[$\delta_k$ (b=3K)]{
\includegraphics[width=0.23\textwidth]{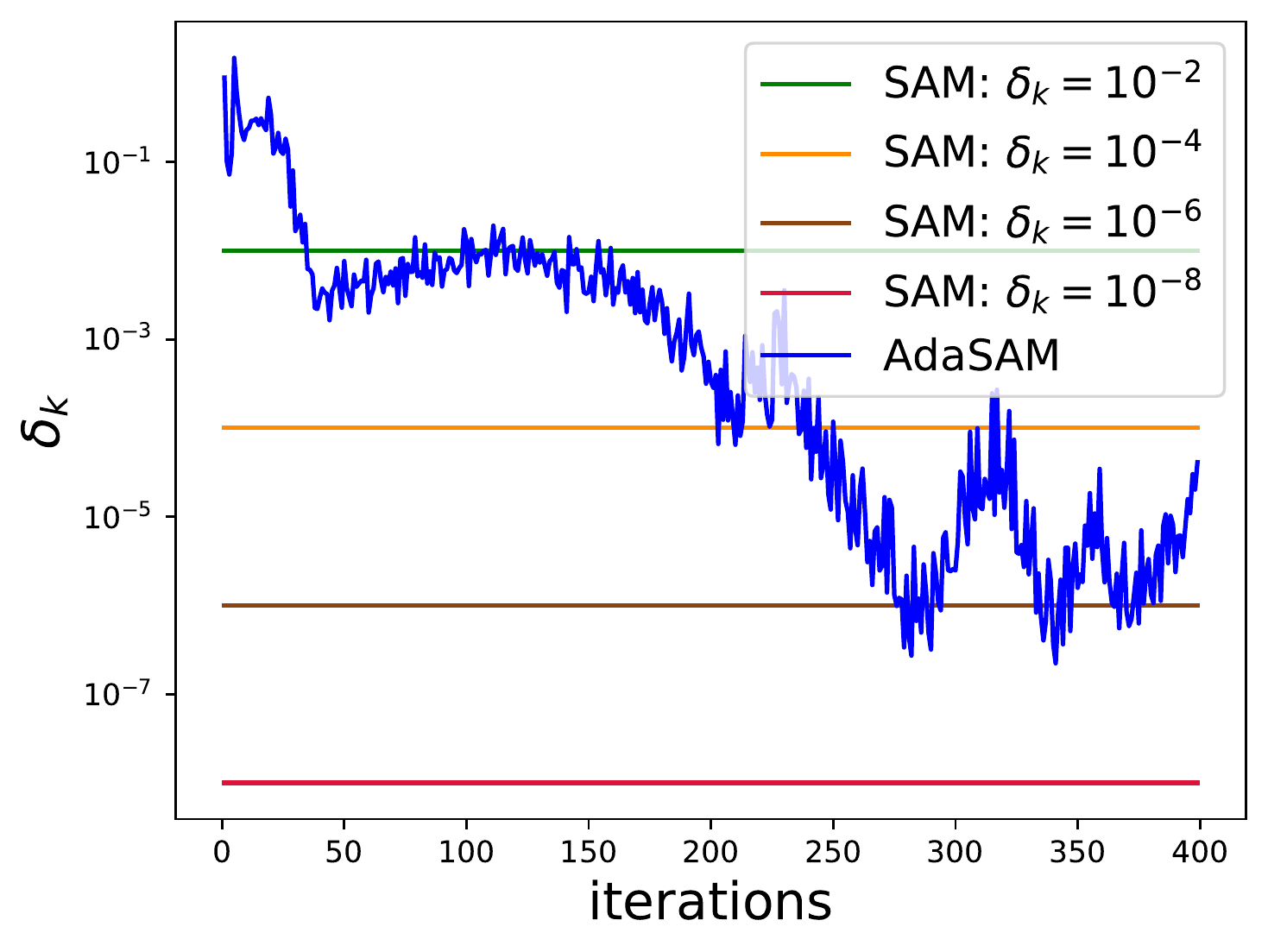}
}
\caption{Experiments on MNIST. SAM$^\dagger $ with $\delta = 10^{-2}, 10^{-4}, 10^{-6}, 10^{-8}$ and AdaSAM. Training loss, training accuracy, squared norm of gradient (abbr. SNG) and $\delta_k$ with batchsize (abbr. b) of 12K, 3K are reported.}
\label{fig:appendix_test_deltak_1}
\end{figure}

 \begin{table*}[ht]
  \caption{Test accuracy on training CIFAR-10/ResNet20. The number on the first row are the regularization parameter $\delta$s.} \label{table:appendix_test_deltak}
  \centering
  \resizebox{\textwidth}{!}{
  \begin{tabular}{lccccccc} 
    \toprule   	
      & $10^{3}$ & $10^{2}$ & $10^{1}$ & 1 & $10^{-1}$ & $10^{-2}$ & $10^{-3}$ \\
    \midrule
    SAM$^{\dagger}$  & 91.57$\pm$.17 & 91.28$\pm$.27 & 91.40$\pm$.03  &  91.63$\pm$.01   & 91.65 $\pm$.24  & 91.60$\pm$.24 & 91.67 $\pm$.31   \\
    SAM$^{\ddagger}$ & 91.48$\pm$.23 & 91.64$\pm$.30 & 91.74$\pm$.20  &  91.91$\pm$.19   & 91.62$\pm$.11   & 91.68$\pm$.34 & 91.48$\pm$.29 
    \\
    \bottomrule
  \end{tabular}
  }
\end{table*} 

 \begin{figure}[ht]
\centering 
\subfigure[Train Accuracy]{
\includegraphics[width=0.31\textwidth]{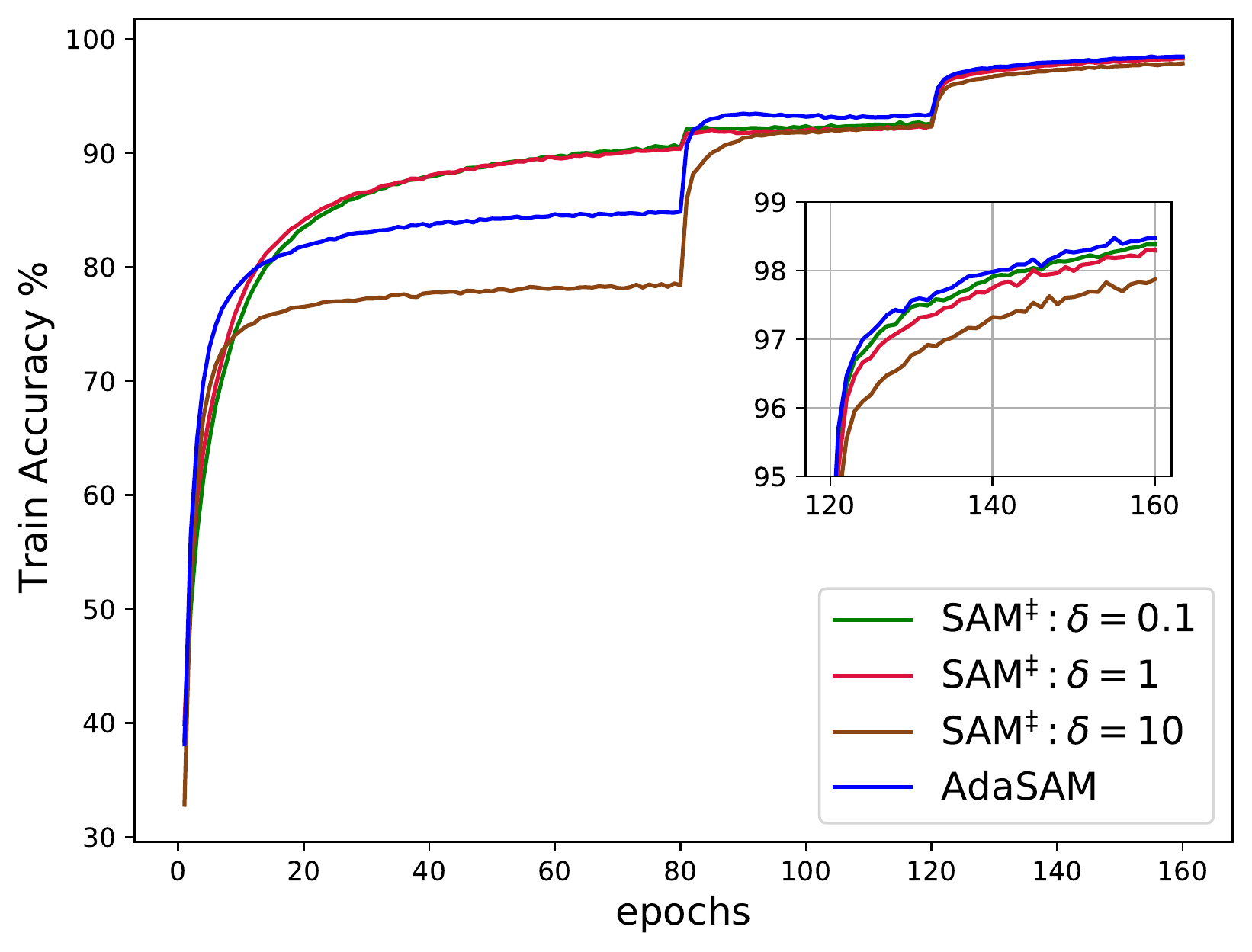}
}
\subfigure[Test Accuracy]{
\includegraphics[width=0.31\textwidth]{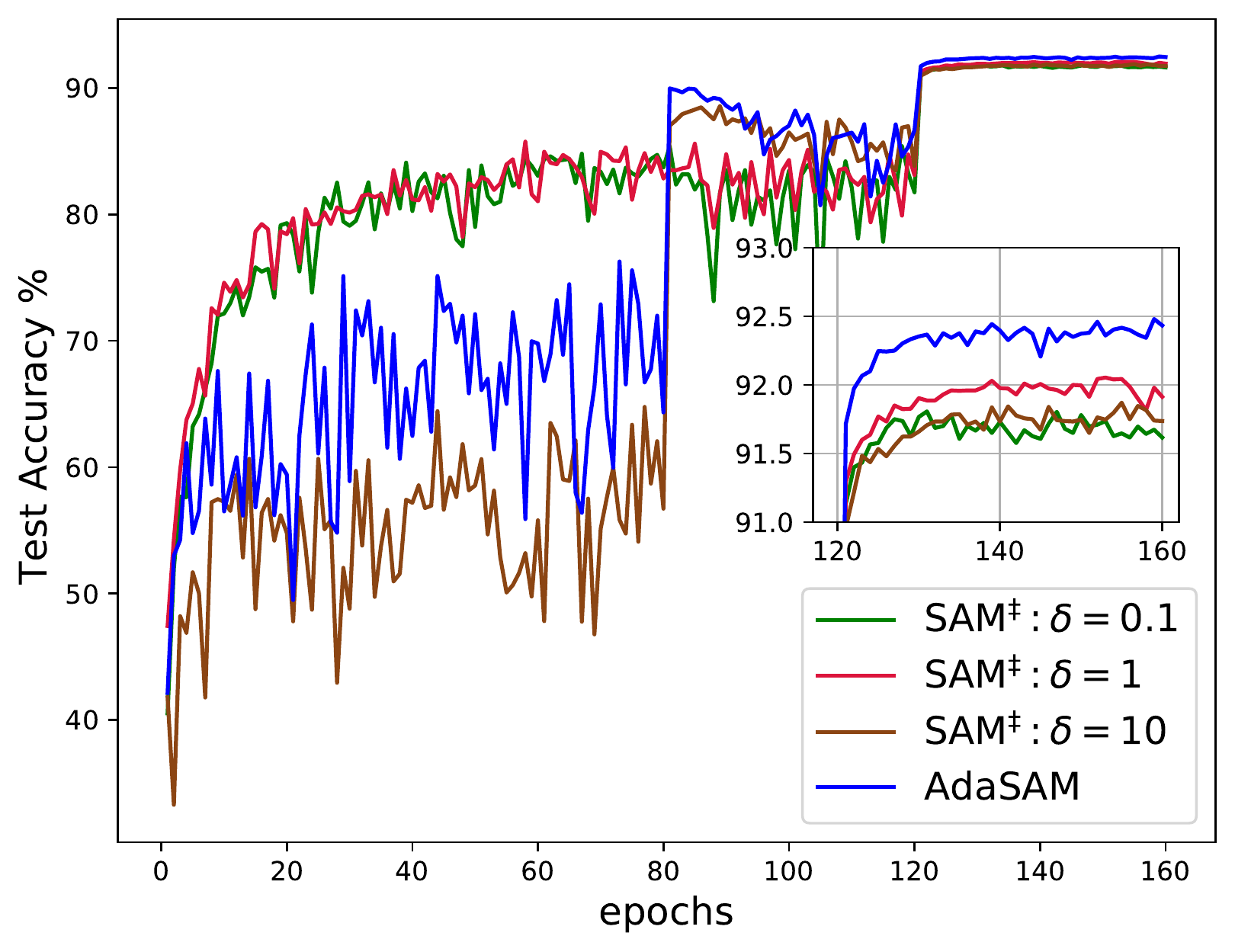}
}
\subfigure[$\delta_k$]{
\includegraphics[width=0.31\textwidth]{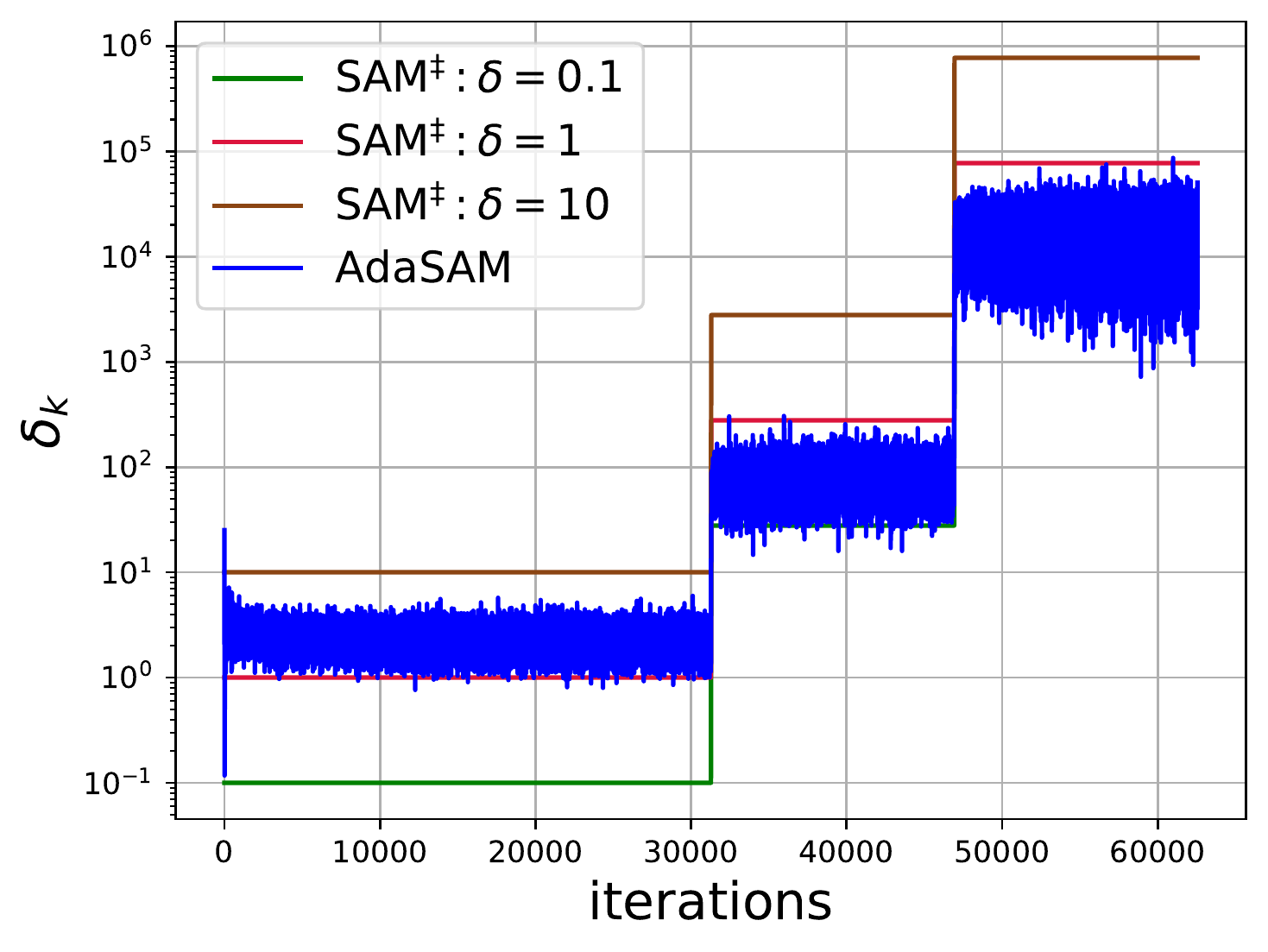}
}
\caption{Experiments on CIFAT-10/ResNet20. Training with SAM$^{\ddagger}$ with $ \delta =$ 0.1, 1, 10. (c) shows the evolution of $ \delta_k $  in SAM$^{\ddagger}$ and AdaSAM during training.}
\label{fig:appendix_test_deltak_2}
\end{figure}

 Experimental results on MNIST when training CNN with batch size of 12K, 3K are reported in Figure~\ref{fig:appendix_test_deltak_1}.  Note that $c_1 = 10^{-4}$ is unchanged across two tests of different batch sizes. We  see AdaSAM adaptively adjusts $ \delta_k $ during training and always achieves the best result. On the contrary, the proper $\delta$ for SAM$^{\dagger} $ is dependent on the batch size.
 
 For the tests on CIFAR-10/ResNet20, we made considerable efforts to tune $\delta$ in SAM$^\dagger$/SAM$^\ddagger$. The results corresponding to different $ \delta $s are shown in Table~\ref{table:appendix_test_deltak}. We also plot related curves of SAM$^{\ddagger} $ and AdaSAM in Figure~\ref{fig:appendix_test_deltak_2}, from which we see the $ \delta_k $ determined in Line~12 in Algorithm~\ref{alg:adasam} roughly matches the scheme of SAM$^{\ddagger}$, i.e. $ \delta_k \geq C \beta_k^{-2} $ for some constant $ C>0$ , thus conforming our heuristic analysis about the convergence of AdaSAM in Section~\ref{anal2}. Observed from Figure~\ref{fig:appendix_test_deltak_2}(c), we set $ \delta = 0.5 $ in SAM$^{\ddagger}$ to roughly match the evolution of $ \delta_k$ in AdaSAM and obtain a slightly better test accuracy 92.05\%.
  These results demonstrate the effectiveness of our choice of $ \delta_k $ in AdaSAM.

\subsection{Moving average}
 For our implementation Algorithm~\ref{alg:adasam} and Algorithm~\ref{alg:padasam}, we incorporate moving average as an option. 
 In deterministic quadratic optimization, the minimal residual property still holds since the relation $ \hat{R}_k = -\nabla f(x_k)\hat{X}_k $ is maintained. In general stochastic optimization, 
 We  find moving average may enhance the robustness to noise or generalization ability.
  
 \begin{figure}[ht]
\centering 
\subfigure[Train Loss]{
\includegraphics[width=0.23\textwidth]{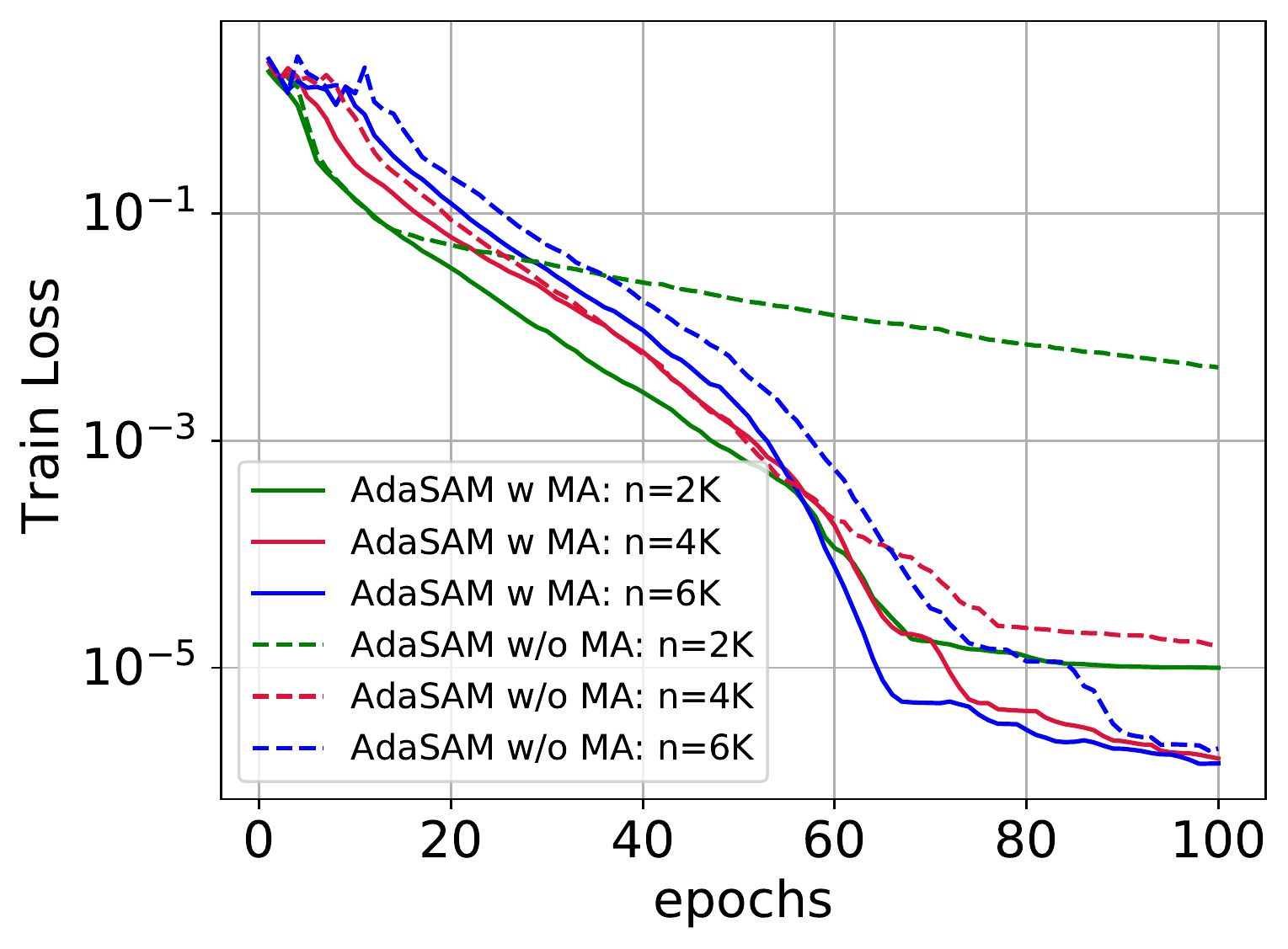}
}
\subfigure[SNG]{
\includegraphics[width=0.23\textwidth]{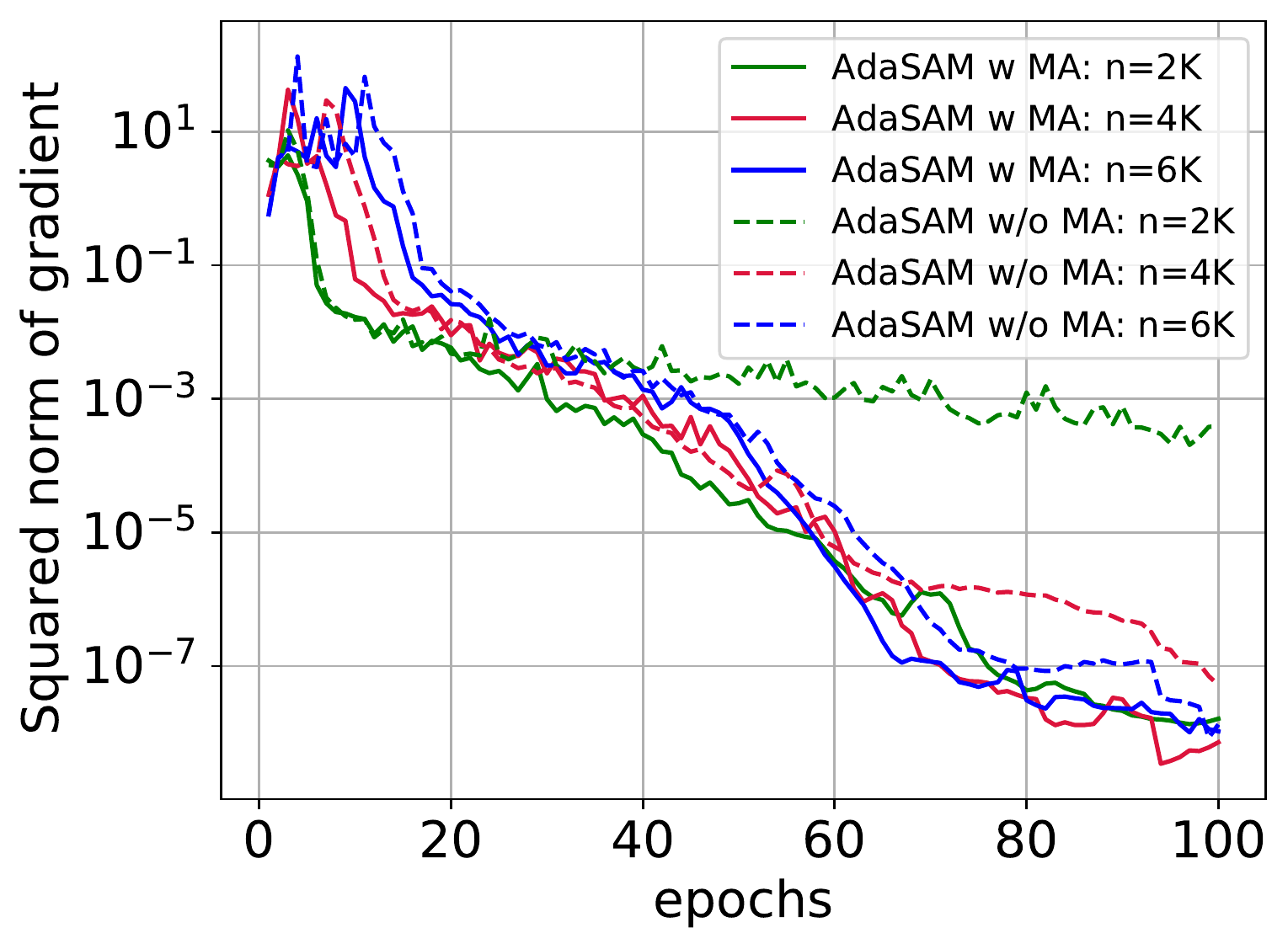}
}
\subfigure[Train Loss]{
\includegraphics[width=0.23\textwidth]{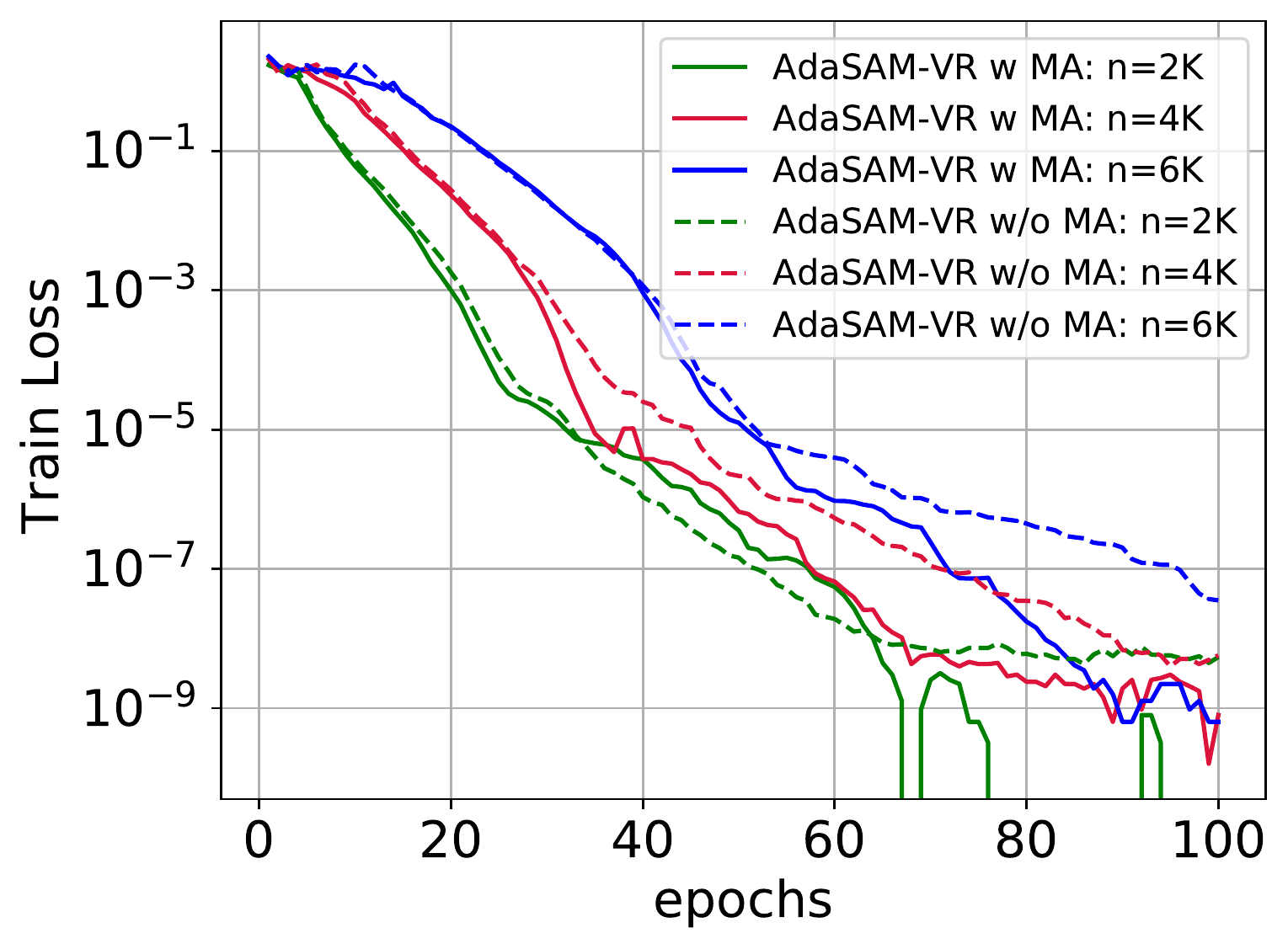}
}
\subfigure[SNG]{
\includegraphics[width=0.23\textwidth]{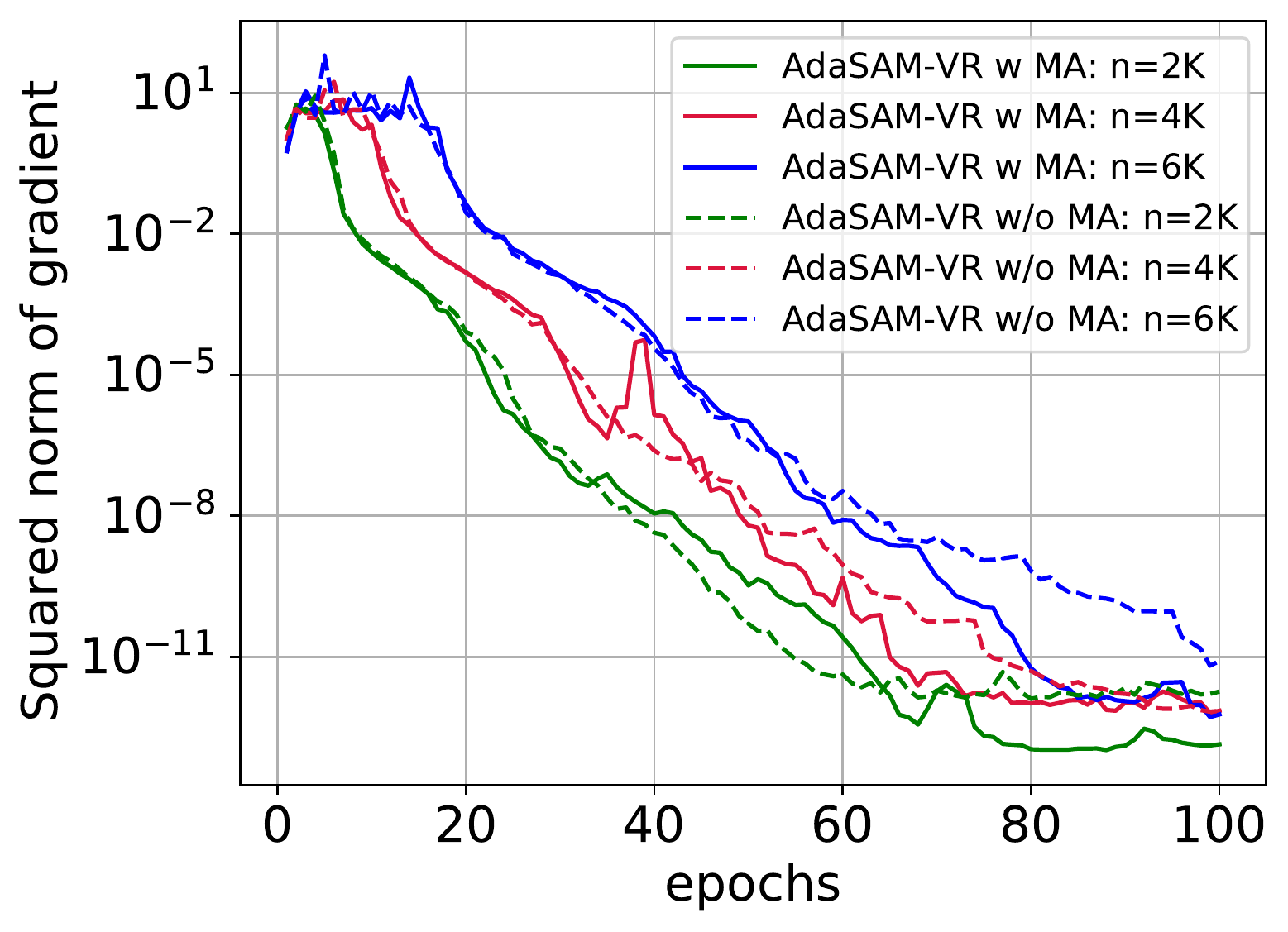}
}
\caption{Experiments on MNIST. (a)(b) Training loss and square norm of gradient (abbr. SNG) of AdaSAM with/without moving average (abbr. MA). (c)(d) Training loss and SNG of AdaSAM-VR with/without MA. Batch size $n = $2K, 4K, 6K.}
\label{fig:appendix_test_movavg_1}
\end{figure}
 
 In Figure~\ref{fig:appendix_test_movavg_1}, we report AdaSAM/AdaSAM-VR with/without moving average for mini-batch training on MNIST. Figure~\ref{fig:appendix_test_movavg_1}(a) indicates that AdaSAM without moving average stagnates when batchsize is 2K due to noise in gradient estimates. By incorporating variance reduction, AdaSAM-VR without moving average recovers the fast convergence rate. From this example, we conclude that moving average may help reduce the variability in gradient estimates and improve convergence.
 
 We also reran the experiments on CIFAR-10/CIFAR-100 to see the effect of moving average. Results are reported in Table~\ref{table:appendix_test_movavg_1} and plotted in Figure~\ref{fig:appendix_test_mvavg_2}. There seems to be no significant differences judging from final test accuracy, while AdaSAM without moving average can be faster at the beginning as indicated from Figure~\ref{fig:appendix_test_mvavg_2}.
 
 We reran the experiments on Penn TreeBank. Results are shown in  Table~\ref{table:appendix_test_mvavg_2} and Figure~\ref{fig:appendix_test_mvavg_3}. Similar to the phenomenon on CIFARs,   pAdaSAM without moving average converges faster at the beginning. However, its final validation perplexity and test perplexity is slightly suboptimal compared with pAdaSAM with moving average. 
 
 With these experimental results, we think although moving average is not needed in our theoretical analysis, it may be beneficial in stabilizing the training or improving generalization ability.
 
 \begin{table*}[ht]
\centering
\caption{ Experiments on CIFAR10/CIFAR100. WideResNet is abbreviated as WResNet. }
\label{table:appendix_test_movavg_1}
\resizebox{\textwidth}{!}{
\begin{tabular}{l c c c c c c c c c}
\toprule
 \multirow{2}{*}{Method} & \multicolumn{6}{c}{CIFAR10} & \multicolumn{3}{c}{CIFAR100}\\
 \cmidrule(lr){2-7} \cmidrule(lr){8-10}
 & ResNet18 & ResNet20 & ResNet32 & ResNet44 & ResNet56 & WResNet & ResNet18  & ResNeXt & DenseNet \\
 \cmidrule(lr){2-7} \cmidrule(lr){8-10}
AdaSAM w MA & 95.17$\pm$.10 & 92.43$\pm$.19 & 93.22$\pm$.32 & 93.57$\pm$.14 & 93.77$\pm$.12 & 95.23$\pm$.07 & 78.13$\pm$.14 & 79.31$\pm$.27 & 80.09$\pm$.52 \\
AdaSAM w/o MA & 95.22$\pm$.13 & 92.52$\pm$.09 & 93.08$\pm$.22 & 93.62$\pm$.05 & 93.89$\pm$.16 &  95.16$\pm$.04 & 78.09$\pm$.27 & 79.57$\pm$.21 & 80.03$\pm$.25 \\
\bottomrule
\end{tabular}}
\end{table*}

 \begin{figure}[ht]
\centering 
\subfigure[Train (ResNet18)]{
\includegraphics[width=0.23\textwidth]{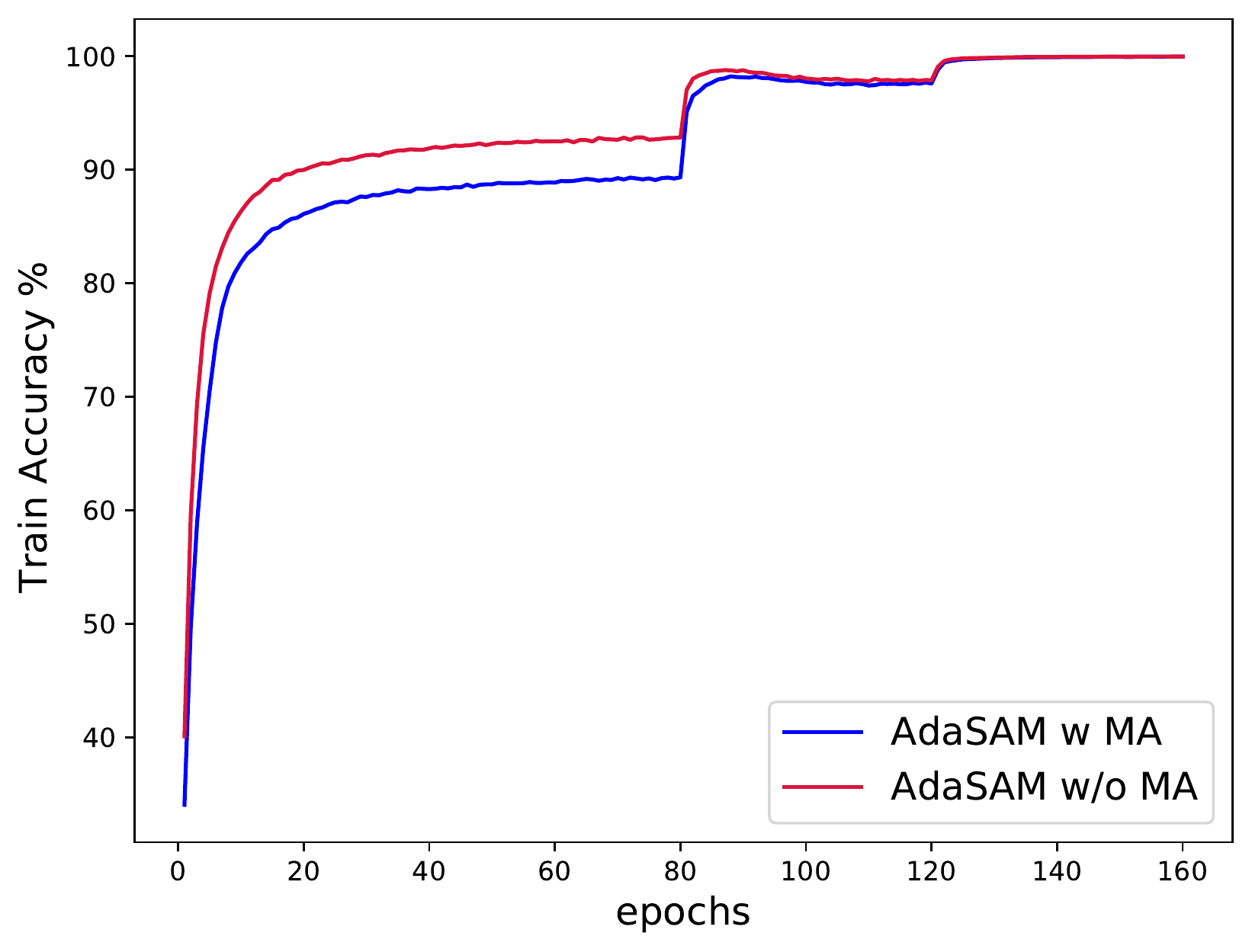}
}
\subfigure[Test (ResNet18)]{
\includegraphics[width=0.23\textwidth]{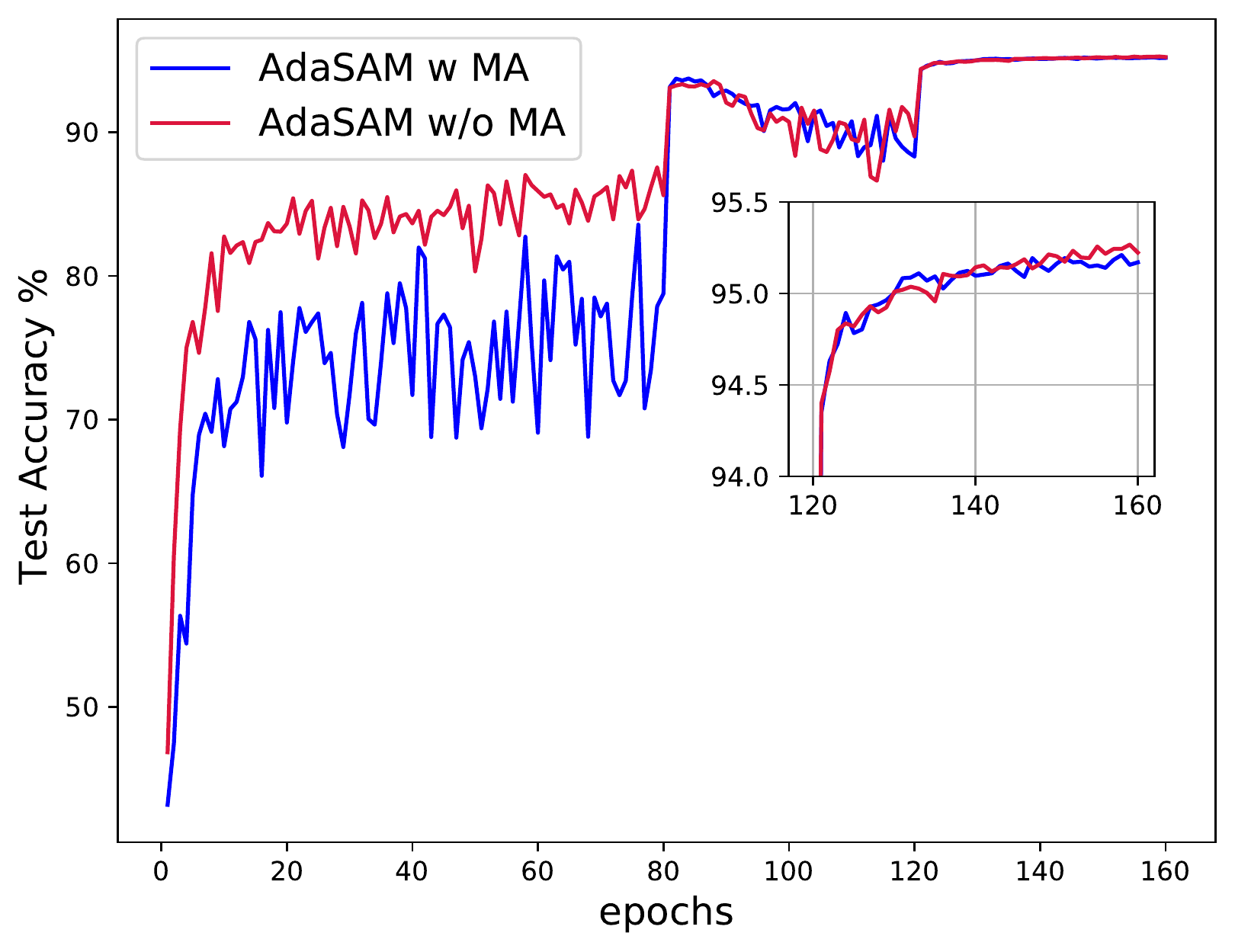}
}
\subfigure[Train (WideResNet) ]{
\includegraphics[width=0.23\textwidth]{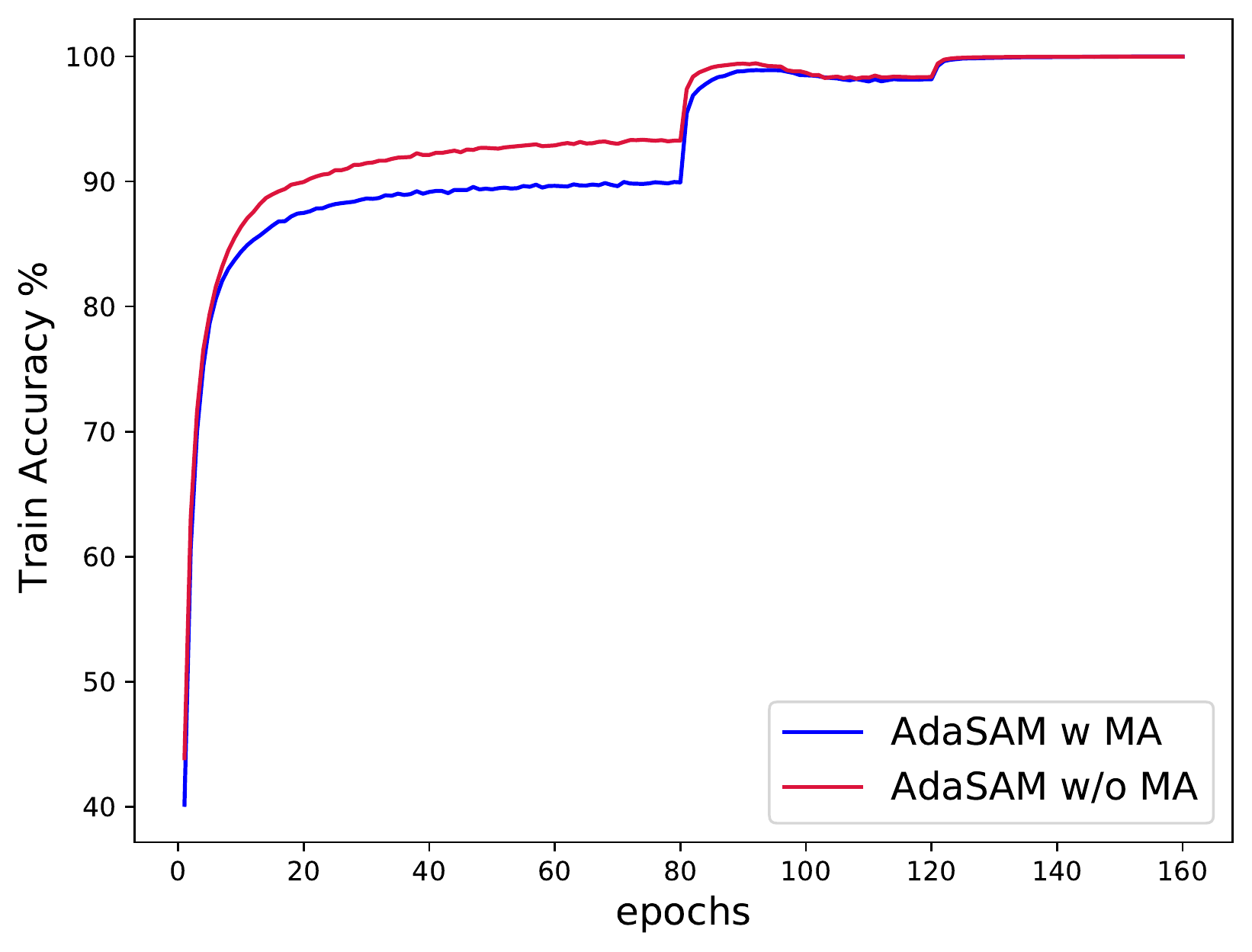}
}
\subfigure[Test (WideResNet)]{
\includegraphics[width=0.23\textwidth]{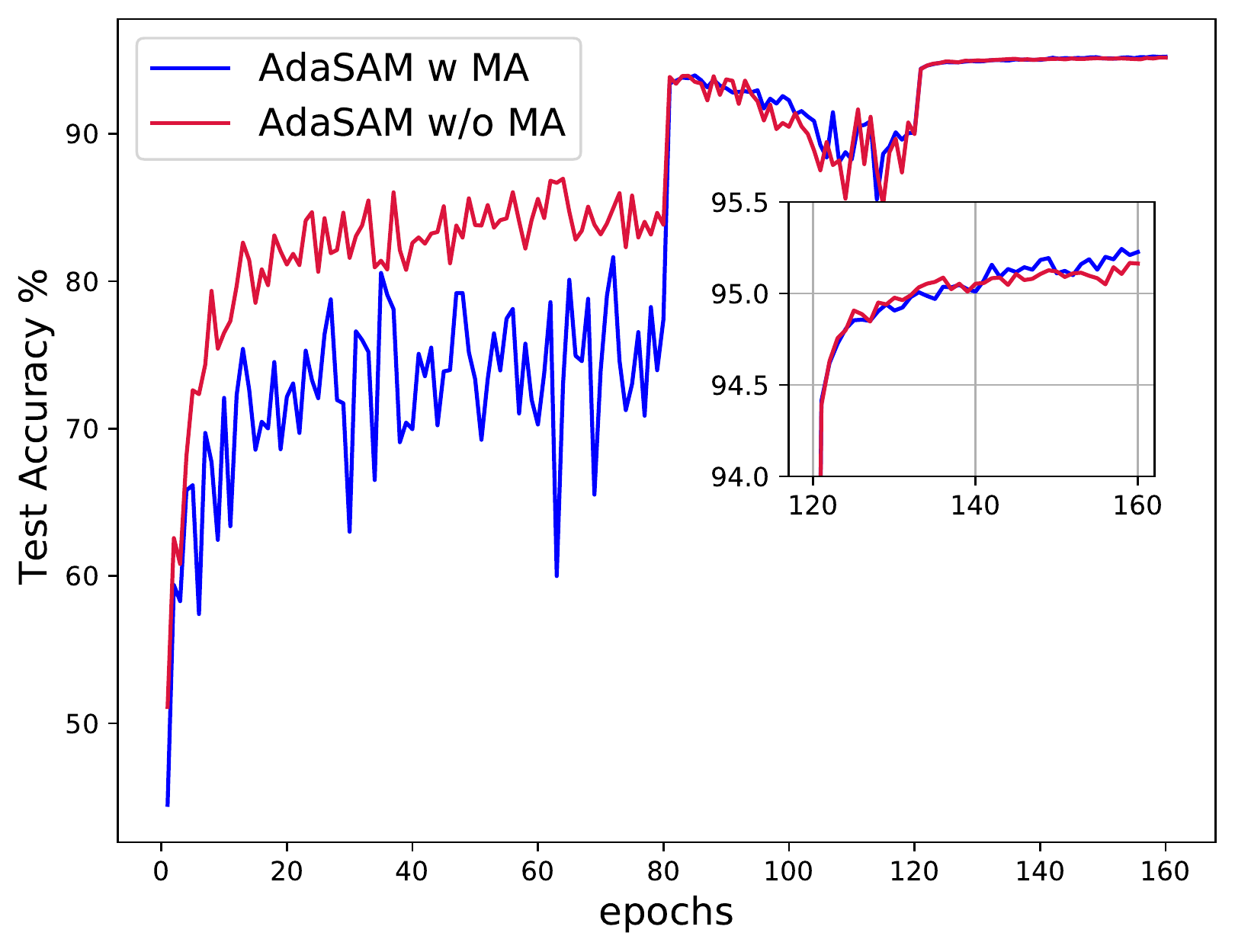}
}
\subfigure[Train (ResNeXt)]{
\includegraphics[width=0.23\textwidth]{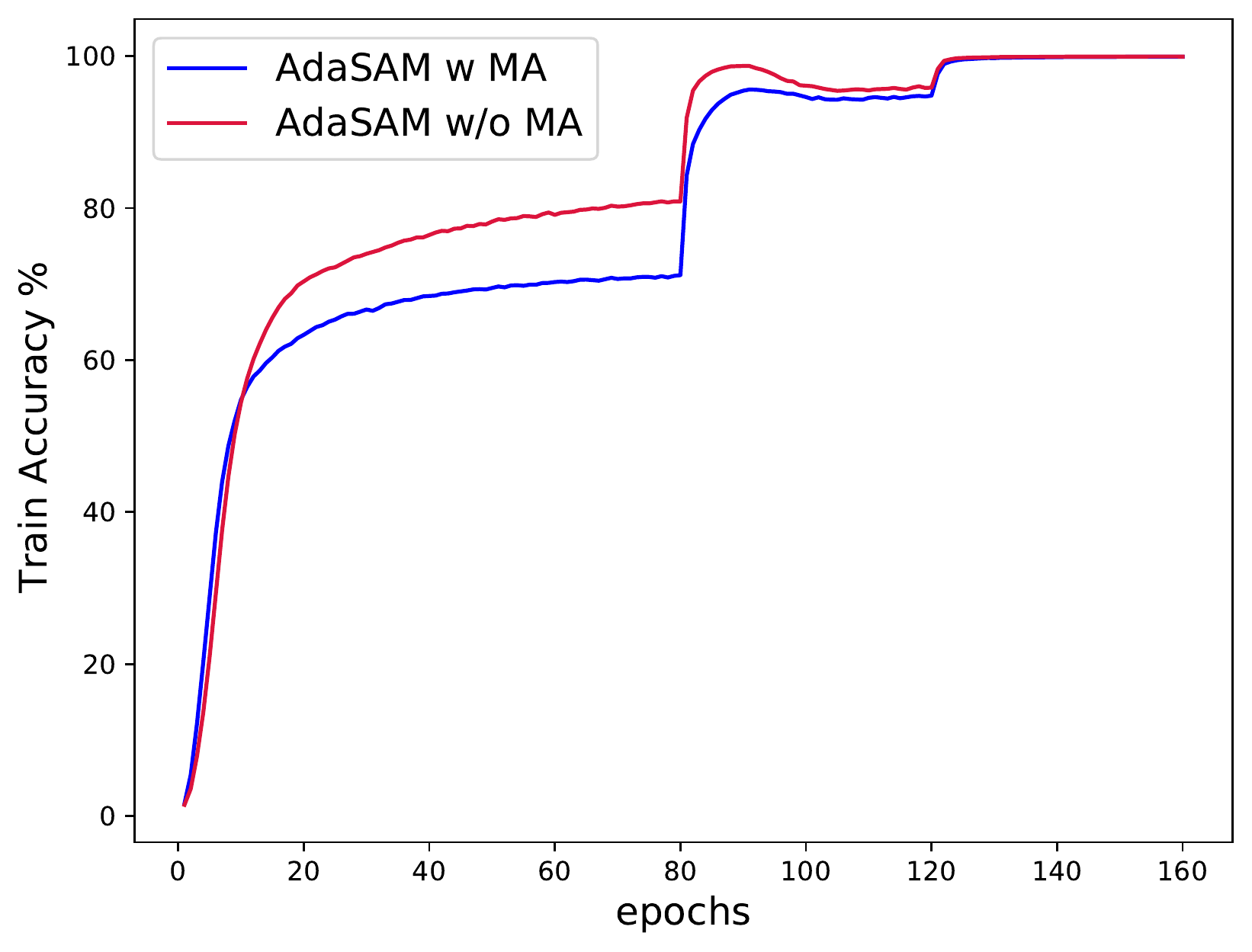}
}
\subfigure[Test (ResNeXt)]{
\includegraphics[width=0.23\textwidth]{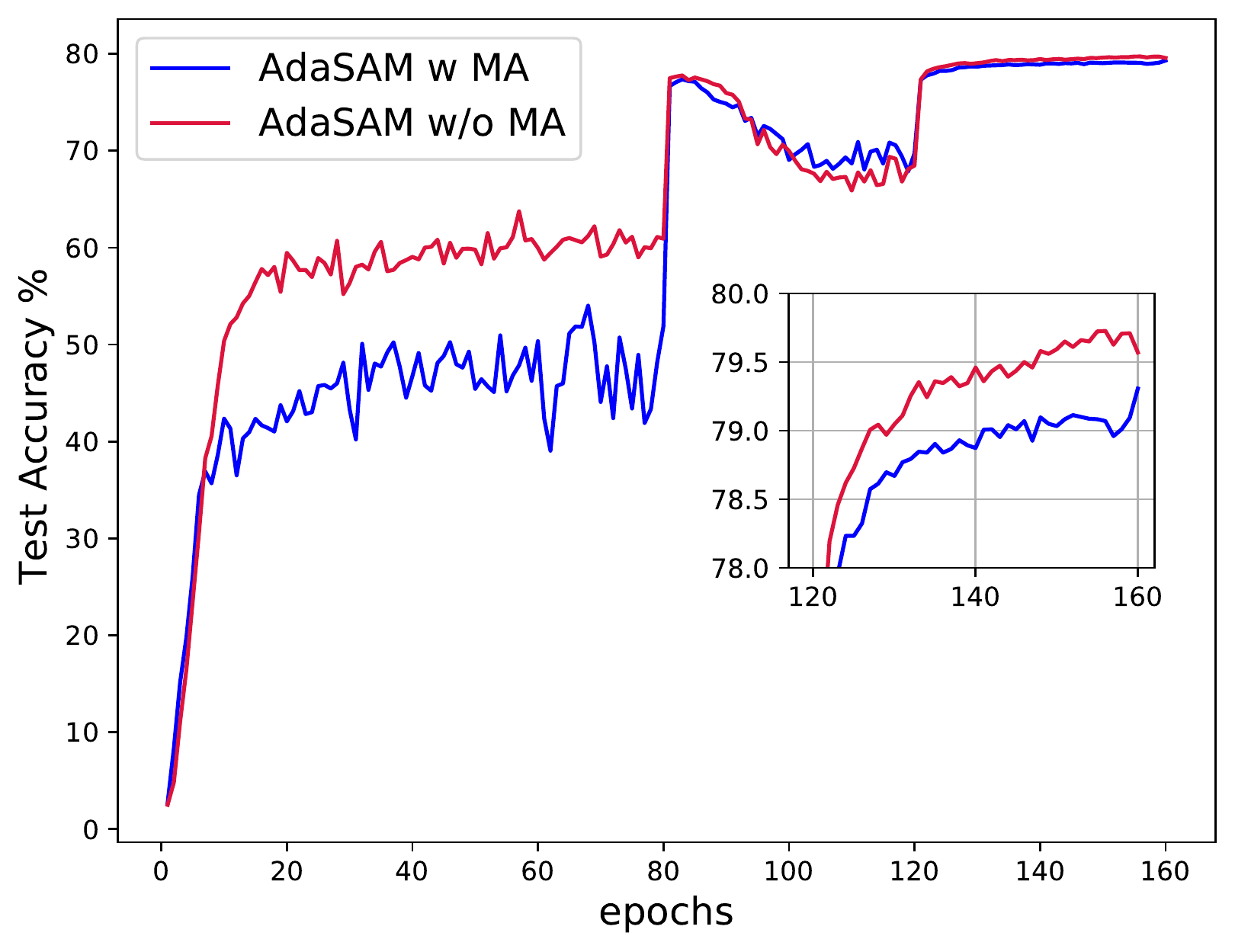}
}
\subfigure[Train (DenseNet)]{
\includegraphics[width=0.23\textwidth]{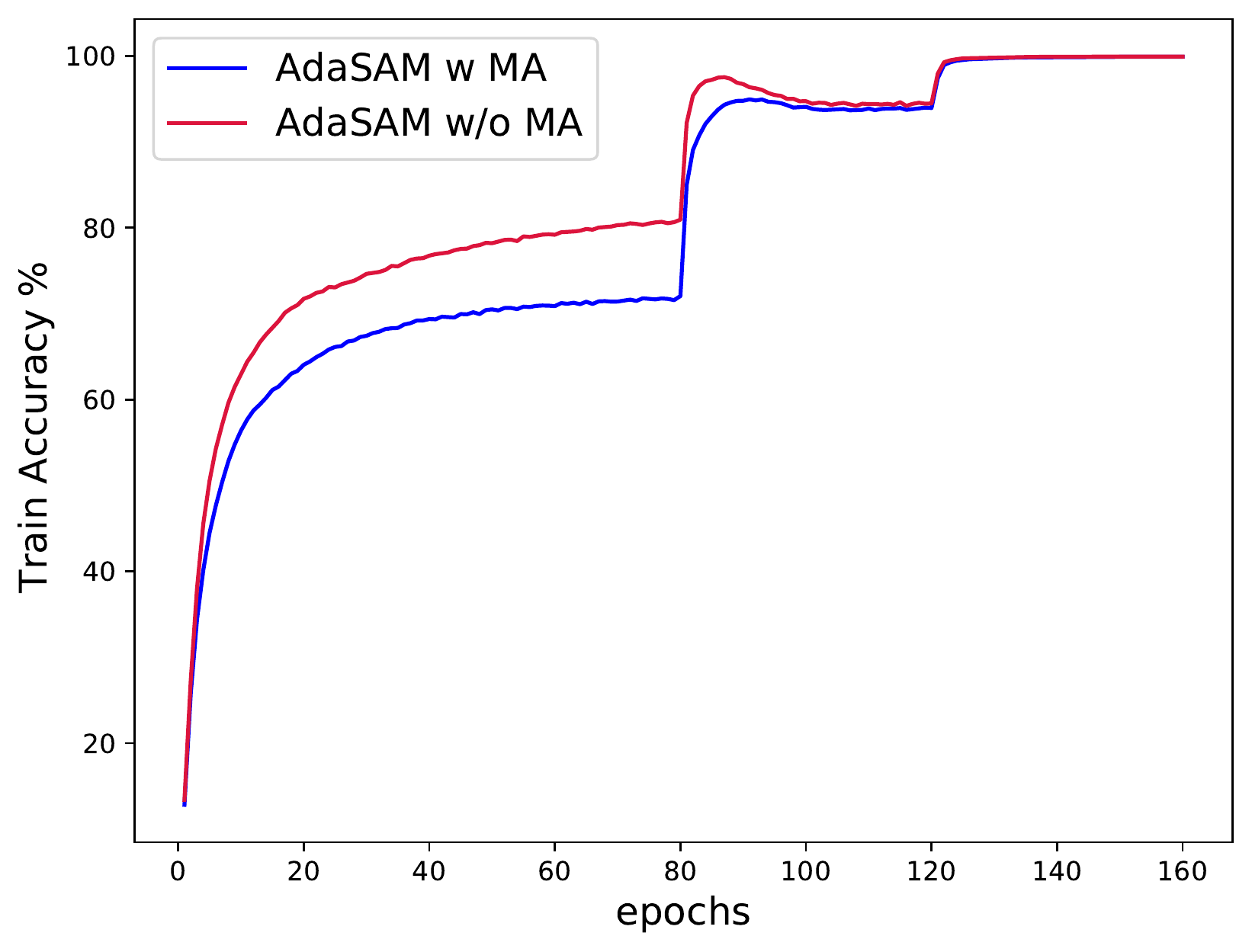}
}
\subfigure[Test (DenseNet)]{
\includegraphics[width=0.23\textwidth]{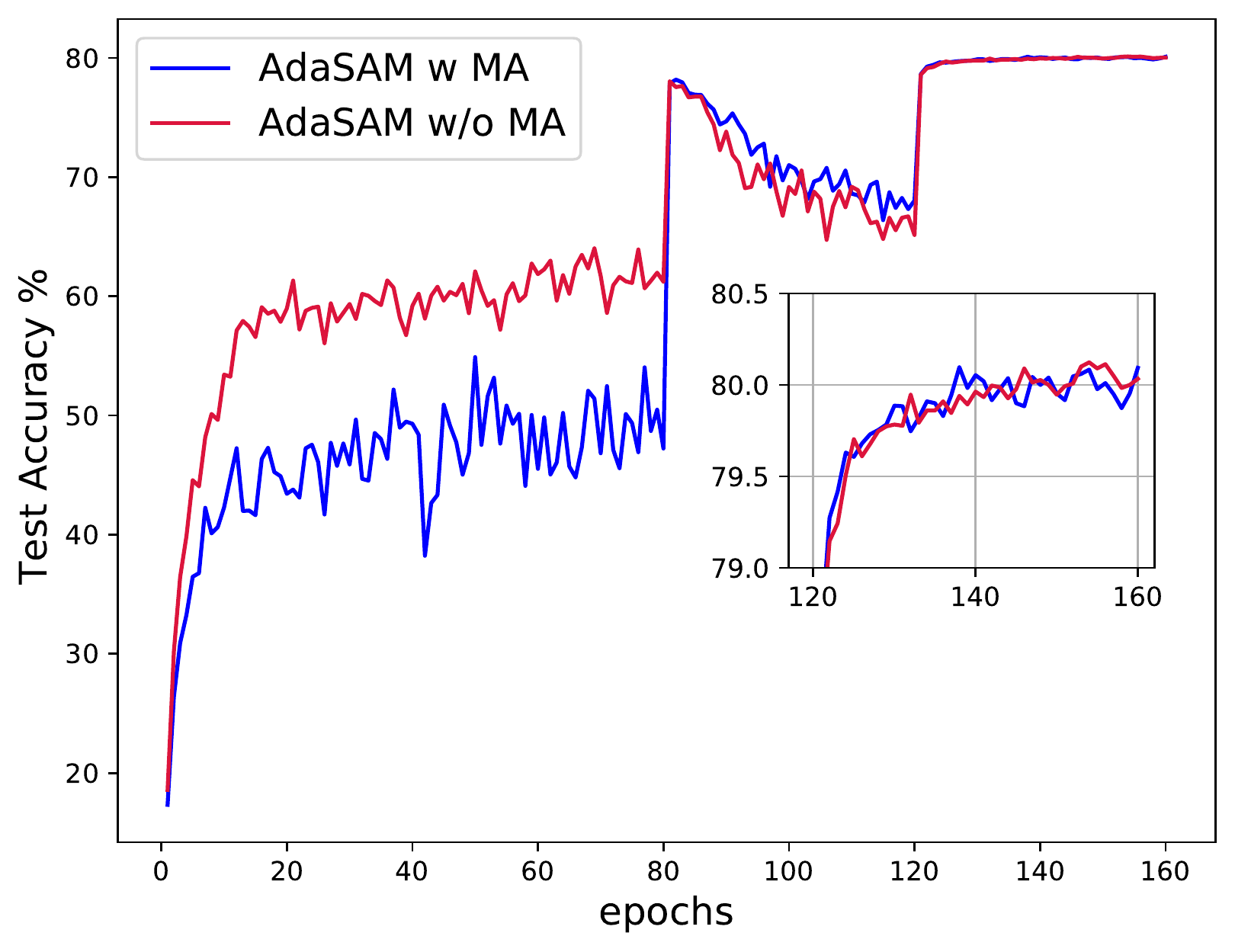}
}
\caption{Experiments on CIFARs. Training CIFAR-10/ResNet18, CIFAR-10/WideResNet16-4, CIFAR-100/ResNeXt50, and CIFAR-100/DenseNet121 using AdaSAM with moving average (abbr. MA) or AdaSAM without moving average. Curves of training accuracy and test accuracy are reported.}
\label{fig:appendix_test_mvavg_2}
\end{figure}
 
 \begin{table*}[ht]
  \caption{Test perplexity on Penn TreeBank for 1,2,3-layer LSTM. Comparison between pAdaSAM with moving average (abbr. MA) and pAdaSAM without MA.} \label{table:appendix_test_mvavg_2}
  \centering
  \begin{tabular}{lccc} 
    \toprule   	
    Method     & 1-Layer     & 2-Layer & 3-Layer \\
    \midrule
    pAdaSAM w/o MA	     & 80.27$\pm$.09   & 64.74$\pm$.02    & 59.72$\pm$.05 \\
    pAdaSAM w MA     & 79.34$\pm$.09    & 63.18$\pm$.22   & 59.47$\pm$.08 \\
    \bottomrule
  \end{tabular}
\end{table*}  
 
 \begin{figure}[H]
\centering 
\subfigure[1-Layer LSTM]{
\includegraphics[width=0.31\textwidth]{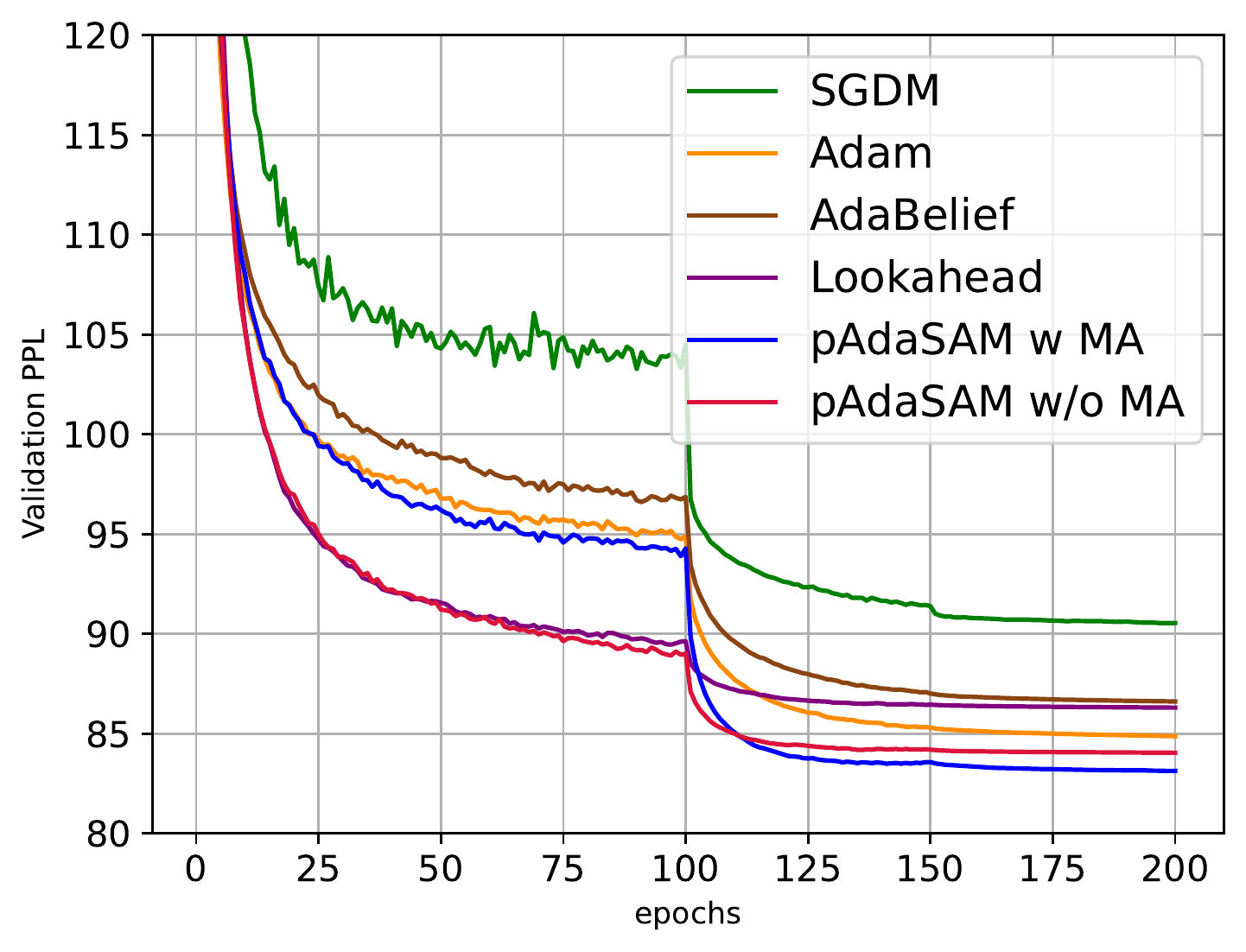}
}
\subfigure[2-Layer LSTM]{
\includegraphics[width=0.31\textwidth]{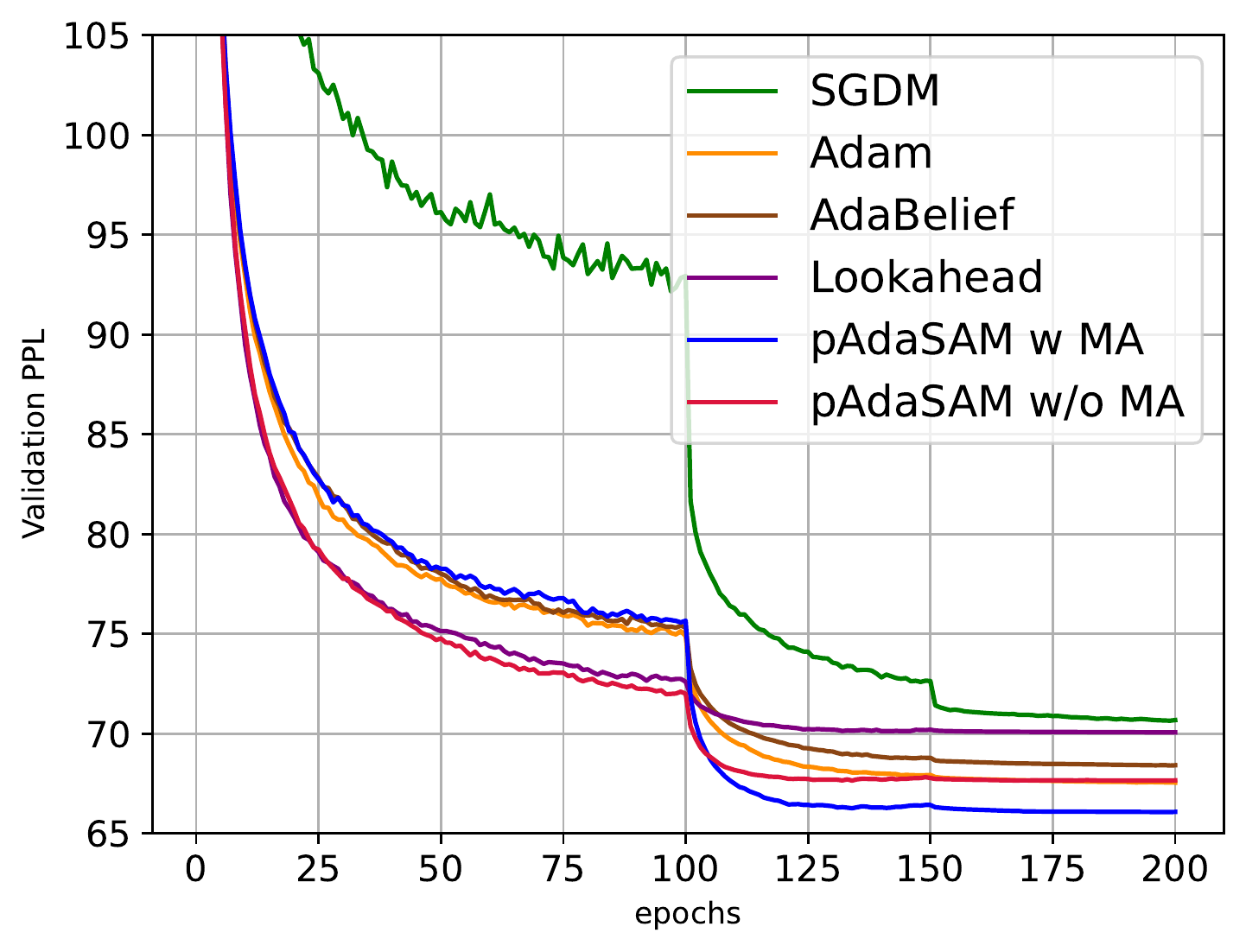}
}
\subfigure[3-Layer LSTM]{
\includegraphics[width=0.31\textwidth]{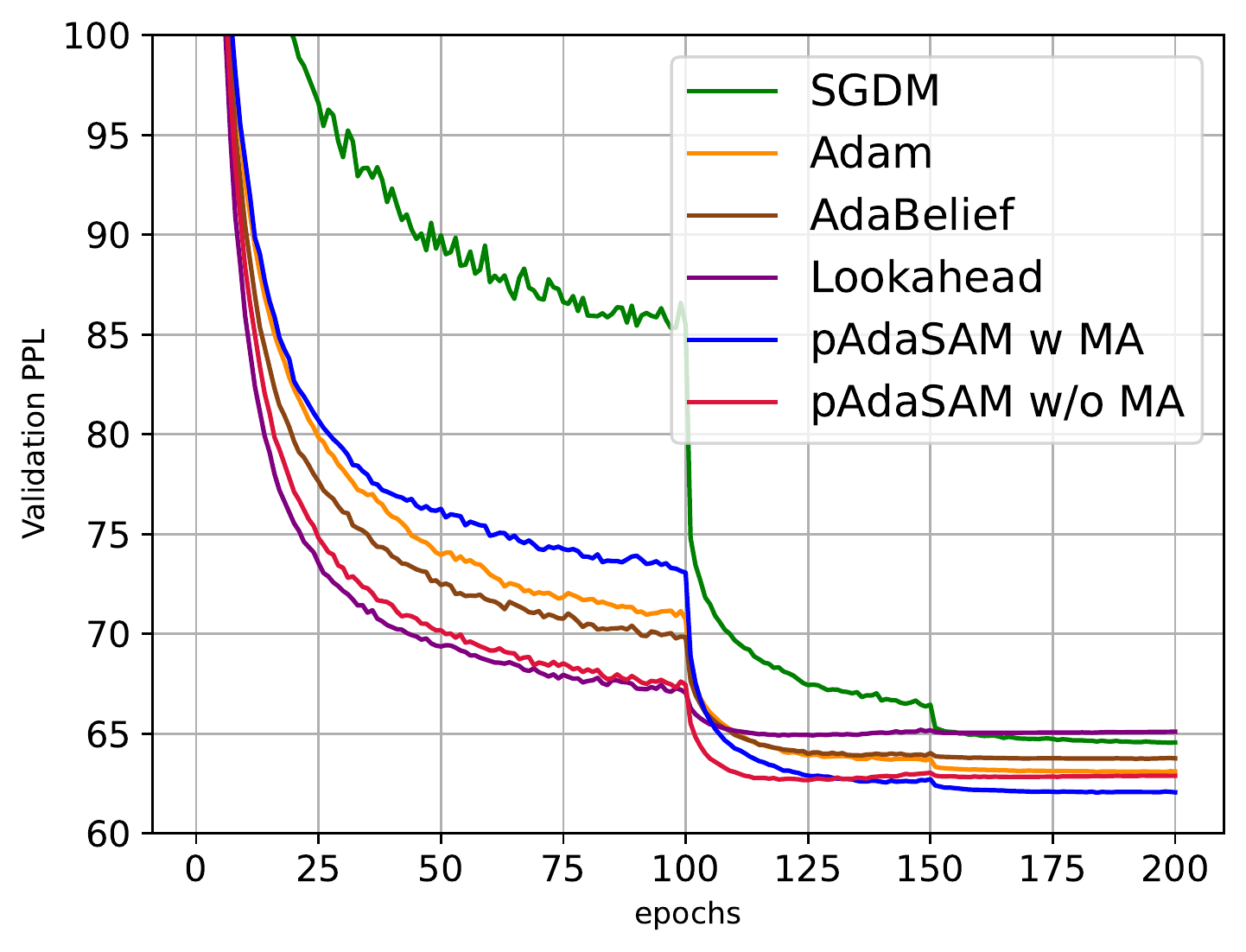}
}
\caption{Experiments on Penn TreeBank. Validation perplexity of training 1,2,3-Layer LSTM. Comparison between pAdaSAM with moving average (abbr. MA) and pAdaSAM without MA.}
\label{fig:appendix_test_mvavg_3}
\end{figure} 
 

\subsection{Additional experiments on MNIST}
 We provide some additional experiments on MNIST that is omitted in the main paper.
 \paragraph{Diminishing stepsize}
  Our theoretical analysis of SAM in Section~\ref{sec:theory} takes the diminishing condition~\eqref{cond:diminish} as an assumption of $ \beta_k $ in Theorem~\ref{them:nonconvexStochastic}, \ref{them:bounded_noisy_gradient}, \ref{them:complexity}. Nonetheless, using constant stepsize and decaying after several epochs is a common way in practice. To test the diminishing condition, we set the $t$-th epoch learning rate for SGD/Adam/SdLBFGS and the $t$-th epoch mixing parameter $ \beta_k $ for RAM/AdaSAM as $ \eta_t =  \eta_0 (1+ \lfloor t/  20 \rfloor )^{-1}$, where $ t$ denotes the number of epochs, $\eta_0$ is tuned for each optimizer. For SGD, Adam and SdLBFGS, $\eta_0$ is 0.2, 0.001, 0.1, respectively. For RAM and AdaSAM, $\eta_0$ is 2. The results of training with batch sizes of 3K and 6K are reported in Figure~\ref{fig:appendix_test_dimi}. AdaSAM still shows the better convergence rate.
  
\begin{figure}[ht]
\centering 
\subfigure[Train Loss ($n$=6K)]{
\includegraphics[width=0.23\textwidth]{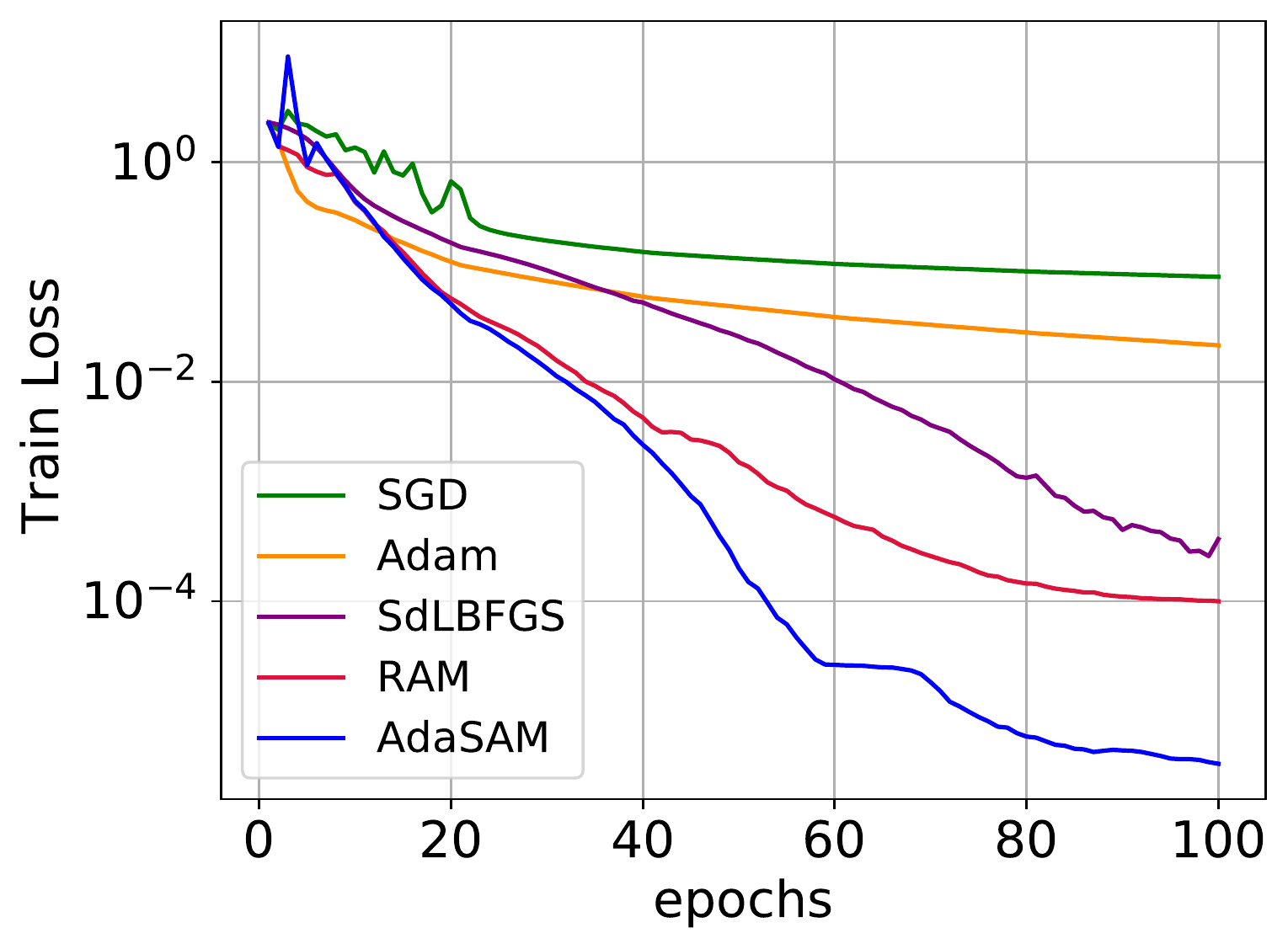}
}
\subfigure[Train SNG ($n$=6K)]{
\includegraphics[width=0.23\textwidth]{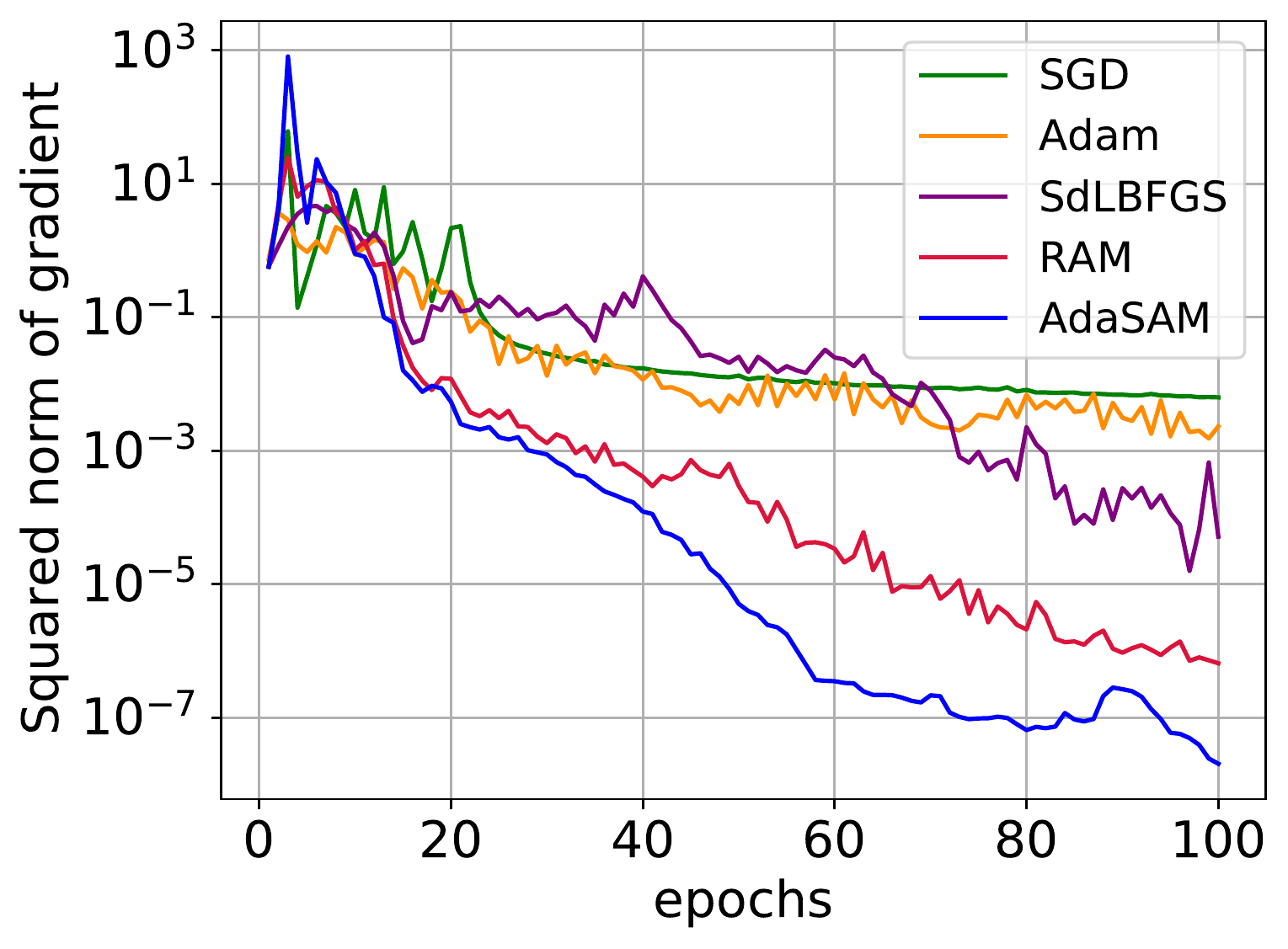}
}
\subfigure[Train Loss ($n$=3K)]{
\includegraphics[width=0.23\textwidth]{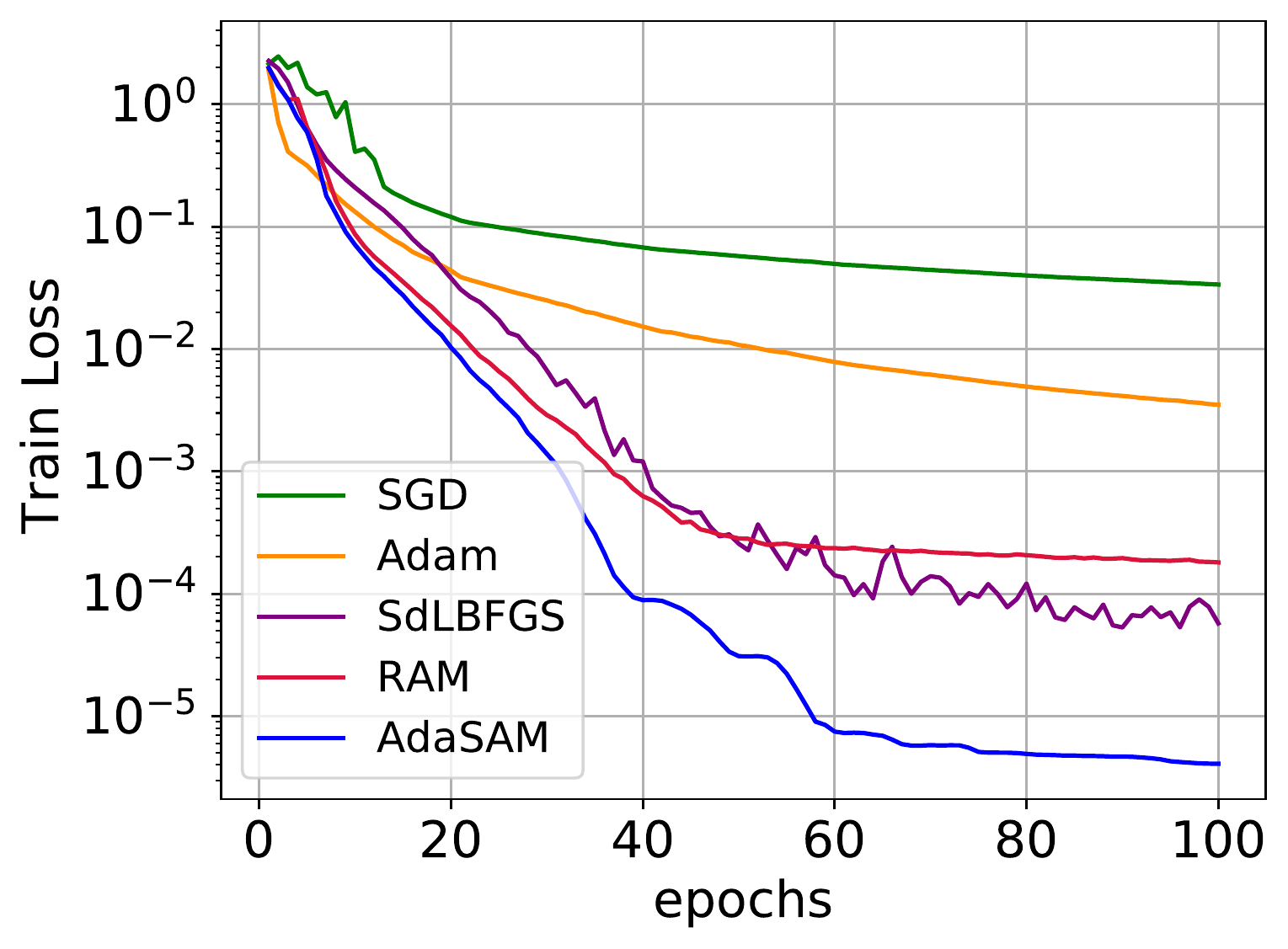}
}
\subfigure[Train SNG ($n$=3K)]{
\includegraphics[width=0.23\textwidth]{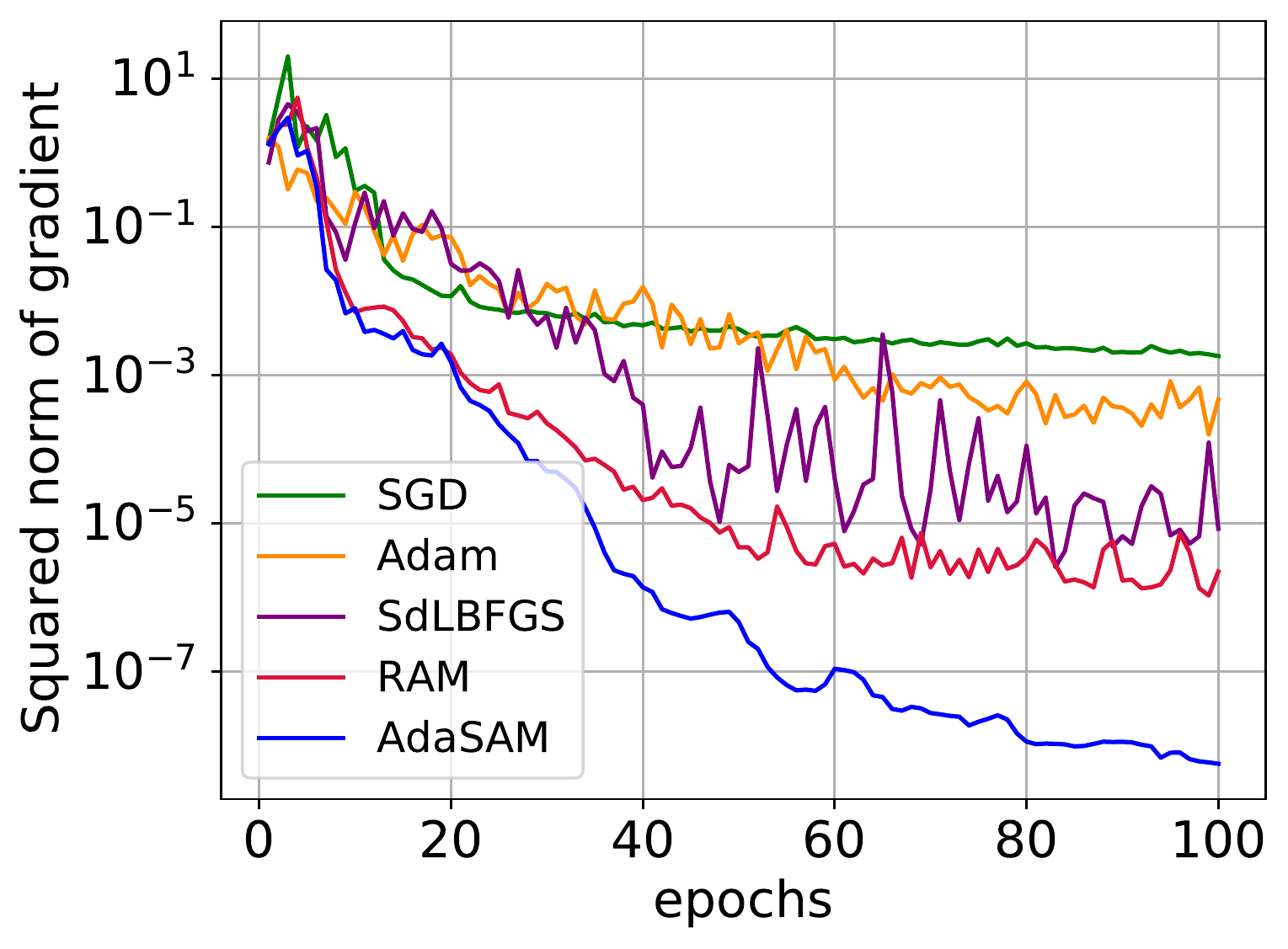}
}
\caption{Experiments on MNIST (with diminishing stepsize). (a)(b) Training loss and square norm of gradient (abbr. SNG) using batchsize $ n$= 6K; (c)(d) Training loss and SNG using $ n$= 3K; }
\label{fig:appendix_test_dimi}
\end{figure}
 
 \begin{figure}[ht]
\centering 
\subfigure[Train Loss ($n$=2K)]{
\includegraphics[width=0.23\textwidth]{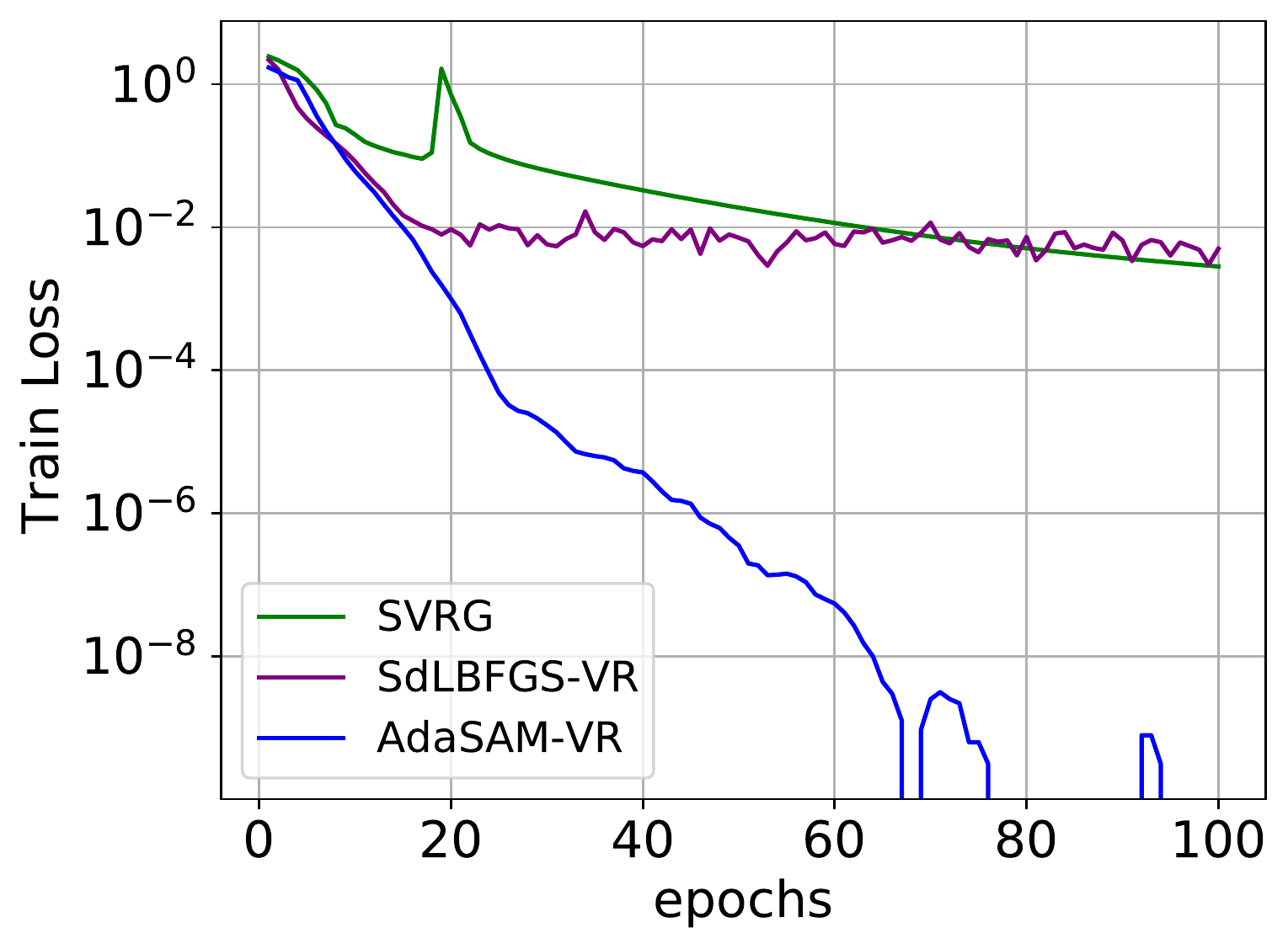}
}
\subfigure[Train SNG ($n$=2K)]{
\includegraphics[width=0.23\textwidth]{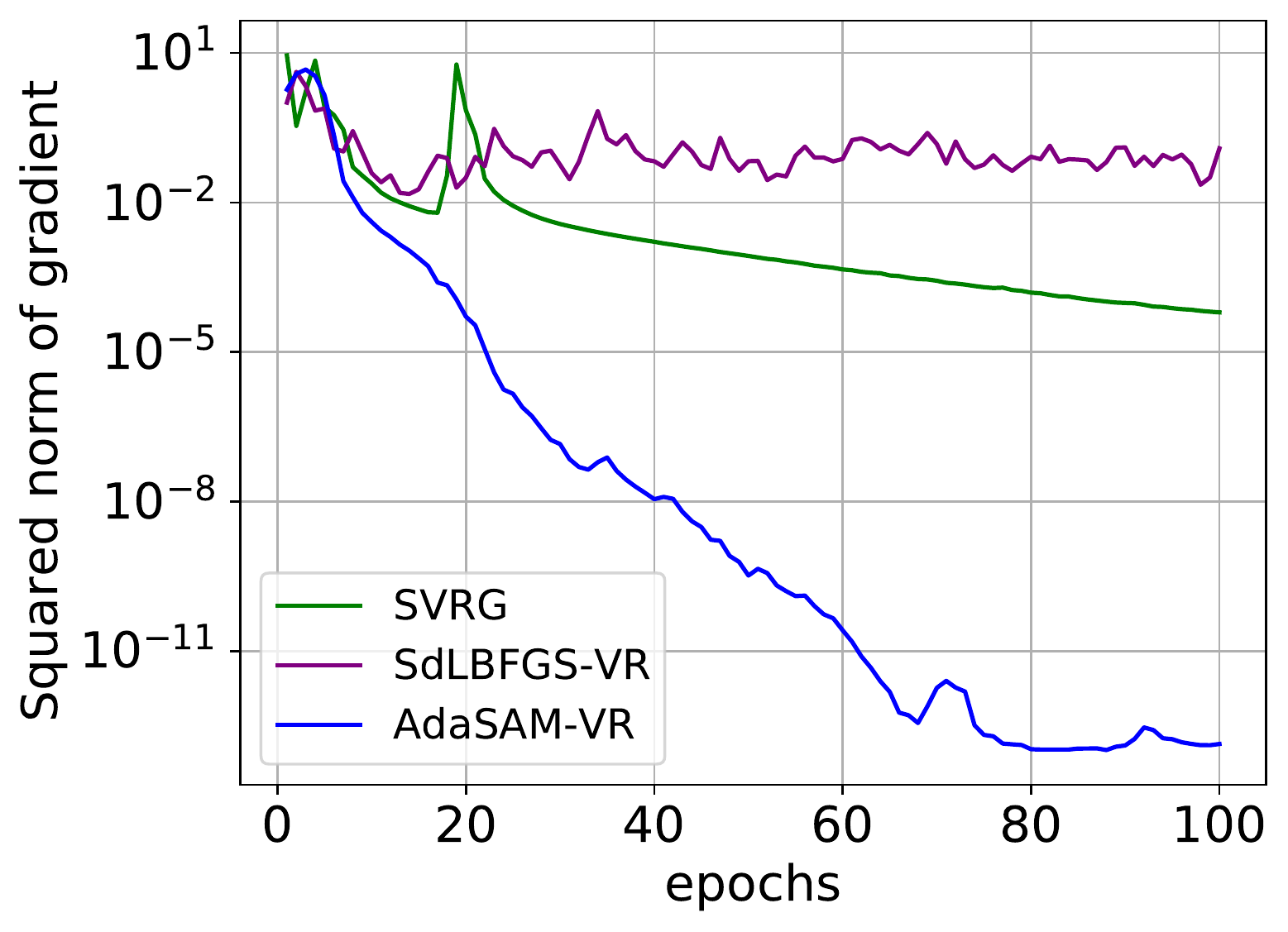}
}
\subfigure[Train Loss ($n$=4K)]{
\includegraphics[width=0.23\textwidth]{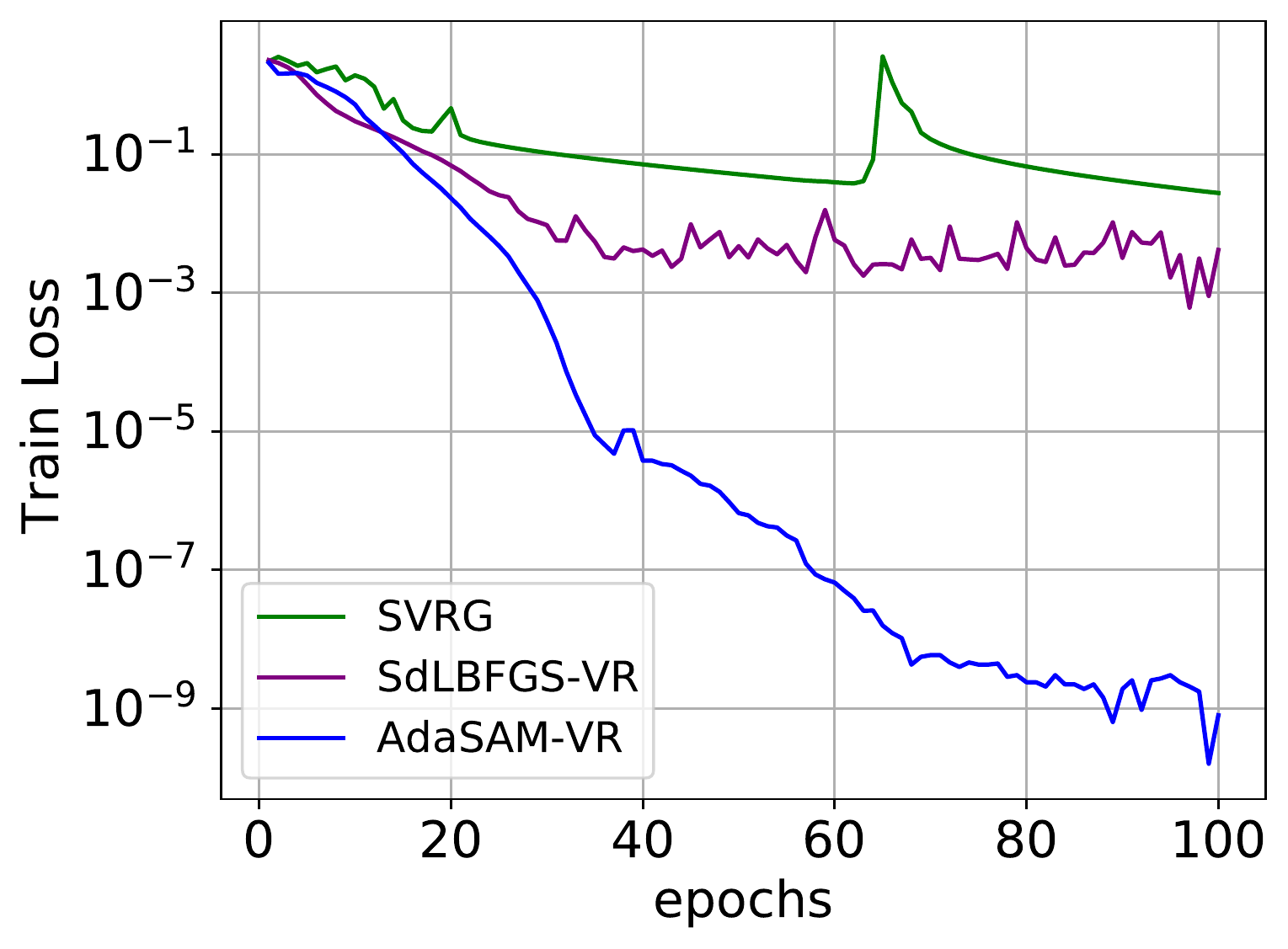}
}
\subfigure[Train SNG ($n$=4K)]{
\includegraphics[width=0.23\textwidth]{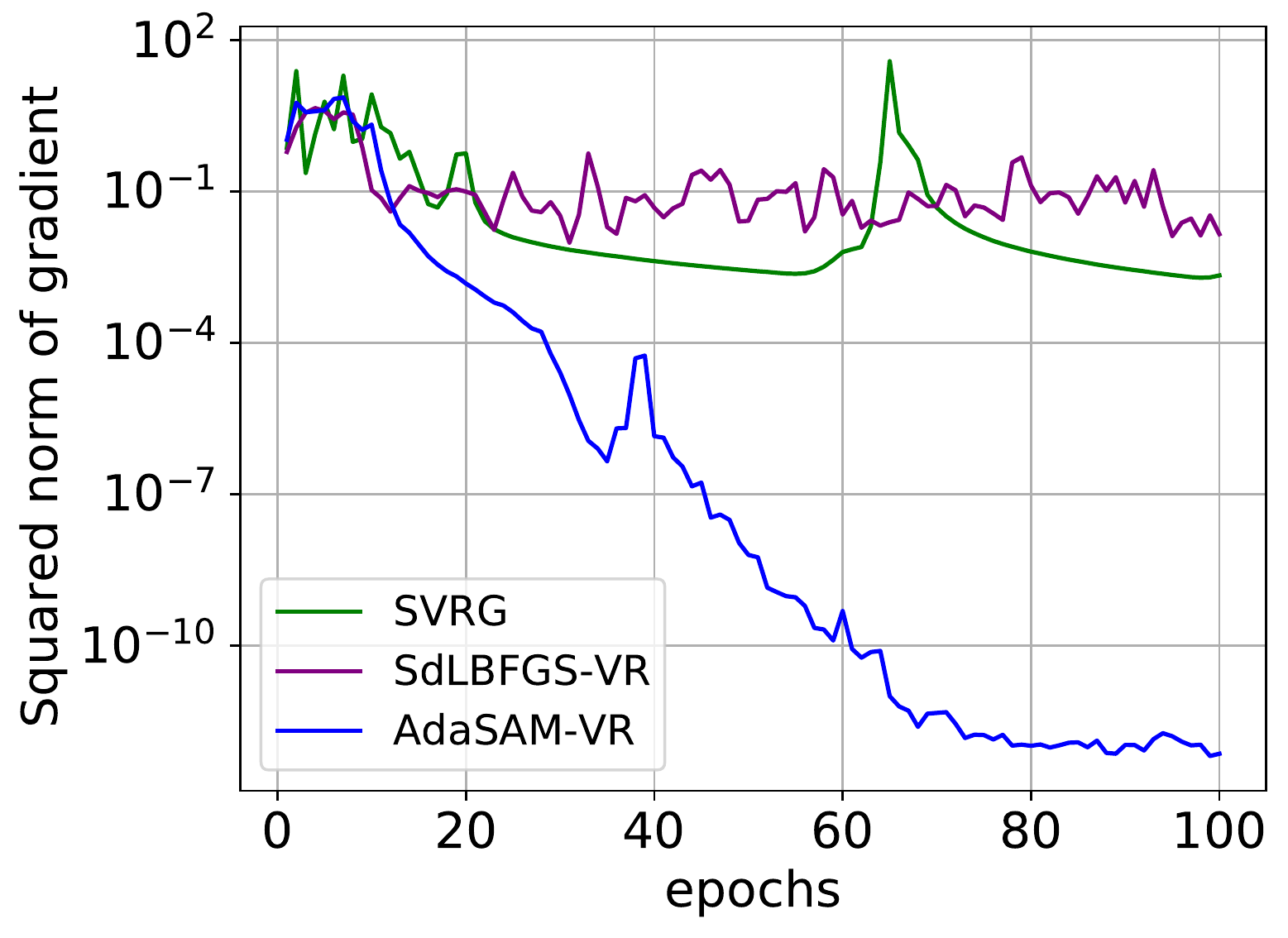}
}
\caption{Experiments on MNIST (with variance reduction). (a)(b) Training loss and square norm of gradient (abbr. SNG) using batchsize $ n$= 2K; (c)(d) Training loss and  SNG using  $ n$= 4K; }
\label{fig:appendix_test_cmpvr}
\end{figure}
  
 \paragraph{Comparisons with SVRG and SdLBFGS-VR}
 We also compared with SVRG and SdLBFGS-VR \citep{wang2017stochastic}. The learning rates for SVRG and SdLBFGS-VR are 0.2 and 0.1. Results of training with batch sizes of 2K and 4K are shown in Figure~\ref{fig:appendix_test_cmpvr}. We find that AdaSAM-VR is more effective compared with the other two variance reduced optimizers.
 
\subsection{Discussion about the hyperparameters}
 As explained in Section~\ref{subsec:para}, though at first glance AdaSAM has several hyperparameters to tune, we actually only need to individually tune the regularization parameter $ c_1 $ except for other common hyperparameter such as weight-decay in almost all the cases. For example, setting $c_1 = 10^{-2}$ is fairly robust in our experiments in image classification on CIFARs and language model on Penn TreeBank. We tested various deep neural networks on CIFAR-10 and CIFAR-100, while the hyperparameters were kept unchanged across different tests. 
 
 We conducted  tests to see the effect of the historical length $ m $ in AdaSAM and pAdaSAM. As pointed in \citep{byrd2016stochastic} that the quasi-Newton updating is inherently an overwriting process rather than an average process, large noise in gradient estimates can make a secant method  rather unstable. On the contrary, since AM is identified as a multisecant method, it leverages more secant conditions in one update which may alleviate the negative impact of a noisy secant condition. Hence, AM may be more tolerant to noise. The historical length $ m $ determines how many secant conditions are taken into consideration at one time, so a larger $ m $ is supposed to make AdaSAM more tolerant to noise.
 
 \begin{figure}[ht]
\centering 
\subfigure[Train Loss ]{
\includegraphics[width=0.23\textwidth]{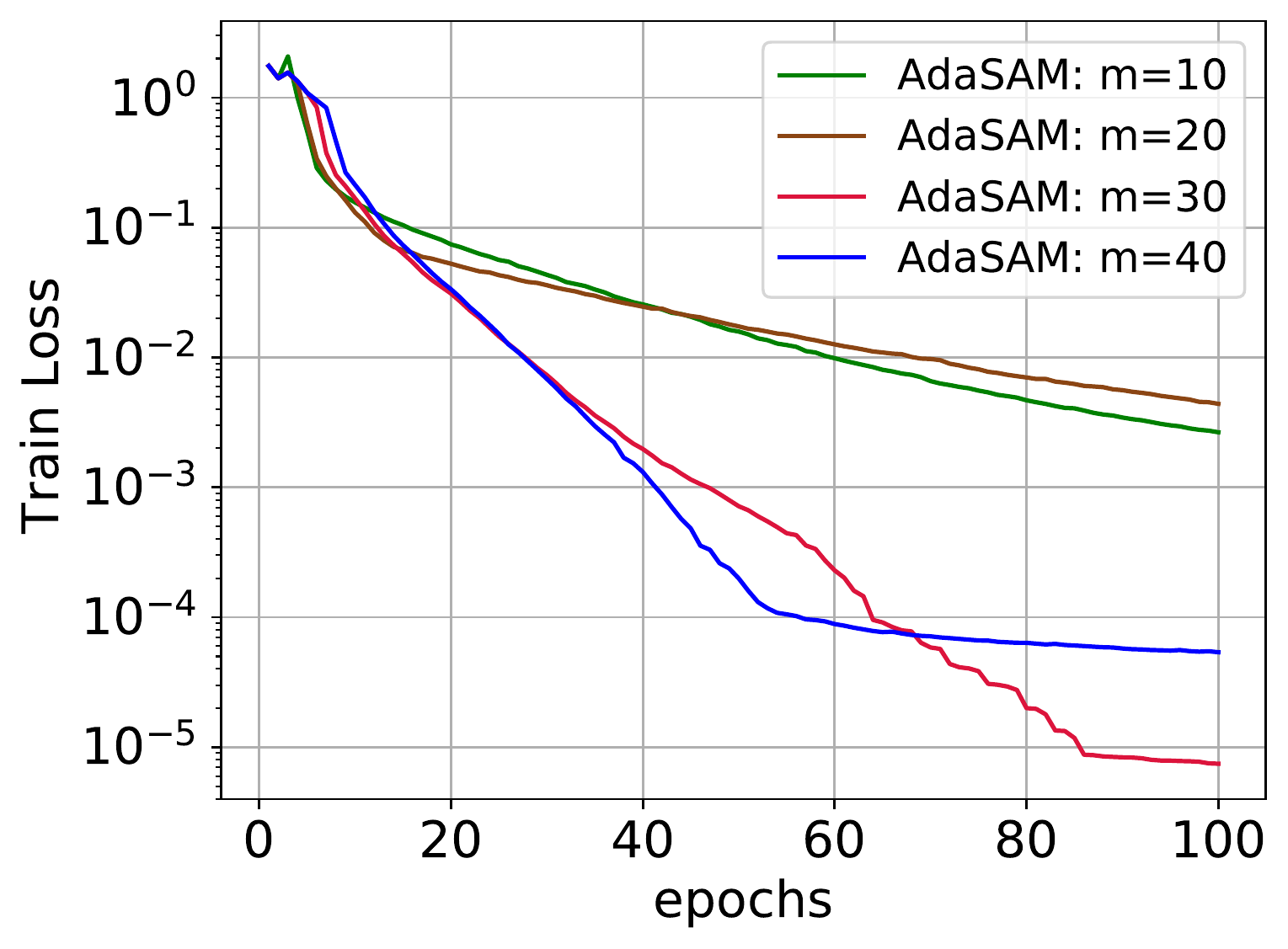}
}
\subfigure[Train SNG ]{
\includegraphics[width=0.23\textwidth]{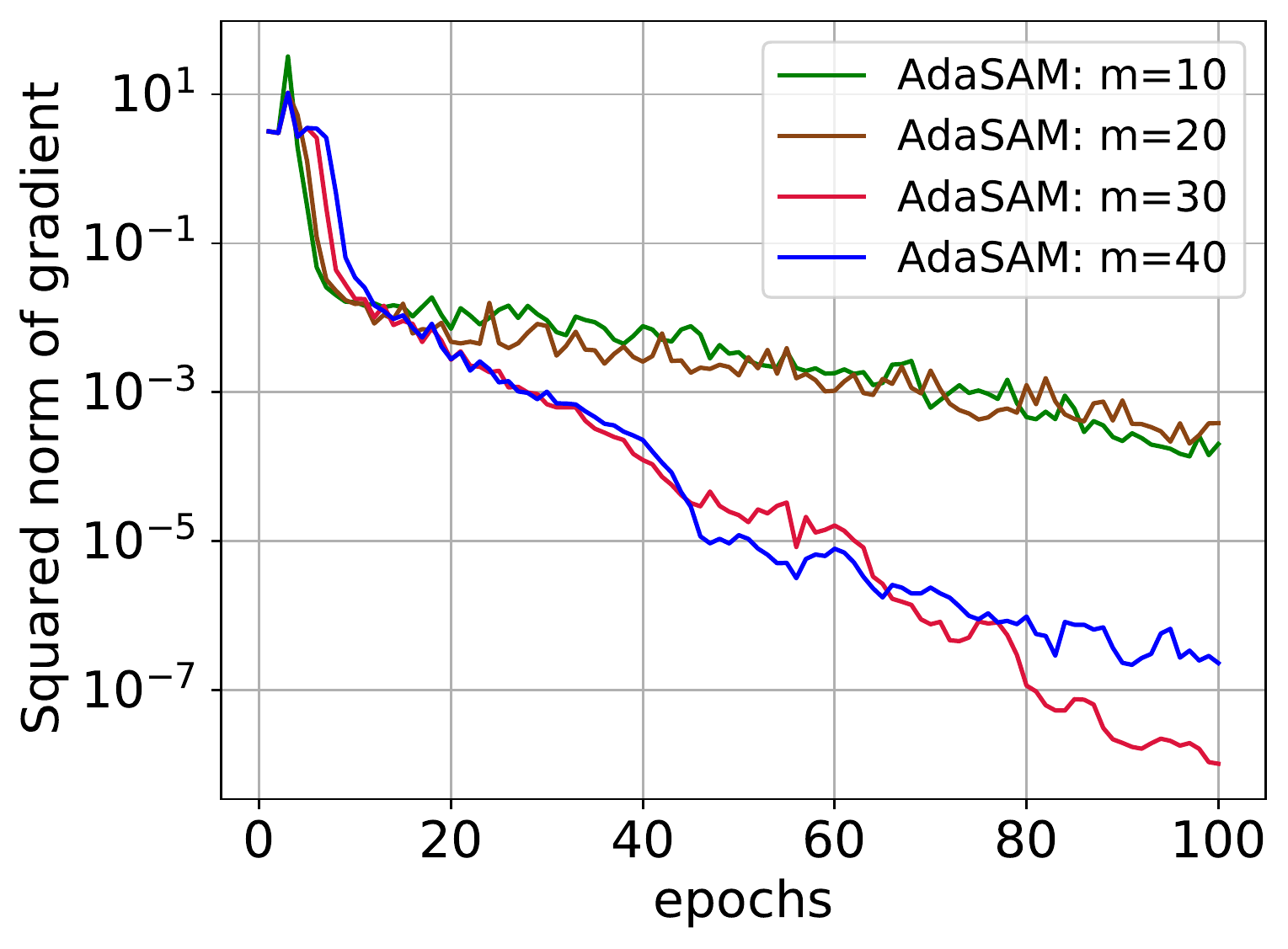}
}
\caption{Experiments on MNIST. Training loss and square norm of gradient (abbr. SNG) using batchsize $ n$= 2K. AdaSAM without moving average and $ m = 10,20,30,40$. }
\label{fig:appendix_test_hist1}
\end{figure}

 We set $ m=20 $ in the experiments on MNIST. In Figure~\ref{fig:appendix_test_movavg_1}(a) and (b), we find AdaSAM without moving average stagnates when training with batchsize = 2K. We set $ m $ to 10,30,40 to see if any difference happens. The result is shown in Figure~\ref{fig:appendix_test_hist1}, from which we see using a larger $ m=30 $ did help convergence. Further increasing $ m=40 $ does not lead to  lower training loss, which may be due to the potential numerical weakness in solving \eqref{adalsq} with \eqref{AdaRAA:Gamma_k} directly.
 
 \begin{figure}[ht]
\centering 
\subfigure[Train Accuracy]{
\includegraphics[width=0.23\textwidth]{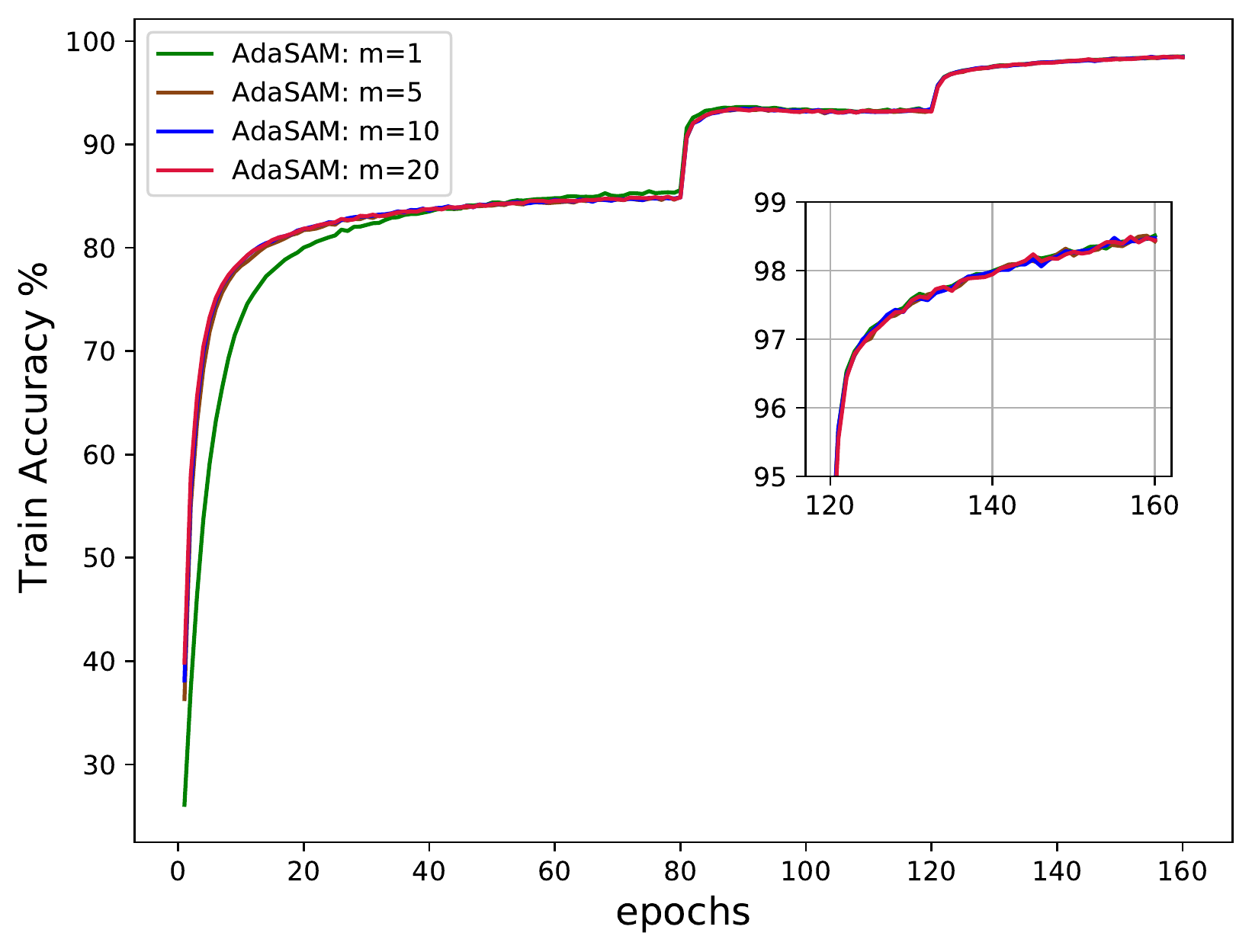}
}
\subfigure[Test Accuracy]{
\includegraphics[width=0.23\textwidth]{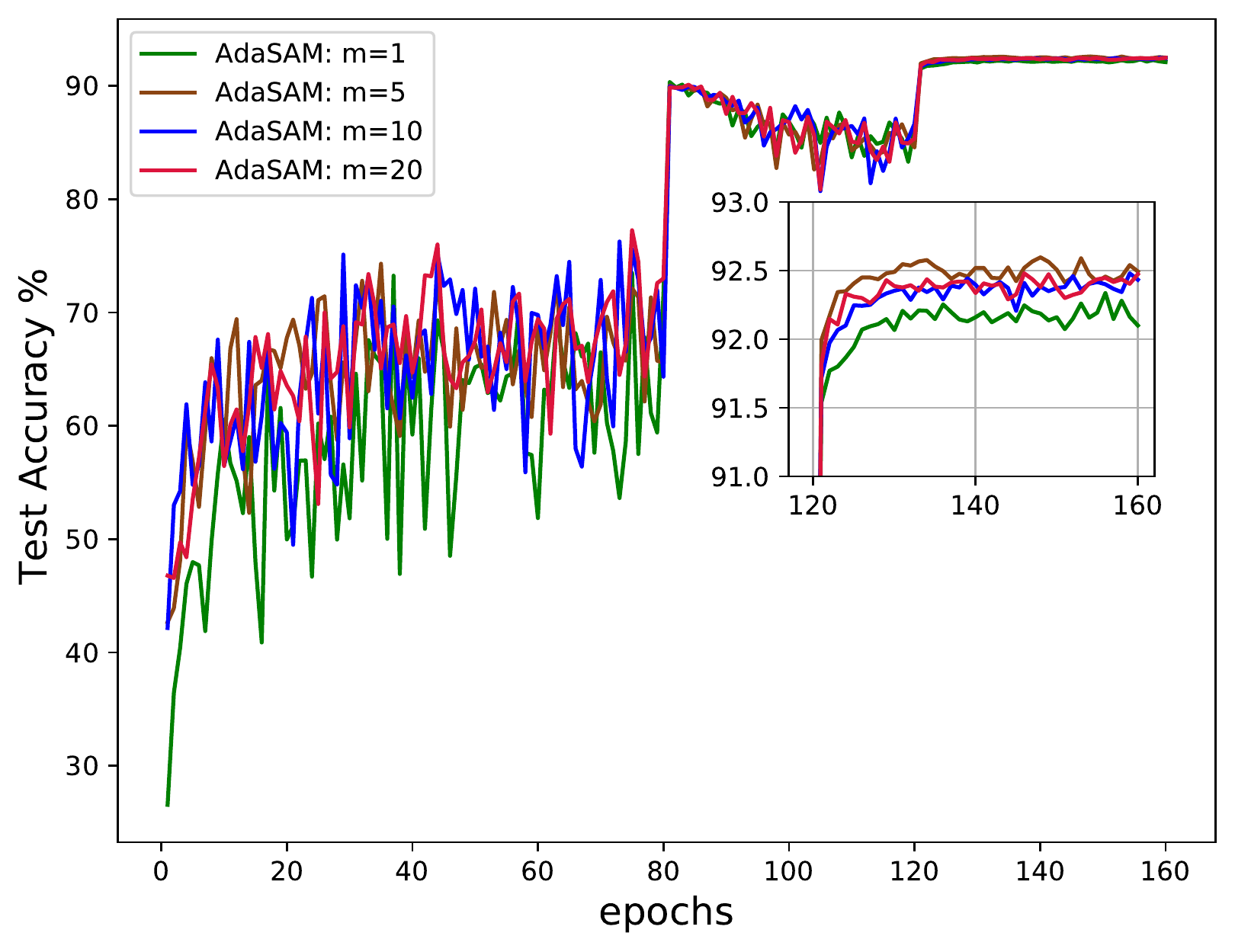}
}
\subfigure[Train Perplexity]{
\includegraphics[width=0.23\textwidth]{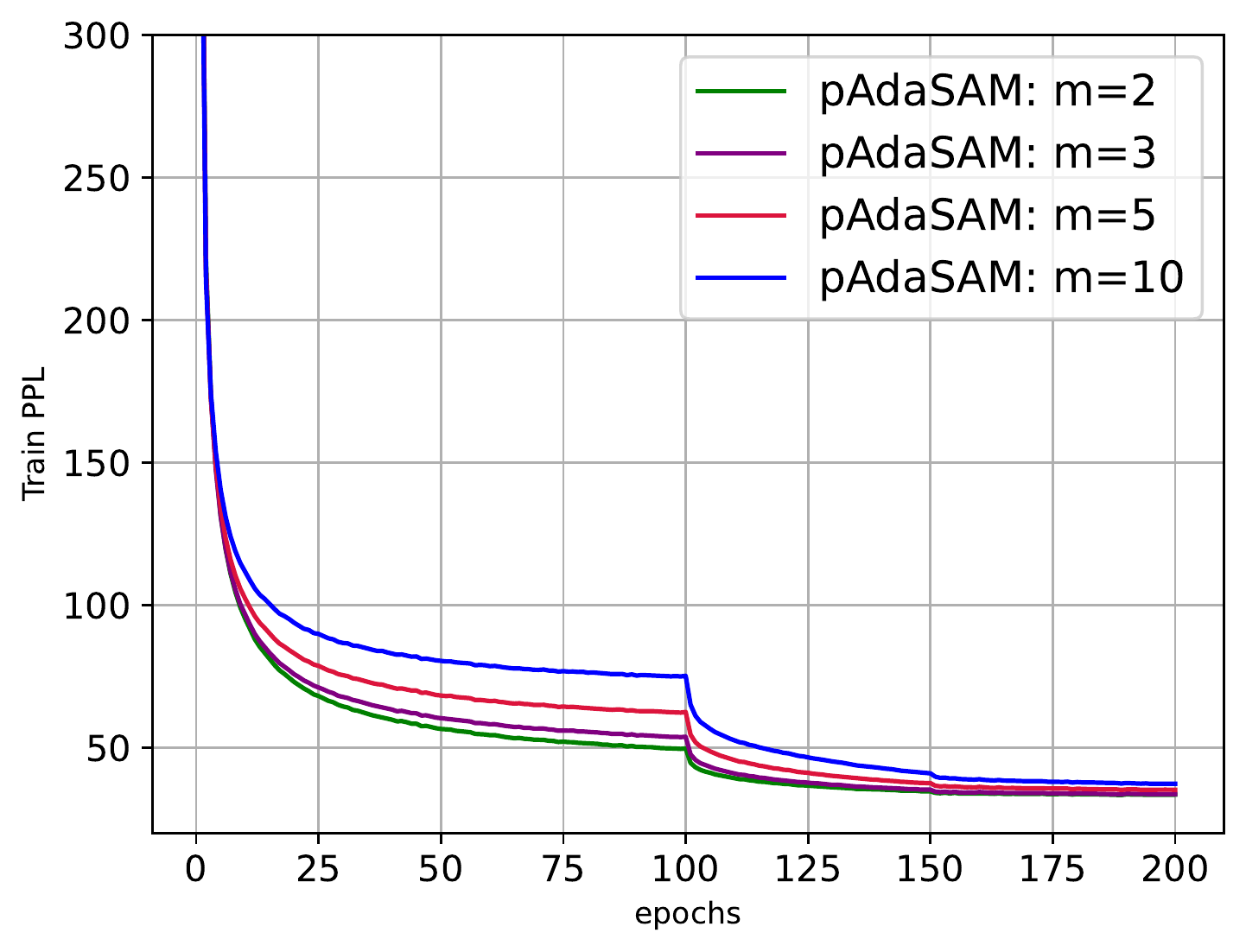}
}
\subfigure[Validation Perplexity]{
\includegraphics[width=0.23\textwidth]{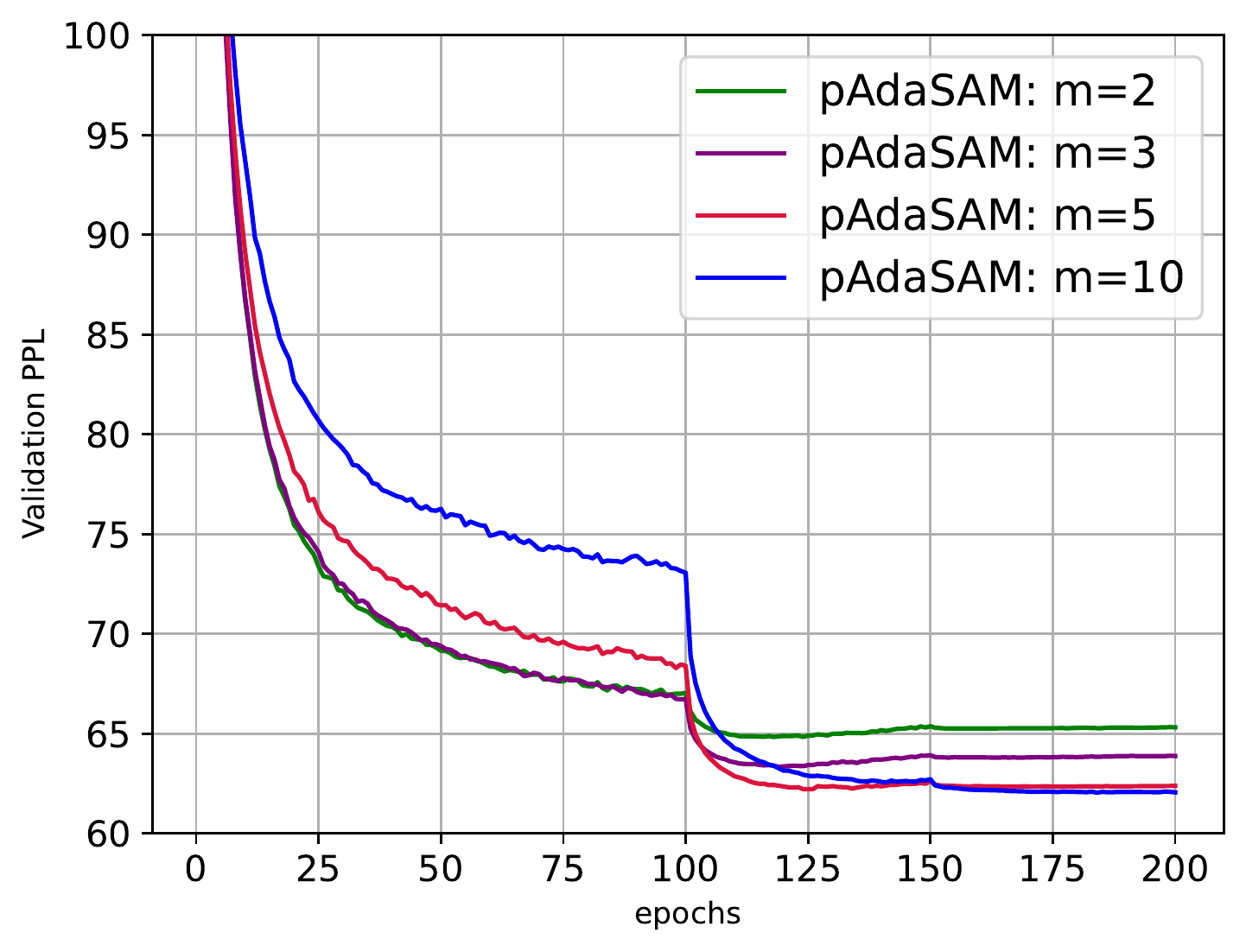}
}
\caption{(a)(b) Experiments on CIFAR-10/ResNet20. AdaSAM with $ m= 1,5,10,20$. (c)(d) Experiments on training a 3-layer LSTM on Penn TreeBank. pAdaSAM with $ m=2,3,5,10. $ }
\label{fig:appendix_test_hist2}
\end{figure}
 
 The results related to different $ m $ in CIFAR-10/ResNet20 and 3-layer LSTM on Penn TreeBank are reported in Figure~\ref{fig:appendix_test_hist2}. A larger $ m $ seems to be beneficial to generalization ability. $m=5$ or 10 is proper for these tests.
 
 \subsection{Computational efficiency}
 The additional computational cost of AdaSAM/pAdaSAM  compared with SGD is mainly due to computing \eqref{AdaRAA:Gamma_k}. The cost is a potential limitation of our method. Fortunately, this part of computation is parallel-friendly since the main operation is dense matrix multiplications. Therefore, when the cost of function evaluations and gradient evaluations dominates the computation, the extra overhead incurred by AdaSAM is negligible. In high performance computing, we expect that the matrix computation in AdaSAM can be further optimized. 
 
 \begin{figure}[ht]
\centering 
\subfigure[Running time of one epoch]{
\includegraphics[width=0.4\textwidth]{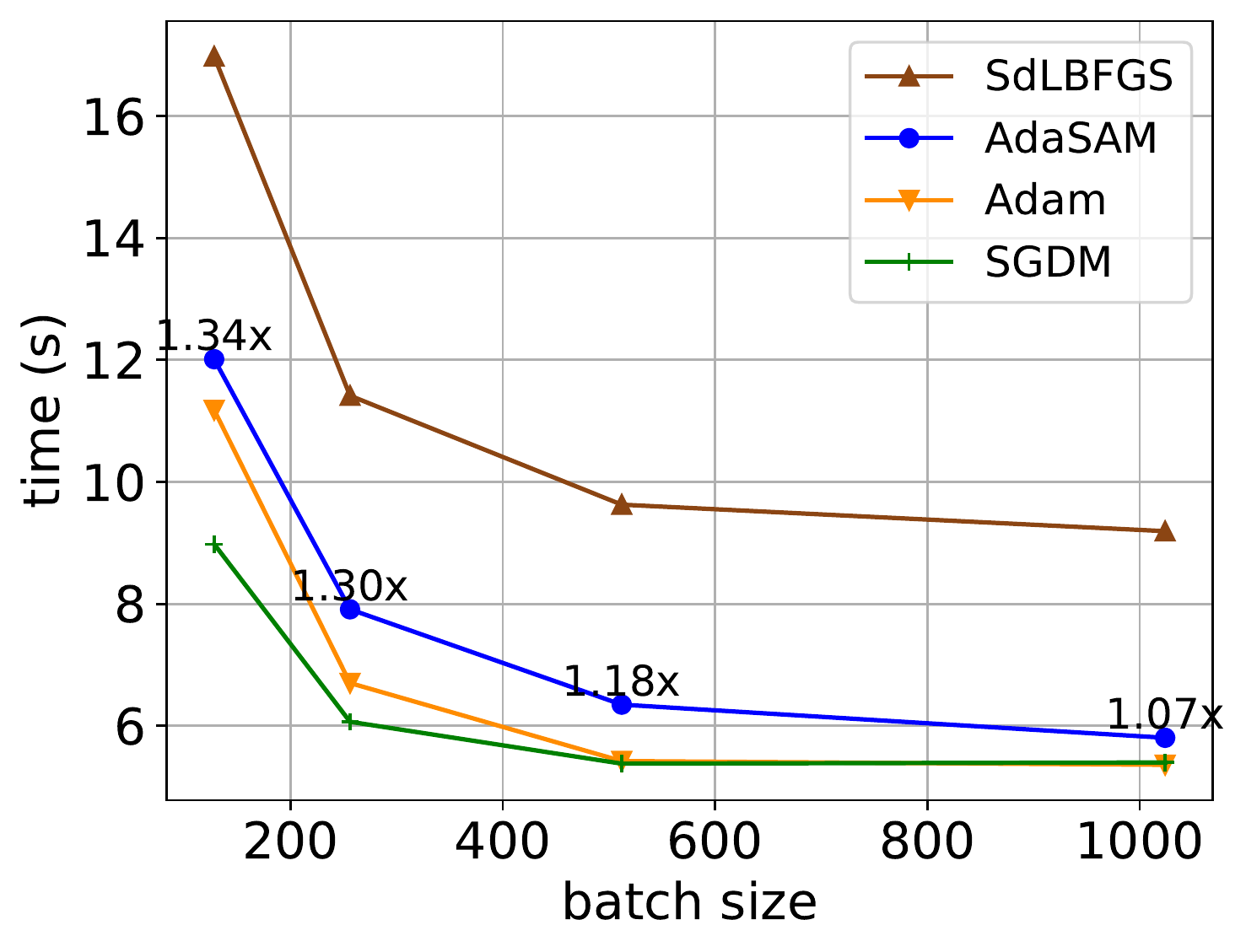}
}
\subfigure[Running time of one epoch]{
\includegraphics[width=0.4\textwidth]{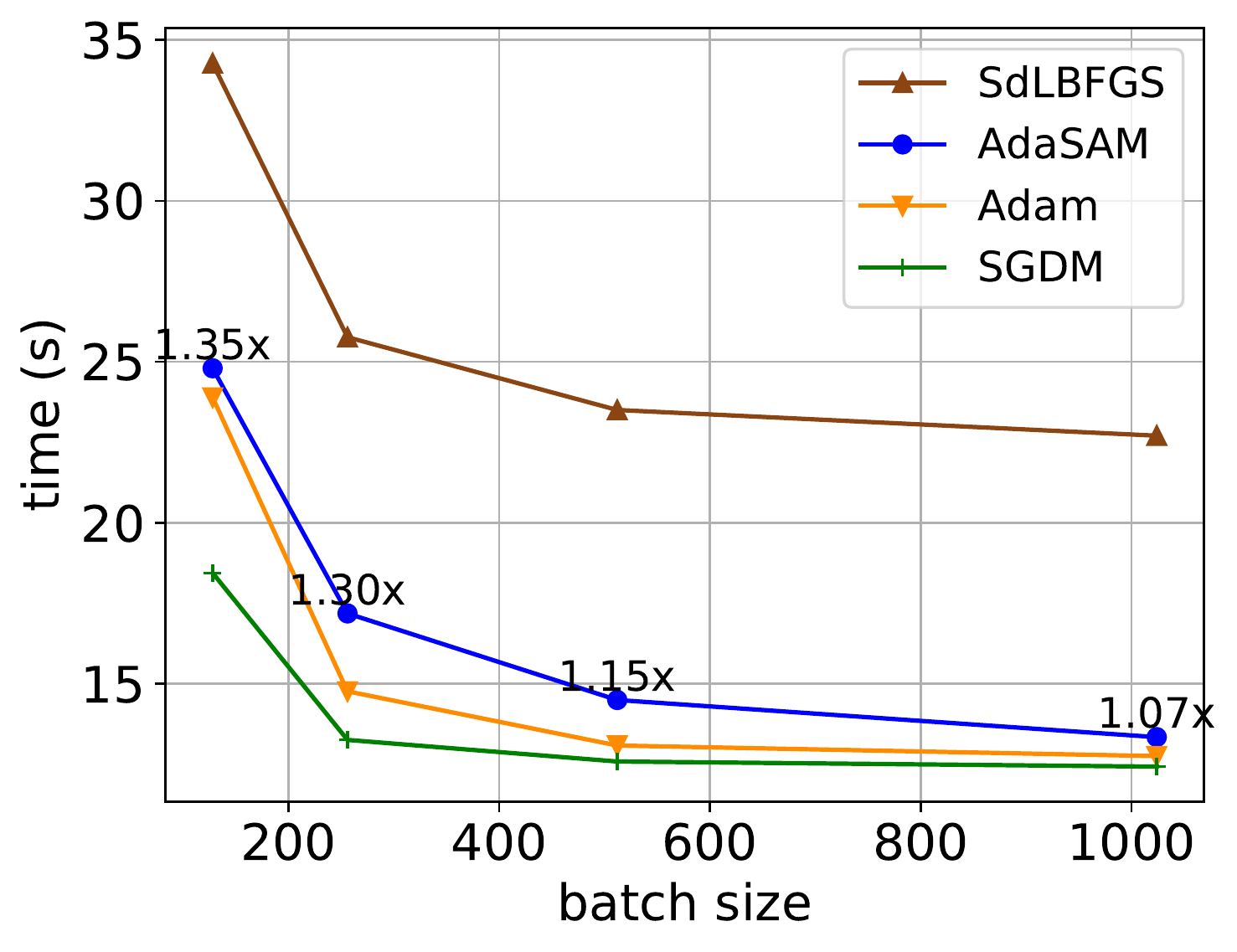}
}
\caption{(a) Running time of one epoch of training CIFAR-10/ResNet20; (b) Running time of one epoch of training CIFAR-10/ResNet56. Batch size = 128, 256, 512, 1024. The numbers marked beside the curve of AdaSAM show the computational time of AdaSAM vs. SGDM.}
\label{fig:appendix_test_batchsize}
\end{figure}
 
 Figure~\ref{fig:appendix_test_batchsize} reports the running time of one epoch of training ResNet20 and ResNet56 on CIFAR-10 with batch size of 128, 256, 512, 1024.   Optimizers are SGDM, Adam, SdLBFGS and AdaSAM. 
 It can be observed that the additional overhead of AdaSAM gradually becomes marginal with the increment of batch size. Therefore, AdaSAM is expected to be more computationally efficient in large mini-batch training. Moreover, as confirmed by the experiments (Figure~\ref{fig:appendix_CIFAR_stop}), AdaSAM can achieve comparable test accuracy while using less epochs of training, thus saving large number of gradient evaluations. The proposed alternating iteration scheme can also serve as a trade-off between computational cost with final accuracy or loss.
 
 

\end{document}